\documentclass[twoside]{article}

\usepackage[accepted]{icml2023}

\usepackage{times}
\usepackage{latexsym}
\usepackage{hyperref,amsmath,amsthm}
\newtheorem{theorem}{Theorem}

\usepackage[T5, T1]{fontenc}

\usepackage[utf8]{inputenc}

\usepackage{graphicx}
\usepackage{tabularx} 
\usepackage{booktabs} 
\usepackage{mathtools}
\usepackage{siunitx}
\usepackage{amsmath}
\usepackage{amssymb}
\usepackage{euscript}
\usepackage{float}
\usepackage{caption}
\usepackage{subcaption}
\usepackage{multirow}
\usepackage{arydshln}
\usepackage{adjustbox}
\usepackage{tablefootnote}
\usepackage{todonotes}
\DeclarePairedDelimiter\norm{\lVert}{\rVert}
\usepackage{relsize}
\usepackage{dsfont}
\usepackage{bm}
\definecolor{mydarkblue}{rgb}{0,0.08,0.45}

\hypersetup{ %
pdftitle={},
pdfsubject={Proceedings of the International Conference on Machine Learning 2023},
pdfkeywords={},
pdfborder=0 0 0,
pdfpagemode=UseNone,
colorlinks=true,
linkcolor=mydarkblue,
citecolor=mydarkblue,
filecolor=mydarkblue,
urlcolor=mydarkblue,
}

\DeclareMathOperator{\argmax}{arg\,max}

\icmltitlerunning{Differential Privacy, Linguistic Fairness, and Training Data Influence}

\begin{document}

\twocolumn[
\icmltitle{Differential Privacy, Linguistic Fairness, and Training Data Influence: \\ Impossibility and Possibility Theorems for Multilingual Language Models}

\icmlsetsymbol{equal}{*}

\begin{icmlauthorlist}
\icmlauthor{Phillip Rust}{ucph}
\icmlauthor{Anders S{\o}gaard}{ucph}
\end{icmlauthorlist}

\icmlaffiliation{ucph}{Department of Computer Science, University of Copenhagen}
\icmlcorrespondingauthor{Phillip Rust}{p.rust@di.ku.dk}
\icmlkeywords{Machine Learning, ICML}

\vskip 0.3in
]

\printAffiliationsAndNotice{}
\begin{abstract}

Language models such as mBERT, XLM-R, and BLOOM aim to achieve multilingual generalization or compression to facilitate transfer to a large number of (potentially unseen) languages. However, these models should ideally also be private, linguistically fair, and transparent, by relating their predictions to training data. Can these requirements be simultaneously satisfied? We show that multilingual compression and linguistic fairness are compatible with differential privacy, but that differential privacy is at odds with training data influence sparsity, an objective for transparency. We further present a series of experiments on two common NLP tasks and evaluate multilingual compression and training data influence sparsity under different privacy guarantees, exploring these trade-offs in more detail. Our results suggest that we need to develop ways to jointly optimize for these objectives in order to find practical trade-offs.

\end{abstract}

\section{Introduction}
\label{sec:intro}

One of the open challenges in AI is bridging the widening digital language divide by providing technologies that work well for all languages. Multilingual language models such as mBERT \citep{devlin-etal-2019-bert}, XLM-R \citep{conneau-etal-2020-unsupervised}, and BLOOM \citep{https://doi.org/10.48550/arxiv.2211.05100}, facilitate transfer between closely related languages, enabling roll-out of technologies for low-resource languages, and are used for a wide range of real-world applications in many languages---e.g., from named entity recognition \citep{khalifa-etal-2021-self} to legal document classification \citep{wang-banko-2021-practical}. Generalization across languages is challenged by typological divides, language families, or scripts \citep{singh-etal-2019-bert,dufter-schutze-2020-identifying} and finding architectures that best facilitate such transfer, achieving optimal {\bf multilingual compression} \citep{ravishankar-sogaard-2021-impact} through parameter sharing (rather than compartmentalization), remains an open research problem. 

With the widespread adaptation of multilingual language models also comes responsibility and requirements that models are trustworthy \citep{trustnlp-2021-trustworthy}. What does trustworthiness amount to for multilingual language models? A crucial requirement is that multilingual NLP models perform equally well across languages, not favoring any languages over others. \citet{choudhury2021how} refer to this property as {\bf linguistic fairness}. Linguistic fairness is defined as zero variance across language-specific losses, typically estimated on held-out data.\footnote{This definition of linguistic fairness is an instantiation of  \emph{equal risk fairness} or overall performance parity, i.e., 
equal model performance across groups \citep{Berk2018FairnessIC,Verma2018FairnessDE,williamson-etal-2019-fairness}, which balances precision-based and recall-based metrics and is considered more relevant than calibration-based metrics for standard NLP applications. Since the three are mutually exclusive \citep{https://doi.org/10.48550/arxiv.1707.01195}, we ignore calibration and balance precision and recall.}

Another crucial requirement is {\em transparency}, i.e., the ability to say {\em why} models make particular predictions. Methods to achieve transparency come in two flavors; Some methods---commonly referred to as feature attribution methods---present rationales behind predictions in terms of input token attributions, but such rationales are limited in that they cannot explain predictions motivated by the absence of input tokens or the presence of particular token combinations. Feature attribution methods have also been shown to be unreliable \citep{Kindermans2019TheO,Arun2020.07.28.20163899}. Other methods highlight training data influence, i.e., provide influential data points as rationales for decisions. Often referred to as instance-based interpretability methods, they are argued to be more useful across different NLP tasks \citep{han-etal-2020-explaining, han-tsvetkov-2021-influence-tuning, zhou2021do}. We refer to the objective of achieving sparse training data influence, i.e. strong instance-interpretability, as \textbf{training data influence sparsity}.
Finally, for many NLP applications, we further need our models to be private, for which {\bf differential privacy} \citep[DP; ][]{dwork-etal-2006-dp} provides a theoretically rigorous framework. 

The trustworthiness objectives as defined above have primarily been considered in a monolingual context, and are often (falsely) assumed to be independent \citep{ruder-etal-2022-square}.\footnote{One exception is a growing body of work showing fairness and differential privacy are at odds \citep{bagdasaryan-etal-2019-differential, cummings-etal-2019, chang-shokri-2021, hansen-etal-2022-impact}. While \citet{naidu2021differential} show that differential privacy and GradCAM \citep{selvaraju2019}, 
 a feature attribution method, are compatible, the interaction between differential privacy and training data influence remains unexplored.} Our paper investigates {\em the extent to which these objectives align or are at odds}.
We do so in a multilingual setting and show how multilinguality presents options and challenges.\footnote{We are, to the best of our knowledge, first to consider differential privacy in a multilingual setting specifically, with the exception of work on differentially private neural machine translation \citep{kim-etal-2021-using}. } Our theoretical contributions show that while privacy and linguistic fairness are compatible through multilingual compression, privacy and training data influence sparsity are not, and our empirical results indicate that these objectives interact in non-linear ways.\footnote{Our code is available at \url{https://github.com/xplip/multilingual-lm-objectives}.}

\paragraph{Contributions}
We begin (in \S\ref{sec:background}) with a theoretical exploration of differential privacy, training data influence, and linguistic fairness in the context of multilingual language models. We show that differential privacy and training data influence sparsity are fundamentally at odds, a result which is not limited to the multilingual setting. While differential privacy and fairness are often said to be at odds, we also show that differential privacy and linguistic fairness over languages are compatible in the multilingual setting, as a result of compression.

Subsequently (in \S\ref{sec:exp_setup}--\S\ref{sec:empirical_interpretability}), we present empirical results on the impact of differentially private fine-tuning on multilingual compression and training data influence: We analyze the effect of such fine-tuning on the multilingual compression of large LMs and find that it is possible to achieve (i) high compression with strong privacy at the cost of performance; (ii) high compression with high performance at the cost of privacy; or (iii) privacy and accuracy at the cost of compression. Since we show in \S\ref{sec:background} that performance, privacy and compression {\em are theoretically} compatible, this leaves us with an open problem: How do we practically optimize for both performance, privacy and compression?

Furthermore, we compare four (proxy) metrics for quantifying multilingual compression---sentence retrieval, centered kernel alignment \citep[CKA; ][]{kornblith-etal-2019}, IsoScore \citep{rudman-etal-2021-iso}, representational similarity analysis \citep[RSA; ][]{kriegeskorte-etal-2008, edelman-1998}---and discuss their usefulness for balancing these trade-offs.

\newpage

Finally, we show that LMs exhibiting high multilingual compression are less instance-interpretable in that they make highlighting training data influence more difficult.

In sum, our work shows that {\em linguistically fair and private high-performance multilingual models are possible, even if  learning them is challenging. However, training data influence methods will fail for such models}.

\section{Theoretical Exploration}
\label{sec:background}

We consider language model learning and fine-tuning in a multilingual setting, in which our training data $D=D_1\cup\ldots\cup D_{|L|}$ is the union of disjoint training data from $|L|$ different languages. We consider the interaction of differential privacy, training data influence and linguistic fairness with performance and compression in this setting.

\paragraph{Preliminaries} We briefly introduce our formal definitions here: A randomized algorithm, here model, ${\mathcal{M}: \mathcal{D} \to \mathcal{Y}}$ is $\varepsilon_p$-\emph{differentially private} \citep{dwork-etal-2006-dp} iff for all adjacent datasets ${D, D' \in \mathcal{D}}$ and all ${Y \subset \mathcal{Y}}$, ${\mathbb{P}(\mathcal{M}(D) \in Y)} \leq {\exp(\varepsilon_p)\cdot \mathbb{P}(\mathcal{M}(D') \in Y)}$.\footnote{
Note how standard empirical risk minimization is not private, since it is a linear combination of training samples near the decision boundary, and if $D$ and ${D'}$ differ in one of those, the classifier changes significantly.} Adjacent means that the datasets differ by exactly one example $x_{\mathit{diff}}$.

A model $\mathcal{M}$ is said to be {\em $\varepsilon_i$-instance-interpretable}, i.e., having sparse training data influence, iff for any ${D, D', D'' \in \mathcal{D}}$ with $D' =D\setminus\{x_{\mathit{diff}}\}$, ${D''=D\setminus\{x'\}}$, and $x_{\mathit{diff}}\neq x'$, where $x_{\mathit{diff}}$ is the most influential training data point under leave-one-out influence,\footnote{Leave-one-out here means $D'=D\setminus \{x_{\mathit{diff}}\}$ and is the gold standard for instance-based methods, which explains the close connection to DP where we also deal with adjacent datasets.} it holds that ${\mathbb{P}(\mathcal{M}(D) \in Y)} - {\mathbb{P}(\mathcal{M}(D') \in Y)} > {\exp(\varepsilon_i)\cdot (\mathbb{P}(\mathcal{M}(D) \in Y) - \mathbb{P}(\mathcal{M}(D'') \in Y))}$. In other words, $x_{\mathit{diff}}$ had more influence on $\mathcal{M}$ than any other data point $x'$ by some margin $\exp(\varepsilon_i)$ \citep{10.5555/3305381.3305576}.

A model $\mathcal{M}$ is said to be fair if for a group partitioning $g(D)\rightarrow D_{g_1},\ldots,D_{g_n}$ into smaller samples and for some loss function $\ell$, e.g., 0-1 loss, $\ell(\mathcal{M}(D_{g_i}))=\ell(\mathcal{M}(D_{g_j}))$ \citep{williamson-etal-2019-fairness}. A model that is fair for a group partitioning by languages is said to be linguistically fair \citep{choudhury2021how}.

Finally, a model $\mathcal{M}$ exhibits perfect multilingual compression when it outputs identical representations for semantically equivalent inputs irrespective of the input language. Formally, for a pair of translation equivalent sentences, ($i_j$, $i_q$), the representations of $i_j$ and $i_q$ are identical at any layer $l$ of the model, i.e $\mathcal{M}^l(i_j)=\mathcal{M}^l(i_q)$. 

\newpage

In the following paragraphs, we discuss under what conditions DP, training data influence, linguistic fairness, and multilingual compression are at odds or are compatible, and how these conditions align with common scenarios in multilingual NLP.\footnote{Differential privacy meaningfully protects any individual training example. However, sensitive information may be repeated across many training examples, so $\varepsilon$-DP does not necessarily prevent leakage of such information at the granularity of individual people, real-world events, etc. For example, in our multilingual setting, an attacker may still gain access to a social security number learned by the model, but they will be unable to identify whether the number was leaked  in a particular language.}

\paragraph{Differential Privacy and Training Data Influence Sparsity} We first show that differential privacy and training data influence sparsity are fundamentally at odds:

\begin{theorem} A model $\mathcal {M}$ becomes less $\varepsilon_i$-instance-interpretable as it becomes more $\varepsilon_p$-differentially private, and vice-versa.
\end{theorem} 
\begin{proof}
Let ${\mathbb{P}(\mathcal {M}(D) \in Y)}$ be abbreviated as $p$, ${\mathbb{P}(\mathcal {M}(D') \in Y)}
=\mathbb{P}(\mathcal {M}(D \setminus \{x_{\mathit{diff}}\} \in Y)$ be abbreviated as $p_d$, and let ${\mathbb{P}(\mathcal {M}(D'') \in Y)}
={\mathbb{P}(\mathcal {M}(D \setminus \{x'\} \in Y)}$ be abbreviated as $p_2$.
Assume that $\mathcal {M}$ is  $\varepsilon_i$-instance-interpretable and $\varepsilon_p$-differentially private.

If $\mathcal {M}$  is $\varepsilon_p$-differentially private, it holds that
\begin{equation}\label{eq1}
    \setlength\abovedisplayskip{7pt}
    \setlength\belowdisplayskip{7pt}
    \begin{aligned}
        &p&\leq&\exp(\varepsilon_p)\cdot p_d\\
        \Rightarrow\quad&\exp(\varepsilon_p)&\geq&\frac{p}{p_d}\\
    \end{aligned}
\end{equation}

If $\mathcal{M}$ is also $\varepsilon_i$-instance-interpretable, it also holds that
\begin{equation}\label{eq2}
    \setlength\abovedisplayskip{7pt}
    \setlength\belowdisplayskip{6pt}
    \begin{aligned}
        (i)&&p-p_d&>&\exp(\varepsilon_i)(p-p_2)\\
        (ii)&\Rightarrow\quad&p&>&\exp(\varepsilon_i)(p-p_2)+p_d\\
        (iii)&\Rightarrow\quad&\frac{p}{p_d}&>&\frac{\exp(\varepsilon_i)(p-p_2)+p_d}{p_d}\\
        (iv)&\Rightarrow\quad&\exp(\varepsilon_p)&>&\frac{\exp(\varepsilon_i)(p-p_2)}{p_d}+1\\
    \end{aligned}
\end{equation}

Step $(iv)$ follows from Equation~\ref{eq1}. We can now see from Equation~\ref{eq2} step $(iv)$ that $\varepsilon_p$ increases with increasing $\varepsilon_i$, i.e. the model becomes less differentially private as it becomes more instance-interpretable, and vice-versa.
\end{proof}

This result is not limited to the multilingual setting.

\paragraph{Differential Privacy and Linguistic Fairness} Fairness and differential privacy are occasionally at odds, as shown by \citet{bagdasaryan-etal-2019-differential, cummings-etal-2019, chang-shokri-2021, hansen-etal-2022-impact},\footnote{Several authors have considered practical trade-offs between privacy and fairness, including \citet{jagielski-etal-2019-private-fair}, \citet{lyu-etal-2020-differentially}, \citet{pannekoek2021investigating}, and \citet{liu2021fair}.} but in the multilingual setting, fairness and privacy can be compatible (for the common definitions above). We first note that there is a trivial solution to obtaining differential privacy and linguistic fairness (a joint optimum), namely randomness. This simply shows that the two objectives can be simultaneously satisfied. Next, imagine a perfectly compressed multilingual language model trained on a multi-parallel dataset.

\begin{theorem} If a model $\mathcal{M}_D$ trained on parallel data from ${|L|} \geq 2$ languages, $D=\{\ldots, i_1, \ldots, i_{|L|}, \ldots \}$, with $i_j$ and $i_q$ being translation equivalents, is perfectly multilingually compressed, then it is $\varepsilon_p$-differentially private.
\end{theorem} 

\begin{proof}
Since $\mathcal{M}_D$ is perfectly compressed, the representation of $i_j$ is identical to $i_q$ at any layer $l$, i.e., $\mathcal{M}_D^l(i_j)=\mathcal{M}_D^l(i_q)$. This gives us strong $k$-anonymity \citep{10.1145/2414456.2414474} in the representation space of $\mathcal{M}_D$, with $k = |L|$ and all dimensions as quasi-identifiers. Since $k$-anonymity is not obtained through a deterministic (reversible) procedure, but a randomly initialized learning procedure with random sampling, and since our attributes are randomly initialized, {$k$-anonymization} entails differential privacy in our setting.\footnote{The procedure also is not dependent on any individual input, because all individual data properties are either random (from initialization) or $k$-anonymous, by construction.}
$\mathcal{M}_D$, given perfect compression and convergence, is 0-differentially private, i.e., the probability distribution of $\mathcal{M}_D$ is unaffected by the removal of any single row.
\end{proof}

It follows directly from perfect compression that $\mathcal{M}_D$ is also linguistically fair because identical representations imply identical performance across languages. It is therefore an immediate corollary of the above result that a linguistically fair model can be differentially private. 

While the assumptions of a perfectly compressed model and clean multi-parallel dataset rarely hold up in practice and there is no obvious way to satisfy them while maintaining utility, the practical significance of this result is a reminder that multilingual training converges toward $k$-anonymization, and that safe $k$-anonymization of the representation space, if obtained, would provide us differential privacy. In the absence of strong guarantees, increasing the number of training languages (larger $k$) would strengthen privacy \citep{10.1145/2414456.2414474}. Our empirical results below (\S\ref{sec:results}) suggest that we can often obtain strong privacy and strong compression, but at the cost of performance.  

\section{Experimental Setup}
\label{sec:exp_setup}

In our experiments, we investigate the relation between the performance and multilingual compression of fine-tuned multilingual language models, and their privacy and training data influence. We rely on a commonly used multilingual pretrained language model, which we fine-tune with different levels of ($\varepsilon$, $\delta$)-differential privacy
on two common NLP tasks and evaluate using metrics of compression and training data influence.\footnote{For completeness, we explain the difference between $\varepsilon$-DP and ($\varepsilon$, $\delta$)-DP in Appendix~\ref{sec:eps_delta_dp}.} This section presents the pretrained language model, the tasks, the training protocol, the metrics of compression and training data influence, and the evaluation procedure. 

\paragraph{Model} We use a pretrained XLM-R Base \citep{conneau-etal-2020-unsupervised}, which is a 12-layer encoder-only transformer with {\raise.17ex\hbox{$\scriptstyle\sim$}}277M parameters and 250k vocabulary size trained on CC-100 (100 languages) via masked language modeling.

\paragraph{Tasks and Data} We fine-tune in a zero-shot cross-lingual transfer setting for part-of-speech (POS) tagging and natural language inference (NLI). 
Why these tasks? First, while POS tagging is driven by lower-level syntactic features, NLI requires a higher-level understanding \citep{lauscher-etal-2020-zero}. Second, we can leverage \emph{multi-parallel} corpora for multilingual fine-tuning and zero-shot cross-lingual transfer in both tasks, which helps eliminate confounders.\footnote{One limitation of this selection is that we only consider classification but no generative tasks, which could be worth exploring in the future.}

For POS tagging, we use the Parallel Universal Dependencies (PUD) treebank from Universal Dependencies (UD) v2.8 \citep{nivre-etal-2020-universal, zeman-etal-2021-ud}, which contains 1000 sentences parallel across 15 languages. We train in 7 of these languages (\textsc{fr}, \textsc{it}, \textsc{ja}, \textsc{pt}, \textsc{th}, \textsc{tr}, \textsc{zh}),\footnote{See Table~\ref{tab:languages} for language details.} exclude English,\footnote{We exclude English to keep the number of languages balanced and because the combined corpus is already biased towards Indo-European with Latin scripts (see Table~\ref{tab:languages}).} and use the remaining 7 languages (\textsc{ar}, \textsc{de}, \textsc{es}, \textsc{hi}, \textsc{id}, \textsc{ko}, \textsc{ru}) for validation. This split ensures that (1) we both train and evaluate on typologically diverse language samples, (2) there exist additional UD v2.8 treebanks in our validation set languages that we can harness for testing, and (3) there exist parallel sentences in our training set languages that we can harness to evaluate multilingual compression. We use the test splits of the following treebanks for testing: Arabic-PADT, German-GSD, Spanish-GSD, Hindi-HDTB, Indonesian-GSD, Korean-Kaist, and Russian-SynTagRus. Appendix Table~\ref{tab:treebank_sizes} lists the treebanks' sizes.\footnote{Regardless of test split size, each language contributes equally to the mean accuracy reported in Figure~\ref{fig:section3}.}

For NLI, we rely on the XNLI dataset \citep{conneau-etal-2018-xnli}, which contains (premise, hypothesis, label)-triplets multi-parallel across 15 languages. We, again, train in 7 of these languages (\textsc{bg}, \textsc{es}, \textsc{fr}, \textsc{hi}, \textsc{tr}, \textsc{vi}, \textsc{zh}), exclude the original English data, and validate in the remaining 7 languages (\textsc{ar}, \textsc{de}, \textsc{el}, \textsc{ru}, \textsc{sw}, \textsc{th}, \textsc{ur}). We train and validate our models on the original XNLI validation data (7500 examples per language), and we test the models on the original test data (15000 examples per language) in the validation set languages.

The idea to train and validate on the same sentences (in different languages) while testing on sentences from different treebanks (as we do for POS) or a different dataset split (as for XNLI) is to induce a slight distributional shift between validation and test data for the same language sample. This shift lets us evaluate the regularization strength of the gradient noise added by the DP-optimizer.

\paragraph{Training} We employ the standard fine-tuning procedures for token classification (POS) and sequence classification (XNLI) proposed by \citet{devlin-etal-2019-bert}. Similar to \citet{li-etal-2021-dp-lm}, we use DP-AdamW (i.e., the DP-SGD algorithm \citep{abadi-etal-2016-dpsgd} applied to the AdamW optimizer with default hyperparameters \citep{loshchilov2018decoupled, kingma-ba-2015-adam}) to train with ($\varepsilon$, $\delta$)-DP. We evaluate 6 different privacy budgets with ${\varepsilon \in \{1, 3, 8, 15, 30, \infty\}}$.\footnote{$\varepsilon = \infty$ refers to the standard, non-private setting.} We set $\delta = \frac{\num{1e-4}}{|D_{train}|}$ for POS, where $|D_{train}|=7000$ is the length of the training dataset, and $\delta = \num{1e-6}$ for XNLI.\footnote{We deliberately use a larger $\delta$ for XNLI because it turned out to be much harder to achieve convergence than for POS. Even with the looser DP bounds from $\delta = 1e-6$, we were unable to find a hyper-parameter setting for $\varepsilon = 1$ where the fine-tuned model was substantially better than random guessing.} The noise multiplier $\sigma$ corresponding to a particular ($\varepsilon$, $\delta$)-budget is determined numerically before training through binary search. Our implementation builds upon the optimized Opacus \citep{yousefpour-etal-2021-opacus} privacy engine by \citet{li-etal-2021-dp-lm}.\footnote{\href{https://github.com/lxuechen/private-transformers}{https://github.com/lxuechen/private-transformers}}\textsuperscript{,}\footnote{We do not use ghost clipping, their proposed technique to fit larger batches on the GPU at the cost of training time, as we can still fit sufficiently large batches on our GPUs without.} We use the R\'{e}nyi differential privacy \citep[RDP; ][]{mironov-2017-rdp, mironov-etal-2019-sampled} accountant with conversion to ($\varepsilon$, $\delta$)-DP \citep{canonne-etal-2020-discrete}. Hyper-parameter tuning on private data---which the POS and XNLI data in our study simulate---has been shown to incur additional privacy leakage \citep{liu-talwar-2019-private, papernot-steinke-2021-hyperparameter}. Therefore, we try to keep hyper-parameter tuning to a minimum and rely on sensible priors to select a suitable range of hyper-parameters. For POS, we find that the range of good hyper-parameters for non-private settings transfers well to private settings if we just use slightly higher learning rates. For XNLI, we select hyper-parameters such that the sampling rate matches that used by \citet{li-etal-2021-dp-lm} for NLI tasks in the GLUE benchmark \citep{wang-etal-2018-glue}.\footnote{The sampling rate $q = \frac{B_{\text{train}}}{|D_{\text{train}}|}$, $B$ denoting the batch size.} Accordingly, we train with a maximum sequence length of 128 for 10 epochs with a total batch size of 96 for POS and 30 epochs with batch size 512 for XNLI.\footnote{Note that using fixed-size batches technically breaks the privacy guarantees of RDP based on the Sampled Gaussian Mechanism \citep{mironov-etal-2019-sampled}. We follow the convention of using fixed-size batches, avoiding potential out-of-memory GPU issues, as a proxy for the true privacy spending and performance (see \citep{li-etal-2021-dp-lm} and Appendix~D.4 in \citep{tramer-2021-differentially}).} At each privacy budget, we train models (3 random initializations each) with 6 learning rates for POS (\num{1e-4}, \num{3e-4}, \num{5e-4}, \num{7e-4}, \num{1e-5}, \num{5e-5}, \num{7e-5}, \num{1e-6}) and 3 learning rates for XNLI (\num{3e-4}, \num{4e-4}, \num{5e-4} for private models and \num{9e-5}, \num{1e-4}, \num{2e-4} for non-private models). Based on the validation accuracy we then select the 5 best settings for each privacy level and task, listed in Appendix~\ref{sec:best_settings}. The learning rate is linearly decayed after 50 warm-up steps for POS and without warm-up for XNLI. We perform gradient clipping (per-sample in private settings) with a threshold of 0.1. Weight decay is set to 0.01.

\paragraph{Quantifying Multilingual Compression}
\label{sec:metrics_compression}
We present four metrics of multilingual compression: A common proxy task to measure the quality of cross-lingual representations is sentence retrieval \citep{artetxe-schwenk-2019-massively, dufter-schutze-2020-identifying, libovicky-etal-2020-language, ravishankar-sogaard-2021-impact, liu-etal-2021-preserving, maronikolakis-etal-2021-wine-v}. \citet{dufter-schutze-2020-identifying} quantify the degree of multilingual compression using bidirectional sentence retrieval precision
as follows:\footnote{Note that \citet{dufter-schutze-2020-identifying} also consider word alignment in their multilinguality score. We omit this task as it is not trivial to obtain ground truth alignments in our setup.}
\begin{equation}
\label{eq:retrieval_precision}
\setlength\abovedisplayskip{7pt}
\setlength\belowdisplayskip{7pt}
\mathrm{P} = \frac{1}{2m}\sum_{i=1}^{m}\mathds{1}_{\argmax_k R_{ik}=i} + \mathds{1}_{\argmax_k R_{ki}=i}.
\end{equation}

Here, $R \in \mathbb{R}^{m \times m}$ denotes the matrix of cosine similarities $R_{ij} = \mathrm{cos}(e_i^q,e_j^r)$ between the $m$ sub-word representations $e_i^q$ and $e_j^r$ from a LM at indices $i$ and $j$ for a set of parallel sentences in the languages $q$ and $r$.\footnote{The sub-word representations are taken from the LM's layer $l$ and mean-pooled over the sequence length (excluding special tokens).}

\citet{kornblith-etal-2019} propose to use linear centered kernel alignment (CKA) as a similarity index for neural network representations. It is defined as 
\begin{equation}
\setlength\abovedisplayskip{10pt}
\setlength\belowdisplayskip{10pt}
\mathrm{CKA}(X, Y) = \frac{\norm{Y^{\text{T}}X}_{F}^{2}}{\norm{X^{\text{T}}X}_{F}\norm{Y^{\text{T}}Y}_{F}}.
\end{equation}
For LMs, the matrices $X$ and $Y$ are obtained by mean-pooling $n$ sub-word representations at model layer $l$ \citep{conneau-etal-2020-emerging, glavas-vulic-2021-supervised}. Typically, $X$ and $Y$ correspond to the representations from two different models for identical examples \citep{kornblith-etal-2019, phang-etal-2021-fine}. We instead use the representations from a single model for a parallel sentence pair $(s_q, s_r)$ in languages $q$ and $r$ as $X$ and $Y$, respectively, to study the similarity of representations across languages, similar to \citet{muller-etal-2021-first} and \citet{conneau-etal-2020-emerging}. \citet{anonymous2022enhancing} also use CKA as a metric of compression. 

IsoScore \citep{rudman-etal-2021-iso} is an isotropy metric, computed as outlined in Appendix~\ref{a:definition_isoscore}, that quantifies the degree to which a point cloud uniformly utilizes the vector space.
In our context, this point cloud corresponds to the $n$ sub-word representations of all examples in a corpus at layer $l$. Prior work has shown that anisotropic representation spaces, such as the embedding spaces of large LMs \citep{ethayarajh-2019-contextual}, suffer from so-called \emph{representation degeneration} \citep{gao-etal-2019-representation}, and that the isotropy of a model's representation space correlates with its task performance \citep[][\textit{inter alia}]{zhou-etal-2019-getting, wang-etal-2020-improv, zhou-etal-2021-isobn, rajaee-pilehvar-2021-cluster}. High isotropy also means languages are not compartmentalized and should therefore correlate with high compression.

Representational similarity analysis \citep[RSA; ][]{kriegeskorte-etal-2008, edelman-1998} was originally introduced in the field of cognitive neuroscience to analyze the similarity of fMRI activity patterns, but it is also applicable to neural network representations \citep[][\textit{inter alia}]{bouchacourt-baroni-2018-agents, chrupala-2019-symbolic, chrupala-alishahi-2019-correlating, lepori-mccoy-2020-picking, he-etal-2021-effectiveness}, e.g., to analyze their similarity across languages. RSA measures the similarity between the representational geometries (i.e., the arrangement in the vector space) of two sets of representations. The representational geometry is determined through pairwise (dis)similarity metrics, and similarity is typically measured using a rank-based correlation metric such as Spearman's~$\rho$ \citep{diederichsen-kriegeskorte-2017-representational}.

\paragraph{Quantifying Training Data Influence}
Training data influence metrics can help us gain an understanding of the inner workings of a model \citep[][\textit{inter alia}]{10.5555/3305381.3305576, yeh-etal-2018-representer, charpiat-etal-2019-input, koh-etal-2019-accuracy, pruthi-etal-2020, basu-etal-2020-second, k-soegaard-2021-revisiting, zhang-etal-2021-sample, kong-chaudhuri-2021-understanding}. Such metrics are approximations of leave-one-out-influence. \citet{pruthi-etal-2020} proposed a both effective and practical method, called $\mathrm{TracInCP}$,\footnote{``CP'' stands for checkpoint; the method approximates $\mathrm{TracInIdeal}$, which is impractical to compute, through model checkpoints taken during training \citep{pruthi-etal-2020}.} to compute the influence of a training example $z$ on the model's prediction for another example $z'$, which could be a test example or $z$ itself (called the self-influence). The influence is computed as follows: 
\begin{equation}
    \setlength\abovedisplayskip{7pt}
    \setlength\belowdisplayskip{7pt}
    \mathrm{TracInCP}(z, z') = \sum_{i=1}^{k} \eta_i \nabla \ell (\theta_i, z) \cdot \nabla \ell (\theta_i, z'),
\end{equation} where $\eta_i$ is the learning rate and $\nabla \ell (\theta_i, z)$ is the gradient of the loss w.r.t. the model parameters $\theta_i$ and inputs $z$ for the $i$-th model checkpoint. We will use $\mathrm{TracInCP}$ as an approximation of training data influence in our experiments.

\begin{figure*}[t!]
    \centering
    \begin{subfigure}[b]{0.19\textwidth}
        \centering
        \includegraphics[width=\textwidth]{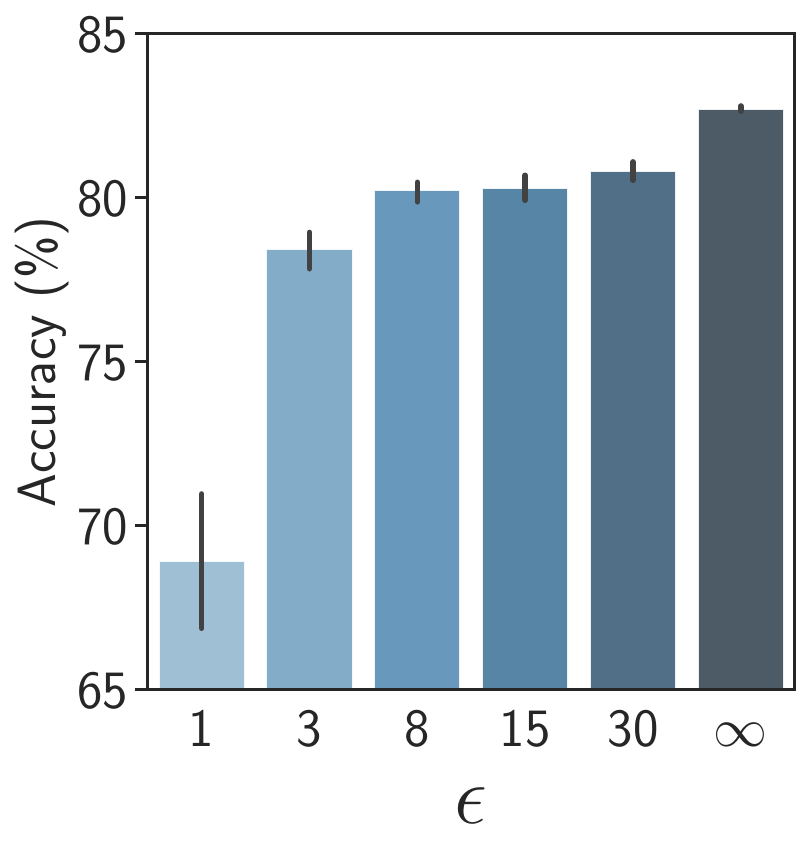}
        \caption{POS Performance}
        \label{fig:pos_acc}
    \end{subfigure}
    \begin{subfigure}[b]{0.192\textwidth}
        \centering
        \includegraphics[width=\textwidth]{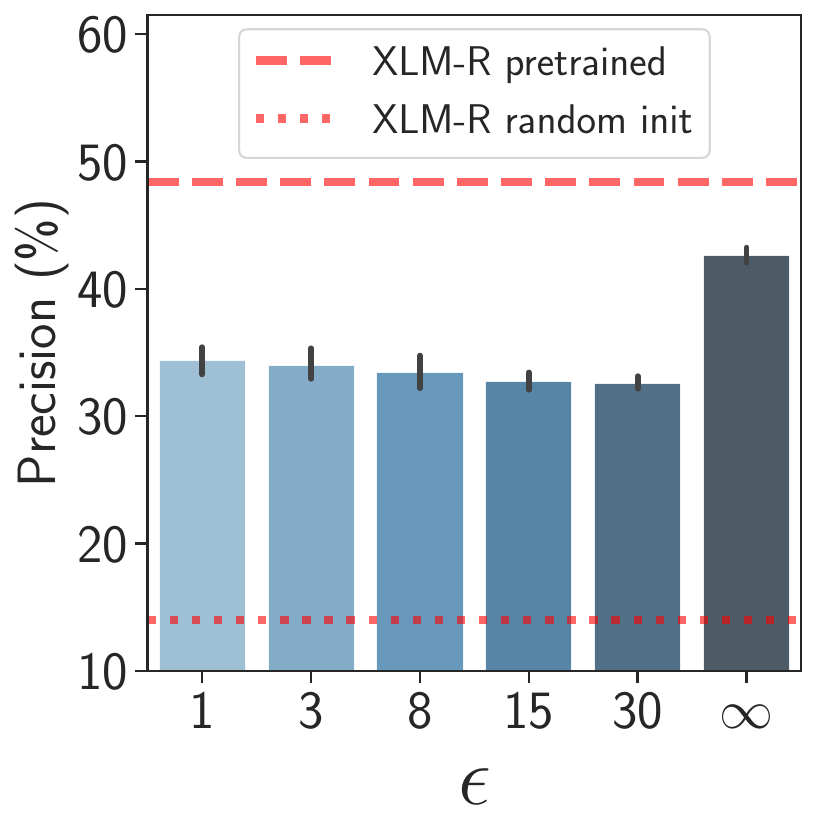}
        \caption{POS Retrieval}
        \label{fig:pos_compression}
    \end{subfigure}
    \begin{subfigure}[b]{0.1984\textwidth}
        \centering
        \includegraphics[width=\textwidth]{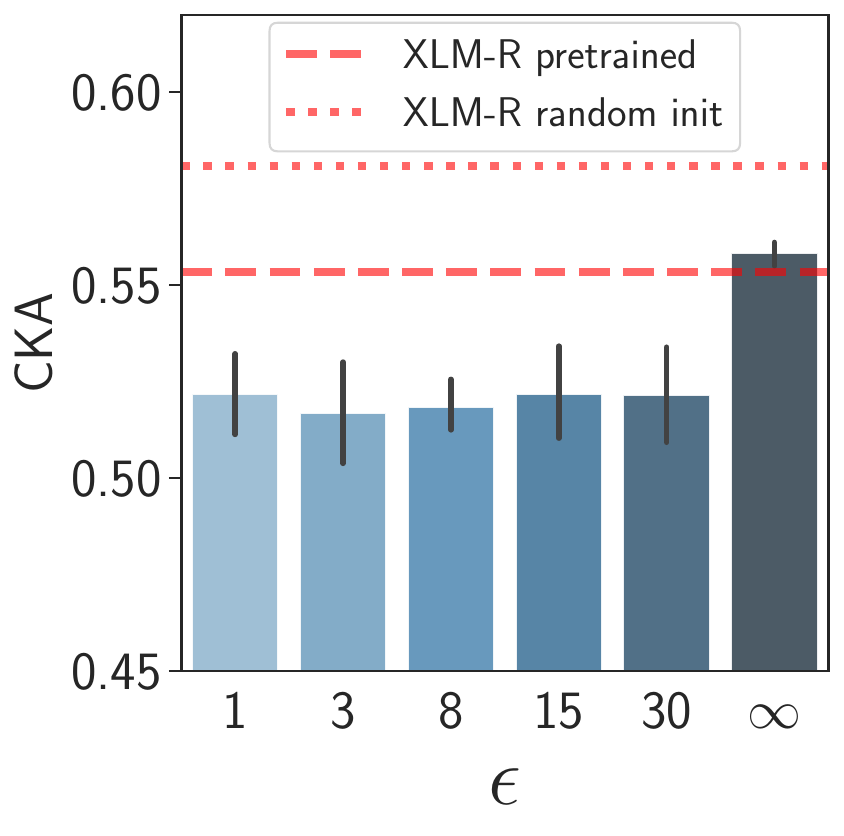}
        \caption{POS CKA}
        \label{fig:pos_cka}
    \end{subfigure}
    \begin{subfigure}[b]{0.196\textwidth}
        \centering
        \includegraphics[width=\textwidth]{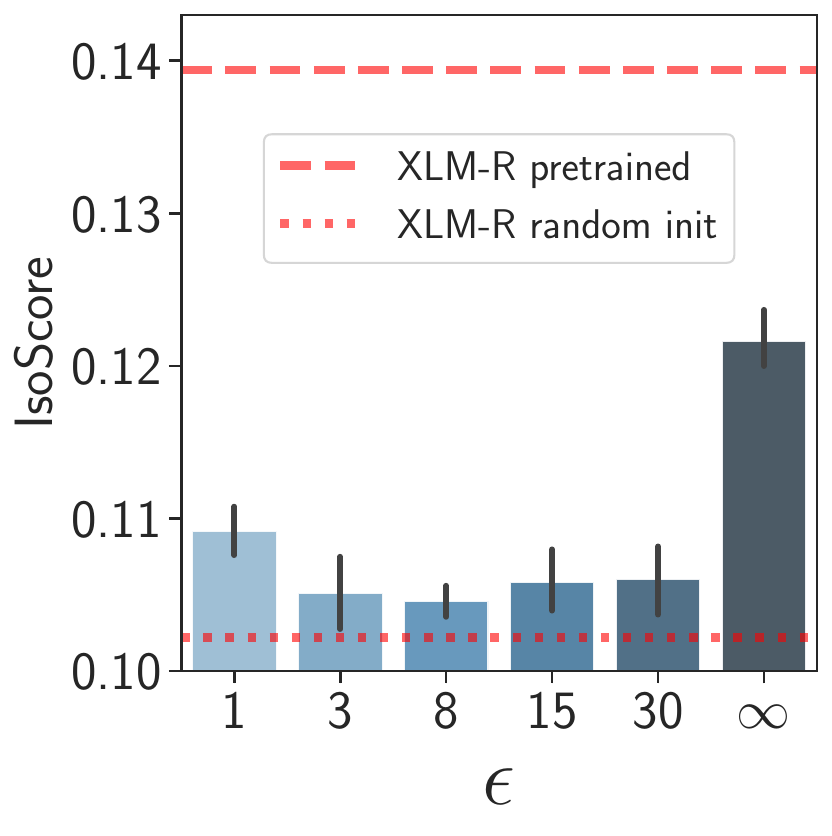}
        \caption{POS IsoScore}
        \label{fig:pos_iso}
    \end{subfigure}
    \begin{subfigure}[b]{0.1992\textwidth}
        \centering
        \includegraphics[width=\textwidth]{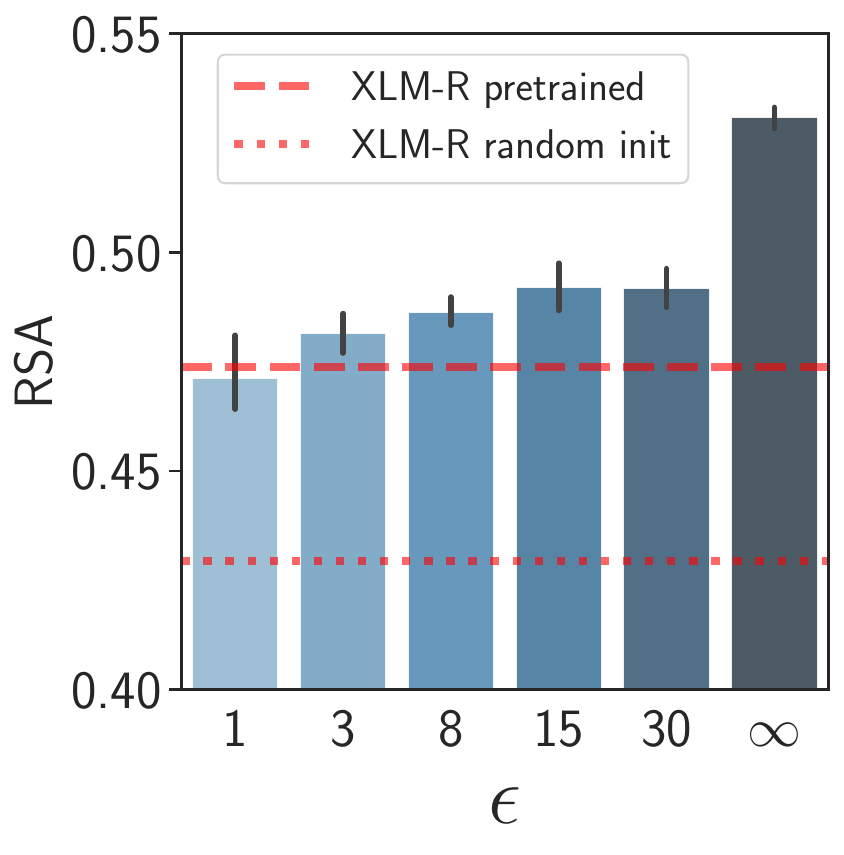}
        \caption{POS RSA}
        \label{fig:pos_rsa}
    \end{subfigure}
    \begin{subfigure}[b]{0.194\textwidth}
        \centering
        \includegraphics[width=\textwidth]{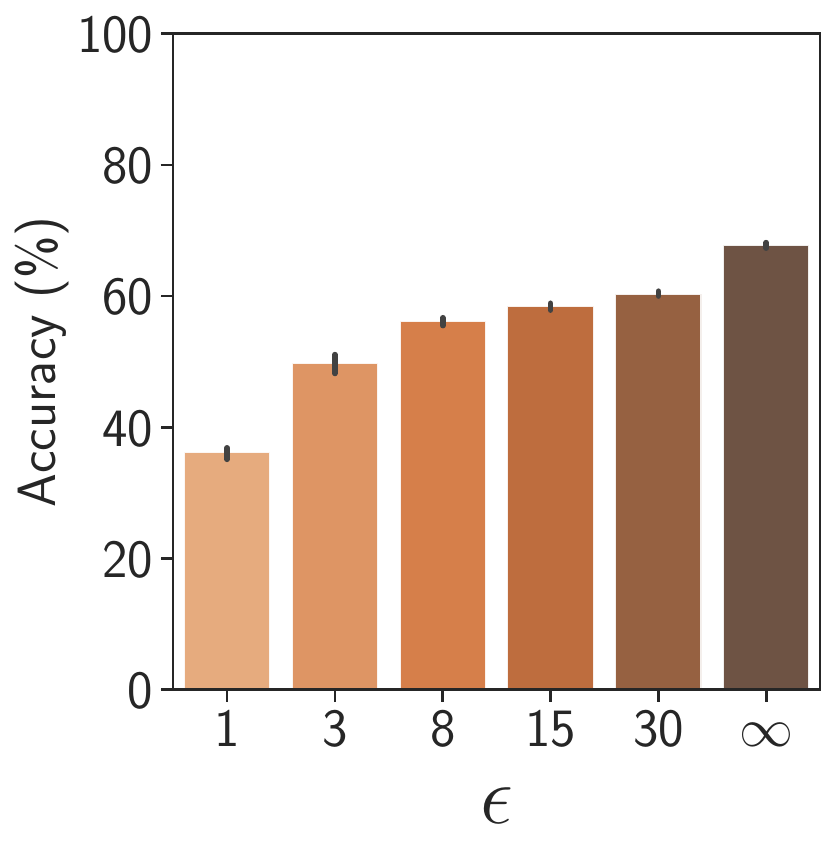}
        \caption{XNLI Performance}
        \label{fig:xnli_acc}
    \end{subfigure}
    \begin{subfigure}[b]{0.19\textwidth}
        \centering
        \includegraphics[width=\textwidth]{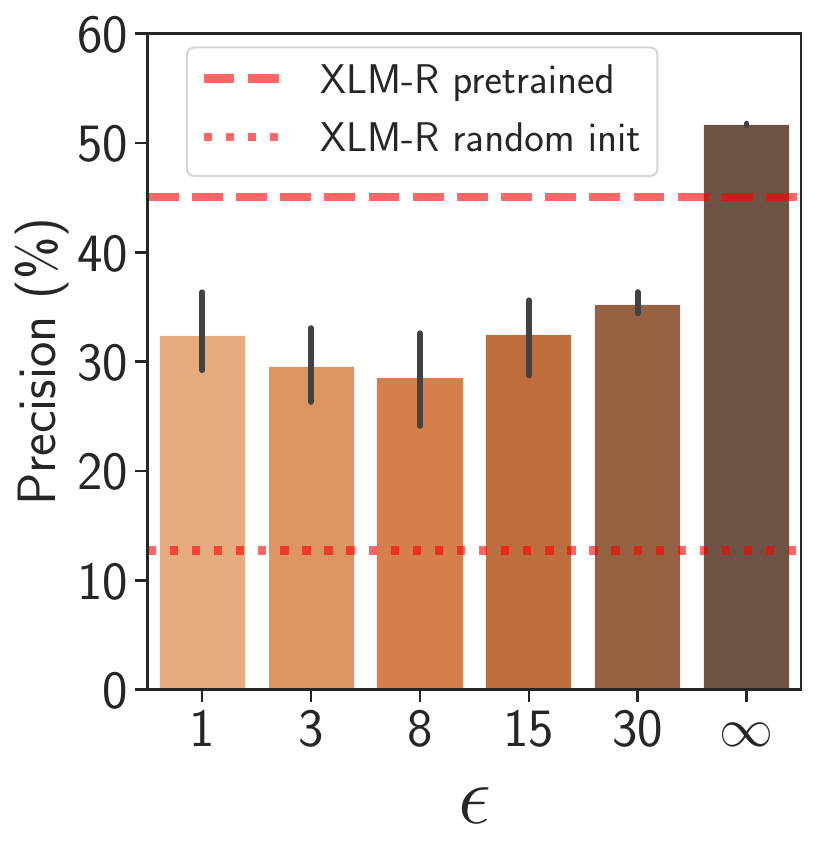}
        \caption{XNLI Retrieval}
        \label{fig:xnli_compression}
    \end{subfigure}
    \begin{subfigure}[b]{0.198\textwidth}
        \centering
        \includegraphics[width=\textwidth]{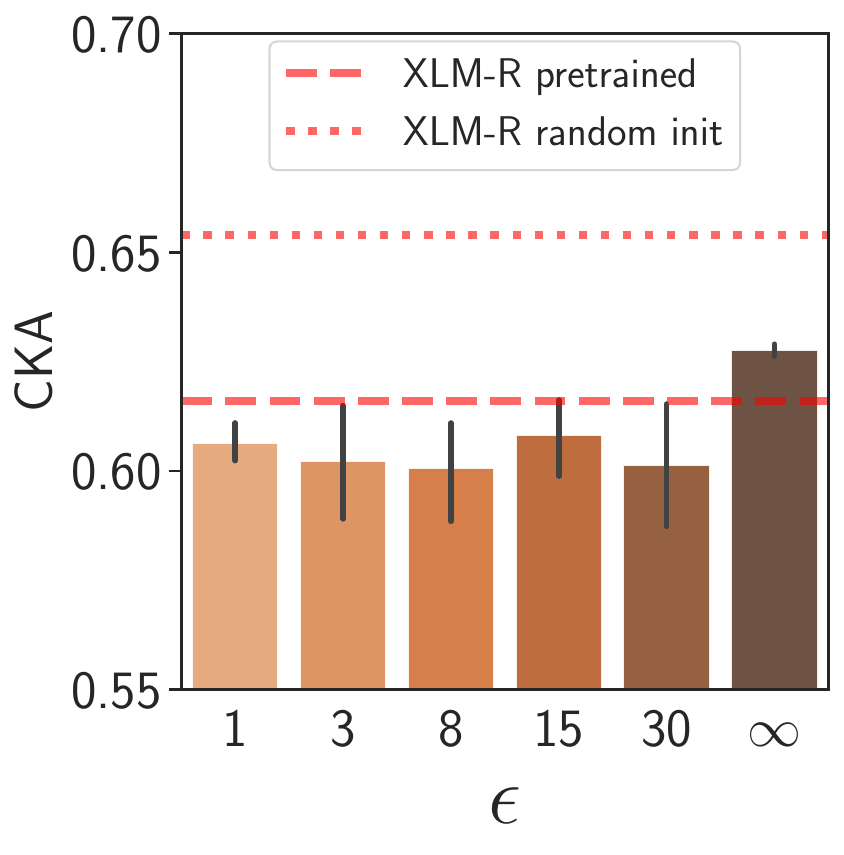}
        \caption{XNLI CKA}
        \label{fig:xnli_cka}
    \end{subfigure}
    \begin{subfigure}[b]{0.194\textwidth}
        \centering
        \includegraphics[width=\textwidth]{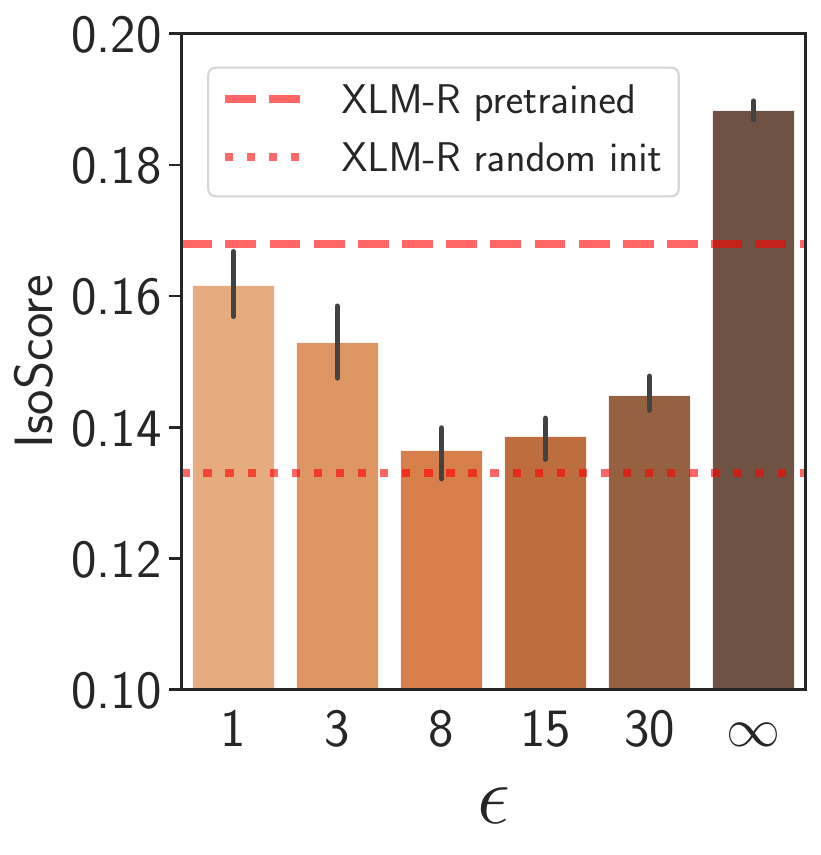}
        \caption{XNLI IsoScore}
        \label{fig:xnli_iso}
    \end{subfigure}
    \begin{subfigure}[b]{0.199\textwidth}
        \centering
        \includegraphics[width=\textwidth]{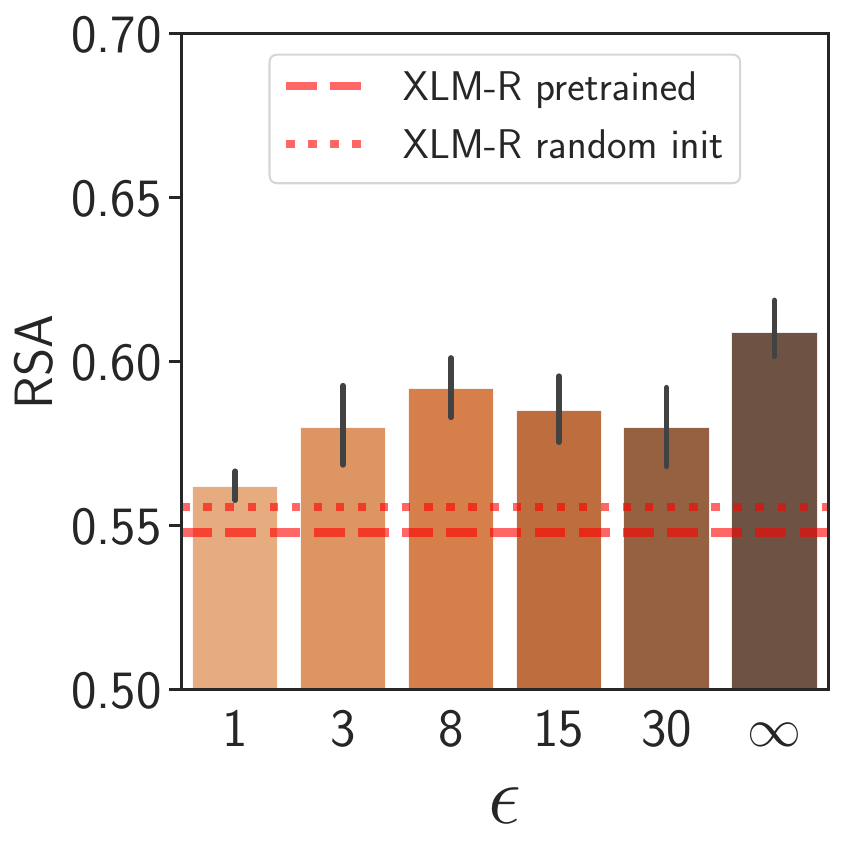}
        \caption{XNLI RSA}
        \label{fig:xnli_rsa}
    \end{subfigure}
    \caption{Task performance, sentence retrieval, CKA, IsoScore, and RSA results when fine-tuning with different privacy guarantees ($\infty$=non-private). We add the original pretrained XLM-R and XLM-R with randomly initialized weights for comparison. The results show how non-private fine-tuning balances multilingual compression and task performance. Strongly private fine-tuning ($\varepsilon=1$) is compatible with high compression (retrieval, CKA, IsoScore), but not with task performance. For medium levels of privacy (e.g., $\varepsilon=8$), we see the result of balancing privacy and task performance at the expense of multilingual compression.}
    \label{fig:section3}
\end{figure*}

\paragraph{Evaluation} We evaluate our models both during and after fine-tuning. For POS, we evaluate every 100 steps, and for XNLI, every 200 steps. We measure zero-shot cross-lingual transfer performance on the validation and test data by accuracy (token-level for POS and sequence-level for XNLI). To account for randomness, we take the mean of the best 5 seeds for each privacy budget.

The measures of multilingual compression (sentence retrieval precision, CKA, IsoScore, RSA) are computed using distinct evaluation corpora comprising parallel sentences for all languages pairs in the respective training set language sample. For models trained on XNLI, we use 3000 sentence pairs per language pair from the TED~2020 corpus \citep{reimers-gurevych-2020-making} and 3500 pairs from the WikiMatrix dataset \citep{schwenk-etal-2021-wikimatrix}. For models trained for POS, we use 3500 pairs from TED~2020, 3500 pairs from WikiMatrix, and 900 pairs from Tatoeba,\footnote{\href{https://tatoeba.org}{https://tatoeba.org}}\textsuperscript{,}\footnote{We extract sentence pairs from Tatoeba using the tatoebatools library (\href{https://github.com/LBeaudoux/tatoebatools}{https://github.com/LBeaudoux/tatoebatools}).}\textsuperscript{,}\footnote{We exclude \textsc{th} from the WikiMatrix and Tatoeba evaluation sets for POS as there are insufficiently many sentence pairs available between \textsc{th} and the remaining languages.} numbers chosen based on availability and memory usage.

Following \citet{dufter-schutze-2020-identifying}, we evaluate the models at layers 0 and 8, which complement each other well with regard to the properties they capture, e.g., multilinguality and task-specificity \citep{choenni2020does, de-vries-etal-2020-whats, muller-etal-2021-first}.
We compute the sentence retrieval precision between language pairs and take the mean.\footnote{Sentence retrieval is bidirectional (see Eq.~\ref{eq:retrieval_precision}). Given $|L|$ languages, we therefore average over the full $\mathbb{R}^{|L| \times |L|}$ language pair matrix, only excluding the main diagonal.}
The IsoScore is computed for the contextualized representations of all examples in the respective corpus at once. In contrast, CKA and RSA scores are also computed per language pair, and then averaged across those.\footnote{CKA and RSA are symmetrical. Given $|L|$ languages, we thus only use the upper triangle of the $\mathbb{R}^{|L| \times |L|}$ language pair matrix, still excluding the main diagonal.}
For RSA, we use ${\mathrm{D} = 1 - \text{Spearman's}\,\rho}$ and ${\mathrm{S} = \text{Spearman's}\,\rho}$ as the dissimilarity and similarity metrics, respectively.\footnote{This is consistent with the results of \citet{zhelezniak-etal-2019-correlation} and \citet{lepori-mccoy-2020-picking} showing that ${\text{Spearman's}\,\rho}$ is more suitable for RSA with embeddings than conventional similarity metrics such as cosine similarity.} Finally, we average results for all four metrics across TED~2020, WikiMatrix, and Tatoeba, the two layers, and the 5 best seeds for each privacy budget. For comparison, we also compute all metrics for the original pretrained and a randomly initialized XLM-R model.

\section{Results} 
\label{sec:results}

\paragraph{Privacy, Compression, Performance}
\label{sec:question1}
We now empirically investigate the relationship between differential privacy, multilingual compression, and cross-lingual transfer performance. We present aggregated results in Figure~\ref{fig:section3} and non-aggregated results in Appendix~\ref{sec:details_q1}. We observe that the zero-shot accuracy {\em decreases} as we fine-tune with stronger privacy guarantees (Figures~\ref{fig:pos_acc} and \ref{fig:xnli_acc}), which is expected due to the \textit{privacy--utility tradeoff} \citep{geng-etal-2020-analysis}. In particular, the relatively small sizes of our training datasets make private LM fine-tuning more challenging  \citep{kerrigan-etal-2020-differentially, habernal-2021-differential, senge-etal-2021-size, yu-etal-2021-differentially} because, for a fixed number of update steps, the gradient noise added per update step grows as the size of the training dataset decreases \citep{tramer-2021-differentially, mcmahan-etal-2018-learning}. Note that although the private models tend to underperform the non-private models by a large margin on the validation set ($>$30\% for XNLI, as shown in Appendix Table~\ref{tab:xnli_full_acc}), the performance gap on the test set is noticeably smaller, showing that training with differential privacy, like other noise injection schemes \citep{6796505}, is also a form of regularization. 

Figures~\ref{fig:pos_compression} and \ref{fig:xnli_compression} display sentence retrieval precision when fine-tuning with different privacy budgets. The highest compression is achieved by the non-private models. The second-highest compression is achieved for $\varepsilon = 1$, our most private models. Both suggest non-linear privacy--compression interactions, with POS showing lowest compression for $\varepsilon=30$ (or higher) and XNLI showing lowest compression for $\varepsilon=8$.
The results are very similar for IsoScore (Figures~\ref{fig:pos_iso}, \ref{fig:xnli_iso}) and also similar, albeit less pronounced for CKA (Figures~\ref{fig:pos_cka}, \ref{fig:xnli_cka}).\footnote{The randomly initialized XLM-R model shows high CKA scores. This is explained by the high dimensionality ($d=768$) of the contextualized representations, considering that CKA saturates with increasing network width \citep{kornblith-etal-2019}, and the high centroid similarity of random activations.} RSA, in contrast, exhibits very low scores for highly private models; see Appendix~\ref{sec:rsa_details}.

These results show that we can achieve \textit{strong compression and strong performance at the cost of privacy} ($\varepsilon = \infty$), \textit{strong compression and strong privacy at the cost of performance} ($\varepsilon = 1$), or \textit{trade-off performance and privacy at the cost of compression} (e.g., $\varepsilon = 8$). It may seem counter-intuitive that multilingual compression and cross-lingual transfer performance are not strictly correlated. However, in the fine-tuning setting, we can sacrifice task-specific knowledge in favor of multilingual compression, which leads to poor performance. Vice-versa, a model may exploit spurious correlations in the data to make correct predictions without actually relying on cross-lingual signal.
An example for the former case is the pretrained (but not fine-tuned) XLM-R, which scores highly in multilingual compression (as displayed in Figure~\ref{fig:section3}) but has poor cross-lingual transfer performance in the downstream tasks.

We also find that in some fine-tuning settings, e.g., $\varepsilon = \infty$, the multilingual compression surpasses that of the pretrained XLM-R. While \citet{{liu-etal-2021-preserving}} have previously shown that sentence retrieval performance typically drops (i.e., compression worsens) over the course of fine-tuning (which we confirm in Appendix Fig.~\ref{fig:retrieval_steps}), this finding clearly shows that there are exceptions. Future work may investigate this further. 

Lastly, retrieval and CKA scores are always highest between typologically similar languages and languages over-represented in pretraining (see Table~\ref{tab:languages} for a comparison across languages) \emph{across all levels of privacy}, as shown by the non-aggregated results in the Appendix Figures~\ref{fig:pos_retrieval_tedwm}--\ref{fig:xnli_cka_full}. This finding thus extends conclusions from prior work \citep{pires-etal-2019-multilingual, wu-dredze-2019-beto, k-etal-2020-ability, lauscher-etal-2020-zero} to private models.

\begin{figure*}[ht!]
    \centering
    \begin{subfigure}[b]{0.23\textwidth}
        \centering
        \includegraphics[width=\textwidth]{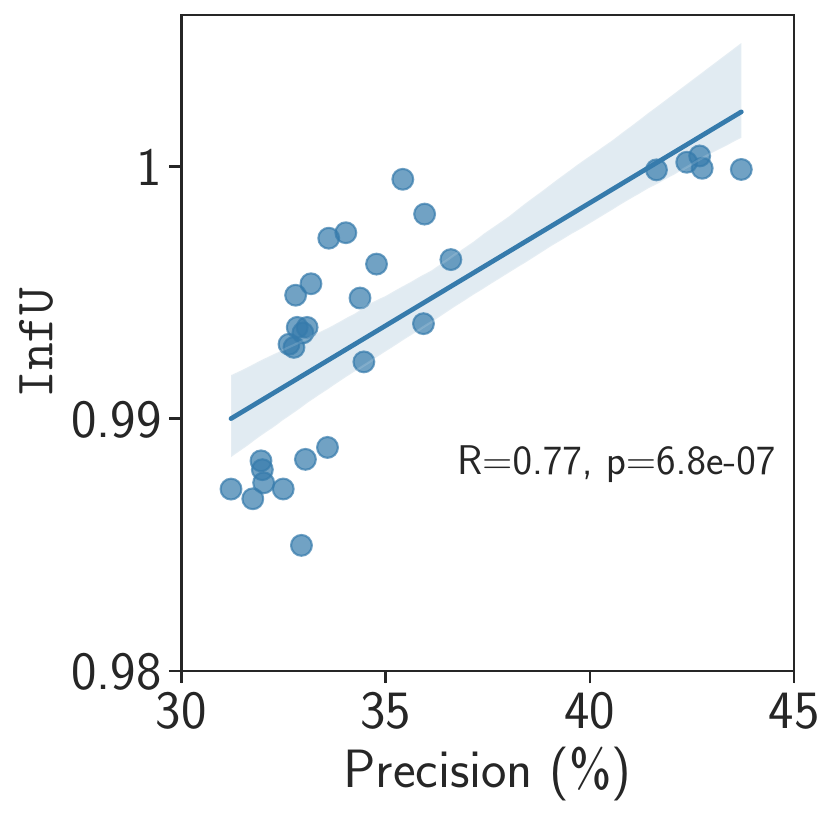}
        \caption{POS $\mathrm{InfU}$ -- Retrieval}
        \label{fig:pos_infu_retrieval}
    \end{subfigure}
    \begin{subfigure}[b]{0.24\textwidth}
        \centering
        \includegraphics[width=\textwidth]{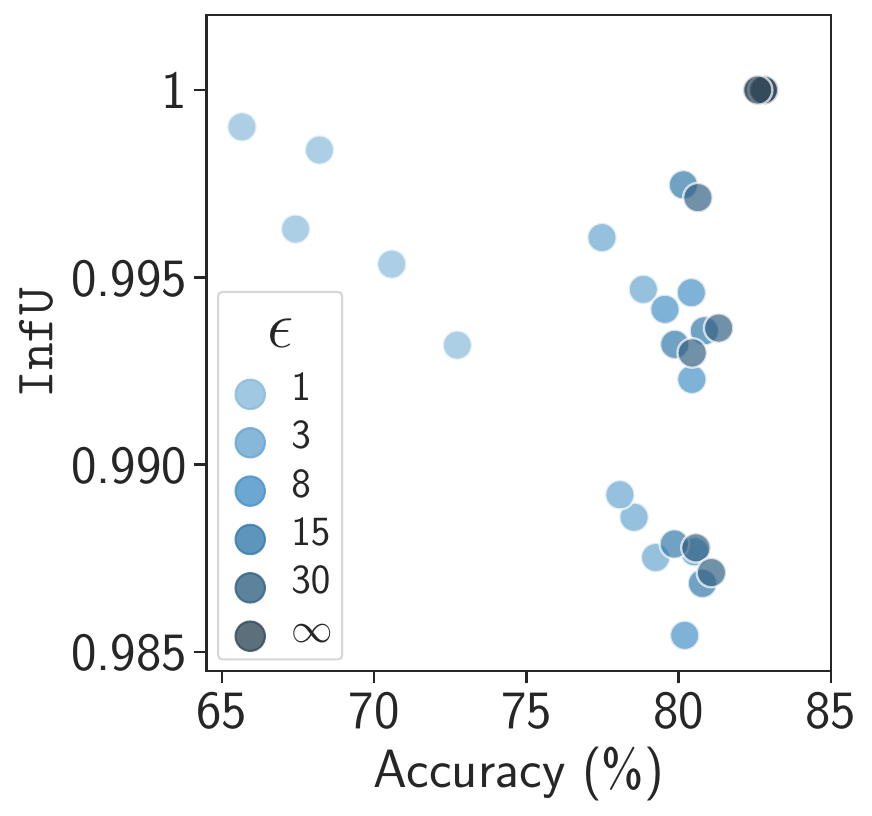}
        \caption{POS $\mathrm{InfU}$ -- Perf.}
        \label{fig:pos_infu_acc}
    \end{subfigure}
    \begin{subfigure}[b]{0.23\textwidth}
        \centering
        \includegraphics[width=\textwidth]{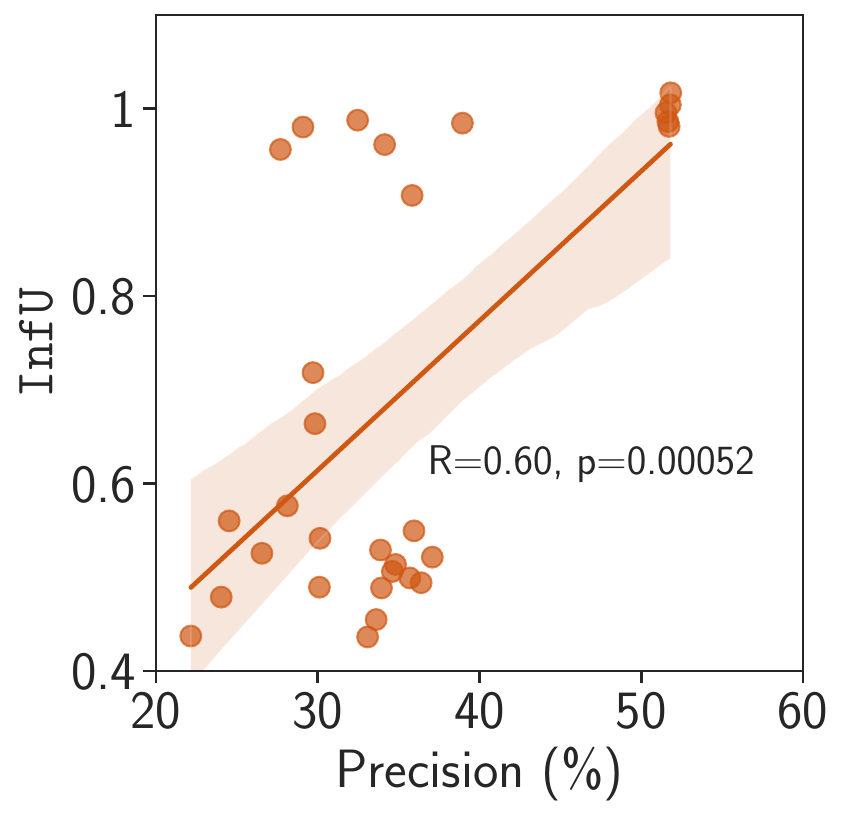}
        \caption{XNLI $\mathrm{InfU}$ -- Retrieval}
        \label{fig:xnli_infu_retrieval}
    \end{subfigure}
    \begin{subfigure}[b]{0.23\textwidth}
        \centering
        \includegraphics[width=\textwidth]{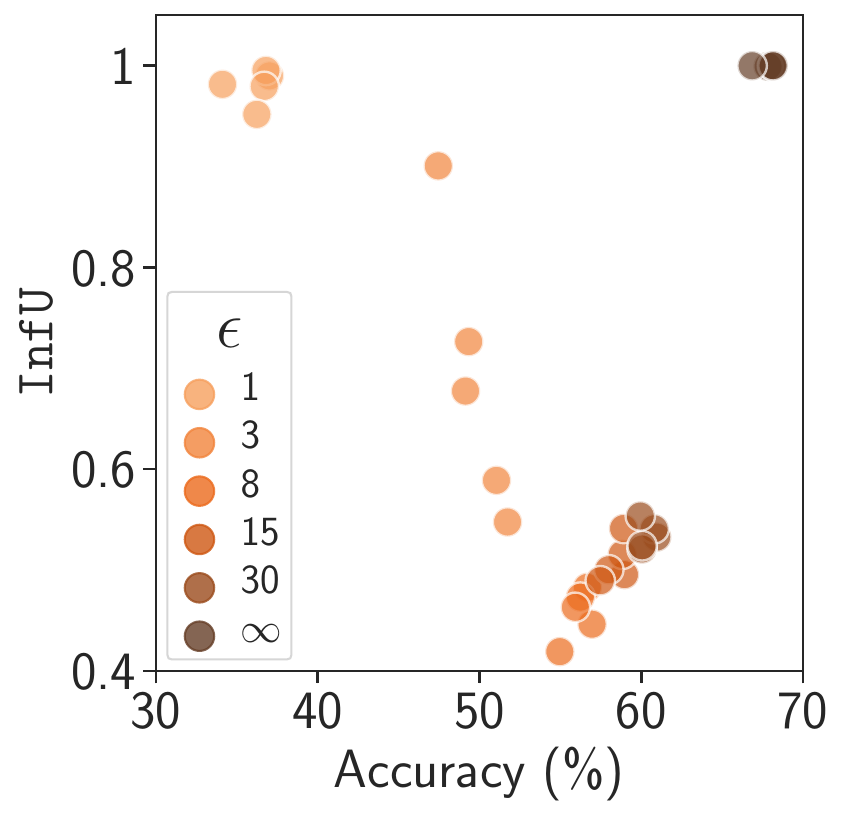}
        \caption{XNLI $\mathrm{InfU}$ -- Perf.}
        \label{fig:xnli_infu_acc}
    \end{subfigure}
    \caption{Linear fit and Pearson correlation between the influence uniformity $\mathrm{InfU}$ and sentence retrieval precision (\ref{fig:pos_infu_retrieval}, \ref{fig:xnli_infu_retrieval}) and  $\mathrm{InfU}$ versus downstream performance for different levels of privacy (\ref{fig:pos_infu_acc}, \ref{fig:xnli_infu_acc}). We see significant positive correlations between retrieval precision and $\mathrm{InfU}$, suggesting a negative correlation between multilingual compression and training data influence sparsity. For task performance, we see the trade-off between training data influence sparsity ($\mathrm{InfU}$) and privacy, which aligns with our theoretical expectations (\S\ref{sec:background}).
    }
    \label{fig:interpretability_plots}
\end{figure*}

\section{More multilingual, less interpretable?}
\label{sec:empirical_interpretability}

\paragraph{Metric} To answer this question, we introduce $\mathrm{InfU}$ (\textbf{Inf}luence \textbf{U}niformity), 
a measure of uniformity based on $\mathrm{TracInCP}$ influence scores for each training example in the multiparallel dataset $D=\{\ldots, i_1, \ldots, i_{|L|}, \ldots \}$, with $i_j$ and $i_k$ translation equivalents.
We compute $\mathrm{InfU}$ for $\mathcal{M}$ and the translation equivalents $i$ = \{$i_1, \ldots, i_{|L|}\}$ as follows:

\vspace{-4mm}
\begin{equation}
\setlength\abovedisplayskip{-5pt}
\setlength\belowdisplayskip{7pt}
\mathrm{InfU}
(i) = \frac{1}{|L|}\sum_{k = 1}^{|L|} \text{H}(\sigma(\mathrm{TracInCP}(i_k, i)))
\end{equation} 

where $H$ is the entropy with $log_{|L|}$ and $\sigma$ is a softmax used to obtain a probability distribution over influence scores. $\mathrm{InfU}$ is maximized ($\mathrm{InfU} = 1$) for uniform influence scores, fulfilling $\mathrm{TracInCP}(i_{j}, i_{k}) = \mathrm{TracInCP}(i_{q}, i_{r})$, $\forall j,k,q,r \in L$. This means a perfectly multilingual model that yields equivalent representations for translation equivalent examples obtains $\mathrm{InfU} = 1$. 
In this scenario of maximum uniformity our model is also the least instance-interpretable because training data influence is minimally sparse, so we cannot easily identify influential examples for a prediction. 
We use $\mathrm{InfU}$ to study to what extent influence sparsity aligns with the metrics privacy and cross-lingual performance.

\paragraph{Setup} We use 1000 training examples and compute $\mathrm{TracInCP}$ scores from the last 3 model checkpoints, taken every 100 steps, with their corresponding learning rates.\footnote{Since the learning rate changes every training step, we use the learning rate from the end of each checkpointing interval.}

\paragraph{Results and Analysis} We plot the mean $\mathrm{InfU}$ against the mean sentence retrieval precision for our fine-tuned models and compute Pearson's $\mathrm{R}$ in Figures~\ref{fig:pos_infu_retrieval} and \ref{fig:xnli_infu_retrieval}. For both tasks, there is a significant ($p < 0.05$) strong positive correlation between the $\mathrm{InfU}$ score and multilingual compression as determined through sentence retrieval. This supports the idea that \emph{multilingual compression is at odds with training data influence}. See also how highly private and low-performing models score highly in $\mathrm{InfU}$ (Figures~\ref{fig:pos_infu_acc}, \ref{fig:xnli_infu_acc}); and non-private and high-performing models do the same. For medium levels of privacy we, however, see a trade-off characterized by lower $\mathrm{InfU}$, i.e., better instance-interpretability, and medium performance.
Strong \emph{privacy} guarantees, sparse training data influence estimates, and performance are incompatible, because the high-performing models are strictly low in privacy and training data influence sparsity, and the models high in privacy are strictly low in performance and training data influence sparsity.

\section{Related Work}
While privacy, fairness, and interpretability \emph{individually} have enjoyed ample attention from the research community in recent years \citep{liu-etal-2021-when, mehrabi-etal-2021-survey, soegaard-2021-survey}, the interactions between these objectives have not been explored much \citep{ruder-etal-2022-square}. Some prior work has focused on the interactions between group fairness and differential privacy, suggesting that the two objectives are at odds, although this relationship also depends on the selected notion of fairness \citep{bagdasaryan-etal-2019-differential, cummings-etal-2019, chang-shokri-2021, hansen-etal-2022-impact}. Somewhat in contrast to this work, we show that linguistic fairness (group fairness over linguistic communities) and differential privacy may align for multilingual language models. Furthermore, \citet{naidu2021differential} and \citet{shokri-etal-2021-risks} have studied the interaction between privacy and feature attribution methods for model explainability. While the former show that privacy and feature attribution methods can align, the latter find that model explanations are at risk of membership inference attacks. 
Closest to our work is contemporaneous work by \citet{strobel-shokri-2022} who discuss the interactions of data privacy with fairness, explainability, and robustness. Our work differs from theirs in that we are particularly concerned with multilingual language models and we consider instance-based interpretability methods while they consider feature attribution methods. \citet{strobel-shokri-2022} also call for more research at the intersection of different objectives rather than working on one at a time.

\section{Conclusion}
We presented a preliminary investigation
of how multilingual compression, differential privacy, training data influence, and linguistic fairness interact in multilingual models.
We found that privacy and influence are incompatible, while privacy and linguistic fairness, often said to be at odds, are theoretically compatible through multilingual compression. We also explored these interactions empirically. Our results support the idea that high multilingual compression can be achieved either while optimizing for performance or while optimizing for privacy, but that by trading off privacy and performance, we compromise compression. Finding practical trade-offs between {\em all} these dimensions remains an open challenge.
Finally, we introduced a new diagnostic metric, influence uniformity, which we used to validate that privacy and training data influence sparsity are incompatible, and that the interactions between privacy, training data influence sparsity, and multilingual compression are, therefore, also non-linear.

\section*{Ethical Aspects and Broader Impact}
It is crucial that NLP goes beyond performance and studies the interaction of objectives such as privacy, interpretability, and fairness, also in multilingual NLP \citep{ruder-etal-2022-square}. Our work aims to provide a starting point for further research in this area. Our empirical investigation, including the models we train, fully relies on publicly available models and data. Moreover, we do not create any new datasets. Therefore, we foresee no misuse of the results of our work.

\section*{Acknowledgements}
We thank the anonymous reviewers and members of the CoAStaL group for their helpful feedback and suggestions. Phillip Rust is funded by the Novo Nordisk Foundation (grant NNF 20SA0066568).

\bibliography{anthology,custom}

\begin{thebibliography}{107}
\providecommand{\natexlab}[1]{#1}
\providecommand{\url}[1]{\texttt{#1}}
\expandafter\ifx\csname urlstyle\endcsname\relax
  \providecommand{\doi}[1]{doi: #1}\else
  \providecommand{\doi}{doi: \begingroup \urlstyle{rm}\Url}\fi

\bibitem[Abadi et~al.(2016)Abadi, Chu, Goodfellow, McMahan, Mironov, Talwar,
  and Zhang]{abadi-etal-2016-dpsgd}
Abadi, M., Chu, A., Goodfellow, I.~J., McMahan, H.~B., Mironov, I., Talwar, K.,
  and Zhang, L.
\newblock Deep learning with differential privacy.
\newblock In \emph{Proceedings of the 2016 {ACM} {SIGSAC} Conference on
  Computer and Communications Security}, pp.\  308--318, Vienna, Austria, 2016.
  {ACM}.
\newblock URL \url{https://doi.org/10.1145/2976749.2978318}.

\bibitem[Artetxe \& Schwenk(2019)Artetxe and
  Schwenk]{artetxe-schwenk-2019-massively}
Artetxe, M. and Schwenk, H.
\newblock Massively multilingual sentence embeddings for zero-shot
  cross-lingual transfer and beyond.
\newblock \emph{Transactions of the Association for Computational Linguistics},
  7:\penalty0 597--610, March 2019.
\newblock URL \url{https://aclanthology.org/Q19-1038}.

\bibitem[Arun et~al.(2020)Arun, Gaw, Singh, Chang, Aggarwal, Chen, Hoebel,
  Gupta, Patel, Gidwani, Adebayo, Li, and
  Kalpathy-Cramer]{Arun2020.07.28.20163899}
Arun, N., Gaw, N., Singh, P., Chang, K., Aggarwal, M., Chen, B., Hoebel, K.,
  Gupta, S., Patel, J., Gidwani, M., Adebayo, J., Li, M.~D., and
  Kalpathy-Cramer, J.
\newblock Assessing the (un)trustworthiness of saliency maps for localizing
  abnormalities in medical imaging.
\newblock \emph{medRxiv}, 2020.
\newblock URL
  \url{https://www.medrxiv.org/content/early/2020/07/30/2020.07.28.20163899}.

\bibitem[Bagdasaryan et~al.(2019)Bagdasaryan, Poursaeed, and
  Shmatikov]{bagdasaryan-etal-2019-differential}
Bagdasaryan, E., Poursaeed, O., and Shmatikov, V.
\newblock Differential privacy has disparate impact on model accuracy.
\newblock In Wallach, H.~M., Larochelle, H., Beygelzimer, A.,
  d'Alch{\'{e}}{-}Buc, F., Fox, E.~B., and Garnett, R. (eds.), \emph{Advances
  in Neural Information Processing Systems 32: Annual Conference on Neural
  Information Processing Systems ({NeurIPS})}, pp.\  15453--15462, Vancouver,
  BC, Canada, 2019. Curran Associates, Inc.
\newblock URL
  \url{https://proceedings.neurips.cc/paper/2019/hash/fc0de4e0396fff257ea362983c2dda5a-Abstract.html}.

\bibitem[Basu et~al.(2020)Basu, You, and Feizi]{basu-etal-2020-second}
Basu, S., You, X., and Feizi, S.
\newblock On second-order group influence functions for black-box predictions.
\newblock In \emph{Proceedings of the 37th International Conference on Machine
  Learning ({ICML})}, volume 119 of \emph{Proceedings of Machine Learning
  Research}, pp.\  715--724, Online, 2020. {PMLR}.
\newblock URL \url{http://proceedings.mlr.press/v119/basu20b.html}.

\bibitem[Berk et~al.(2018)Berk, Heidari, Jabbari, Kearns, and
  Roth]{Berk2018FairnessIC}
Berk, R.~A., Heidari, H., Jabbari, S., Kearns, M., and Roth, A.
\newblock Fairness in criminal justice risk assessments: The state of the art.
\newblock \emph{Sociological Methods \& Research}, 50:\penalty0 3 -- 44, 2018.
\newblock URL \url{https://doi.org/10.1177/0049124118782533}.

\bibitem[Bishop(1995)]{6796505}
Bishop, C.~M.
\newblock Training with noise is equivalent to tikhonov regularization.
\newblock \emph{Neural Computation}, 7\penalty0 (1):\penalty0 108--116, 1995.
\newblock URL \url{https://dl.acm.org/doi/10.1162/neco.1995.7.1.108}.

\bibitem[Bouchacourt \& Baroni(2018)Bouchacourt and
  Baroni]{bouchacourt-baroni-2018-agents}
Bouchacourt, D. and Baroni, M.
\newblock How agents see things: On visual representations in an emergent
  language game.
\newblock In \emph{Proceedings of the 2018 Conference on Empirical Methods in
  Natural Language Processing}, pp.\  981--985, Brussels, Belgium,
  October-November 2018. Association for Computational Linguistics.
\newblock URL \url{https://aclanthology.org/D18-1119}.

\bibitem[Canonne et~al.(2020)Canonne, Kamath, and
  Steinke]{canonne-etal-2020-discrete}
Canonne, C.~L., Kamath, G., and Steinke, T.
\newblock The discrete gaussian for differential privacy.
\newblock In Larochelle, H., Ranzato, M., Hadsell, R., Balcan, M., and Lin, H.
  (eds.), \emph{Advances in Neural Information Processing Systems 33: Annual
  Conference on Neural Information Processing Systems ({NeurIPS})}, pp.\
  15676--15688, Online, 2020. Curran Associates, Inc.
\newblock URL
  \url{https://proceedings.neurips.cc/paper/2020/hash/b53b3a3d6ab90ce0268229151c9bde11-Abstract.html}.

\bibitem[Chang \& Shokri(2021)Chang and Shokri]{chang-shokri-2021}
Chang, H. and Shokri, R.
\newblock On the privacy risks of algorithmic fairness.
\newblock In \emph{2021 IEEE European Symposium on Security and Privacy
  (EuroS\&P)}, pp.\  292--303, Los Alamitos, CA, USA, 2021. IEEE Computer
  Society.
\newblock URL
  \url{https://doi.ieeecomputersociety.org/10.1109/EuroSP51992.2021.00028}.

\bibitem[Charpiat et~al.(2019)Charpiat, Girard, Felardos, and
  Tarabalka]{charpiat-etal-2019-input}
Charpiat, G., Girard, N., Felardos, L., and Tarabalka, Y.
\newblock Input similarity from the neural network perspective.
\newblock In Wallach, H.~M., Larochelle, H., Beygelzimer, A.,
  d'Alch{\'{e}}{-}Buc, F., Fox, E.~B., and Garnett, R. (eds.), \emph{Advances
  in Neural Information Processing Systems 32: Annual Conference on Neural
  Information Processing Systems ({NeurIPS})}, pp.\  5343--5352, Vancouver, BC,
  Canada, 2019. Curran Associates, Inc.
\newblock URL
  \url{https://proceedings.neurips.cc/paper/2019/hash/c61f571dbd2fb949d3fe5ae1608dd48b-Abstract.html}.

\bibitem[Choenni \& Shutova(2020)Choenni and Shutova]{choenni2020does}
Choenni, R. and Shutova, E.
\newblock What does it mean to be language-agnostic? probing multilingual
  sentence encoders for typological properties.
\newblock \emph{arXiv preprint}, 2020.
\newblock URL \url{https://arxiv.org/abs/2009.12862}.

\bibitem[Choudhury \& Deshpande(2021)Choudhury and Deshpande]{choudhury2021how}
Choudhury, M. and Deshpande, A.
\newblock How linguistically fair are multilingual pre-trained language models?
\newblock In \emph{Proceedings of the 35th {AAAI} Conference on Artificial
  Intelligence}, pp.\  12710--12718, Online, 2021. {AAAI} Press.
\newblock URL \url{https://ojs.aaai.org/index.php/AAAI/article/view/17505}.

\bibitem[Chrupa{\l}a(2019)]{chrupala-2019-symbolic}
Chrupa{\l}a, G.
\newblock Symbolic inductive bias for visually grounded learning of spoken
  language.
\newblock In \emph{Proceedings of the 57th Annual Meeting of the Association
  for Computational Linguistics}, pp.\  6452--6462, Florence, Italy, July 2019.
  Association for Computational Linguistics.
\newblock URL \url{https://aclanthology.org/P19-1647}.

\bibitem[Chrupa{\l}a \& Alishahi(2019)Chrupa{\l}a and
  Alishahi]{chrupala-alishahi-2019-correlating}
Chrupa{\l}a, G. and Alishahi, A.
\newblock Correlating neural and symbolic representations of language.
\newblock In \emph{Proceedings of the 57th Annual Meeting of the Association
  for Computational Linguistics}, pp.\  2952--2962, Florence, Italy, July 2019.
  Association for Computational Linguistics.
\newblock URL \url{https://aclanthology.org/P19-1283}.

\bibitem[Conneau et~al.(2018)Conneau, Rinott, Lample, Williams, Bowman,
  Schwenk, and Stoyanov]{conneau-etal-2018-xnli}
Conneau, A., Rinott, R., Lample, G., Williams, A., Bowman, S., Schwenk, H., and
  Stoyanov, V.
\newblock {XNLI}: Evaluating cross-lingual sentence representations.
\newblock In \emph{Proceedings of the 2018 Conference on Empirical Methods in
  Natural Language Processing}, pp.\  2475--2485, Brussels, Belgium,
  October-November 2018. Association for Computational Linguistics.
\newblock URL \url{https://aclanthology.org/D18-1269}.

\bibitem[Conneau et~al.(2020{\natexlab{a}})Conneau, Khandelwal, Goyal,
  Chaudhary, Wenzek, Guzm{\'a}n, Grave, Ott, Zettlemoyer, and
  Stoyanov]{conneau-etal-2020-unsupervised}
Conneau, A., Khandelwal, K., Goyal, N., Chaudhary, V., Wenzek, G., Guzm{\'a}n,
  F., Grave, E., Ott, M., Zettlemoyer, L., and Stoyanov, V.
\newblock Unsupervised cross-lingual representation learning at scale.
\newblock In \emph{Proceedings of the 58th Annual Meeting of the Association
  for Computational Linguistics}, pp.\  8440--8451, Online, July
  2020{\natexlab{a}}. Association for Computational Linguistics.
\newblock URL \url{https://aclanthology.org/2020.acl-main.747}.

\bibitem[Conneau et~al.(2020{\natexlab{b}})Conneau, Wu, Li, Zettlemoyer, and
  Stoyanov]{conneau-etal-2020-emerging}
Conneau, A., Wu, S., Li, H., Zettlemoyer, L., and Stoyanov, V.
\newblock Emerging cross-lingual structure in pretrained language models.
\newblock In \emph{Proceedings of the 58th Annual Meeting of the Association
  for Computational Linguistics}, pp.\  6022--6034, Online, July
  2020{\natexlab{b}}. Association for Computational Linguistics.
\newblock URL \url{https://aclanthology.org/2020.acl-main.536}.

\bibitem[Cummings et~al.(2019)Cummings, Gupta, Kimpara, and
  Morgenstern]{cummings-etal-2019}
Cummings, R., Gupta, V., Kimpara, D., and Morgenstern, J.
\newblock On the compatibility of privacy and fairness.
\newblock In \emph{Adjunct Publication of the 27th Conference on User Modeling,
  Adaptation and Personalization}, UMAP'19 Adjunct, pp.\  309–315, New York,
  NY, USA, 2019. Association for Computing Machinery.
\newblock ISBN 9781450367110.
\newblock URL \url{https://doi.org/10.1145/3314183.3323847}.

\bibitem[de~Vries et~al.(2020)de~Vries, van Cranenburgh, and
  Nissim]{de-vries-etal-2020-whats}
de~Vries, W., van Cranenburgh, A., and Nissim, M.
\newblock What{'}s so special about {BERT}{'}s layers? a closer look at the
  {NLP} pipeline in monolingual and multilingual models.
\newblock In \emph{Findings of the Association for Computational Linguistics:
  EMNLP 2020}, pp.\  4339--4350, Online, November 2020. Association for
  Computational Linguistics.
\newblock URL \url{https://aclanthology.org/2020.findings-emnlp.389}.

\bibitem[Devlin et~al.(2019)Devlin, Chang, Lee, and
  Toutanova]{devlin-etal-2019-bert}
Devlin, J., Chang, M.-W., Lee, K., and Toutanova, K.
\newblock {BERT}: Pre-training of deep bidirectional transformers for language
  understanding.
\newblock In \emph{Proceedings of the 2019 Conference of the North {A}merican
  Chapter of the Association for Computational Linguistics: Human Language
  Technologies, Volume 1 (Long and Short Papers)}, pp.\  4171--4186,
  Minneapolis, Minnesota, June 2019. Association for Computational Linguistics.
\newblock URL \url{https://aclanthology.org/N19-1423}.

\bibitem[Diedrichsen \& Kriegeskorte(2017)Diedrichsen and
  Kriegeskorte]{diederichsen-kriegeskorte-2017-representational}
Diedrichsen, J. and Kriegeskorte, N.
\newblock Representational models: A common framework for understanding
  encoding, pattern-component, and representational-similarity analysis.
\newblock \emph{PLOS Computational Biology}, 13\penalty0 (4):\penalty0 1--33,
  04 2017.
\newblock URL \url{https://doi.org/10.1371/journal.pcbi.1005508}.

\bibitem[Dufter \& Sch{\"u}tze(2020)Dufter and
  Sch{\"u}tze]{dufter-schutze-2020-identifying}
Dufter, P. and Sch{\"u}tze, H.
\newblock Identifying elements essential for {BERT}{'}s multilinguality.
\newblock In \emph{Proceedings of the 2020 Conference on Empirical Methods in
  Natural Language Processing (EMNLP)}, pp.\  4423--4437, Online, November
  2020. Association for Computational Linguistics.
\newblock URL \url{https://aclanthology.org/2020.emnlp-main.358}.

\bibitem[Dwork(2006)]{dwork-etal-2006-dp}
Dwork, C.
\newblock Differential privacy.
\newblock In Bugliesi, M., Preneel, B., Sassone, V., and Wegener, I. (eds.),
  \emph{Proceedings of the 33rd International Colloquium on Automata, Languages
  and Programming ({ICALP}), Part {II}}, volume 4052 of \emph{Lecture Notes in
  Computer Science}, pp.\  1--12, Venice, Italy, 2006. Springer.
\newblock URL \url{https://doi.org/10.1007/11787006\_1}.

\bibitem[Dwork \& Roth(2014)Dwork and Roth]{dwork-roth-2014-foundations}
Dwork, C. and Roth, A.
\newblock The algorithmic foundations of differential privacy.
\newblock \emph{Foundations and Trends in Theoretical Computer Science},
  9\penalty0 (3-4):\penalty0 211--407, 2014.
\newblock URL \url{https://doi.org/10.1561/0400000042}.

\bibitem[Edelman(1998)]{edelman-1998}
Edelman, S.
\newblock Representation is representation of similarities.
\newblock \emph{Behavioral and Brain Sciences}, 21\penalty0 (4):\penalty0
  449–467, 1998.
\newblock URL \url{https://doi.org/10.1017/s0140525x98001253}.

\bibitem[Ethayarajh(2019)]{ethayarajh-2019-contextual}
Ethayarajh, K.
\newblock How contextual are contextualized word representations? {C}omparing
  the geometry of {BERT}, {ELM}o, and {GPT}-2 embeddings.
\newblock In \emph{Proceedings of the 2019 Conference on Empirical Methods in
  Natural Language Processing and the 9th International Joint Conference on
  Natural Language Processing (EMNLP-IJCNLP)}, pp.\  55--65, Hong Kong, China,
  November 2019. Association for Computational Linguistics.
\newblock URL \url{https://aclanthology.org/D19-1006}.

\bibitem[Gao et~al.(2019)Gao, He, Tan, Qin, Wang, and
  Liu]{gao-etal-2019-representation}
Gao, J., He, D., Tan, X., Qin, T., Wang, L., and Liu, T.
\newblock Representation degeneration problem in training natural language
  generation models.
\newblock In \emph{Proceedings of the 7th International Conference on Learning
  Representations ({ICLR})}, New Orleans, LA, USA, 2019. OpenReview.net.
\newblock URL \url{https://openreview.net/forum?id=SkEYojRqtm}.

\bibitem[Geng et~al.(2020)Geng, Ding, Guo, and Kumar]{geng-etal-2020-analysis}
Geng, Q., Ding, W., Guo, R., and Kumar, S.
\newblock Tight analysis of privacy and utility tradeoff in approximate
  differential privacy.
\newblock In Chiappa, S. and Calandra, R. (eds.), \emph{Proceedings of the 23rd
  International Conference on Artificial Intelligence and Statistics
  ({AISTATS})}, volume 108 of \emph{Proceedings of Machine Learning Research},
  pp.\  89--99, Online, 26--28 Aug 2020. PMLR.
\newblock URL \url{https://proceedings.mlr.press/v108/geng20a.html}.

\bibitem[Glava{\v{s}} \& Vuli{\'c}(2021)Glava{\v{s}} and
  Vuli{\'c}]{glavas-vulic-2021-supervised}
Glava{\v{s}}, G. and Vuli{\'c}, I.
\newblock Is supervised syntactic parsing beneficial for language understanding
  tasks? an empirical investigation.
\newblock In \emph{Proceedings of the 16th Conference of the European Chapter
  of the Association for Computational Linguistics: Main Volume}, pp.\
  3090--3104, Online, April 2021. Association for Computational Linguistics.
\newblock URL \url{https://aclanthology.org/2021.eacl-main.270}.

\bibitem[Habernal(2021)]{habernal-2021-differential}
Habernal, I.
\newblock When differential privacy meets {NLP}: The devil is in the detail.
\newblock In \emph{Proceedings of the 2021 Conference on Empirical Methods in
  Natural Language Processing}, pp.\  1522--1528, Online and Punta Cana,
  Dominican Republic, November 2021. Association for Computational Linguistics.
\newblock URL \url{https://aclanthology.org/2021.emnlp-main.114}.

\bibitem[Han \& Tsvetkov(2021)Han and
  Tsvetkov]{han-tsvetkov-2021-influence-tuning}
Han, X. and Tsvetkov, Y.
\newblock Influence tuning: Demoting spurious correlations via instance
  attribution and instance-driven updates.
\newblock In \emph{Findings of the Association for Computational Linguistics:
  EMNLP 2021}, pp.\  4398--4409, Punta Cana, Dominican Republic, November 2021.
  Association for Computational Linguistics.
\newblock URL \url{https://aclanthology.org/2021.findings-emnlp.374}.

\bibitem[Han et~al.(2020)Han, Wallace, and Tsvetkov]{han-etal-2020-explaining}
Han, X., Wallace, B.~C., and Tsvetkov, Y.
\newblock Explaining black box predictions and unveiling data artifacts through
  influence functions.
\newblock In \emph{Proceedings of the 58th Annual Meeting of the Association
  for Computational Linguistics}, pp.\  5553--5563, Online, July 2020.
  Association for Computational Linguistics.
\newblock URL \url{https://aclanthology.org/2020.acl-main.492}.

\bibitem[Hansen et~al.(2022)Hansen, Neerkaje, Sawhney, Flek, and
  Søgaard]{hansen-etal-2022-impact}
Hansen, V. P.~B., Neerkaje, A.~T., Sawhney, R., Flek, L., and Søgaard, A.
\newblock The impact of differential privacy on group disparity mitigation.
\newblock In \emph{Proceedings of the Fourth Workshop on Privacy in Natural
  Language Processing at NAACL 2022}, Seattle, United States, 2022. Association
  for Computational Linguistics.
\newblock URL \url{https://arxiv.org/abs/2203.02745}.

\bibitem[He et~al.(2021)He, Liu, Ye, Tan, Ding, Cheng, Low, Bing, and
  Si]{he-etal-2021-effectiveness}
He, R., Liu, L., Ye, H., Tan, Q., Ding, B., Cheng, L., Low, J., Bing, L., and
  Si, L.
\newblock On the effectiveness of adapter-based tuning for pretrained language
  model adaptation.
\newblock In \emph{Proceedings of the 59th Annual Meeting of the Association
  for Computational Linguistics and the 11th International Joint Conference on
  Natural Language Processing (Volume 1: Long Papers)}, pp.\  2208--2222,
  Online, August 2021. Association for Computational Linguistics.
\newblock URL \url{https://aclanthology.org/2021.acl-long.172}.

\bibitem[Jagielski et~al.(2019)Jagielski, Kearns, Mao, Oprea, Roth, Malvajerdi,
  and Ullman]{jagielski-etal-2019-private-fair}
Jagielski, M., Kearns, M., Mao, J., Oprea, A., Roth, A., Malvajerdi, S.~S., and
  Ullman, J.
\newblock Differentially private fair learning.
\newblock In Chaudhuri, K. and Salakhutdinov, R. (eds.), \emph{Proceedings of
  the 36th International Conference on Machine Learning ({ICML})}, volume~97 of
  \emph{Proceedings of Machine Learning Research}, pp.\  3000--3008, Long
  Beach, CA, USA, 09--15 Jun 2019. PMLR.
\newblock URL \url{https://proceedings.mlr.press/v97/jagielski19a.html}.

\bibitem[K \& S{\o}gaard(2021)K and S{\o}gaard]{k-soegaard-2021-revisiting}
K, K. and S{\o}gaard, A.
\newblock Revisiting methods for finding influential examples.
\newblock \emph{arXiv preprint}, 2021.
\newblock URL \url{https://arxiv.org/abs/2111.04683}.

\bibitem[K et~al.(2020)K, Wang, Mayhew, and Roth]{k-etal-2020-ability}
K, K., Wang, Z., Mayhew, S., and Roth, D.
\newblock Cross-lingual ability of multilingual {BERT:} an empirical study.
\newblock In \emph{Proceedings of the 8th International Conference on Learning
  Representations ({ICLR})}, Online, 2020. OpenReview.net.
\newblock URL \url{https://openreview.net/forum?id=HJeT3yrtDr}.

\bibitem[Kerrigan et~al.(2020)Kerrigan, Slack, and
  Tuyls]{kerrigan-etal-2020-differentially}
Kerrigan, G., Slack, D., and Tuyls, J.
\newblock Differentially private language models benefit from public
  pre-training.
\newblock In \emph{Proceedings of the Second Workshop on Privacy in NLP}, pp.\
  39--45, Online, November 2020. Association for Computational Linguistics.
\newblock URL \url{https://aclanthology.org/2020.privatenlp-1.5}.

\bibitem[Khalifa et~al.(2021)Khalifa, Abdul-Mageed, and
  Shaalan]{khalifa-etal-2021-self}
Khalifa, M., Abdul-Mageed, M., and Shaalan, K.
\newblock Self-training pre-trained language models for zero- and few-shot
  multi-dialectal {A}rabic sequence labeling.
\newblock In \emph{Proceedings of the 16th Conference of the European Chapter
  of the Association for Computational Linguistics: Main Volume}, pp.\
  769--782, Online, April 2021. Association for Computational Linguistics.
\newblock URL \url{https://aclanthology.org/2021.eacl-main.65}.

\bibitem[Kim et~al.(2021)Kim, Bisazza, and Turkmen]{kim-etal-2021-using}
Kim, S., Bisazza, A., and Turkmen, F.
\newblock Using confidential data for domain adaptation of neural machine
  translation.
\newblock In \emph{Proceedings of the Third Workshop on Privacy in Natural
  Language Processing}, pp.\  46--52, Online, June 2021. Association for
  Computational Linguistics.
\newblock URL \url{https://aclanthology.org/2021.privatenlp-1.6}.

\bibitem[Kindermans et~al.(2019)Kindermans, Hooker, Adebayo, Alber, Sch{\"u}tt,
  D{\"a}hne, Erhan, and Kim]{Kindermans2019TheO}
Kindermans, P.-J., Hooker, S., Adebayo, J., Alber, M., Sch{\"u}tt, K.~T.,
  D{\"a}hne, S., Erhan, D., and Kim, B.
\newblock The (un)reliability of saliency methods.
\newblock In \emph{Explainable AI}, 2019.
\newblock URL \url{https://doi.org/10.1007/978-3-030-28954-6_14}.

\bibitem[Kingma \& Ba(2015)Kingma and Ba]{kingma-ba-2015-adam}
Kingma, D.~P. and Ba, J.
\newblock Adam: {A} method for stochastic optimization.
\newblock In Bengio, Y. and LeCun, Y. (eds.), \emph{Proceedings of the 3rd
  International Conference on Learning Representations ({ICLR})}, San Diego,
  CA, USA, 2015.
\newblock URL \url{http://arxiv.org/abs/1412.6980}.

\bibitem[Koh \& Liang(2017)Koh and Liang]{10.5555/3305381.3305576}
Koh, P.~W. and Liang, P.
\newblock Understanding black-box predictions via influence functions.
\newblock In Precup, D. and Teh, Y.~W. (eds.), \emph{Proceedings of the 34th
  International Conference on Machine Learning ({ICML})}, volume~70 of
  \emph{Proceedings of Machine Learning Research}, pp.\  1885--1894, Sydney,
  NSW, Australia, 2017. {PMLR}.
\newblock URL \url{http://proceedings.mlr.press/v70/koh17a.html}.

\bibitem[Koh et~al.(2019)Koh, Ang, Teo, and Liang]{koh-etal-2019-accuracy}
Koh, P.~W., Ang, K., Teo, H. H.~K., and Liang, P.
\newblock On the accuracy of influence functions for measuring group effects.
\newblock In Wallach, H.~M., Larochelle, H., Beygelzimer, A.,
  d'Alch{\'{e}}{-}Buc, F., Fox, E.~B., and Garnett, R. (eds.), \emph{Advances
  in Neural Information Processing Systems 32: Annual Conference on Neural
  Information Processing Systems ({NeurIPS})}, pp.\  5255--5265, Vancouver, BC,
  Canada, 2019. Curran Associates, Inc.
\newblock URL
  \url{https://proceedings.neurips.cc/paper/2019/hash/a78482ce76496fcf49085f2190e675b4-Abstract.html}.

\bibitem[Kong \& Chaudhuri(2021)Kong and
  Chaudhuri]{kong-chaudhuri-2021-understanding}
Kong, Z. and Chaudhuri, K.
\newblock Understanding instance-based interpretability of variational
  auto-encoders.
\newblock In \emph{Advances in Neural Information Processing Systems 34: Annual
  Conference on Neural Information Processing Systems ({NeurIPS})}, Online,
  2021. Curran Associates, Inc.
\newblock URL \url{https://arxiv.org/abs/2105.14203}.

\bibitem[Kornblith et~al.(2019)Kornblith, Norouzi, Lee, and
  Hinton]{kornblith-etal-2019}
Kornblith, S., Norouzi, M., Lee, H., and Hinton, G.
\newblock Similarity of neural network representations revisited.
\newblock In Chaudhuri, K. and Salakhutdinov, R. (eds.), \emph{Proceedings of
  the 36th International Conference on Machine Learning ({ICML})}, volume~97 of
  \emph{Proceedings of Machine Learning Research}, pp.\  3519--3529, Long
  Beach, CA, USA, 2019. PMLR.
\newblock URL \url{https://proceedings.mlr.press/v97/kornblith19a.html}.

\bibitem[Kriegeskorte et~al.(2008)Kriegeskorte, Mur, and
  Bandettini]{kriegeskorte-etal-2008}
Kriegeskorte, N., Mur, M., and Bandettini, P.
\newblock Representational similarity analysis - connecting the branches of
  systems neuroscience.
\newblock \emph{Frontiers in Systems Neuroscience}, 2:\penalty0 4, 2008.
\newblock ISSN 1662-5137.
\newblock URL
  \url{https://www.frontiersin.org/article/10.3389/neuro.06.004.2008}.

\bibitem[Lauscher et~al.(2020)Lauscher, Ravishankar, Vuli{\'c}, and
  Glava{\v{s}}]{lauscher-etal-2020-zero}
Lauscher, A., Ravishankar, V., Vuli{\'c}, I., and Glava{\v{s}}, G.
\newblock From zero to hero: {O}n the limitations of zero-shot language
  transfer with multilingual {T}ransformers.
\newblock In \emph{Proceedings of the 2020 Conference on Empirical Methods in
  Natural Language Processing (EMNLP)}, pp.\  4483--4499, Online, November
  2020. Association for Computational Linguistics.
\newblock URL \url{https://aclanthology.org/2020.emnlp-main.363}.

\bibitem[Lepori \& McCoy(2020)Lepori and McCoy]{lepori-mccoy-2020-picking}
Lepori, M. and McCoy, R.~T.
\newblock Picking {BERT}{'}s brain: Probing for linguistic dependencies in
  contextualized embeddings using representational similarity analysis.
\newblock In \emph{Proceedings of the 28th International Conference on
  Computational Linguistics}, pp.\  3637--3651, Barcelona, Spain (Online),
  December 2020. International Committee on Computational Linguistics.
\newblock URL \url{https://aclanthology.org/2020.coling-main.325}.

\bibitem[Lhoest et~al.(2021)Lhoest, Villanova~del Moral, Jernite, Thakur, von
  Platen, Patil, Chaumond, Drame, Plu, Tunstall, Davison, {\v{S}}a{\v{s}}ko,
  Chhablani, Malik, Brandeis, Le~Scao, Sanh, Xu, Patry, McMillan-Major, Schmid,
  Gugger, Delangue, Matussi{\`e}re, Debut, Bekman, Cistac, Goehringer, Mustar,
  Lagunas, Rush, and Wolf]{lhoest-etal-2021-datasets}
Lhoest, Q., Villanova~del Moral, A., Jernite, Y., Thakur, A., von Platen, P.,
  Patil, S., Chaumond, J., Drame, M., Plu, J., Tunstall, L., Davison, J.,
  {\v{S}}a{\v{s}}ko, M., Chhablani, G., Malik, B., Brandeis, S., Le~Scao, T.,
  Sanh, V., Xu, C., Patry, N., McMillan-Major, A., Schmid, P., Gugger, S.,
  Delangue, C., Matussi{\`e}re, T., Debut, L., Bekman, S., Cistac, P.,
  Goehringer, T., Mustar, V., Lagunas, F., Rush, A., and Wolf, T.
\newblock Datasets: A community library for natural language processing.
\newblock In \emph{Proceedings of the 2021 Conference on Empirical Methods in
  Natural Language Processing: System Demonstrations}, pp.\  175--184, Online
  and Punta Cana, Dominican Republic, November 2021. Association for
  Computational Linguistics.
\newblock URL \url{https://aclanthology.org/2021.emnlp-demo.21}.

\bibitem[Li et~al.(2012)Li, Qardaji, and Su]{10.1145/2414456.2414474}
Li, N., Qardaji, W., and Su, D.
\newblock On sampling, anonymization, and differential privacy or,
  k-anonymization meets differential privacy.
\newblock In \emph{Proceedings of the 7th ACM Symposium on Information,
  Computer and Communications Security}, ASIACCS '12, pp.\  32–33, New York,
  NY, USA, 2012. Association for Computing Machinery.
\newblock ISBN 9781450316484.
\newblock URL \url{https://doi.org/10.1145/2414456.2414474}.

\bibitem[Li et~al.(2022)Li, Tramer, Liang, and Hashimoto]{li-etal-2021-dp-lm}
Li, X., Tramer, F., Liang, P., and Hashimoto, T.
\newblock Large language models can be strong differentially private learners.
\newblock In \emph{Proceedings of the 10th International Conference on Learning
  Representations ({ICLR})}, Online, 2022. OpenReview.net.
\newblock URL \url{https://openreview.net/forum?id=bVuP3ltATMz}.

\bibitem[Libovick{\'y} et~al.(2020)Libovick{\'y}, Rosa, and
  Fraser]{libovicky-etal-2020-language}
Libovick{\'y}, J., Rosa, R., and Fraser, A.
\newblock On the language neutrality of pre-trained multilingual
  representations.
\newblock In \emph{Findings of the Association for Computational Linguistics:
  EMNLP 2020}, pp.\  1663--1674, Online, November 2020. Association for
  Computational Linguistics.
\newblock URL \url{https://aclanthology.org/2020.findings-emnlp.150}.

\bibitem[Liu et~al.(2021{\natexlab{a}})Liu, Ding, Shaham, Rahayu, Farokhi, and
  Lin]{liu-etal-2021-when}
Liu, B., Ding, M., Shaham, S., Rahayu, W., Farokhi, F., and Lin, Z.
\newblock When machine learning meets privacy: A survey and outlook.
\newblock \emph{ACM Comput. Surv.}, 54\penalty0 (2), 2021{\natexlab{a}}.
\newblock ISSN 0360-0300.
\newblock URL \url{https://doi.org/10.1145/3436755}.

\bibitem[Liu \& Talwar(2019)Liu and Talwar]{liu-talwar-2019-private}
Liu, J. and Talwar, K.
\newblock Private selection from private candidates.
\newblock In \emph{Proceedings of the 51st Annual ACM SIGACT Symposium on
  Theory of Computing ({STOC})}, pp.\  298–309, New York, NY, USA, 2019.
  Association for Computing Machinery.
\newblock ISBN 9781450367059.
\newblock URL \url{https://doi.org/10.1145/3313276.3316377}.

\bibitem[Liu et~al.(2021{\natexlab{b}})Liu, Wang, Lu, Cheng, Jin, Wang, and
  Zha]{liu2021fair}
Liu, W., Wang, X., Lu, X., Cheng, J., Jin, B., Wang, X., and Zha, H.
\newblock Fair differential privacy can mitigate the disparate impact on model
  accuracy.
\newblock \emph{Submitted to the 9th International Conference on Learning
  Representations ({ICLR})}, 2021{\natexlab{b}}.
\newblock URL \url{https://openreview.net/forum?id=IqVB8e0DlUd}.

\bibitem[Liu et~al.(2021{\natexlab{c}})Liu, Winata, Madotto, and
  Fung]{liu-etal-2021-preserving}
Liu, Z., Winata, G.~I., Madotto, A., and Fung, P.
\newblock Preserving cross-linguality of pre-trained models via continual
  learning.
\newblock In \emph{Proceedings of the 6th Workshop on Representation Learning
  for NLP (RepL4NLP-2021)}, pp.\  64--71, Online, August 2021{\natexlab{c}}.
  Association for Computational Linguistics.
\newblock URL \url{https://aclanthology.org/2021.repl4nlp-1.8}.

\bibitem[Loshchilov \& Hutter(2019)Loshchilov and
  Hutter]{loshchilov2018decoupled}
Loshchilov, I. and Hutter, F.
\newblock Decoupled weight decay regularization.
\newblock In \emph{Proceedings of the 7th International Conference on Learning
  Representations ({ICLR})}, New Orleans, LA, USA, 2019. OpenReview.net.
\newblock URL \url{https://openreview.net/forum?id=Bkg6RiCqY7}.

\bibitem[Lyu et~al.(2020)Lyu, He, and Li]{lyu-etal-2020-differentially}
Lyu, L., He, X., and Li, Y.
\newblock Differentially private representation for {NLP}: Formal guarantee and
  an empirical study on privacy and fairness.
\newblock In \emph{Findings of the Association for Computational Linguistics:
  EMNLP 2020}, pp.\  2355--2365, Online, November 2020. Association for
  Computational Linguistics.
\newblock URL \url{https://aclanthology.org/2020.findings-emnlp.213}.

\bibitem[Maronikolakis et~al.(2021)Maronikolakis, Dufter, and
  Sch{\"u}tze]{maronikolakis-etal-2021-wine-v}
Maronikolakis, A., Dufter, P., and Sch{\"u}tze, H.
\newblock Wine is not v i n. on the compatibility of tokenizations across
  languages.
\newblock In \emph{Findings of the Association for Computational Linguistics:
  EMNLP 2021}, pp.\  2382--2399, Punta Cana, Dominican Republic, November 2021.
  Association for Computational Linguistics.
\newblock URL \url{https://aclanthology.org/2021.findings-emnlp.205}.

\bibitem[McMahan et~al.(2018)McMahan, Ramage, Talwar, and
  Zhang]{mcmahan-etal-2018-learning}
McMahan, H.~B., Ramage, D., Talwar, K., and Zhang, L.
\newblock Learning differentially private recurrent language models.
\newblock In \emph{Proceedings of the 6th International Conference on Learning
  Representations ({ICLR})}, Vancouver, BC, Canada, 2018. OpenReview.net.
\newblock URL \url{https://openreview.net/forum?id=BJ0hF1Z0b}.

\bibitem[Mehrabi et~al.(2021)Mehrabi, Morstatter, Saxena, Lerman, and
  Galstyan]{mehrabi-etal-2021-survey}
Mehrabi, N., Morstatter, F., Saxena, N., Lerman, K., and Galstyan, A.
\newblock A survey on bias and fairness in machine learning.
\newblock \emph{ACM Comput. Surv.}, 54\penalty0 (6), 2021.
\newblock ISSN 0360-0300.
\newblock URL \url{https://doi.org/10.1145/3457607}.

\bibitem[Miconi(2017)]{https://doi.org/10.48550/arxiv.1707.01195}
Miconi, T.
\newblock The impossibility of "fairness": {A} generalized impossibility result
  for decisions.
\newblock \emph{arXiv preprint}, 2017.
\newblock URL \url{https://arxiv.org/abs/1707.01195}.

\bibitem[Mironov(2017)]{mironov-2017-rdp}
Mironov, I.
\newblock R{\'{e}}nyi differential privacy.
\newblock In \emph{30th {IEEE} Computer Security Foundations Symposium,
  ({CSF})}, pp.\  263--275, Santa Barbara, CA, USA, 2017. {IEEE} Computer
  Society.
\newblock URL \url{https://doi.org/10.1109/CSF.2017.11}.

\bibitem[Mironov et~al.(2019)Mironov, Talwar, and
  Zhang]{mironov-etal-2019-sampled}
Mironov, I., Talwar, K., and Zhang, L.
\newblock R\'enyi differential privacy of the sampled gaussian mechanism.
\newblock \emph{arXiv preprint}, 2019.
\newblock URL \url{https://arxiv.org/abs/1908.10530}.

\bibitem[Muller et~al.(2021)Muller, Elazar, Sagot, and
  Seddah]{muller-etal-2021-first}
Muller, B., Elazar, Y., Sagot, B., and Seddah, D.
\newblock First align, then predict: Understanding the cross-lingual ability of
  multilingual {BERT}.
\newblock In \emph{Proceedings of the 16th Conference of the European Chapter
  of the Association for Computational Linguistics: Main Volume}, pp.\
  2214--2231, Online, April 2021. Association for Computational Linguistics.
\newblock URL \url{https://aclanthology.org/2021.eacl-main.189}.

\bibitem[Naidu et~al.(2021)Naidu, Priyanshu, Kumar, Kotti, Wang, and
  Mireshghallah]{naidu2021differential}
Naidu, R., Priyanshu, A., Kumar, A., Kotti, S., Wang, H., and Mireshghallah, F.
\newblock When differential privacy meets interpretability: A case study.
\newblock In \emph{CVPR 2021 Workshop for Responsible Computer Vision (RCV)},
  2021.
\newblock URL \url{https://arxiv.org/abs/2106.13203}.

\bibitem[Nivre et~al.(2020)Nivre, de~Marneffe, Ginter, Haji{\v{c}}, Manning,
  Pyysalo, Schuster, Tyers, and Zeman]{nivre-etal-2020-universal}
Nivre, J., de~Marneffe, M.-C., Ginter, F., Haji{\v{c}}, J., Manning, C.~D.,
  Pyysalo, S., Schuster, S., Tyers, F., and Zeman, D.
\newblock {U}niversal {D}ependencies v2: An evergrowing multilingual treebank
  collection.
\newblock In \emph{Proceedings of the 12th Language Resources and Evaluation
  Conference}, pp.\  4034--4043, Marseille, France, May 2020. European Language
  Resources Association.
\newblock ISBN 979-10-95546-34-4.
\newblock URL \url{https://aclanthology.org/2020.lrec-1.497}.

\bibitem[Pannekoek \& Spigler(2021)Pannekoek and
  Spigler]{pannekoek2021investigating}
Pannekoek, M. and Spigler, G.
\newblock Investigating trade-offs in utility, fairness and differential
  privacy in neural networks.
\newblock \emph{arXiv preprint}, 2021.
\newblock URL \url{https://arxiv.org/abs/2102.05975}.

\bibitem[Papernot \& Steinke(2022)Papernot and
  Steinke]{papernot-steinke-2021-hyperparameter}
Papernot, N. and Steinke, T.
\newblock Hyperparameter tuning with renyi differential privacy.
\newblock In \emph{Proceedings of the 10th International Conference on Learning
  Representations ({ICLR})}, Online, 2022. OpenReview.net.
\newblock URL \url{https://openreview.net/forum?id=-70L8lpp9DF}.

\bibitem[Paszke et~al.(2019)Paszke, Gross, Massa, Lerer, Bradbury, Chanan,
  Killeen, Lin, Gimelshein, Antiga, Desmaison, K{\"{o}}pf, Yang, DeVito,
  Raison, Tejani, Chilamkurthy, Steiner, Fang, Bai, and
  Chintala]{paszke-etal-2019-pytorch}
Paszke, A., Gross, S., Massa, F., Lerer, A., Bradbury, J., Chanan, G., Killeen,
  T., Lin, Z., Gimelshein, N., Antiga, L., Desmaison, A., K{\"{o}}pf, A., Yang,
  E.~Z., DeVito, Z., Raison, M., Tejani, A., Chilamkurthy, S., Steiner, B.,
  Fang, L., Bai, J., and Chintala, S.
\newblock Pytorch: An imperative style, high-performance deep learning library.
\newblock In Wallach, H.~M., Larochelle, H., Beygelzimer, A.,
  d'Alch{\'{e}}{-}Buc, F., Fox, E.~B., and Garnett, R. (eds.), \emph{Advances
  in Neural Information Processing Systems 32: Annual Conference on Neural
  Information Processing Systems ({NeurIPS})}, pp.\  8024--8035, Vancouver, BC,
  Canada, 2019. Curran Associates, Inc.
\newblock URL
  \url{https://proceedings.neurips.cc/paper/2019/hash/bdbca288fee7f92f2bfa9f7012727740-Abstract.html}.

\bibitem[Phang et~al.(2021)Phang, Liu, and Bowman]{phang-etal-2021-fine}
Phang, J., Liu, H., and Bowman, S.~R.
\newblock Fine-tuned transformers show clusters of similar representations
  across layers.
\newblock In \emph{Proceedings of the Fourth BlackboxNLP Workshop on Analyzing
  and Interpreting Neural Networks for NLP}, pp.\  529--538, Punta Cana,
  Dominican Republic, November 2021. Association for Computational Linguistics.
\newblock URL \url{https://aclanthology.org/2021.blackboxnlp-1.42}.

\bibitem[Pires et~al.(2019)Pires, Schlinger, and
  Garrette]{pires-etal-2019-multilingual}
Pires, T., Schlinger, E., and Garrette, D.
\newblock How multilingual is multilingual {BERT}?
\newblock In \emph{Proceedings of the 57th Annual Meeting of the Association
  for Computational Linguistics}, pp.\  4996--5001, Florence, Italy, July 2019.
  Association for Computational Linguistics.
\newblock URL \url{https://aclanthology.org/P19-1493}.

\bibitem[Pruksachatkun et~al.(2021)Pruksachatkun, Ramakrishna, Chang, Krishna,
  Dhamala, Guha, and Ren]{trustnlp-2021-trustworthy}
Pruksachatkun, Y., Ramakrishna, A., Chang, K.-W., Krishna, S., Dhamala, J.,
  Guha, T., and Ren, X. (eds.).
\newblock \emph{Proceedings of the First Workshop on Trustworthy Natural
  Language Processing}, Online, June 2021. Association for Computational
  Linguistics.
\newblock URL \url{https://aclanthology.org/2021.trustnlp-1.0}.

\bibitem[Pruthi et~al.(2020)Pruthi, Liu, Kale, and
  Sundararajan]{pruthi-etal-2020}
Pruthi, G., Liu, F., Kale, S., and Sundararajan, M.
\newblock Estimating training data influence by tracing gradient descent.
\newblock In Larochelle, H., Ranzato, M., Hadsell, R., Balcan, M., and Lin, H.
  (eds.), \emph{Advances in Neural Information Processing Systems 33: Annual
  Conference on Neural Information Processing Systems ({NeurIPS})}, Online,
  2020. Curran Associates, Inc.
\newblock URL
  \url{https://proceedings.neurips.cc/paper/2020/hash/e6385d39ec9394f2f3a354d9d2b88eec-Abstract.html}.

\bibitem[Rajaee \& Pilehvar(2021)Rajaee and
  Pilehvar]{rajaee-pilehvar-2021-cluster}
Rajaee, S. and Pilehvar, M.~T.
\newblock A cluster-based approach for improving isotropy in contextual
  embedding space.
\newblock In \emph{Proceedings of the 59th Annual Meeting of the Association
  for Computational Linguistics and the 11th International Joint Conference on
  Natural Language Processing (Volume 2: Short Papers)}, pp.\  575--584,
  Online, August 2021. Association for Computational Linguistics.
\newblock URL \url{https://aclanthology.org/2021.acl-short.73}.

\bibitem[Ravishankar \& S{\o}gaard(2021)Ravishankar and
  S{\o}gaard]{ravishankar-sogaard-2021-impact}
Ravishankar, V. and S{\o}gaard, A.
\newblock The impact of positional encodings on multilingual compression.
\newblock In \emph{Proceedings of the 2021 Conference on Empirical Methods in
  Natural Language Processing}, pp.\  763--777, Online and Punta Cana,
  Dominican Republic, November 2021. Association for Computational Linguistics.
\newblock URL \url{https://aclanthology.org/2021.emnlp-main.59}.

\bibitem[Reimers \& Gurevych(2020)Reimers and
  Gurevych]{reimers-gurevych-2020-making}
Reimers, N. and Gurevych, I.
\newblock Making monolingual sentence embeddings multilingual using knowledge
  distillation.
\newblock In \emph{Proceedings of the 2020 Conference on Empirical Methods in
  Natural Language Processing (EMNLP)}, pp.\  4512--4525, Online, November
  2020. Association for Computational Linguistics.
\newblock URL \url{https://aclanthology.org/2020.emnlp-main.365}.

\bibitem[Ruder et~al.(2022)Ruder, Vuli{\'c}, and
  S{\o}gaard]{ruder-etal-2022-square}
Ruder, S., Vuli{\'c}, I., and S{\o}gaard, A.
\newblock Square one bias in {NLP}: Towards a multi-dimensional exploration of
  the research manifold.
\newblock In \emph{Findings of the Association for Computational Linguistics:
  ACL 2022}, pp.\  2340--2354, Dublin, Ireland, 2022. Association for
  Computational Linguistics.
\newblock URL \url{https://aclanthology.org/2022.findings-acl.184}.

\bibitem[Rudman et~al.(2022)Rudman, Gillman, Rayne, and
  Eickhoff]{rudman-etal-2021-iso}
Rudman, W., Gillman, N., Rayne, T., and Eickhoff, C.
\newblock {I}so{S}core: Measuring the uniformity of embedding space
  utilization.
\newblock In \emph{Findings of the Association for Computational Linguistics:
  ACL 2022}, pp.\  3325--3339, Dublin, Ireland, May 2022. Association for
  Computational Linguistics.
\newblock URL \url{https://aclanthology.org/2022.findings-acl.262}.

\bibitem[Scao et~al.(2022)Scao, Fan, Akiki, Pavlick, Ilić, Hesslow, Castagné,
  Luccioni, Yvon, Gallé, Tow, Rush, Biderman, Webson, Ammanamanchi, Wang,
  Sagot, Muennighoff, del Moral, Ruwase, Bawden, Bekman, McMillan-Major,
  Beltagy, Nguyen, Saulnier, Tan, Suarez, Sanh, Laurençon, Jernite, Launay,
  Mitchell, Raffel, Gokaslan, Simhi, Soroa, Aji, Alfassy, Rogers, Nitzav, Xu,
  Mou, Emezue, Klamm, Leong, van Strien, Adelani, Radev, Ponferrada, Levkovizh,
  Kim, Natan, De~Toni, Dupont, Kruszewski, Pistilli, Elsahar, Benyamina, Tran,
  Yu, Abdulmumin, Johnson, Gonzalez-Dios, de~la Rosa, Chim, Dodge, Zhu, Chang,
  Frohberg, Tobing, Bhattacharjee, Almubarak, Chen, Lo, Von~Werra, Weber, Phan,
  allal, Tanguy, Dey, Muñoz, Masoud, Grandury, Šaško, Huang, Coavoux, Singh,
  Jiang, Vu, Jauhar, Ghaleb, Subramani, Kassner, Khamis, Nguyen, Espejel,
  de~Gibert, Villegas, Henderson, Colombo, Amuok, Lhoest, Harliman, Bommasani,
  López, Ribeiro, Osei, Pyysalo, Nagel, Bose, Muhammad, Sharma, Longpre,
  Nikpoor, Silberberg, Pai, Zink, Torrent, Schick, Thrush, Danchev, Nikoulina,
  Laippala, Lepercq, Prabhu, Alyafeai, Talat, Raja, Heinzerling, Si, Taşar,
  Salesky, Mielke, Lee, Sharma, Santilli, Chaffin, Stiegler, Datta, Szczechla,
  Chhablani, Wang, Pandey, Strobelt, Fries, Rozen, Gao, Sutawika, Bari,
  Al-shaibani, Manica, Nayak, Teehan, Albanie, Shen, Ben-David, Bach, Kim,
  Bers, Fevry, Neeraj, Thakker, Raunak, Tang, Yong, Sun, Brody, Uri, Tojarieh,
  Roberts, Chung, Tae, Phang, Press, Li, Narayanan, Bourfoune, Casper, Rasley,
  Ryabinin, Mishra, Zhang, Shoeybi, Peyrounette, Patry, Tazi, Sanseviero, von
  Platen, Cornette, Lavallée, Lacroix, Rajbhandari, Gandhi, Smith, Requena,
  Patil, Dettmers, Baruwa, Singh, Cheveleva, Ligozat, Subramonian, Névéol,
  Lovering, Garrette, Tunuguntla, Reiter, Taktasheva, Voloshina, Bogdanov,
  Winata, Schoelkopf, Kalo, Novikova, Forde, Clive, Kasai, Kawamura, Hazan,
  Carpuat, Clinciu, Kim, Cheng, Serikov, Antverg, van~der Wal, Zhang, Zhang,
  Gehrmann, Mirkin, Pais, Shavrina, Scialom, Yun, Limisiewicz, Rieser,
  Protasov, Mikhailov, Pruksachatkun, Belinkov, Bamberger, Kasner, Rueda,
  Pestana, Feizpour, Khan, Faranak, Santos, Hevia, Unldreaj, Aghagol,
  Abdollahi, Tammour, HajiHosseini, Behroozi, Ajibade, Saxena, Ferrandis,
  Contractor, Lansky, David, Kiela, Nguyen, Tan, Baylor, Ozoani, Mirza,
  Ononiwu, Rezanejad, Jones, Bhattacharya, Solaiman, Sedenko, Nejadgholi,
  Passmore, Seltzer, Sanz, Dutra, Samagaio, Elbadri, Mieskes, Gerchick,
  Akinlolu, McKenna, Qiu, Ghauri, Burynok, Abrar, Rajani, Elkott, Fahmy,
  Samuel, An, Kromann, Hao, Alizadeh, Shubber, Wang, Roy, Viguier, Le, Oyebade,
  Le, Yang, Nguyen, Kashyap, Palasciano, Callahan, Shukla, Miranda-Escalada,
  Singh, Beilharz, Wang, Brito, Zhou, Jain, Xu, Fourrier, Periñán, Molano,
  Yu, Manjavacas, Barth, Fuhrimann, Altay, Bayrak, Burns, Vrabec, Bello, Dash,
  Kang, Giorgi, Golde, Posada, Sivaraman, Bulchandani, Liu, Shinzato,
  de~Bykhovetz, Takeuchi, Pàmies, Castillo, Nezhurina, Sänger, Samwald,
  Cullan, Weinberg, De~Wolf, Mihaljcic, Liu, Freidank, Kang, Seelam, Dahlberg,
  Broad, Muellner, Fung, Haller, Chandrasekhar, Eisenberg, Martin, Canalli, Su,
  Su, Cahyawijaya, Garda, Deshmukh, Mishra, Kiblawi, Ott, Sang-aroonsiri,
  Kumar, Schweter, Bharati, Laud, Gigant, Kainuma, Kusa, Labrak, Bajaj,
  Venkatraman, Xu, Xu, Xu, Tan, Xie, Ye, Bras, Belkada, and
  Wolf]{https://doi.org/10.48550/arxiv.2211.05100}
Scao, T.~L., Fan, A., Akiki, C., Pavlick, E., Ilić, S., Hesslow, D.,
  Castagné, R., Luccioni, A.~S., Yvon, F., Gallé, M., Tow, J., Rush, A.~M.,
  Biderman, S., Webson, A., Ammanamanchi, P.~S., Wang, T., Sagot, B.,
  Muennighoff, N., del Moral, A.~V., Ruwase, O., Bawden, R., Bekman, S.,
  McMillan-Major, A., Beltagy, I., Nguyen, H., Saulnier, L., Tan, S., Suarez,
  P.~O., Sanh, V., Laurençon, H., Jernite, Y., Launay, J., Mitchell, M.,
  Raffel, C., Gokaslan, A., Simhi, A., Soroa, A., Aji, A.~F., Alfassy, A.,
  Rogers, A., Nitzav, A.~K., Xu, C., Mou, C., Emezue, C., Klamm, C., Leong, C.,
  van Strien, D., Adelani, D.~I., Radev, D., Ponferrada, E.~G., Levkovizh, E.,
  Kim, E., Natan, E.~B., De~Toni, F., Dupont, G., Kruszewski, G., Pistilli, G.,
  Elsahar, H., Benyamina, H., Tran, H., Yu, I., Abdulmumin, I., Johnson, I.,
  Gonzalez-Dios, I., de~la Rosa, J., Chim, J., Dodge, J., Zhu, J., Chang, J.,
  Frohberg, J., Tobing, J., Bhattacharjee, J., Almubarak, K., Chen, K., Lo, K.,
  Von~Werra, L., Weber, L., Phan, L., allal, L.~B., Tanguy, L., Dey, M.,
  Muñoz, M.~R., Masoud, M., Grandury, M., Šaško, M., Huang, M., Coavoux, M.,
  Singh, M., Jiang, M. T.-J., Vu, M.~C., Jauhar, M.~A., Ghaleb, M., Subramani,
  N., Kassner, N., Khamis, N., Nguyen, O., Espejel, O., de~Gibert, O.,
  Villegas, P., Henderson, P., Colombo, P., Amuok, P., Lhoest, Q., Harliman,
  R., Bommasani, R., López, R.~L., Ribeiro, R., Osei, S., Pyysalo, S., Nagel,
  S., Bose, S., Muhammad, S.~H., Sharma, S., Longpre, S., Nikpoor, S.,
  Silberberg, S., Pai, S., Zink, S., Torrent, T.~T., Schick, T., Thrush, T.,
  Danchev, V., Nikoulina, V., Laippala, V., Lepercq, V., Prabhu, V., Alyafeai,
  Z., Talat, Z., Raja, A., Heinzerling, B., Si, C., Taşar, D.~E., Salesky, E.,
  Mielke, S.~J., Lee, W.~Y., Sharma, A., Santilli, A., Chaffin, A., Stiegler,
  A., Datta, D., Szczechla, E., Chhablani, G., Wang, H., Pandey, H., Strobelt,
  H., Fries, J.~A., Rozen, J., Gao, L., Sutawika, L., Bari, M.~S., Al-shaibani,
  M.~S., Manica, M., Nayak, N., Teehan, R., Albanie, S., Shen, S., Ben-David,
  S., Bach, S.~H., Kim, T., Bers, T., Fevry, T., Neeraj, T., Thakker, U.,
  Raunak, V., Tang, X., Yong, Z.-X., Sun, Z., Brody, S., Uri, Y., Tojarieh, H.,
  Roberts, A., Chung, H.~W., Tae, J., Phang, J., Press, O., Li, C., Narayanan,
  D., Bourfoune, H., Casper, J., Rasley, J., Ryabinin, M., Mishra, M., Zhang,
  M., Shoeybi, M., Peyrounette, M., Patry, N., Tazi, N., Sanseviero, O., von
  Platen, P., Cornette, P., Lavallée, P.~F., Lacroix, R., Rajbhandari, S.,
  Gandhi, S., Smith, S., Requena, S., Patil, S., Dettmers, T., Baruwa, A.,
  Singh, A., Cheveleva, A., Ligozat, A.-L., Subramonian, A., Névéol, A.,
  Lovering, C., Garrette, D., Tunuguntla, D., Reiter, E., Taktasheva, E.,
  Voloshina, E., Bogdanov, E., Winata, G.~I., Schoelkopf, H., Kalo, J.-C.,
  Novikova, J., Forde, J.~Z., Clive, J., Kasai, J., Kawamura, K., Hazan, L.,
  Carpuat, M., Clinciu, M., Kim, N., Cheng, N., Serikov, O., Antverg, O.,
  van~der Wal, O., Zhang, R., Zhang, R., Gehrmann, S., Mirkin, S., Pais, S.,
  Shavrina, T., Scialom, T., Yun, T., Limisiewicz, T., Rieser, V., Protasov,
  V., Mikhailov, V., Pruksachatkun, Y., Belinkov, Y., Bamberger, Z., Kasner,
  Z., Rueda, A., Pestana, A., Feizpour, A., Khan, A., Faranak, A., Santos, A.,
  Hevia, A., Unldreaj, A., Aghagol, A., Abdollahi, A., Tammour, A.,
  HajiHosseini, A., Behroozi, B., Ajibade, B., Saxena, B., Ferrandis, C.~M.,
  Contractor, D., Lansky, D., David, D., Kiela, D., Nguyen, D.~A., Tan, E.,
  Baylor, E., Ozoani, E., Mirza, F., Ononiwu, F., Rezanejad, H., Jones, H.,
  Bhattacharya, I., Solaiman, I., Sedenko, I., Nejadgholi, I., Passmore, J.,
  Seltzer, J., Sanz, J.~B., Dutra, L., Samagaio, M., Elbadri, M., Mieskes, M.,
  Gerchick, M., Akinlolu, M., McKenna, M., Qiu, M., Ghauri, M., Burynok, M.,
  Abrar, N., Rajani, N., Elkott, N., Fahmy, N., Samuel, O., An, R., Kromann,
  R., Hao, R., Alizadeh, S., Shubber, S., Wang, S., Roy, S., Viguier, S., Le,
  T., Oyebade, T., Le, T., Yang, Y., Nguyen, Z., Kashyap, A.~R., Palasciano,
  A., Callahan, A., Shukla, A., Miranda-Escalada, A., Singh, A., Beilharz, B.,
  Wang, B., Brito, C., Zhou, C., Jain, C., Xu, C., Fourrier, C., Periñán,
  D.~L., Molano, D., Yu, D., Manjavacas, E., Barth, F., Fuhrimann, F., Altay,
  G., Bayrak, G., Burns, G., Vrabec, H.~U., Bello, I., Dash, I., Kang, J.,
  Giorgi, J., Golde, J., Posada, J.~D., Sivaraman, K.~R., Bulchandani, L., Liu,
  L., Shinzato, L., de~Bykhovetz, M.~H., Takeuchi, M., Pàmies, M., Castillo,
  M.~A., Nezhurina, M., Sänger, M., Samwald, M., Cullan, M., Weinberg, M.,
  De~Wolf, M., Mihaljcic, M., Liu, M., Freidank, M., Kang, M., Seelam, N.,
  Dahlberg, N., Broad, N.~M., Muellner, N., Fung, P., Haller, P.,
  Chandrasekhar, R., Eisenberg, R., Martin, R., Canalli, R., Su, R., Su, R.,
  Cahyawijaya, S., Garda, S., Deshmukh, S.~S., Mishra, S., Kiblawi, S., Ott,
  S., Sang-aroonsiri, S., Kumar, S., Schweter, S., Bharati, S., Laud, T.,
  Gigant, T., Kainuma, T., Kusa, W., Labrak, Y., Bajaj, Y.~S., Venkatraman, Y.,
  Xu, Y., Xu, Y., Xu, Y., Tan, Z., Xie, Z., Ye, Z., Bras, M., Belkada, Y., and
  Wolf, T.
\newblock Bloom: A 176b-parameter open-access multilingual language model,
  2022.
\newblock URL \url{https://arxiv.org/abs/2211.05100}.

\bibitem[Schwenk et~al.(2021)Schwenk, Chaudhary, Sun, Gong, and
  Guzm{\'a}n]{schwenk-etal-2021-wikimatrix}
Schwenk, H., Chaudhary, V., Sun, S., Gong, H., and Guzm{\'a}n, F.
\newblock {W}iki{M}atrix: Mining 135{M} parallel sentences in 1620 language
  pairs from {W}ikipedia.
\newblock In \emph{Proceedings of the 16th Conference of the European Chapter
  of the Association for Computational Linguistics: Main Volume}, pp.\
  1351--1361, Online, April 2021. Association for Computational Linguistics.
\newblock URL \url{https://aclanthology.org/2021.eacl-main.115}.

\bibitem[Selvaraju et~al.(2019)Selvaraju, Cogswell, Das, Vedantam, Parikh, and
  Batra]{selvaraju2019}
Selvaraju, R.~R., Cogswell, M., Das, A., Vedantam, R., Parikh, D., and Batra,
  D.
\newblock Grad-cam: Visual explanations from deep networks via gradient-based
  localization.
\newblock \emph{International Journal of Computer Vision}, 128\penalty0
  (2):\penalty0 336–359, Oct 2019.
\newblock ISSN 1573-1405.
\newblock URL \url{http://dx.doi.org/10.1007/s11263-019-01228-7}.

\bibitem[Senge et~al.(2022)Senge, Igamberdiev, and
  Habernal]{senge-etal-2021-size}
Senge, M., Igamberdiev, T., and Habernal, I.
\newblock One size does not fit all: Investigating strategies for
  differentially-private learning across {NLP} tasks.
\newblock In \emph{Proceedings of the 2022 Conference on Empirical Methods in
  Natural Language Processing}, pp.\  7340--7353, Abu Dhabi, United Arab
  Emirates, December 2022. Association for Computational Linguistics.
\newblock URL \url{https://aclanthology.org/2022.emnlp-main.496}.

\bibitem[Shokri et~al.(2021)Shokri, Strobel, and Zick]{shokri-etal-2021-risks}
Shokri, R., Strobel, M., and Zick, Y.
\newblock On the privacy risks of model explanations.
\newblock In \emph{Proceedings of the 2021 AAAI/ACM Conference on AI, Ethics,
  and Society}, AIES '21, pp.\  231–241, New York, NY, USA, 2021. Association
  for Computing Machinery.
\newblock ISBN 9781450384735.
\newblock URL \url{https://doi.org/10.1145/3461702.3462533}.

\bibitem[Singh et~al.(2019)Singh, McCann, Socher, and
  Xiong]{singh-etal-2019-bert}
Singh, J., McCann, B., Socher, R., and Xiong, C.
\newblock {BERT} is not an interlingua and the bias of tokenization.
\newblock In \emph{Proceedings of the 2nd Workshop on Deep Learning Approaches
  for Low-Resource NLP (DeepLo 2019)}, pp.\  47--55, Hong Kong, China, November
  2019. Association for Computational Linguistics.
\newblock URL \url{https://aclanthology.org/D19-6106}.

\bibitem[S\o{}gaard(2021)]{soegaard-2021-survey}
S\o{}gaard, A.
\newblock Explainable natural language processing.
\newblock \emph{Synthesis Lectures on Human Language Technologies}, 14\penalty0
  (3):\penalty0 1--123, 2021.
\newblock URL \url{https://doi.org/10.2200/S01118ED1V01Y202107HLT051}.

\bibitem[Strobel \& Shokri(2022)Strobel and Shokri]{strobel-shokri-2022}
Strobel, M. and Shokri, R.
\newblock Data privacy and trustworthy machine learning.
\newblock \emph{IEEE Security \& Privacy}, 20\penalty0 (5):\penalty0 44--49,
  2022.
\newblock URL \url{https://doi.org/10.1109/MSEC.2022.3178187}.

\bibitem[Tramèr \& Boneh(2021)Tramèr and Boneh]{tramer-2021-differentially}
Tramèr, F. and Boneh, D.
\newblock Differentially private learning needs better features (or much more
  data).
\newblock In \emph{Proceedings of the 9th International Conference on Learning
  Representations ({ICLR})}, Online, 2021. OpenReview.net.
\newblock URL \url{https://openreview.net/forum?id=YTWGvpFOQD-}.

\bibitem[Verma \& Rubin(2018)Verma and Rubin]{Verma2018FairnessDE}
Verma, S. and Rubin, J.~S.
\newblock Fairness definitions explained.
\newblock \emph{2018 IEEE/ACM International Workshop on Software Fairness
  (FairWare)}, pp.\  1--7, 2018.
\newblock URL \url{https://doi.org/10.1145/3194770.3194776}.

\bibitem[Wang et~al.(2018)Wang, Singh, Michael, Hill, Levy, and
  Bowman]{wang-etal-2018-glue}
Wang, A., Singh, A., Michael, J., Hill, F., Levy, O., and Bowman, S.
\newblock {GLUE}: A multi-task benchmark and analysis platform for natural
  language understanding.
\newblock In \emph{Proceedings of the 2018 {EMNLP} Workshop {B}lackbox{NLP}:
  Analyzing and Interpreting Neural Networks for {NLP}}, pp.\  353--355,
  Brussels, Belgium, November 2018. Association for Computational Linguistics.
\newblock URL \url{https://aclanthology.org/W18-5446}.

\bibitem[Wang \& Banko(2021)Wang and Banko]{wang-banko-2021-practical}
Wang, C. and Banko, M.
\newblock Practical transformer-based multilingual text classification.
\newblock In \emph{Proceedings of the 2021 Conference of the North American
  Chapter of the Association for Computational Linguistics: Human Language
  Technologies: Industry Papers}, pp.\  121--129, Online, June 2021.
  Association for Computational Linguistics.
\newblock URL \url{https://aclanthology.org/2021.naacl-industry.16}.

\bibitem[Wang et~al.(2020)Wang, Huang, Huang, Hu, Wang, and
  Gu]{wang-etal-2020-improv}
Wang, L., Huang, J., Huang, K., Hu, Z., Wang, G., and Gu, Q.
\newblock Improving neural language generation with spectrum control.
\newblock In \emph{Proceedings of the 8th International Conference on Learning
  Representations ({ICLR})}, Online, 2020. OpenReview.net.
\newblock URL \url{https://openreview.net/forum?id=ByxY8CNtvr}.

\bibitem[Williamson \& Menon(2019)Williamson and
  Menon]{williamson-etal-2019-fairness}
Williamson, R.~C. and Menon, A.~K.
\newblock Fairness risk measures.
\newblock In Chaudhuri, K. and Salakhutdinov, R. (eds.), \emph{Proceedings of
  the 36th International Conference on Machine Learning ({ICML})}, volume~97 of
  \emph{Proceedings of Machine Learning Research}, pp.\  6786--6797, Long
  Beach, CA, {USA}, 2019. {PMLR}.
\newblock URL \url{http://proceedings.mlr.press/v97/williamson19a.html}.

\bibitem[Wolf et~al.(2020)Wolf, Debut, Sanh, Chaumond, Delangue, Moi, Cistac,
  Rault, Louf, Funtowicz, Davison, Shleifer, von Platen, Ma, Jernite, Plu, Xu,
  Le~Scao, Gugger, Drame, Lhoest, and Rush]{wolf-etal-2020-transformers}
Wolf, T., Debut, L., Sanh, V., Chaumond, J., Delangue, C., Moi, A., Cistac, P.,
  Rault, T., Louf, R., Funtowicz, M., Davison, J., Shleifer, S., von Platen,
  P., Ma, C., Jernite, Y., Plu, J., Xu, C., Le~Scao, T., Gugger, S., Drame, M.,
  Lhoest, Q., and Rush, A.
\newblock Transformers: State-of-the-art natural language processing.
\newblock In \emph{Proceedings of the 2020 Conference on Empirical Methods in
  Natural Language Processing: System Demonstrations}, pp.\  38--45, Online,
  October 2020. Association for Computational Linguistics.
\newblock URL \url{https://aclanthology.org/2020.emnlp-demos.6}.

\bibitem[Wu \& Dredze(2019)Wu and Dredze]{wu-dredze-2019-beto}
Wu, S. and Dredze, M.
\newblock Beto, bentz, becas: The surprising cross-lingual effectiveness of
  {BERT}.
\newblock In \emph{Proceedings of the 2019 Conference on Empirical Methods in
  Natural Language Processing and the 9th International Joint Conference on
  Natural Language Processing (EMNLP-IJCNLP)}, pp.\  833--844, Hong Kong,
  China, November 2019. Association for Computational Linguistics.
\newblock URL \url{https://aclanthology.org/D19-1077}.

\bibitem[Yang et~al.(2022)Yang, Chen, Zhou, and Li]{anonymous2022enhancing}
Yang, H., Chen, H., Zhou, H., and Li, L.
\newblock Enhancing cross-lingual transfer by manifold mixup.
\newblock In \emph{Proceedings of the 10th International Conference on Learning
  Representations ({ICLR})}, Online, 2022. OpenReview.net.
\newblock URL \url{https://openreview.net/forum?id=OjPmfr9GkVv}.

\bibitem[Yeh et~al.(2018)Yeh, Kim, Yen, and
  Ravikumar]{yeh-etal-2018-representer}
Yeh, C., Kim, J.~S., Yen, I.~E., and Ravikumar, P.
\newblock Representer point selection for explaining deep neural networks.
\newblock In Bengio, S., Wallach, H.~M., Larochelle, H., Grauman, K.,
  Cesa{-}Bianchi, N., and Garnett, R. (eds.), \emph{Advances in Neural
  Information Processing Systems 31: Annual Conference on Neural Information
  Processing Systems ({NeurIPS})}, pp.\  9311--9321, Montr{\'{e}}al, Canada,
  2018. Curran Associates, Inc.
\newblock URL
  \url{https://proceedings.neurips.cc/paper/2018/hash/8a7129b8f3edd95b7d969dfc2c8e9d9d-Abstract.html}.

\bibitem[Yousefpour et~al.(2021)Yousefpour, Shilov, Sablayrolles, Testuggine,
  Prasad, Malek, Nguyen, Ghosh, Bharadwaj, Zhao, Cormode, and
  Mironov]{yousefpour-etal-2021-opacus}
Yousefpour, A., Shilov, I., Sablayrolles, A., Testuggine, D., Prasad, K.,
  Malek, M., Nguyen, J., Ghosh, S., Bharadwaj, A., Zhao, J., Cormode, G., and
  Mironov, I.
\newblock Opacus: User-friendly differential privacy library in pytorch.
\newblock In \emph{NeurIPS 2021 Workshop Privacy in Machine Learning}, Online,
  2021.
\newblock URL \url{https://openreview.net/forum?id=EopKEYBoI-}.

\bibitem[Yu et~al.(2022)Yu, Naik, Backurs, Gopi, Inan, Kamath, Kulkarni, Lee,
  Manoel, Wutschitz, Yekhanin, and Zhang]{yu-etal-2021-differentially}
Yu, D., Naik, S., Backurs, A., Gopi, S., Inan, H.~A., Kamath, G., Kulkarni, J.,
  Lee, Y.~T., Manoel, A., Wutschitz, L., Yekhanin, S., and Zhang, H.
\newblock Differentially private fine-tuning of language models.
\newblock In \emph{Proceedings of the 10th International Conference on Learning
  Representations ({ICLR})}, Online, 2022. OpenReview.net.
\newblock URL \url{https://openreview.net/forum?id=Q42f0dfjECO}.

\bibitem[Zeman et~al.(2021)Zeman, Nivre, Abrams, Ackermann, Aepli, Aghaei,
  Agi{\'c}, Ahmadi, Ahrenberg, Ajede, Aleksandravi{\v c}i{\=u}t{\.e}, Alfina,
  Antonsen, Aplonova, Aquino, Aragon, Aranzabe, Ar{\i}can, Arnard{\'o}ttir,
  Arutie, Arwidarasti, Asahara, Aslan, Ateyah, Atmaca, Attia, Atutxa,
  Augustinus, Badmaeva, Balasubramani, Ballesteros, Banerjee, Bank,
  Barbu~Mititelu, Barkarson, Basmov, Batchelor, Bauer, Bedir, Bengoetxea, Berk,
  Berzak, Bhat, Bhat, Biagetti, Bick, Bielinskien{\.e}, Bjarnad{\'o}ttir,
  Blokland, Bobicev, Boizou, Borges~V{\"o}lker, B{\"o}rstell, Bosco, Bouma,
  Bowman, Boyd, Braggaar, Brokait{\.e}, Burchardt, Candito, Caron, Caron,
  Cassidy, Cavalcanti, Cebiro{\u g}lu~Eryi{\u g}it, Cecchini, Celano, {\v
  C}{\'e}pl{\"o}, Cesur, Cetin, {\c C}etino{\u g}lu, Chalub, Chauhan, Chi,
  Chika, Cho, Choi, Chun, Cignarella, Cinkov{\'a}, Collomb, {\c C}{\"o}ltekin,
  Connor, Courtin, Cristescu, Daniel, Davidson, de~Marneffe, de~Paiva, Derin,
  de~Souza, Diaz~de Ilarraza, Dickerson, Dinakaramani, Di~Nuovo, Dione, Dirix,
  Dobrovoljc, Dozat, Droganova, Dwivedi, Eckhoff, Eiche, Eli, Elkahky, Ephrem,
  Erina, Erjavec, Etienne, Evelyn, Facundes, Farkas, Fernanda,
  Fernandez~Alcalde, Foster, Freitas, Fujita, Gajdo{\v s}ov{\'a}, Galbraith,
  Garcia, G{\"a}rdenfors, Garza, Gerardi, Gerdes, Ginter, Godoy, Goenaga,
  Gojenola, G{\"o}k{\i}rmak, Goldberg, G{\'o}mez~Guinovart,
  Gonz{\'a}lez~Saavedra, Grici{\=u}t{\.e}, Grioni, Grobol, Gr{\=
  u}z{\={\i}}tis, Guillaume, Guillot-Barbance, G{\"u}ng{\"o}r, Habash,
  Hafsteinsson, Haji{\v c}, Haji{\v c}~jr., H{\"a}m{\"a}l{\"a}inen,
  H{\`a}~M{\~y}, Han, Hanifmuti, Hardwick, Harris, Haug, Heinecke, Hellwig,
  Hennig, Hladk{\'a}, Hlav{\'a}{\v c}ov{\'a}, Hociung, Hohle, Huber, Hwang,
  Ikeda, Ingason, Ion, Irimia, Ishola, Ito, Jel{\'{\i}}nek, Jha, Johannsen,
  J{\'o}nsd{\'o}ttir, J{\o}rgensen, Juutinen, K, Ka{\c s}{\i}kara, Kaasen,
  Kabaeva, Kahane, Kanayama, Kanerva, Kara, Katz, Kayadelen, Kenney,
  Kettnerov{\'a}, Kirchner, Klementieva, K{\"o}hn, K{\"o}ksal, Kopacewicz,
  Korkiakangas, Kotsyba, Kovalevskait{\.e}, Krek, Krishnamurthy, Kuyruk{\c c}u,
  Kuzgun, Kwak, Laippala, Lam, Lambertino, Lando, Larasati, Lavrentiev, Lee,
  \fontencoding{T5}\selectfont {Phương L{\^e}~H{\`{\^o}}ng}, Lenci,
  Lertpradit, Leung, Levina, Li, Li, Li, Li, Lim, Lima~Padovani, Lind{\'e}n,
  Ljube{\v s}i{\'c}, Loginova, Luthfi, Luukko, Lyashevskaya, Lynn, Macketanz,
  Makazhanov, Mandl, Manning, Manurung, Mar{\c s}an, M{\u a}r{\u a}nduc,
  Mare{\v c}ek, Marheinecke, Mart{\'{\i}}nez~Alonso, Martins, Ma{\v s}ek,
  Matsuda, Matsumoto, Mazzei, {McDonald}, {McGuinness}, Mendon{\c c}a, Miekka,
  Mischenkova, Misirpashayeva, Missil{\"a}, Mititelu, Mitrofan, Miyao,
  Mojiri~Foroushani, Moln{\'a}r, Moloodi, Montemagni, More, Moreno~Romero,
  Moretti, Mori, Mori, Morioka, Moro, Mortensen, Moskalevskyi, Muischnek,
  Munro, Murawaki, M{\"u}{\"u}risep, Nainwani, Nakhl{\'e},
  Navarro~Hor{\~n}iacek, Nedoluzhko, Ne{\v s}pore-B{\=e}rzkalne, Nevaci,
  \fontencoding{T5}\selectfont{Lương Nguy{\~{\^e}}n Th{\d i}}, Nguy{\~{\^e}}n
  Th{\d i}~Minh, Nikaido, Nikolaev, Nitisaroj, Nourian, Nurmi, Ojala, Ojha,
  Ol{\'u}{\`o}kun, Omura, Onwuegbuzia, Osenova, {\"O}stling, {\O}vrelid,
  {\"O}zate{\c s}, {\"O}z{\c c}elik, {\"O}zg{\"u}r, {\"O}zt{\"u}rk~Ba{\c
  s}aran, Park, Partanen, Pascual, Passarotti, Patejuk, Paulino-Passos,
  Peljak-{\L}api{\'n}ska, Peng, Perez, Perkova, Perrier, Petrov, Petrova,
  Phelan, Piitulainen, Pirinen, Pitler, Plank, Poibeau, Ponomareva, Popel,
  Pretkalni{\c n}a, Pr{\'e}vost, Prokopidis, Przepi{\'o}rkowski, Puolakainen,
  Pyysalo, Qi, R{\"a}{\"a}bis, Rademaker, Rama, Ramasamy, Ramisch, Rashel,
  Rasooli, Ravishankar, Real, Rebeja, Reddy, Rehm, Riabov, Rie{\ss}ler,
  Rimkut{\.e}, Rinaldi, Rituma, Rocha, R{\"o}gnvaldsson, Romanenko, Rosa,
  Roșca, Rovati, Rudina, Rueter, R{\'u}narsson, Sadde, Safari, Sagot, Sahala,
  Saleh, Salomoni, Samard{\v z}i{\'c}, Samson, Sanguinetti, San{\i}yar,
  S{\"a}rg, Saul{\={\i}}te, Sawanakunanon, Saxena, Scannell, Scarlata,
  Schneider, Schuster, Schwartz, Seddah, Seeker, Seraji, Shen, Shimada,
  Shirasu, Shishkina, Shohibussirri, Sichinava, Siewert,
  \fontencoding{T1}\selectfont{Einar Freyr Sigurðsson}, Silveira, Silveira,
  Simi, Simionescu, Simk{\'o}, {\v S}imkov{\'a}, Simov, Skachedubova, Smith,
  Soares-Bastos, Spadine, Sprugnoli, Steingr{\'{\i}}msson, Stella, Straka,
  Strickland, Strnadov{\'a}, Suhr, Sulestio, Sulubacak, Suzuki, Sz{\'a}nt{\'o},
  Taji, Takahashi, Tamburini, Tan, Tanaka, Tella, Tellier, Testori, Thomas,
  Torga, Toska, Trosterud, Trukhina, Tsarfaty, T{\"u}rk, Tyers, Uematsu,
  Untilov, Ure{\v s}ov{\'a}, Uria, Uszkoreit, Utka, Vajjala, van~der Goot,
  Vanhove, van Niekerk, van Noord, Varga, Villemonte de~la Clergerie, Vincze,
  Vlasova, Wakasa, Wallenberg, Wallin, Walsh, Wang, Washington, Wendt, Widmer,
  Williams, Wir{\'e}n, Wittern, Woldemariam, Wong, Wr{\'o}blewska, Yako,
  Yamashita, Yamazaki, Yan, Yasuoka, Yavrumyan, Yenice, Y{\i}ld{\i}z, Yu, {\v
  Z}abokrtsk{\'y}, Zahra, Zeldes, Zhu, Zhuravleva, and
  Ziane]{zeman-etal-2021-ud}
Zeman, D., Nivre, J., Abrams, M., Ackermann, E., Aepli, N., Aghaei, H.,
  Agi{\'c}, {\v Z}., Ahmadi, A., Ahrenberg, L., Ajede, C.~K., Aleksandravi{\v
  c}i{\=u}t{\.e}, G., Alfina, I., Antonsen, L., Aplonova, K., Aquino, A.,
  Aragon, C., Aranzabe, M.~J., Ar{\i}can, B.~N., Arnard{\'o}ttir, {\t H}.,
  Arutie, G., Arwidarasti, J.~N., Asahara, M., Aslan, D.~B., Ateyah, L.,
  Atmaca, F., Attia, M., Atutxa, A., Augustinus, L., Badmaeva, E.,
  Balasubramani, K., Ballesteros, M., Banerjee, E., Bank, S., Barbu~Mititelu,
  V., Barkarson, S., Basmov, V., Batchelor, C., Bauer, J., Bedir, S.~T.,
  Bengoetxea, K., Berk, G., Berzak, Y., Bhat, I.~A., Bhat, R.~A., Biagetti, E.,
  Bick, E., Bielinskien{\.e}, A., Bjarnad{\'o}ttir, K., Blokland, R., Bobicev,
  V., Boizou, L., Borges~V{\"o}lker, E., B{\"o}rstell, C., Bosco, C., Bouma,
  G., Bowman, S., Boyd, A., Braggaar, A., Brokait{\.e}, K., Burchardt, A.,
  Candito, M., Caron, B., Caron, G., Cassidy, L., Cavalcanti, T., Cebiro{\u
  g}lu~Eryi{\u g}it, G., Cecchini, F.~M., Celano, G. G.~A., {\v C}{\'e}pl{\"o},
  S., Cesur, N., Cetin, S., {\c C}etino{\u g}lu, {\"O}., Chalub, F., Chauhan,
  S., Chi, E., Chika, T., Cho, Y., Choi, J., Chun, J., Cignarella, A.~T.,
  Cinkov{\'a}, S., Collomb, A., {\c C}{\"o}ltekin, {\c C}., Connor, M.,
  Courtin, M., Cristescu, M., Daniel, P., Davidson, E., de~Marneffe, M.-C.,
  de~Paiva, V., Derin, M.~O., de~Souza, E., Diaz~de Ilarraza, A., Dickerson,
  C., Dinakaramani, A., Di~Nuovo, E., Dione, B., Dirix, P., Dobrovoljc, K.,
  Dozat, T., Droganova, K., Dwivedi, P., Eckhoff, H., Eiche, S., Eli, M.,
  Elkahky, A., Ephrem, B., Erina, O., Erjavec, T., Etienne, A., Evelyn, W.,
  Facundes, S., Farkas, R., Fernanda, M., Fernandez~Alcalde, H., Foster, J.,
  Freitas, C., Fujita, K., Gajdo{\v s}ov{\'a}, K., Galbraith, D., Garcia, M.,
  G{\"a}rdenfors, M., Garza, S., Gerardi, F.~F., Gerdes, K., Ginter, F., Godoy,
  G., Goenaga, I., Gojenola, K., G{\"o}k{\i}rmak, M., Goldberg, Y.,
  G{\'o}mez~Guinovart, X., Gonz{\'a}lez~Saavedra, B., Grici{\=u}t{\.e}, B.,
  Grioni, M., Grobol, L., Gr{\= u}z{\={\i}}tis, N., Guillaume, B.,
  Guillot-Barbance, C., G{\"u}ng{\"o}r, T., Habash, N., Hafsteinsson, H.,
  Haji{\v c}, J., Haji{\v c}~jr., J., H{\"a}m{\"a}l{\"a}inen, M.,
  H{\`a}~M{\~y}, L., Han, N.-R., Hanifmuti, M.~Y., Hardwick, S., Harris, K.,
  Haug, D., Heinecke, J., Hellwig, O., Hennig, F., Hladk{\'a}, B., Hlav{\'a}{\v
  c}ov{\'a}, J., Hociung, F., Hohle, P., Huber, E., Hwang, J., Ikeda, T.,
  Ingason, A.~K., Ion, R., Irimia, E., Ishola, {\d O}., Ito, K.,
  Jel{\'{\i}}nek, T., Jha, A., Johannsen, A., J{\'o}nsd{\'o}ttir, H.,
  J{\o}rgensen, F., Juutinen, M., K, S., Ka{\c s}{\i}kara, H., Kaasen, A.,
  Kabaeva, N., Kahane, S., Kanayama, H., Kanerva, J., Kara, N., Katz, B.,
  Kayadelen, T., Kenney, J., Kettnerov{\'a}, V., Kirchner, J., Klementieva, E.,
  K{\"o}hn, A., K{\"o}ksal, A., Kopacewicz, K., Korkiakangas, T., Kotsyba, N.,
  Kovalevskait{\.e}, J., Krek, S., Krishnamurthy, P., Kuyruk{\c c}u, O.,
  Kuzgun, A., Kwak, S., Laippala, V., Lam, L., Lambertino, L., Lando, T.,
  Larasati, S.~D., Lavrentiev, A., Lee, J., \fontencoding{T5}\selectfont
  {Phương L{\^e}~H{\`{\^o}}ng}, Lenci, A., Lertpradit, S., Leung, H., Levina,
  M., Li, C.~Y., Li, J., Li, K., Li, Y., Lim, K., Lima~Padovani, B.,
  Lind{\'e}n, K., Ljube{\v s}i{\'c}, N., Loginova, O., Luthfi, A., Luukko, M.,
  Lyashevskaya, O., Lynn, T., Macketanz, V., Makazhanov, A., Mandl, M.,
  Manning, C., Manurung, R., Mar{\c s}an, B., M{\u a}r{\u a}nduc, C., Mare{\v
  c}ek, D., Marheinecke, K., Mart{\'{\i}}nez~Alonso, H., Martins, A., Ma{\v
  s}ek, J., Matsuda, H., Matsumoto, Y., Mazzei, A., {McDonald}, R.,
  {McGuinness}, S., Mendon{\c c}a, G., Miekka, N., Mischenkova, K.,
  Misirpashayeva, M., Missil{\"a}, A., Mititelu, C., Mitrofan, M., Miyao, Y.,
  Mojiri~Foroushani, A., Moln{\'a}r, J., Moloodi, A., Montemagni, S., More, A.,
  Moreno~Romero, L., Moretti, G., Mori, K.~S., Mori, S., Morioka, T., Moro, S.,
  Mortensen, B., Moskalevskyi, B., Muischnek, K., Munro, R., Murawaki, Y.,
  M{\"u}{\"u}risep, K., Nainwani, P., Nakhl{\'e}, M., Navarro~Hor{\~n}iacek,
  J.~I., Nedoluzhko, A., Ne{\v s}pore-B{\=e}rzkalne, G., Nevaci, M.,
  \fontencoding{T5}\selectfont{Lương Nguy{\~{\^e}}n Th{\d i}}, Nguy{\~{\^e}}n
  Th{\d i}~Minh, H., Nikaido, Y., Nikolaev, V., Nitisaroj, R., Nourian, A.,
  Nurmi, H., Ojala, S., Ojha, A.~K., Ol{\'u}{\`o}kun, A., Omura, M.,
  Onwuegbuzia, E., Osenova, P., {\"O}stling, R., {\O}vrelid, L., {\"O}zate{\c
  s}, {\c S}.~B., {\"O}z{\c c}elik, M., {\"O}zg{\"u}r, A., {\"O}zt{\"u}rk~Ba{\c
  s}aran, B., Park, H.~H., Partanen, N., Pascual, E., Passarotti, M., Patejuk,
  A., Paulino-Passos, G., Peljak-{\L}api{\'n}ska, A., Peng, S., Perez, C.-A.,
  Perkova, N., Perrier, G., Petrov, S., Petrova, D., Phelan, J., Piitulainen,
  J., Pirinen, T.~A., Pitler, E., Plank, B., Poibeau, T., Ponomareva, L.,
  Popel, M., Pretkalni{\c n}a, L., Pr{\'e}vost, S., Prokopidis, P.,
  Przepi{\'o}rkowski, A., Puolakainen, T., Pyysalo, S., Qi, P., R{\"a}{\"a}bis,
  A., Rademaker, A., Rama, T., Ramasamy, L., Ramisch, C., Rashel, F., Rasooli,
  M.~S., Ravishankar, V., Real, L., Rebeja, P., Reddy, S., Rehm, G., Riabov,
  I., Rie{\ss}ler, M., Rimkut{\.e}, E., Rinaldi, L., Rituma, L., Rocha, L.,
  R{\"o}gnvaldsson, E., Romanenko, M., Rosa, R., Roșca, V., Rovati, D.,
  Rudina, O., Rueter, J., R{\'u}narsson, K., Sadde, S., Safari, P., Sagot, B.,
  Sahala, A., Saleh, S., Salomoni, A., Samard{\v z}i{\'c}, T., Samson, S.,
  Sanguinetti, M., San{\i}yar, E., S{\"a}rg, D., Saul{\={\i}}te, B.,
  Sawanakunanon, Y., Saxena, S., Scannell, K., Scarlata, S., Schneider, N.,
  Schuster, S., Schwartz, L., Seddah, D., Seeker, W., Seraji, M., Shen, M.,
  Shimada, A., Shirasu, H., Shishkina, Y., Shohibussirri, M., Sichinava, D.,
  Siewert, J., \fontencoding{T1}\selectfont{Einar Freyr Sigurðsson}, Silveira,
  A., Silveira, N., Simi, M., Simionescu, R., Simk{\'o}, K., {\v S}imkov{\'a},
  M., Simov, K., Skachedubova, M., Smith, A., Soares-Bastos, I., Spadine, C.,
  Sprugnoli, R., Steingr{\'{\i}}msson, S., Stella, A., Straka, M., Strickland,
  E., Strnadov{\'a}, J., Suhr, A., Sulestio, Y.~L., Sulubacak, U., Suzuki, S.,
  Sz{\'a}nt{\'o}, Z., Taji, D., Takahashi, Y., Tamburini, F., Tan, M. A.~C.,
  Tanaka, T., Tella, S., Tellier, I., Testori, M., Thomas, G., Torga, L.,
  Toska, M., Trosterud, T., Trukhina, A., Tsarfaty, R., T{\"u}rk, U., Tyers,
  F., Uematsu, S., Untilov, R., Ure{\v s}ov{\'a}, Z., Uria, L., Uszkoreit, H.,
  Utka, A., Vajjala, S., van~der Goot, R., Vanhove, M., van Niekerk, D., van
  Noord, G., Varga, V., Villemonte de~la Clergerie, E., Vincze, V., Vlasova,
  N., Wakasa, A., Wallenberg, J.~C., Wallin, L., Walsh, A., Wang, J.~X.,
  Washington, J.~N., Wendt, M., Widmer, P., Williams, S., Wir{\'e}n, M.,
  Wittern, C., Woldemariam, T., Wong, T.-s., Wr{\'o}blewska, A., Yako, M.,
  Yamashita, K., Yamazaki, N., Yan, C., Yasuoka, K., Yavrumyan, M.~M., Yenice,
  A.~B., Y{\i}ld{\i}z, O.~T., Yu, Z., {\v Z}abokrtsk{\'y}, Z., Zahra, S.,
  Zeldes, A., Zhu, H., Zhuravleva, A., and Ziane, R.
\newblock Universal dependencies 2.8, 2021.
\newblock URL \url{http://hdl.handle.net/11234/1-3683}.
\newblock {LINDAT}/{CLARIAH}-{CZ} digital library at the Institute of Formal
  and Applied Linguistics ({{\'U}FAL}), Faculty of Mathematics and Physics,
  Charles University.

\bibitem[Zhang et~al.(2021)Zhang, Huang, Zhu, Ye, Cui, and
  Zhang]{zhang-etal-2021-sample}
Zhang, W., Huang, Z., Zhu, Y., Ye, G., Cui, X., and Zhang, F.
\newblock On sample based explanation methods for {NLP}: Faithfulness,
  efficiency and semantic evaluation.
\newblock In \emph{Proceedings of the 59th Annual Meeting of the Association
  for Computational Linguistics and the 11th International Joint Conference on
  Natural Language Processing (Volume 1: Long Papers)}, pp.\  5399--5411,
  Online, August 2021. Association for Computational Linguistics.
\newblock URL \url{https://aclanthology.org/2021.acl-long.419}.

\bibitem[Zhelezniak et~al.(2019)Zhelezniak, Savkov, Shen, and
  Hammerla]{zhelezniak-etal-2019-correlation}
Zhelezniak, V., Savkov, A., Shen, A., and Hammerla, N.
\newblock Correlation coefficients and semantic textual similarity.
\newblock In \emph{Proceedings of the 2019 Conference of the North {A}merican
  Chapter of the Association for Computational Linguistics: Human Language
  Technologies, Volume 1 (Long and Short Papers)}, pp.\  951--962, Minneapolis,
  Minnesota, June 2019. Association for Computational Linguistics.
\newblock URL \url{https://aclanthology.org/N19-1100}.

\bibitem[Zhou et~al.(2019)Zhou, Sedoc, and Rodu]{zhou-etal-2019-getting}
Zhou, T., Sedoc, J., and Rodu, J.
\newblock Getting in shape: Word embedding subspaces.
\newblock In Kraus, S. (ed.), \emph{Proceedings of the 28th International Joint
  Conference on Artificial Intelligence ({IJCAI})}, pp.\  5478--5484, Macao,
  China, 2019. ijcai.org.
\newblock URL \url{https://doi.org/10.24963/ijcai.2019/761}.

\bibitem[Zhou et~al.(2021{\natexlab{a}})Zhou, Lin, and
  Ren]{zhou-etal-2021-isobn}
Zhou, W., Lin, B.~Y., and Ren, X.
\newblock Isobn: Fine-tuning {BERT} with isotropic batch normalization.
\newblock In \emph{Proceedings of the 35th {AAAI} Conference on Artificial
  Intelligence}, pp.\  14621--14629, Online, 2021{\natexlab{a}}. {AAAI} Press.
\newblock URL \url{https://ojs.aaai.org/index.php/AAAI/article/view/17718}.

\bibitem[Zhou et~al.(2021{\natexlab{b}})Zhou, Booth, Ribeiro, and
  Shah]{zhou2021do}
Zhou, Y., Booth, S., Ribeiro, M.~T., and Shah, J.
\newblock Do feature attribution methods correctly attribute features?
\newblock In \emph{XAI 4 Debugging Workshop at NeurIPS 2021}, Online,
  2021{\natexlab{b}}. OpenReview.
\newblock URL \url{https://openreview.net/forum?id=h4J41lQqaJ3}.

\end{thebibliography}
\bibliographystyle{icml2023}

\appendix
\onecolumn

\section{Reproducibility}
We make our code available at \url{https://github.com/xplip/multilingual-lm-objectives}.

\paragraph{Implementation}
Our implementation is written in PyTorch v1.10.0 \citep{paszke-etal-2019-pytorch} for Python 3.9.5 and builds on code from the following repositories:
\begin{itemize}
    \itemsep-0.2em 
    \item \href{https://github.com/huggingface/transformers}{https://github.com/huggingface/transformers} v4.9.2 \citep{wolf-etal-2020-transformers} for model training and evaluation
    \item \href{https://github.com/lxuechen/private-transformers}{https://github.com/lxuechen/private-transformers} v0.1.0 \citep{li-etal-2021-dp-lm} for DP-training 
    \item \href{https://github.com/pdufter/minimult}{https://github.com/pdufter/minimult} \citep{dufter-schutze-2020-identifying} for computing sentence retrieval precision
    \item \href{https://github.com/jayroxis/CKA-similarity}{https://github.com/jayroxis/CKA-similarity} for computing CKA scores
    \item \href{https://github.com/mlepori1/Picking\_BERTs\_Brain}{https://github.com/mlepori1/Picking\_BERTs\_Brain} \citep{lepori-mccoy-2020-picking} for computing RSA scores
    \item \href{https://github.com/bcbi-edu/p\_eickhoff\_isoscore}{https://github.com/bcbi-edu/p\_eickhoff\_isoscore} \citep{rudman-etal-2021-iso} for computing IsoScores
    \item \href{https://github.com/FengNiMa/VAE-TracIn-pytorch}{https://github.com/FengNiMa/VAE-TracIn-pytorch} \citep{kong-chaudhuri-2021-understanding} for computing $\mathrm{TracInCP}$ scores.
\end{itemize}

\paragraph{Models}
We primarily use the pretrained XLM-RoBERTa \citep[XLM-R; ][]{conneau-etal-2020-unsupervised} base model and tokenizer from \href{https://huggingface.co/xlm-roberta-base}{https://huggingface.co/xlm-roberta-base}. XLM-R (base) is a 12-layer encoder-only transformer with a vocabulary size of 250k and {\raise.17ex\hbox{$\scriptstyle\sim$}}277M total parameters pretrained via masked language modeling on the 100-language CC-100 dataset.

In Appendix~\ref{mbert-results}, we further conduct experiments with multilingual BERT \citep[mBERT; ][]{devlin-etal-2019-bert}, using the base model and tokenizer from \href{https://huggingface.co/bert-base-multilingual-cased}{https://huggingface.co/bert-base-multilingual-cased}. mBERT is a 12-layer encoder-only transformer with a vocabulary size of 120k and {\raise.17ex\hbox{$\scriptstyle\sim$}}177M total parameters pretrained via masked language modeling on Wikipedia data in 104 languages.

\paragraph{Data}
We provide download links and references for the various datasets we used in Table~\ref{tab:dataset_links}.

\paragraph{Hardware}
We train on single Nvidia Titan RTX, A100 (both with CUDA version 11.0), and RTX 3090 (with CUDA version 11.5) GPUs. All machines have at least 64GB of RAM, which is required to compute the IsoScore for our larger evaluation sets (e.g., TED 2020 for POS).

\paragraph{Runtime}
Fine-tuning with evaluation during training on the Titan RTX, which is the slowest of the GPUs used, takes 2--3 hours for POS and 5--6 hours for XNLI. Computing $\mathrm{TracInCP}$ influence scores for one fine-tuned model takes about 30--45 minutes.

\paragraph{Carbon Footprint}
Our fine-tuning runs accumulated {\raise.17ex\hbox{$\scriptstyle\sim$}}36 compute days on the hardware mentioned above (most experiments were conducted on the less powerful Titan RTX GPUs) according to Weights \& Biases\footnote{\href{https://wandb.ai/}{https://wandb.ai/}}, where we logged our experiments. Although we do not have precise numbers, a highly conservative estimate of the total compute spent including prototyping, hyper-parameter search, and all our evaluations is {\raise.17ex\hbox{$\scriptstyle\sim$}}75 compute days.

\section{(\texorpdfstring{$\pmb{\varepsilon}$, $\pmb{\delta}$}{ε, δ})-Differential Privacy}
\label{sec:eps_delta_dp}

In \S\ref{sec:background}, we provide the definition of $\varepsilon$-differential privacy (DP), also called pure DP, as the basis for our theoretical exploration. In our experiments, we rely on ($\varepsilon$, $\delta$)-DP \citep{dwork-roth-2014-foundations}, also called approximate-DP, which is typically used in practice and relaxes the privacy guarantees by a (small) $\delta$ as follows:

\noindent
A randomized algorithm ${\mathcal{M}: \mathcal{D} \to \mathcal{Y}}$ is {\em ($\varepsilon$, $\delta$)-differentially private} \citep{dwork-etal-2006-dp} iff for all adjacent datasets ${D, D' \in \mathcal{D}}$ and all ${Y \subset \mathcal{Y}}$, ${\mathbb{P}(\mathcal{M}(D) \in Y)} \leq {\exp(\varepsilon_p)\cdot \mathbb{P}(\mathcal{M}(D') \in Y)} + \delta$.

\newpage 

\section{Best Fine-Tuning Settings}
\label{sec:best_settings}
As mentioned in \S\ref{sec:exp_setup}, we pre-selected a set of suitable learning rates (LRs) for each task and ran 3 random initializations each. Based on the validation performance, we then selected the following 5 best settings for each privacy budget and task:

\begin{table}[H]
\centering
\caption{Best 5 settings for each task and privacy budget. Includes LR and the corresponding number of random initializations (\# seeds).}
\label{tab:best_settings_pos}
\resizebox{0.48\textwidth}{!}{%
\begin{tabular}{@{}cll@{}}
\toprule
\textbf{$\pmb{\varepsilon}$} & \textbf{POS LR (\# Seeds)} & \textbf{XNLI LR (\# Seeds)}\\\midrule
1        & \num{5e-4} (2); \num{7e-4} (3) & \num{3e-4} (1); \num{4e-4} (2); \num{5e-4} (2)  \\
3        & \num{5e-4} (2); \num{7e-4} (3) & \num{3e-4} (1); \num{4e-4} (2); \num{5e-4} (2)                 \\
8        & \num{5e-4} (3); \num{7e-4} (2) & \num{4e-4} (2); \num{5e-4} (3)                 \\
15       & \num{3e-4} (1); \num{5e-4} (2); \num{7e-4} (2) & \num{3e-4} (1); \num{4e-4} (2); \num{5e-4} (2)                \\
30       & \num{3e-4} (1); \num{5e-4} (2); \num{7e-4} (2) & \num{3e-4} (1); \num{4e-4} (2); \num{5e-4} (2)                 \\
$\infty$ & \num{5e-5} (2); \num{7e-5} (2); \num{1e-4} (1) & \num{9e-5} (2); \num{1e-4} (3)                \\              \bottomrule
\end{tabular}%
}
\end{table}

\vfill

\section{IsoScore Algorithm}
\label{a:definition_isoscore}
Algorithm~\ref{alg:isoscore} describes the IsoScore algorithm \citep{rudman-etal-2021-iso}.

\begin{algorithm*}[hbtp]
\caption{IsoScore \citep{rudman-etal-2021-iso}}\label{alg:isoscore}
\begin{algorithmic}[1]
    \STATE \textbf{begin} Let $X \subset \mathbb{R}^n$ be a finite collection of points.
    \STATE \hspace{3mm} Let $X^{PCA}$ denote the points in $X$ transformed by the first $n$ principal components.
    \STATE \hspace{3mm} Define $\Sigma_{D} \in \mathbb{R}^n$ as the diagonal of the covariance matrix of $X^{PCA}$.
    \STATE \hspace{3mm} Normalize diagonal to $\hat{\Sigma}_D:=\sqrt{n} \cdot \Sigma_{D}/\norm{\Sigma_{D}}$, where $\norm{\cdot}$ is the standard Euclidean norm.
    \STATE \hspace{3mm} The isotropy defect is $\delta(X):={\norm{\hat{\Sigma}_{D} - \mathbf{1}}}/\sqrt{2(n - \sqrt{n})}$, where $\mathbf{1} = (1, \ldots, 1)^{\text{T}} \in \mathbb{R}^n$
    \STATE \hspace{3mm} $X$ uniformly occupies $\phi(X):=(n - \delta(X)^{2}(n -  \sqrt{n}))^{2}/n^{2}$ percent of ambient dimensions.
    \STATE \hspace{3mm} Transform $\phi(X)$ so it can take values in $[0, 1]$, via $\iota(X):=(n \cdot \phi(X) - 1)/(n - 1)$.
    \STATE \hspace{3mm} \textbf{return:} $\iota(X)$
    \STATE \textbf{end}
\end{algorithmic}
\end{algorithm*}

\vfill

\section{Further Analysis of RSA Results}
\label{sec:rsa_details}
As we see in \S\ref{sec:question1}, RSA aligns with sentence retrieval precision, CKA, and IsoScore in producing higher scores for non-private models. However, there is a mismatch between RSA and the other metrics in highly private regimes, where our most private models ($\varepsilon = 1$) do not exhibit high RSA scores. Instead, the aggregated RSA scores peak at medium levels of privacy ($\varepsilon \in \{8, 15\}$) and for the non-private ($\varepsilon = \infty$) models. Unlike for the other metrics, there is also no clear trend among our two tasks in terms of whether the pretrained or a randomly initialized XLM-R model scores higher in RSA.

A closer look at the non-aggregated results (Appendix Figures~\ref{fig:pos_rsa_tedwm}, \ref{fig:pos_rsa_tatoeba}, and \ref{fig:xnli_rsa_full}) shows how the similarity patterns obtained from RSA are often unexpected. For instance, the similarities between the typologically distant languages \textsc{fr} and \textsc{zh} are consistently high for the TED~2020 corpus whereas scores for typologically closer languages are lower (Fig.~\ref{fig:pos_rsa_tedwm}). Based on prior work by, for example, \citet{pires-etal-2019-multilingual}, \citet{wu-dredze-2019-beto}, and \citet{lauscher-etal-2020-zero},  we would expect the model to first compress similar languages before achieving compression for distant ones. Sometimes, we also observe extreme jumps in similarity between layers 0 and 8, for instance, between \textsc{it} and \textsc{tr} in the Tatoeba corpus (Fig.~\ref{fig:pos_rsa_tatoeba}). We do not find these jumps in CKA and sentence retrieval.

One reason why RSA scores may be more sensitive to stricter privacy guarantees (e.g., $\varepsilon=1$) is that the correlation between sentence vector distances is very sensitive to outliers. Differential privacy reduces the number of such outliers, effectively regularizing the correlation coefficients.

\newpage

\section{Multilingual BERT Results}
\label{mbert-results}

In Figures~\ref{fig:mbert_results_1} and \ref{fig:interpretability_plots_mbert}, we present results from re-running the experiments from \S\ref{sec:question1} and \S\ref{sec:empirical_interpretability} with mBERT. We make two changes to the experimental setup outlined above: We use representations extracted at layer 8, which showed to be more meaningful than layer 0 in the XLM-R experiments, to compute the multilinguality metrics. We also include two additional privacy settings, $\varepsilon=0.5$ and $\varepsilon=0.7$, as we found mBERT to be easier to finetune with strong privacy guarantees than XLM-R.

We see the same trends as for XLM-R: performance strictly increases with decreasing privacy while the multilinguality metrics tend to follow a U-shape,\footnote{We again refer to Appendix~\ref{sec:rsa_details} for a discussion of the RSA results.} i.e., they are high for strong privacy settings (small $\varepsilon$) and low privacy settings (large $\varepsilon$) and decrease towards medium privacy.
Likewise, we find a positive correlation between InfU and cross-lingual sentence retrieval precision. The correlation is strong for part-of-speech tagging (POS) but it is mild for XNLI. We believe this may be due to mBERT being less sensitive to the privacy parameter (Figure~\ref{fig:xnli_compression_mbert} is not symmetrical; considering even stronger privacy settings would likely even out the U-shape). Overall, these results further support our finding that there is a negative correlation between multilingual compression and training data influence sparsity.

\vfill 

\begin{figure*}[ht!]
    \centering
    \begin{subfigure}[b]{0.19\textwidth}
        \centering
        \includegraphics[width=\textwidth]{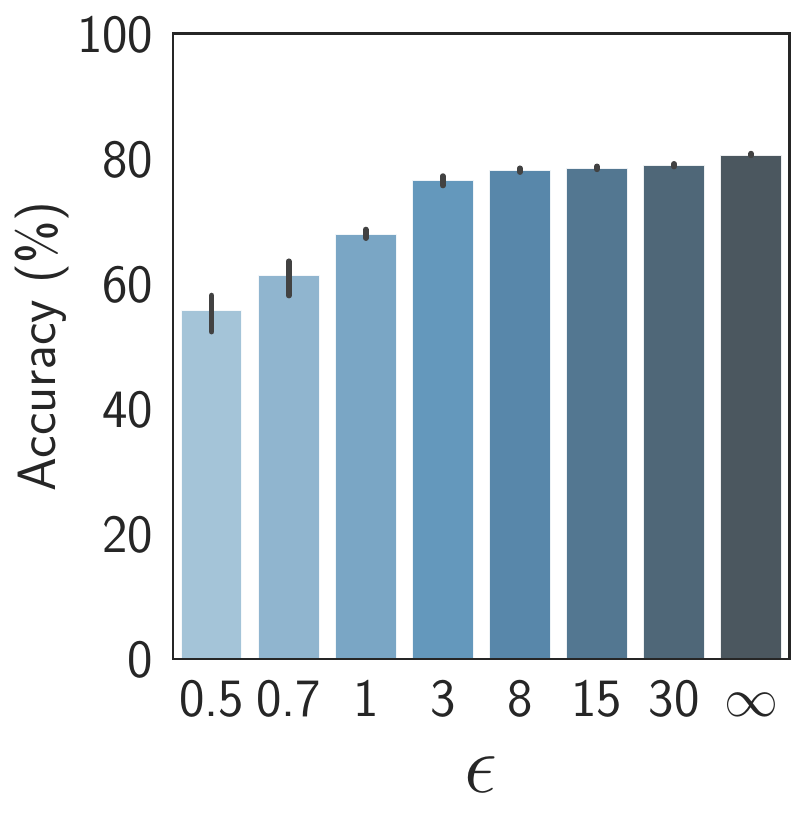}
        \caption{POS Performance}
        \label{fig:pos_acc_mbert}
    \end{subfigure}
    \begin{subfigure}[b]{0.189\textwidth}
        \centering
        \includegraphics[width=\textwidth]{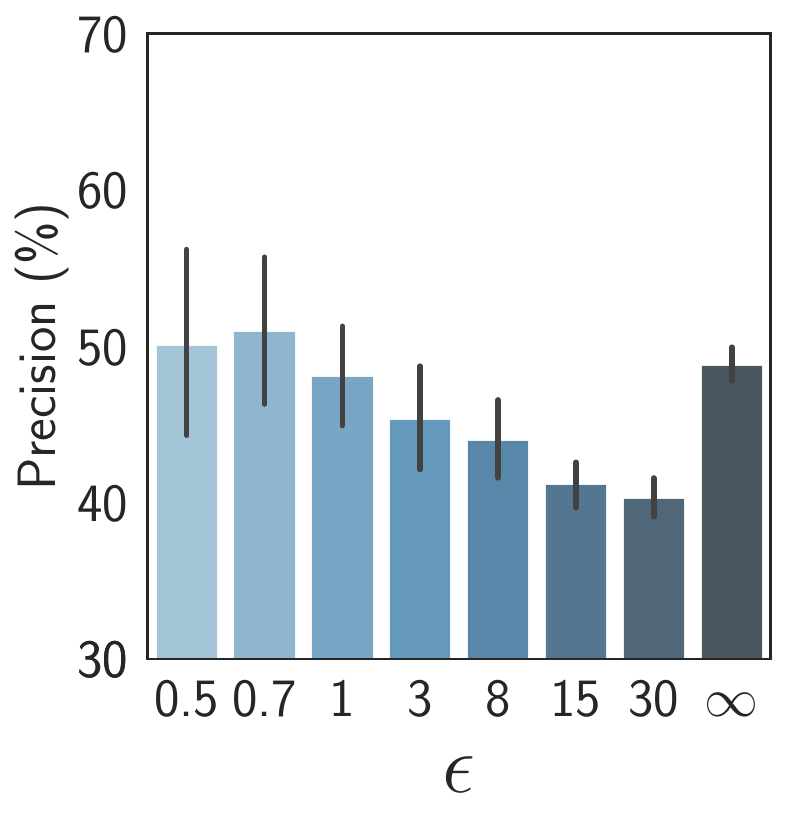}
        \caption{POS Retrieval}
        \label{fig:pos_compression_mbert}
    \end{subfigure}
    \begin{subfigure}[b]{0.1984\textwidth}
        \centering
        \includegraphics[width=\textwidth]{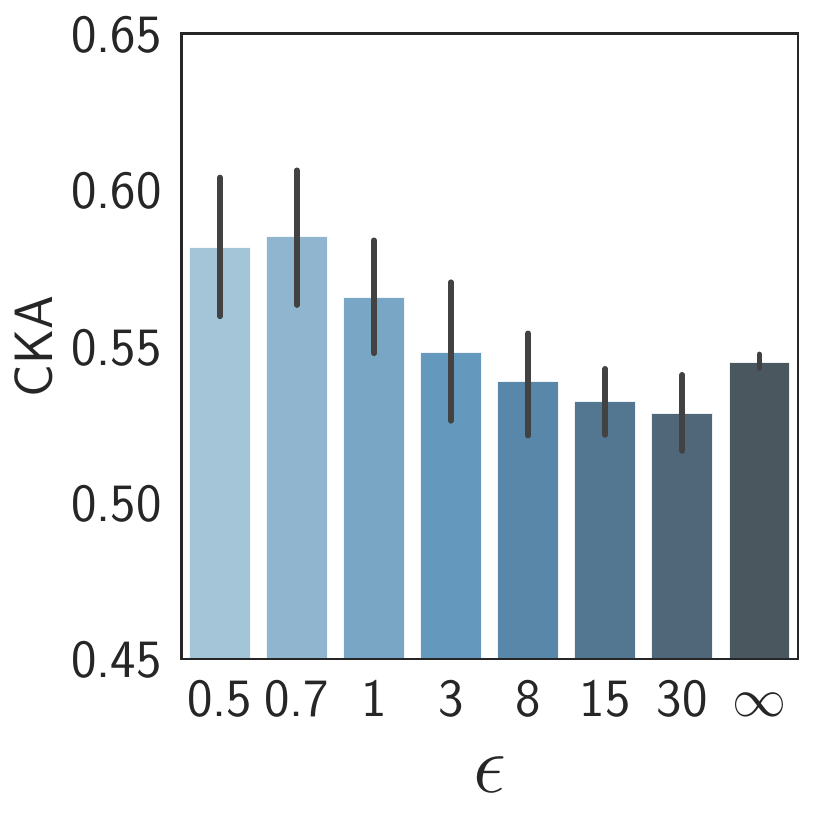}
        \caption{POS CKA}
        \label{fig:pos_cka_mbert}
    \end{subfigure}
    \begin{subfigure}[b]{0.196\textwidth}
        \centering
        \includegraphics[width=\textwidth]{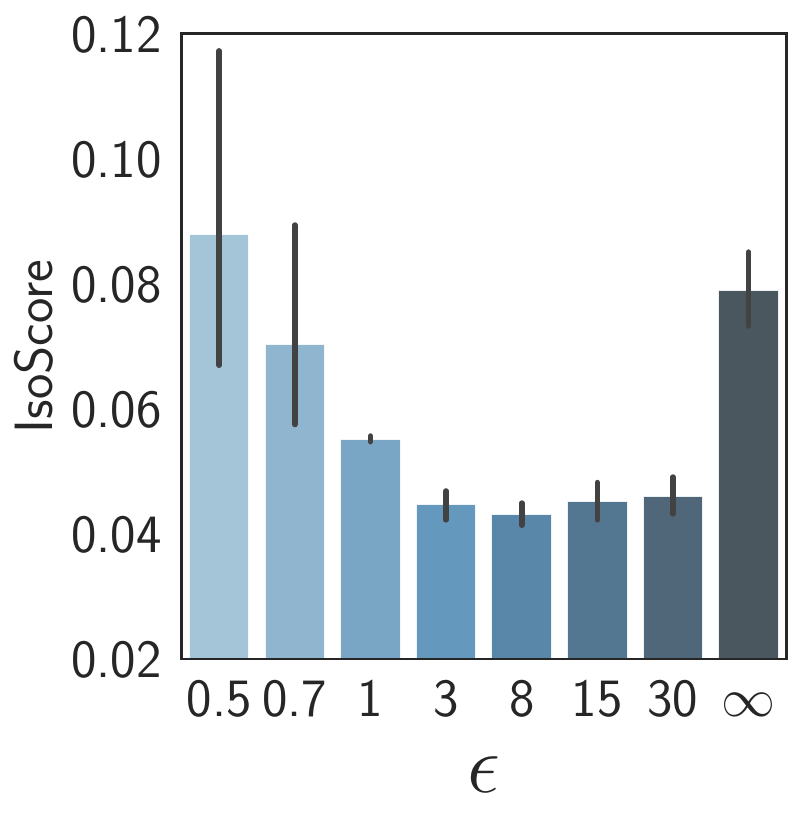}
        \caption{POS IsoScore}
        \label{fig:pos_iso_mbert}
    \end{subfigure}
    \begin{subfigure}[b]{0.1992\textwidth}
        \centering
        \includegraphics[width=\textwidth]{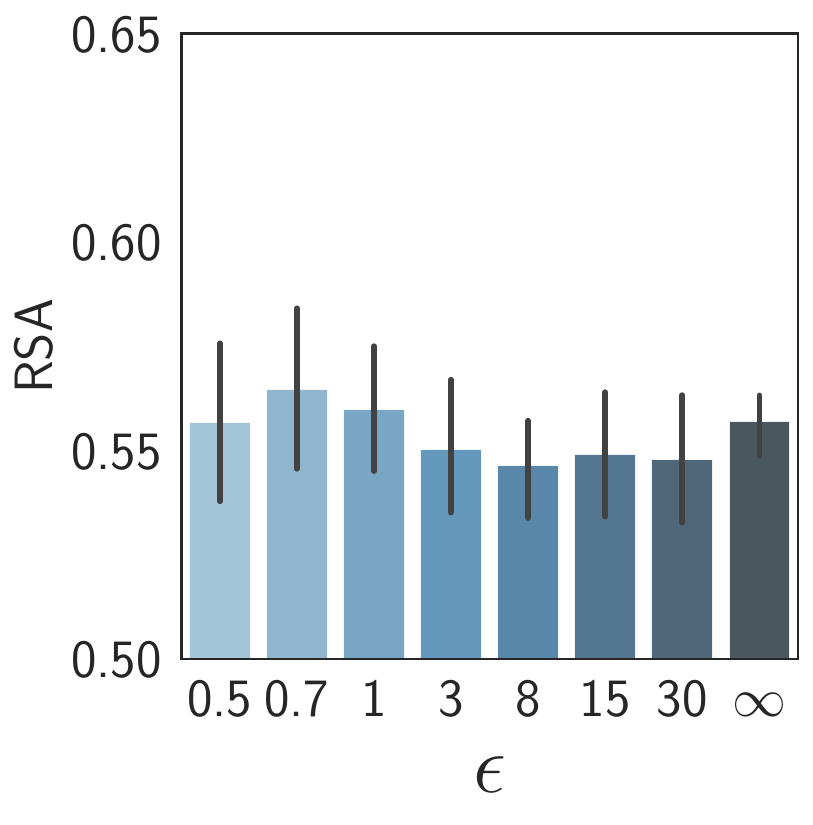}
        \caption{POS RSA}
        \label{fig:pos_rsa_mbert}
    \end{subfigure}
    \begin{subfigure}[b]{0.194\textwidth}
        \centering
        \includegraphics[width=\textwidth]{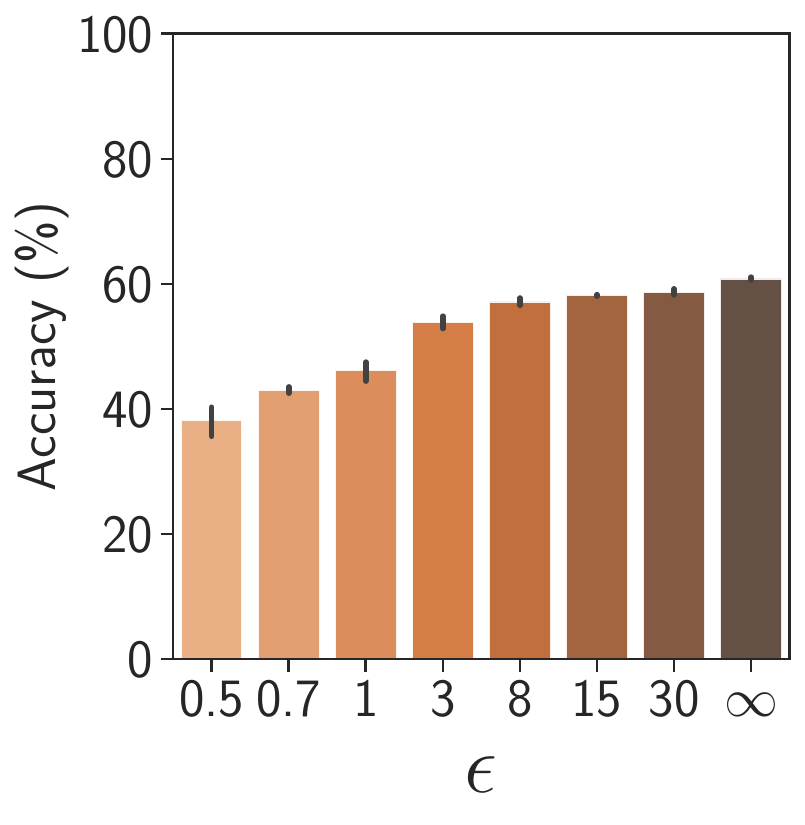}
        \caption{XNLI Performance}
        \label{fig:xnli_acc_mbert}
    \end{subfigure}
    \begin{subfigure}[b]{0.19\textwidth}
        \centering
        \includegraphics[width=\textwidth]{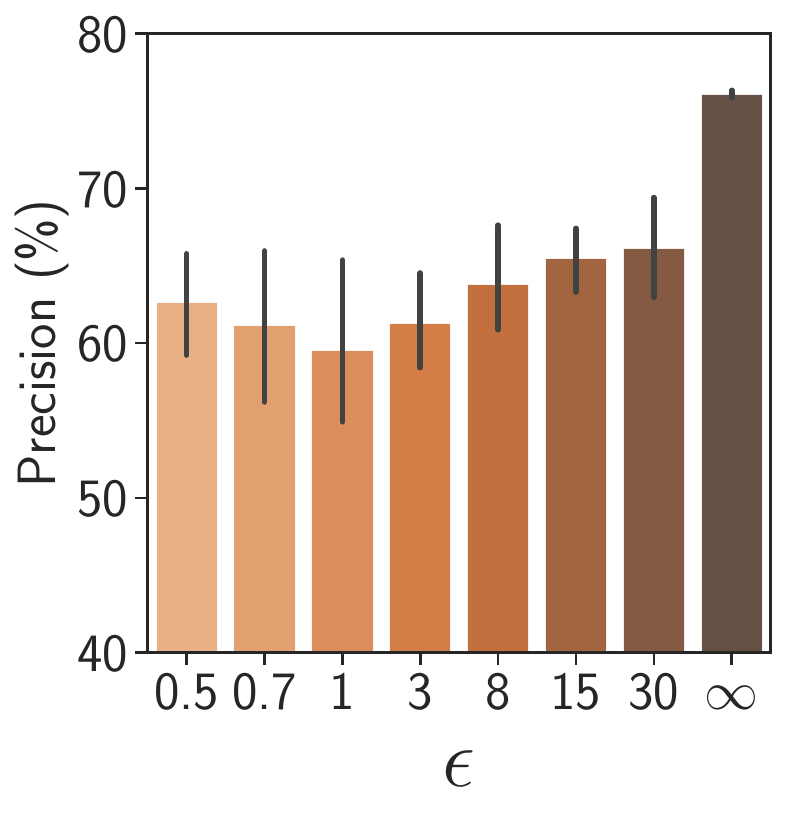}
        \caption{XNLI Retrieval}
        \label{fig:xnli_compression_mbert}
    \end{subfigure}
    \begin{subfigure}[b]{0.198\textwidth}
        \centering
        \includegraphics[width=\textwidth]{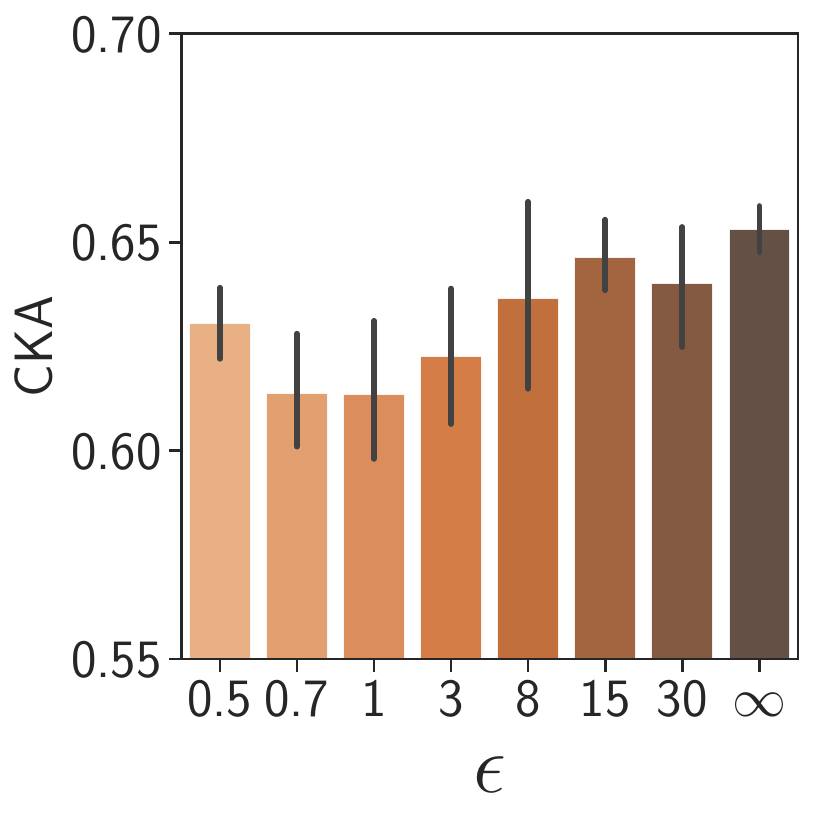}
        \caption{XNLI CKA}
        \label{fig:xnli_cka_mbert}
    \end{subfigure}
    \begin{subfigure}[b]{0.194\textwidth}
        \centering
        \includegraphics[width=\textwidth]{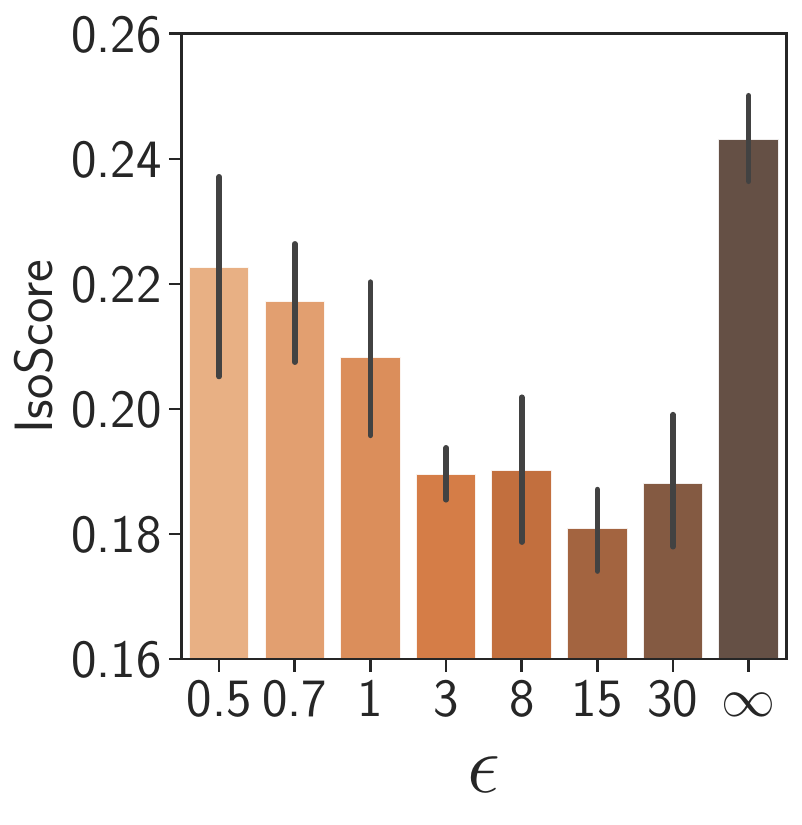}
        \caption{XNLI IsoScore}
        \label{fig:xnli_iso_mbert}
    \end{subfigure}
    \begin{subfigure}[b]{0.199\textwidth}
        \centering
        \includegraphics[width=\textwidth]{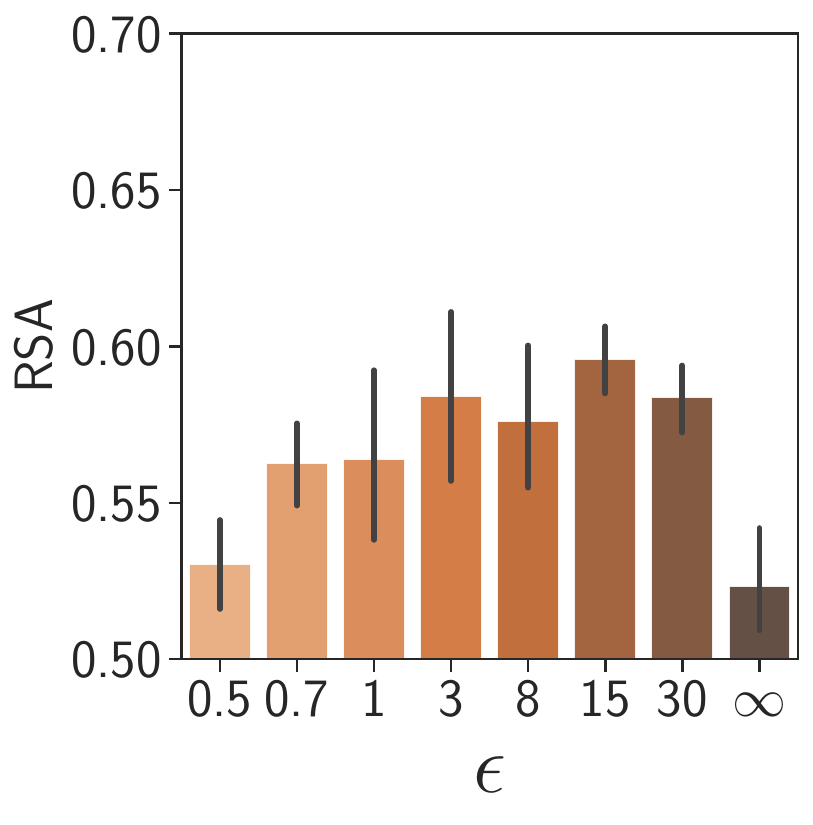}
        \caption{XNLI RSA}
        \label{fig:xnli_rsa_mbert}
    \end{subfigure}
    \caption{Aggregated \textbf{mBERT} results, analogous to Figure~\ref{fig:section3}.}
    \label{fig:mbert_results_1}
\end{figure*}

\vfill

\begin{figure*}[ht!]
    \centering
    \begin{subfigure}[b]{0.23\textwidth}
        \centering
        \includegraphics[width=\textwidth]{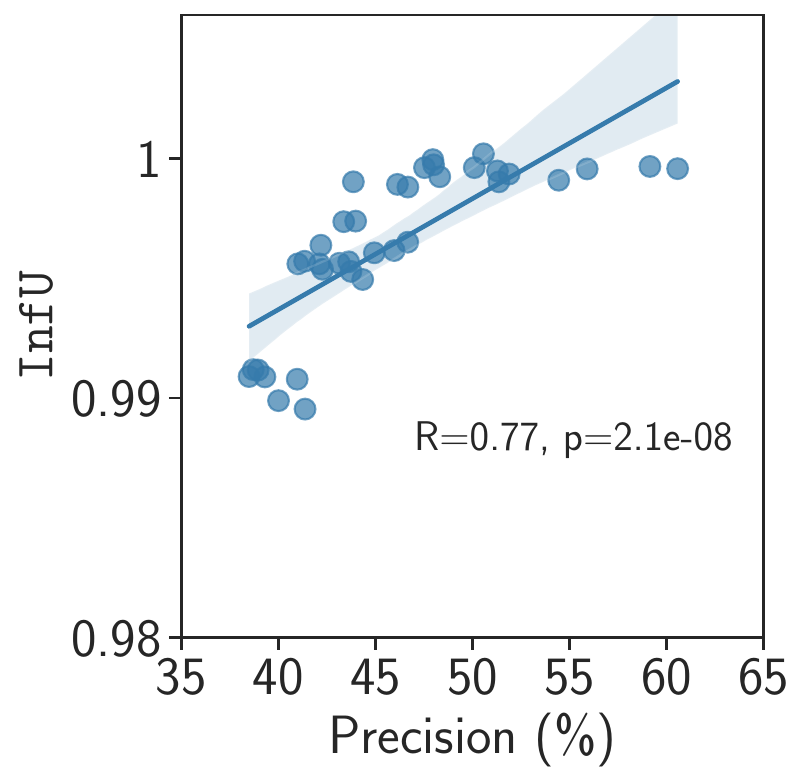}
        \caption{POS $\mathrm{InfU}$ -- Retrieval}
        \label{fig:pos_infu_retrieval_mbert}
    \end{subfigure}
    \begin{subfigure}[b]{0.2397\textwidth}
        \centering
        \includegraphics[width=\textwidth]{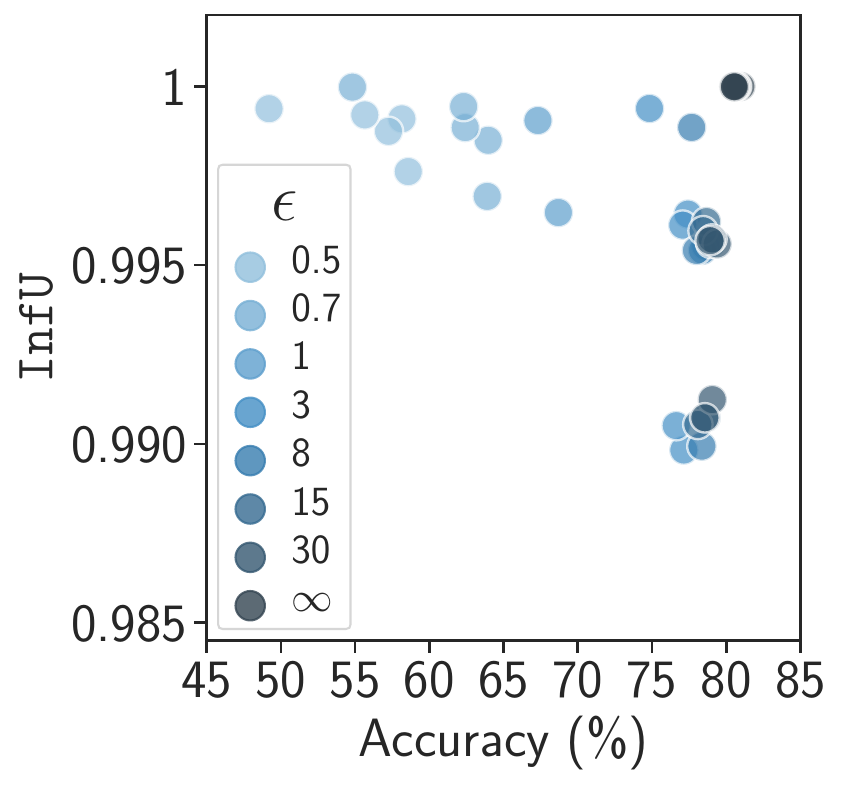}
        \caption{POS $\mathrm{InfU}$ -- Perf.}
        \label{fig:pos_infu_acc_mbert}
    \end{subfigure}
    \begin{subfigure}[b]{0.232\textwidth}
        \centering
        \includegraphics[width=\textwidth]{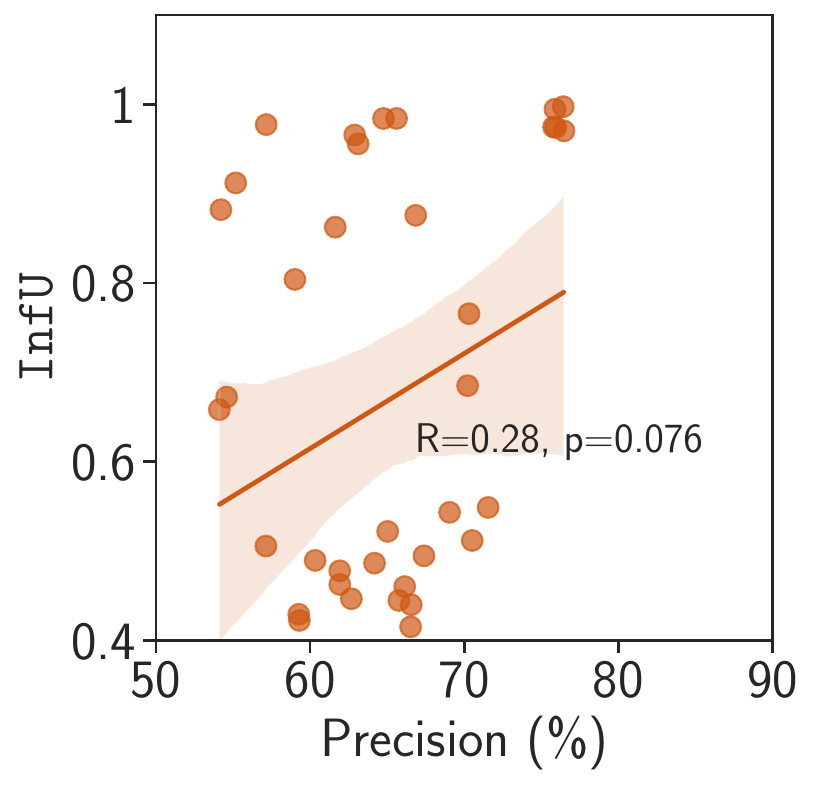}
        \caption{XNLI $\mathrm{InfU}$ -- Retrieval}
        \label{fig:xnli_infu_retrieval_mbert}
    \end{subfigure}
    \begin{subfigure}[b]{0.225\textwidth}
        \centering
        \includegraphics[width=\textwidth]{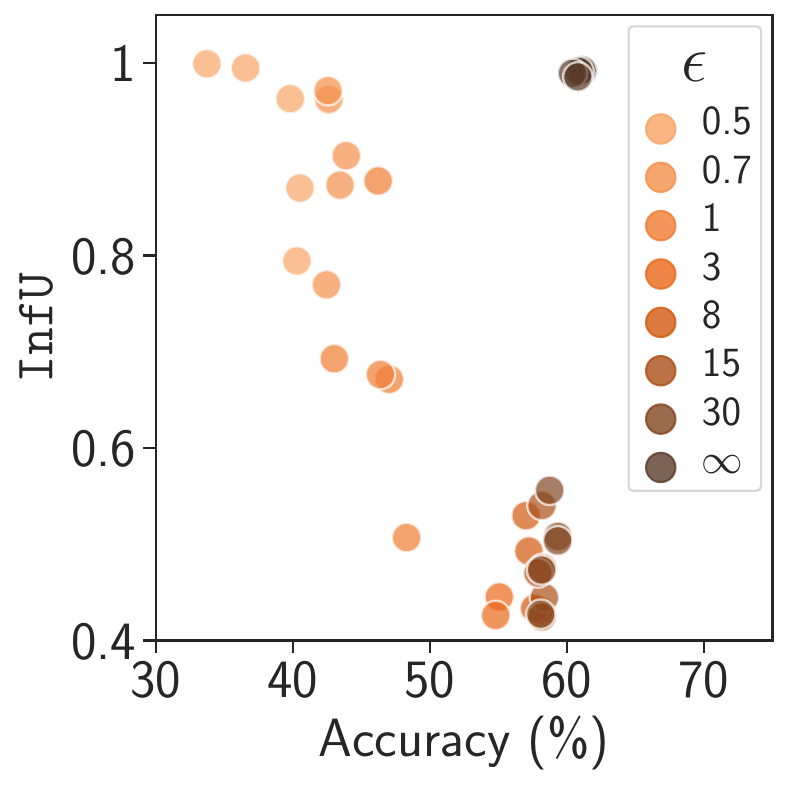}
        \caption{XNLI $\mathrm{InfU}$ -- Perf.}
        \label{fig:xnli_infu_acc_mbert}
    \end{subfigure}
    \caption{Aggregated \textbf{mBERT} results, analogous to Figure~\ref{fig:interpretability_plots}.}
    \label{fig:interpretability_plots_mbert}
\end{figure*}

\newpage

\section{Detailed Results for Experiments in \S\ref{sec:question1}}
\label{sec:details_q1}

Figure~\ref{fig:retrieval_steps} shows the development of the mean sentence retrieval precision at layer 8 for POS and XNLI over the course of fine-tuning with different privacy budgets.

\noindent
We further present non-aggregated results for
\begin{itemize}
    \itemsep-0.3em 
    \item POS performance in Table~\ref{tab:pos_full_acc}
    \item XNLI performance in Table~\ref{tab:xnli_full_acc}
    \item Sentence retrieval for POS in Figures~\ref{fig:pos_retrieval_tedwm} and \ref{fig:pos_retrieval_tatoeba}
    \item Sentence retrieval for XNLI in Figure~\ref{fig:xnli_retrieval_full}
    \item CKA for POS in Figures~\ref{fig:pos_cka_tedwm} and \ref{fig:pos_cka_tatoeba}
    \item CKA for XNLI in Figure~\ref{fig:xnli_cka_full}
    \item IsoScore for POS in Table~\ref{tab:pos_all_isoscores}
    \item IsoScore for XNLI in Table~\ref{tab:xnli_all_isoscores}
    \item RSA for POS in Figures~\ref{fig:pos_rsa_tedwm} and \ref{fig:pos_rsa_tatoeba}
    \item RSA for XNLI in Figure~\ref{fig:xnli_rsa_full}.
\end{itemize}

\begin{table*}[ht!]
\centering
\caption{Overview of languages used in our experiments. Tokens (in millions) and size (in Gibibytes) refer to the respective monolingual corpora in XLM-R's pretraining corpus. Numbers taken from \citet{conneau-etal-2020-unsupervised}. *: includes romanized variants also used in pretraining.}
\label{tab:languages}
\resizebox{0.6\textwidth}{!}{%
\begin{tabular}{@{}lcllll@{}}
\toprule
\textbf{Language} & \textbf{ISO} & \textbf{Family} & \textbf{Script} & \textbf{Tokens} (M) & \textbf{Size} (GiB) \\ \midrule
Arabic     & \textsc{ar} & Afro-Asiatic   & Arabic     & 2869  & 28.0  \\
Bulgarian  & \textsc{bg} & Indo-European  & Cyrillic   & 5487  & 57.5  \\
Chinese    & \textsc{zh} & Sino-Tibetan   & Chinese    & 435   & 63.5  \\
French     & \textsc{fr} & Indo-European  & Latin      & 9780  & 56.8  \\
German     & \textsc{de} & Indo-European  & Latin      & 10297 & 66.6  \\
Greek      & \textsc{el} & Indo-European  & Greek      & 4285  & 46.9  \\
Hindi      & \textsc{hi} & Indo-European  & Devanagari & 1803* & 20.7* \\
Indonesian & \textsc{id} & Austronesian   & Latin      & 22704 & 148.3 \\
Italian    & \textsc{it} & Indo-European  & Latin      & 4983  & 30.2  \\
Japanese   & \textsc{ja} & Japonic        & Japanese   & 530   & 69.3  \\
Kiswahili  & \textsc{sw} & Niger-Congo    & Latin      & 275   & 1.6   \\
Korean     & \textsc{ko} & Koreanic       & Korean     & 5644  & 54.2  \\
Portuguese & \textsc{pt} & Indo-European  & Latin      & 8405  & 49.1  \\
Russian    & \textsc{ru} & Indo-European  & Cyrillic   & 23408 & 278.0 \\
Thai       & \textsc{th} & Kra-Dai        & Thai       & 1834  & 71.7  \\
Turkish    & \textsc{tr} & Turkic         & Latin      & 2736  & 20.9  \\
Urdu       & \textsc{ur} & Indo-European  & Arabic     & 815*  & 6.2*  \\
Vietnamese & \textsc{vi} & Austro-Asiatic & Latin      & 24757 & 137.3 \\ \bottomrule
\end{tabular}%
}
\end{table*}

\begin{table*}[!hb]
\centering
\caption{Links and references to the datasets we used in our experiments. License information are also available via these links. We ensure that we comply with respective license conditions and only use the data within their intended use policy where applicable.}
\label{tab:dataset_links}
\resizebox{\textwidth}{!}{%
\begin{tabular}{@{}lll@{}}
\toprule
\textbf{Dataset} & \textbf{Download Link}                                                                      & \textbf{Reference}                                        \\ \midrule
UD v2.8 (POS) &
  \href{https://lindat.mff.cuni.cz/repository/xmlui/handle/11234/1-3683}{https://lindat.mff.cuni.cz/repository/xmlui/handle/11234/1-3683} &
  \citep{nivre-etal-2020-universal, zeman-etal-2021-ud} \\
XNLI             & \href{https://huggingface.co/datasets/xnli}{https://huggingface.co/datasets/xnli}           & \citep{conneau-etal-2018-xnli, lhoest-etal-2021-datasets} \\
TED 2020 &
  \href{https://github.com/UKPLab/sentence-transformers/blob/master/docs/datasets/TED2020.md}{https://github.com/UKPLab/sentence-transformers/blob/master/docs/datasets/TED2020.md} &
  \citep{reimers-gurevych-2020-making} \\
WikiMatrix &
  \href{https://github.com/facebookresearch/LASER/tree/main/tasks/WikiMatrix}{https://github.com/facebookresearch/LASER/tree/main/tasks/WikiMatrix} &
  \citep{schwenk-etal-2021-wikimatrix} \\
Tatoeba          & \href{https://github.com/LBeaudoux/tatoebatools}{https://github.com/LBeaudoux/tatoebatools} &                                                          \\ \bottomrule
\end{tabular}%
}
\end{table*}

\begin{table*}[!ht]
\centering
\caption{Overview of the UD v2.8 \citep{nivre-etal-2020-universal, zeman-etal-2021-ud} treebanks (test splits only) that we use as test sets in our POS tagging experiments (\S\ref{sec:exp_setup},\ref{sec:question1}) including their respective sizes (number of sentences).}
\label{tab:treebank_sizes}
\resizebox{0.4\textwidth}{!}{%
\begin{tabular}{@{}clc@{}}
\toprule
\textbf{Language} & \textbf{Treebank} & \textbf{\# Sentences} \\\midrule
\textsc{ar}       & Arabic-PADT       & 680                  \\
\textsc{de}       & German-GSD        & 977                  \\
\textsc{es}       & Spanish-GSD       & 426                  \\
\textsc{hi}       & Hindi-HDTB        & 1684                 \\
\textsc{id}       & Indonesian-GSD    & 557                  \\
\textsc{ko}       & Korean-Kaist      & 2287                 \\
\textsc{ru}       & Russian-SynTagRus & 6491                 \\\bottomrule
\end{tabular}%
}
\end{table*}

\vfill 

\begin{figure*}[ht!]
    \centering
    \begin{subfigure}[b]{0.48\textwidth}
        \centering
        \includegraphics[width=\textwidth]{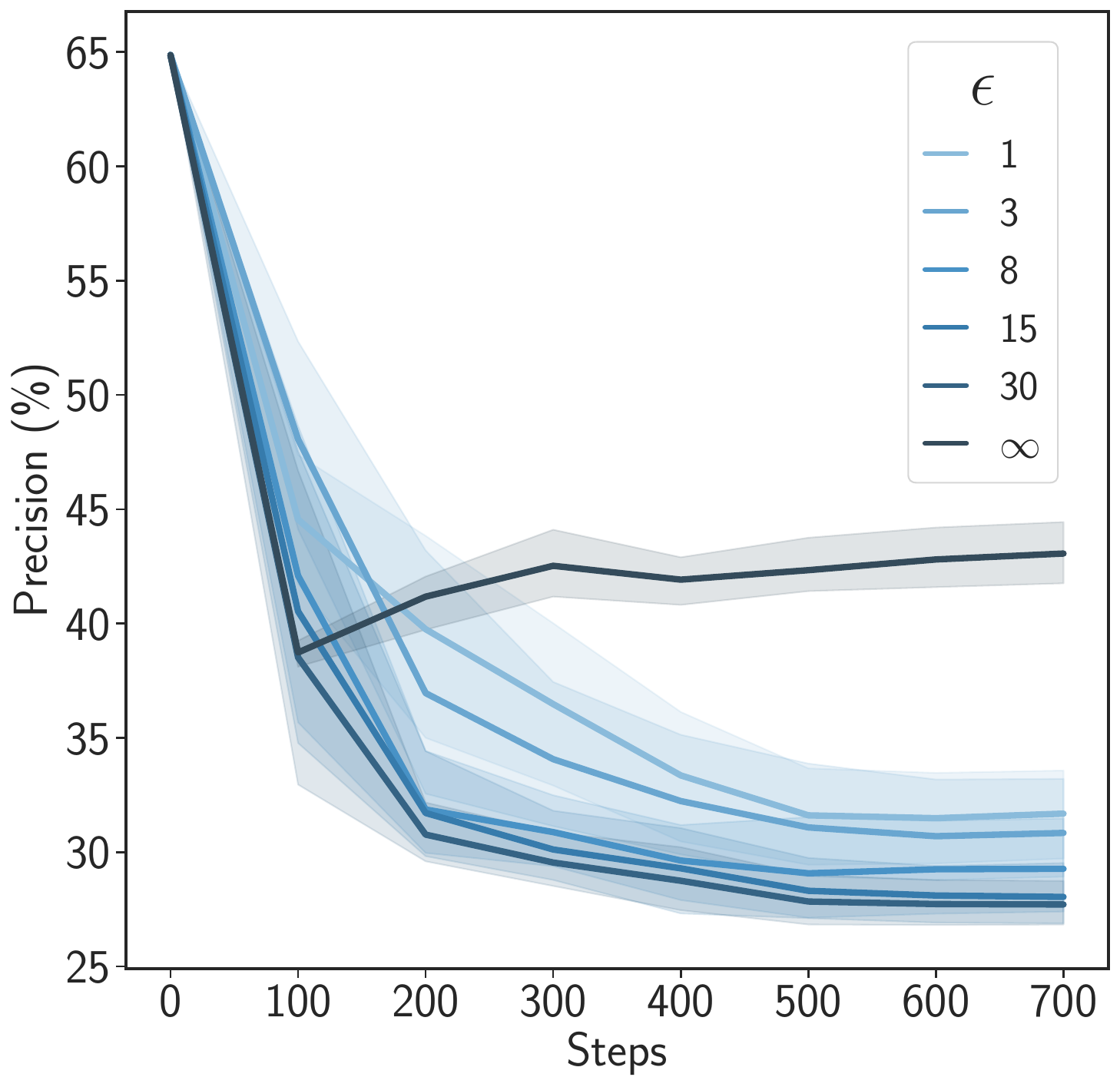}
        \caption{POS}
        \label{fig:pos_retrieval_steps}
    \end{subfigure}
    \begin{subfigure}[b]{0.48\textwidth}
        \centering
        \includegraphics[width=\textwidth]{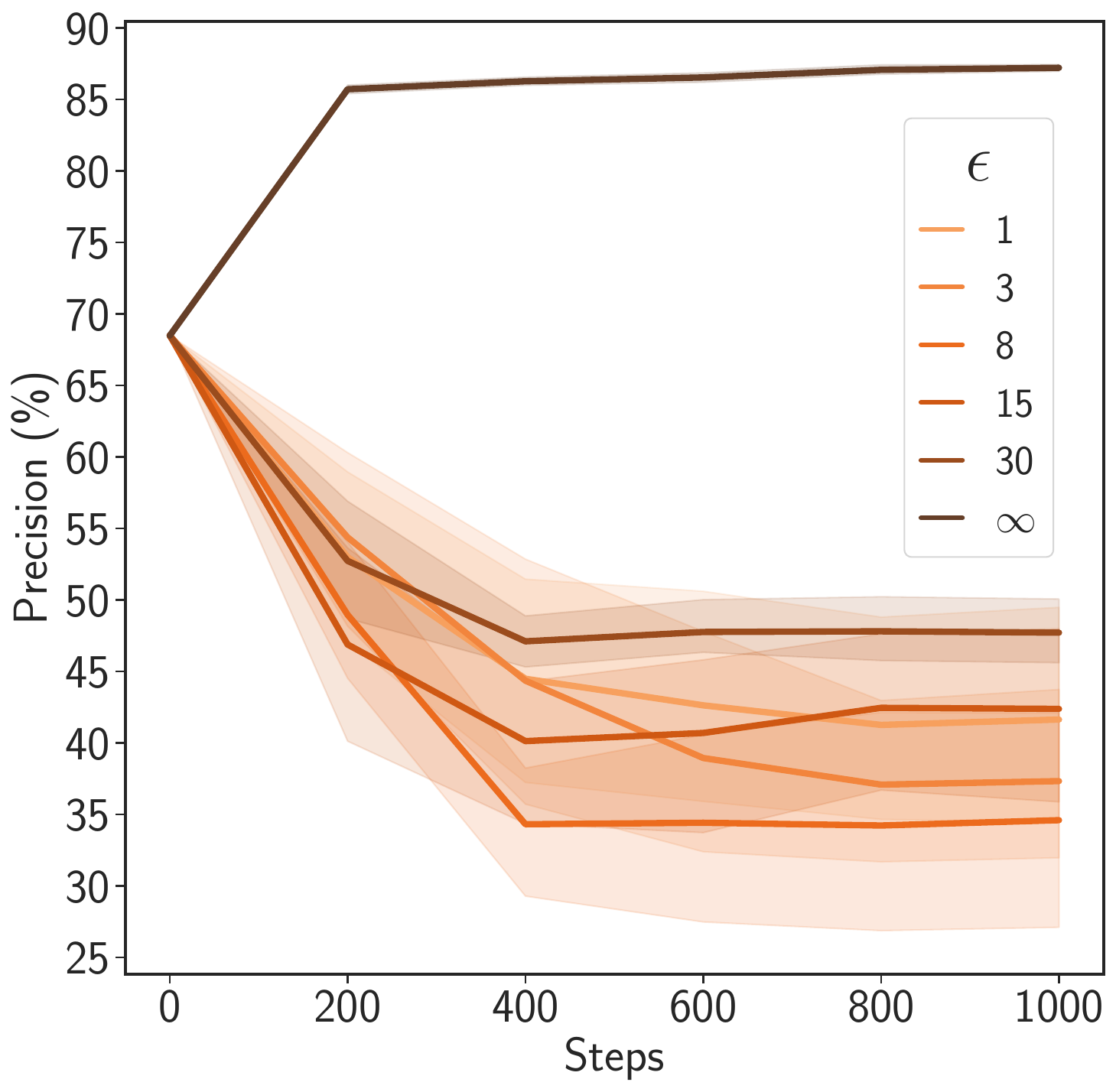}
        \caption{XNLI}
        \label{fig:xnli_retrieval_steps}
    \end{subfigure}
    \caption{Mean sentence retrieval precision for our TED 2020  splits (different languages/data for POS and XNLI) at layer 8 over the course of fine-tuning with different privacy budgets ($\varepsilon$). $\varepsilon = \infty$ denotes non-private models. Error bands show variation around the mean over 5 random seeds. At $\mathrm{Steps}=0$, all models are equivalent to the pretrained XLM-R Base. We see that the non-private models can retain (and for XNLI even improve) their multilingual compression much better than the private models and have less variation.}
    \label{fig:retrieval_steps}
\end{figure*}

\begin{table*}[htp]
\centering
\caption{\textbf{POS} Performance (validation / test accuracy) when fine-tuning XLM-R Base with different privacy budgets ($\varepsilon$). We show results averaged over 5 random seeds each. $\varepsilon = \infty$ denotes non-private models. \textsc{avg} is the average over the 7 languages. See \S\ref{sec:exp_setup} for our experimental setup. We see that performance increases with decreased privacy across all languages.}
\label{tab:pos_full_acc}
\resizebox{\textwidth}{!}{%
\begin{tabular}{@{}ccccccccc@{}}
\toprule
$\varepsilon$ & \textsc{ar} & \textsc{de} & \textsc{es} & \textsc{hi} & \textsc{id} & \textsc{ko} & \textsc{ru} & \textsc{avg} \\ \midrule
1             & 68.3 / 64.6 & 75.5 / 75.1 & 79.8 / 79.0 & 65.0 / 63.3 & 73.8 / 71.9 & 66.1 / 54.2 & 74.8 / 74.0 & 71.9 / 68.9  \\
3             & 79.1 / 76.6 & 86.6 / 86.8 & 90.3 / 89.3 & 74.4 / 70.9 & 82.6 / 79.4 & 71.1 / 59.4 & 86.1 / 86.3 & 81.4 / 78.4  \\
8             & 81.0 / 77.6 & 88.4 / 88.3 & 91.6 / 90.2 & 78.2 / 75.6 & 84.2 / 81.2 & 70.8 / 60.9 & 87.1 / 87.4 & 83.0 / 80.2  \\
15            & 81.3 / 78.4 & 88.8 / 89.0 & 92.4 / 90.9 & 77.0 / 73.2 & 83.9 / 80.7 & 71.9 / 61.8 & 87.7 / 87.8 & 83.3 / 80.3  \\
30            & 81.8 / 78.7 & 89.4 / 89.6 & 92.9 / 91.5 & 77.6 / 74.0 & 84.3 / 81.1 & 72.3 / 62.2 & 88.2 / 88.4 & 83.8 / 80.8  \\
$\infty$ &
  \textbf{83.8} / \textbf{79.7} &
  \textbf{91.5} / \textbf{91.2} &
  \textbf{95.0} / \textbf{93.2} &
  \textbf{82.8} / \textbf{80.2} &
  \textbf{86.2} / \textbf{81.3} &
  \textbf{74.2} / \textbf{62.9} &
  \textbf{89.9} / \textbf{90.2} &
  \textbf{86.2} / \textbf{82.7} \\ \bottomrule
\end{tabular}%
}

\vspace*{\floatsep}
\vspace*{\floatsep}
\vspace*{\floatsep}
\vspace*{\floatsep}
\vspace*{\floatsep}
\vspace*{\floatsep}
    
\centering
\caption{\textbf{XNLI} Performance (validation / test accuracy) when fine-tuning XLM-R Base with different privacy budgets ($\varepsilon$). We show results averaged over 5 random seeds each. $\varepsilon = \infty$ denotes non-private models. \textsc{avg} is the average over the 7 languages. See \S\ref{sec:exp_setup} for our experimental setup. We see that performance increases with decreased privacy across all languages. Here, we also particularly observe that the gap between validation and test performance is substantially lower for private models, which shows the strong regularization effect of training with differential privacy.}
\label{tab:xnli_full_acc}
\resizebox{\textwidth}{!}{%
\begin{tabular}{@{}cllllllll@{}}
\toprule
$\varepsilon$ &
  \multicolumn{1}{c}{\textsc{ar}} &
  \multicolumn{1}{c}{\textsc{de}} &
  \multicolumn{1}{c}{\textsc{el}} &
  \multicolumn{1}{c}{\textsc{ru}} &
  \multicolumn{1}{c}{\textsc{sw}} &
  \multicolumn{1}{c}{\textsc{th}} &
  \multicolumn{1}{c}{\textsc{ur}} &
  \multicolumn{1}{c}{\textsc{avg}} \\ \midrule
1             & 37.3 / 37.4 & 36.8 / 37.0 & 36.6 / 36.5 & 36.3 / 36.2 & 34.3 / 34.5 & 35.6 / 35.7 & 35.6 / 35.6 & 36.1 / 36.1  \\
3             & 49.6 / 50.3 & 49.3 / 51.0 & 50.8 / 51.5 & 49.7 / 50.2 & 45.9 / 47.2 & 48.8 / 49.5 & 47.6 / 48.2 & 48.8 / 49.7  \\
8             & 55.9 / 56.4 & 56.8 / 58.5 & 58.2 / 58.1 & 56.3 / 57.1 & 52.0 / 53.2 & 55.6 / 55.7 & 53.3 / 53.7 & 55.5 / 56.1  \\
15            & 59.1 / 58.3 & 60.4 / 60.8 & 61.5 / 60.9 & 59.7 / 59.5 & 54.4 / 54.8 & 58.9 / 58.2 & 56.4 / 56.1 & 58.6 / 58.4  \\
30            & 61.6 / 60.8 & 63.6 / 63.1 & 64.8 / 62.0 & 62.0 / 61.1 & 56.5 / 57.3 & 61.2 / 60.2 & 58.6 / 57.8 & 61.2 / 60.3  \\
$\infty$ &
  \textbf{90.9} / \textbf{67.8} &
  \textbf{96.2} / \textbf{70.5} &
  \textbf{95.5} / \textbf{70.1} &
  \textbf{93.4} / \textbf{69.7} &
  \textbf{79.0} / \textbf{62.5} &
  \textbf{91.6} / \textbf{68.5} &
  \textbf{86.8} / \textbf{65.4} &
  \textbf{90.5} / \textbf{67.8} \\ \bottomrule
\end{tabular}%
}
\end{table*}

\begin{table*}[btp!]

\end{table*}

\label{sec:appendix}

\begin{figure*}[ht!]
    \centering
    \begin{subfigure}[b]{0.21\textwidth}
        \centering
        \includegraphics[width=\textwidth]{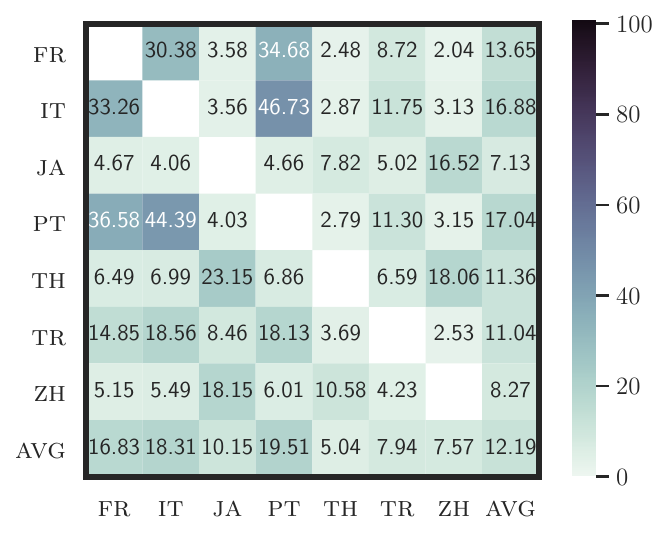}
        \caption{TED, $\varepsilon = 1$, $l = 0$}
        \label{fig:pos_retrieval_ted2020_1_lay0}
    \end{subfigure}
    \begin{subfigure}[b]{0.21\textwidth}
        \centering
        \includegraphics[width=\textwidth]{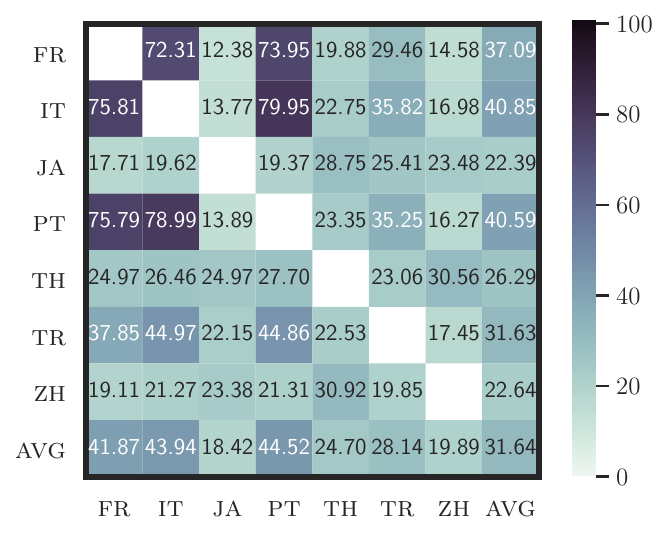}
        \caption{TED, $\varepsilon = 1$, $l = 8$}
        \label{fig:pos_retrieval_ted2020_1_lay8}
    \end{subfigure}
        \begin{subfigure}[b]{0.21\textwidth}
        \centering
        \includegraphics[width=\textwidth]{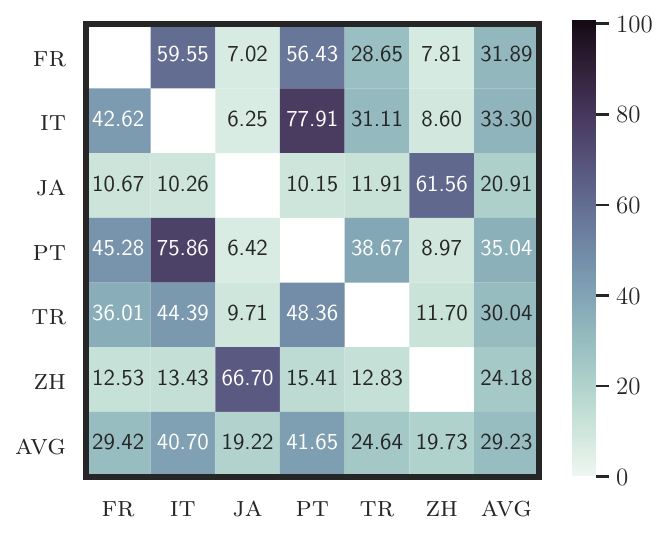}
        \caption{WM, $\varepsilon = 1$, $l = 0$}
        \label{fig:pos_retrieval_wikimatrix_1_lay0}
    \end{subfigure}
    \begin{subfigure}[b]{0.21\textwidth}
        \centering
        \includegraphics[width=\textwidth]{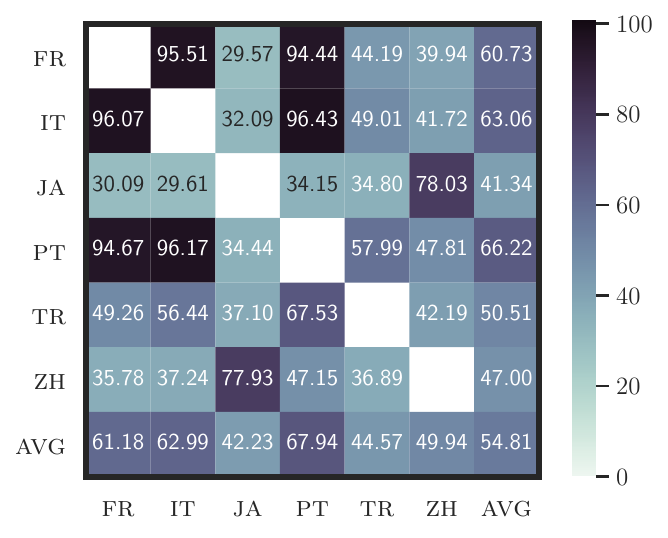}
        \caption{WM, $\varepsilon = 1$, $l = 8$}
        \label{fig:pos_retrieval_wikimatrix_1_lay8}
    \end{subfigure}
    \begin{subfigure}[b]{0.21\textwidth}
        \centering
        \includegraphics[width=\textwidth]{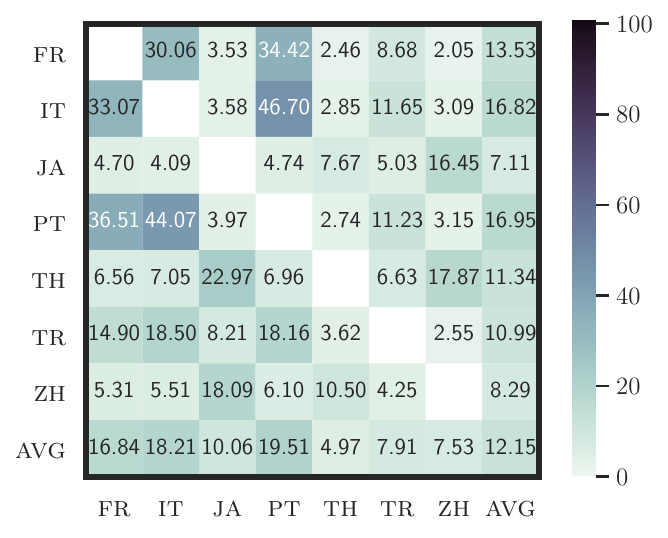}
        \caption{TED, $\varepsilon = 3$, $l = 0$}
        \label{fig:pos_retrieval_ted2020_3_lay0}
    \end{subfigure}
    \begin{subfigure}[b]{0.21\textwidth}
        \centering
        \includegraphics[width=\textwidth]{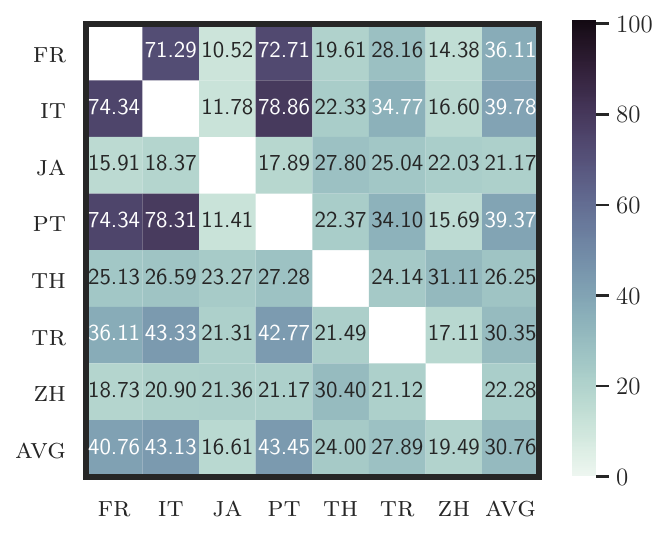}
        \caption{TED, $\varepsilon = 3$, $l = 8$}
        \label{fig:pos_retrieval_ted2020_3_lay8}
    \end{subfigure}
        \begin{subfigure}[b]{0.21\textwidth}
        \centering
        \includegraphics[width=\textwidth]{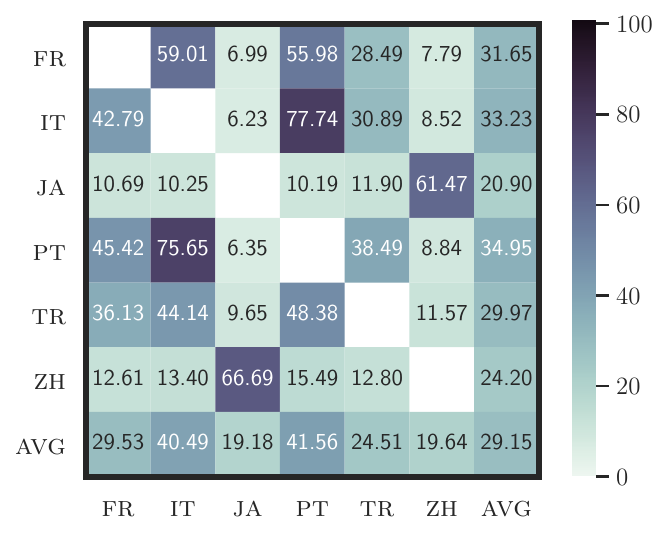}
        \caption{WM, $\varepsilon = 3$, $l = 0$}
        \label{fig:pos_retrieval_wikimatrix_3_lay0}
    \end{subfigure}
    \begin{subfigure}[b]{0.21\textwidth}
        \centering
        \includegraphics[width=\textwidth]{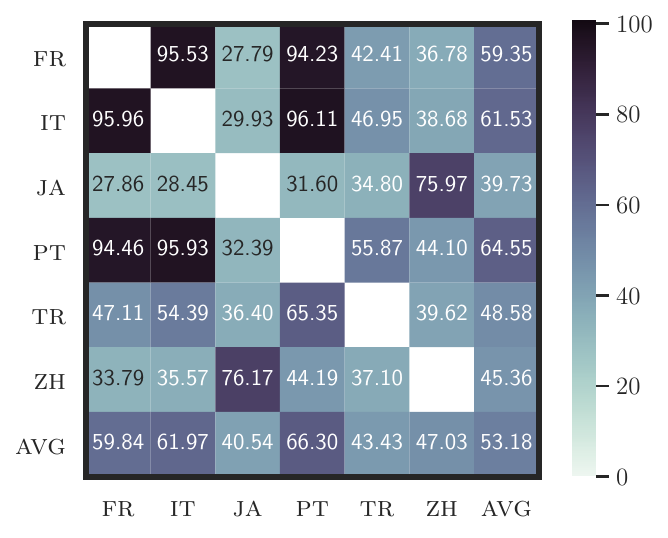}
        \caption{WM, $\varepsilon = 3$, $l = 8$}
        \label{fig:pos_retrieval_wikimatrix_3_lay8}
    \end{subfigure}
    \begin{subfigure}[b]{0.21\textwidth}
        \centering
        \includegraphics[width=\textwidth]{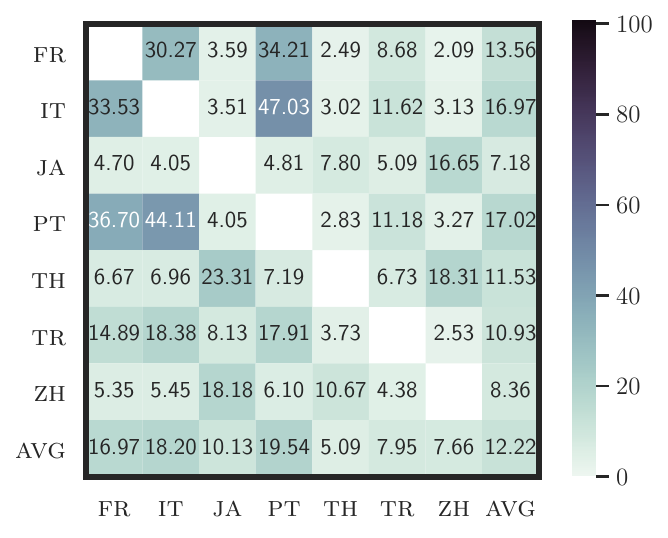}
        \caption{TED, $\varepsilon = 8$, $l = 0$}
        \label{fig:pos_retrieval_ted2020_8_lay0}
    \end{subfigure}
    \begin{subfigure}[b]{0.21\textwidth}
        \centering
        \includegraphics[width=\textwidth]{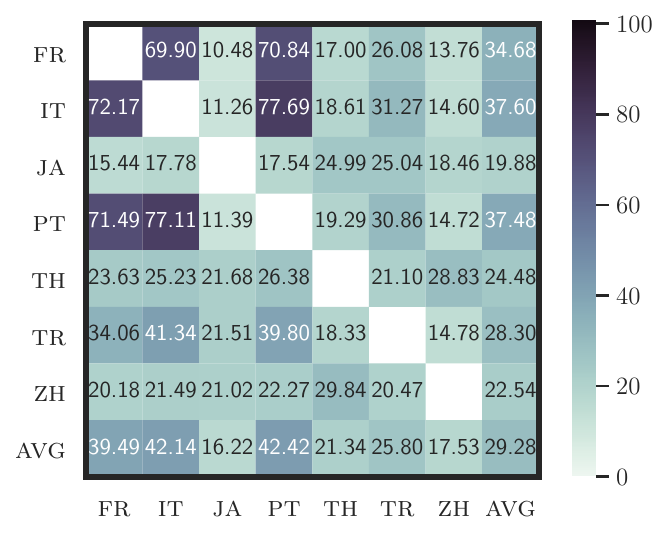}
        \caption{TED, $\varepsilon = 8$, $l = 8$}
        \label{fig:pos_retrieval_ted2020_8_lay8}
    \end{subfigure}
        \begin{subfigure}[b]{0.21\textwidth}
        \centering
        \includegraphics[width=\textwidth]{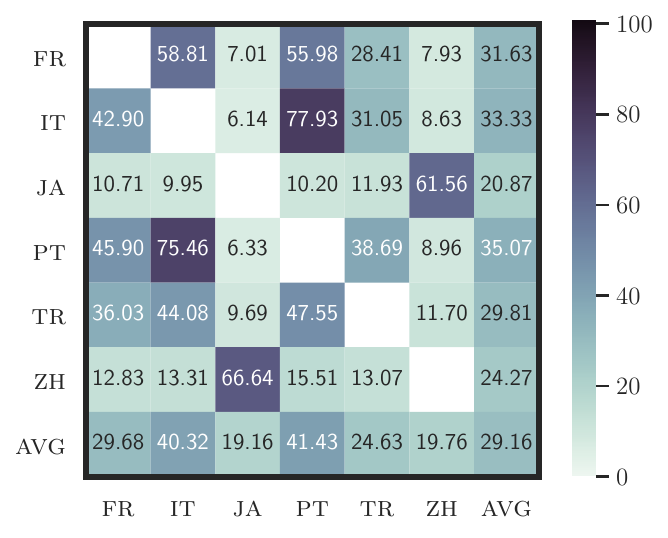}
        \caption{WM, $\varepsilon = 8$, $l = 0$}
        \label{fig:pos_retrieval_wikimatrix_8_lay0}
    \end{subfigure}
    \begin{subfigure}[b]{0.21\textwidth}
        \centering
        \includegraphics[width=\textwidth]{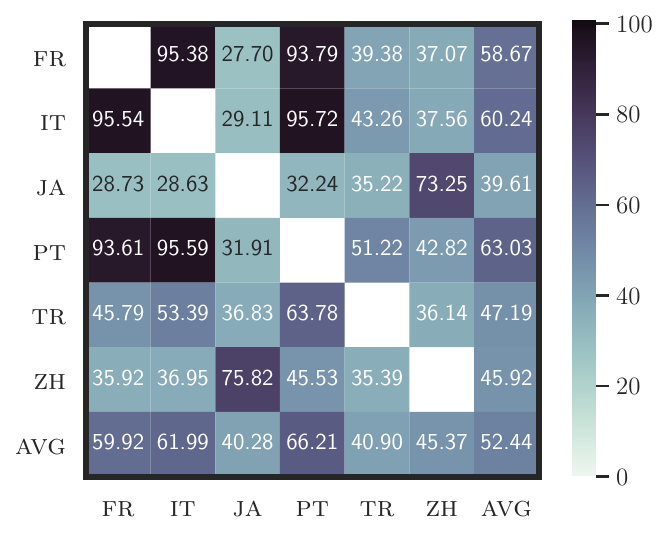}
        \caption{WM, $\varepsilon = 8$, $l = 8$}
        \label{fig:pos_retrieval_wikimatrix_8_lay8}
    \end{subfigure}
    \begin{subfigure}[b]{0.21\textwidth}
        \centering
        \includegraphics[width=\textwidth]{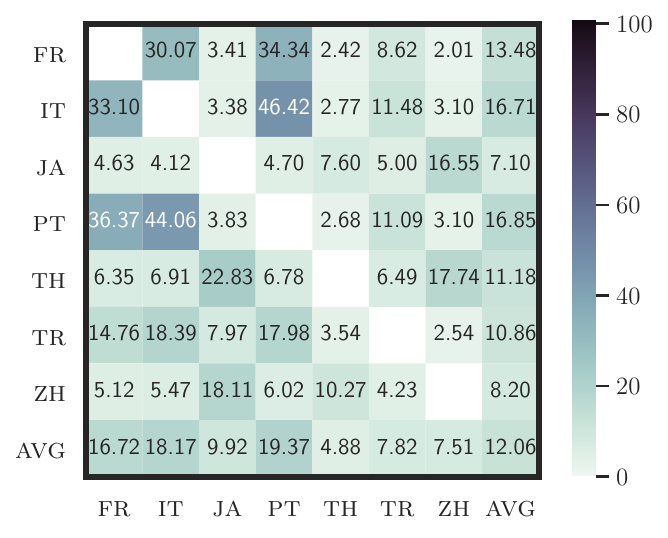}
        \caption{TED, $\varepsilon = 15$, $l = 0$}
        \label{fig:pos_retrieval_ted2020_15_lay0}
    \end{subfigure}
    \begin{subfigure}[b]{0.21\textwidth}
        \centering
        \includegraphics[width=\textwidth]{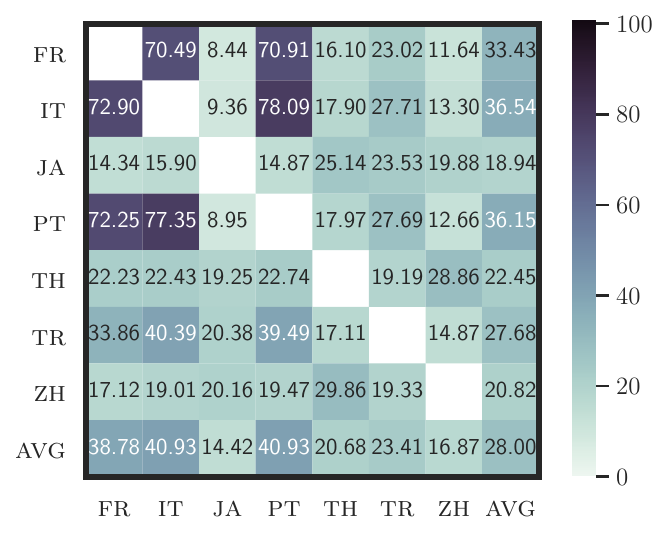}
        \caption{TED, $\varepsilon = 15$, $l = 8$}
        \label{fig:pos_retrieval_ted2020_15_lay8}
    \end{subfigure}
        \begin{subfigure}[b]{0.21\textwidth}
        \centering
        \includegraphics[width=\textwidth]{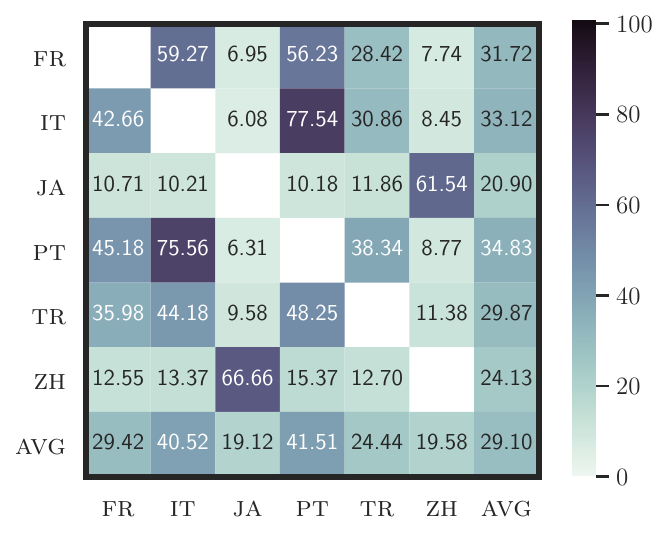}
        \caption{WM, $\varepsilon = 15$, $l = 0$}
        \label{fig:pos_retrieval_wikimatrix_15_lay0}
    \end{subfigure}
    \begin{subfigure}[b]{0.21\textwidth}
        \centering
        \includegraphics[width=\textwidth]{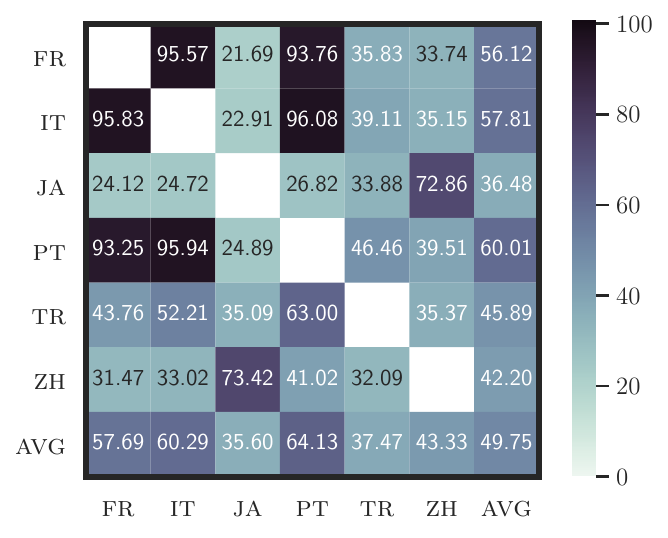}
        \caption{WM, $\varepsilon = 15$, $l = 8$}
        \label{fig:pos_retrieval_wikimatrix_15_lay8}
    \end{subfigure}
    \begin{subfigure}[b]{0.21\textwidth}
        \centering
        \includegraphics[width=\textwidth]{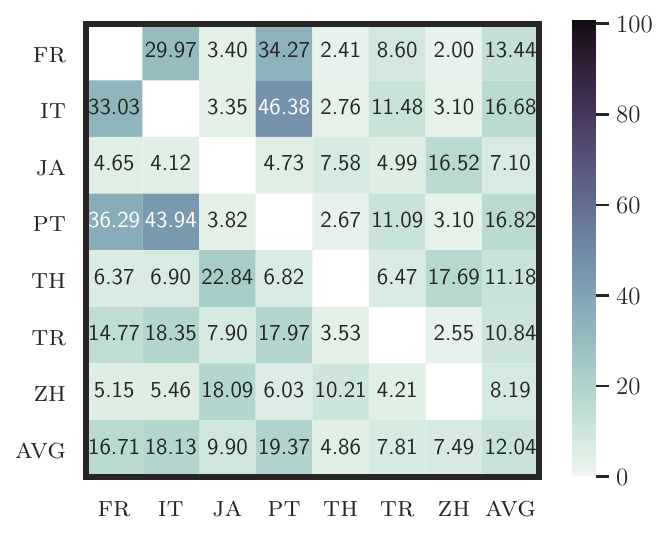}
        \caption{TED, $\varepsilon = 30$, $l = 0$}
        \label{fig:pos_retrieval_ted2020_30_lay0}
    \end{subfigure}
    \begin{subfigure}[b]{0.21\textwidth}
        \centering
        \includegraphics[width=\textwidth]{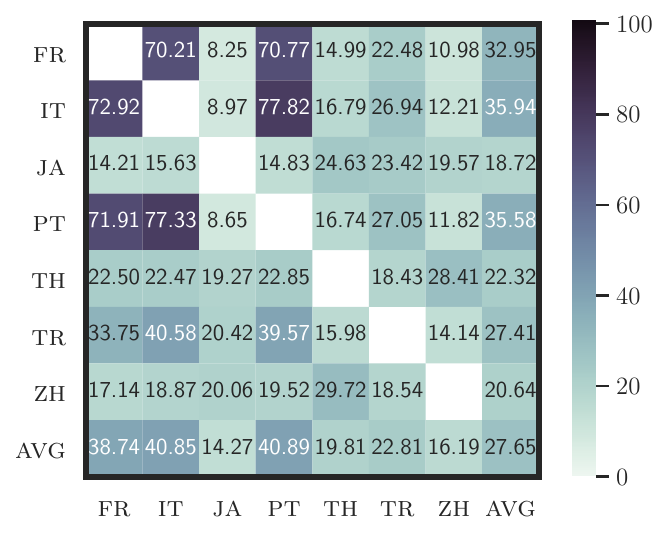}
        \caption{TED, $\varepsilon = 30$, $l = 8$}
        \label{fig:pos_retrieval_ted2020_30_lay8}
    \end{subfigure}
        \begin{subfigure}[b]{0.21\textwidth}
        \centering
        \includegraphics[width=\textwidth]{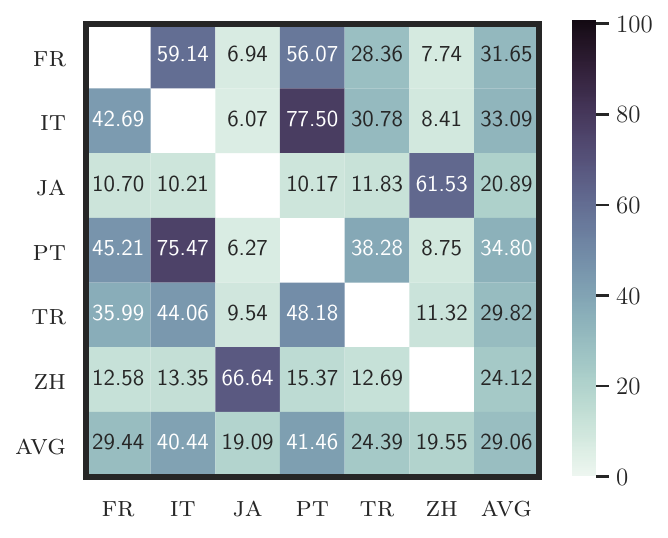}
        \caption{WM, $\varepsilon = 30$, $l = 0$}
        \label{fig:pos_retrieval_wikimatrix_30_lay0}
    \end{subfigure}
    \begin{subfigure}[b]{0.21\textwidth}
        \centering
        \includegraphics[width=\textwidth]{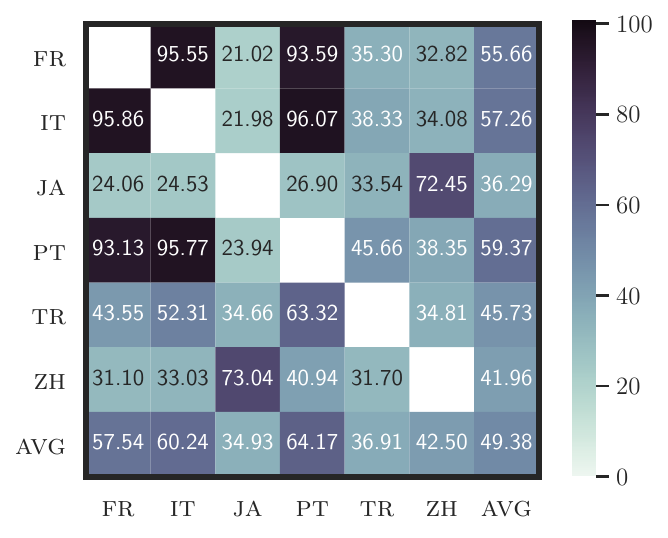}
        \caption{WM, $\varepsilon = 30$, $l = 8$}
        \label{fig:pos_retrieval_wikimatrix_30_lay8}
    \end{subfigure}
    \begin{subfigure}[b]{0.21\textwidth}
        \centering
        \includegraphics[width=\textwidth]{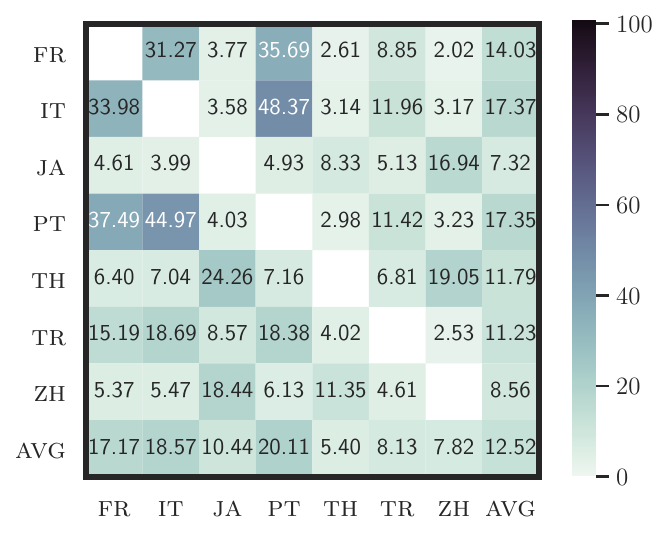}
        \caption{TED, $\varepsilon = \infty$, $l = 0$}
        \label{fig:pos_retrieval_ted2020_inf_lay0}
    \end{subfigure}
    \begin{subfigure}[b]{0.21\textwidth}
        \centering
        \includegraphics[width=\textwidth]{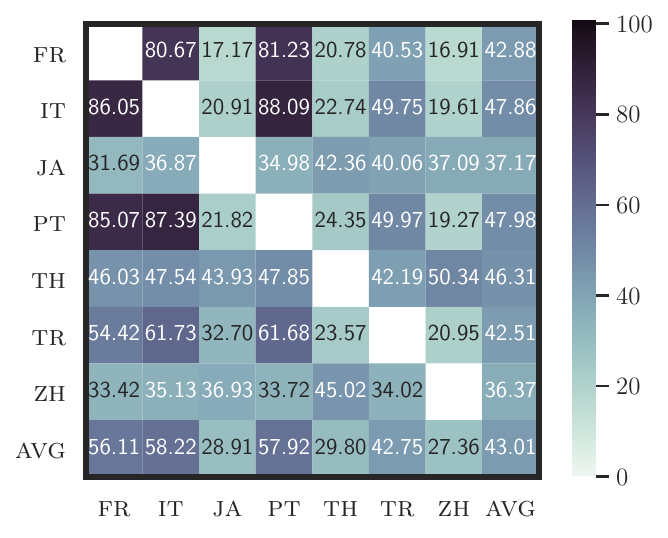}
        \caption{TED, $\varepsilon = \infty$, $l = 8$}
        \label{fig:pos_retrieval_ted2020_inf_lay8}
    \end{subfigure}
        \begin{subfigure}[b]{0.21\textwidth}
        \centering
        \includegraphics[width=\textwidth]{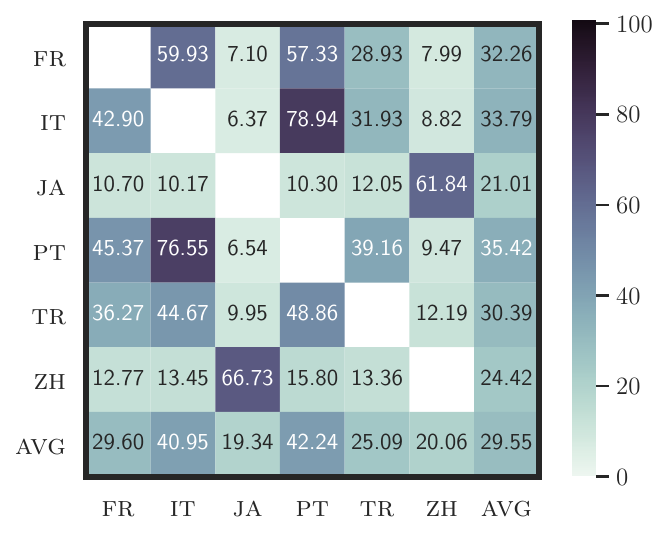}
        \caption{WM, $\varepsilon = \infty$, $l = 0$}
        \label{fig:pos_retrieval_wikimatrix_inf_lay0}
    \end{subfigure}
    \begin{subfigure}[b]{0.21\textwidth}
        \centering
        \includegraphics[width=\textwidth]{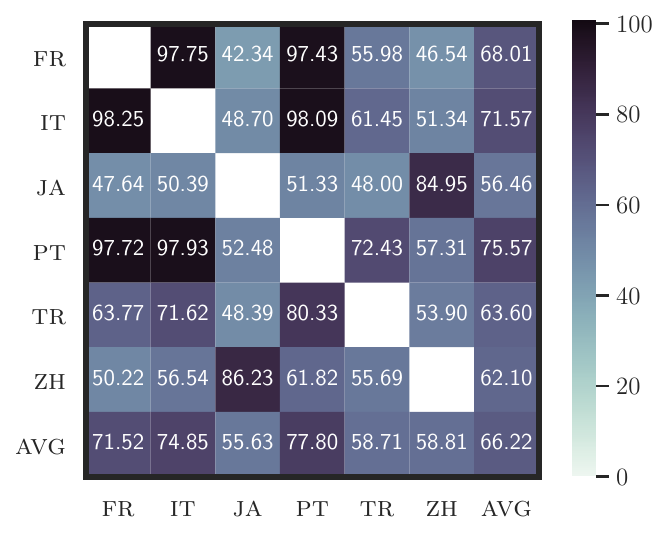}
        \caption{WM, $\varepsilon = \infty$, $l = 8$}
        \label{fig:pos_retrieval_wikimatrix_inf_lay8}
    \end{subfigure}
    \caption{\textbf{POS} Sentence retrieval results for the TED 2020 (TED) and WikiMatrix (WM) datasets and different combinations of privacy budgets ($\varepsilon$) and layers ($l$). Each heatmap cell corresponds to the average over 5 random seeds. We observe that the overall patterns are highly similar across all levels of privacy, particularly at layer 0.}
    \label{fig:pos_retrieval_tedwm}
\end{figure*}

\begin{figure*}[ht!]
    \centering
    \begin{subfigure}[b]{0.21\textwidth}
        \centering
        \includegraphics[width=\textwidth]{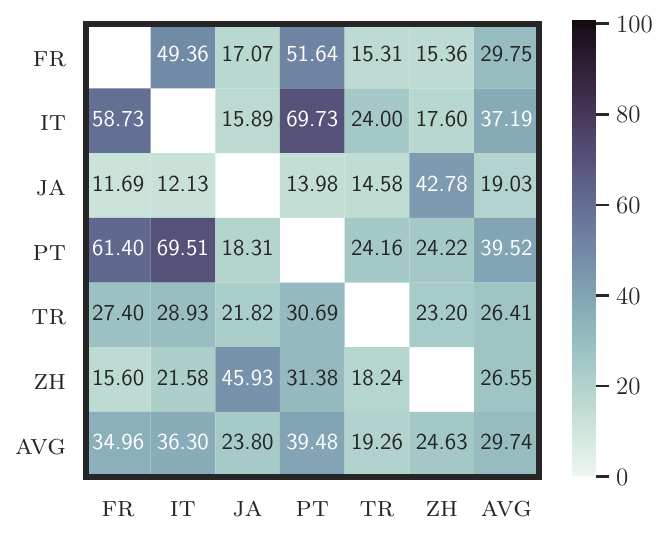}
        \caption{$\varepsilon = 1$, $l = 0$}
        \label{fig:pos_retrieval_tatoeba_1_lay0}
    \end{subfigure}
    \begin{subfigure}[b]{0.21\textwidth}
        \centering
        \includegraphics[width=\textwidth]{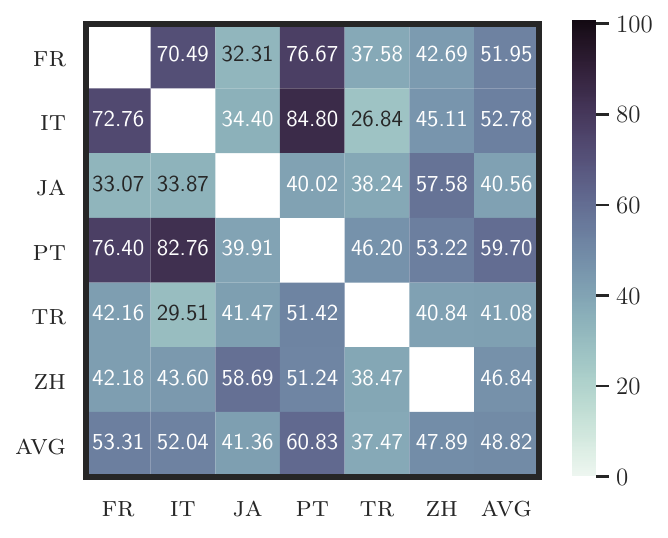}
        \caption{$\varepsilon = 1$, $l = 8$}
        \label{fig:pos_retrieval_tatoeba_1_lay8}
    \end{subfigure}
    \hspace{0.21\textwidth}
    \hspace{0.21\textwidth}
    \begin{subfigure}[b]{0.21\textwidth}
        \centering
        \includegraphics[width=\textwidth]{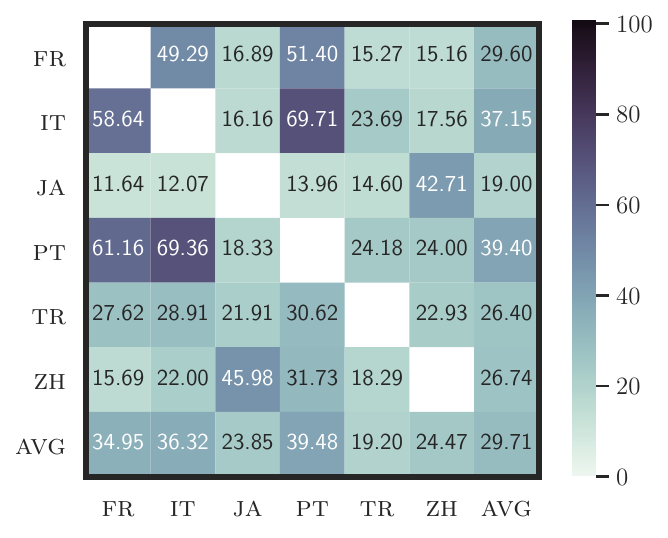}
        \caption{$\varepsilon = 3$, $l = 0$}
        \label{fig:pos_retrieval_tatoeba_3_lay0}
    \end{subfigure}
    \begin{subfigure}[b]{0.21\textwidth}
        \centering
        \includegraphics[width=\textwidth]{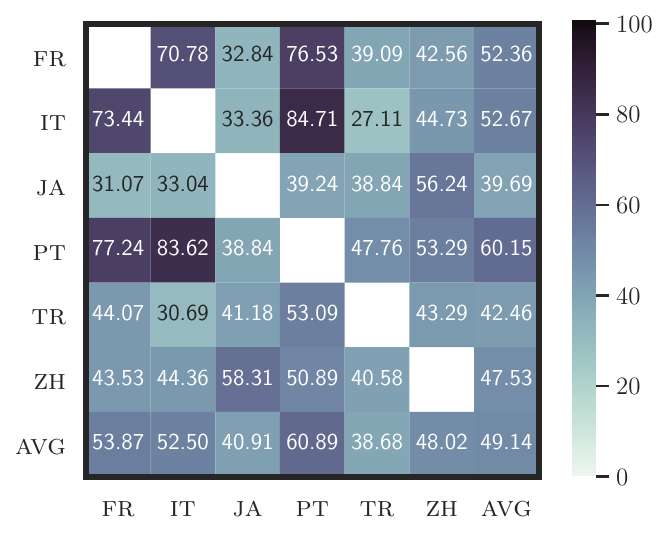}
        \caption{$\varepsilon = 3$, $l = 8$}
        \label{fig:pos_retrieval_tatoeba_3_lay8}
    \end{subfigure}
    \hspace{0.21\textwidth}
    \hspace{0.21\textwidth}
    \begin{subfigure}[b]{0.21\textwidth}
        \centering
        \includegraphics[width=\textwidth]{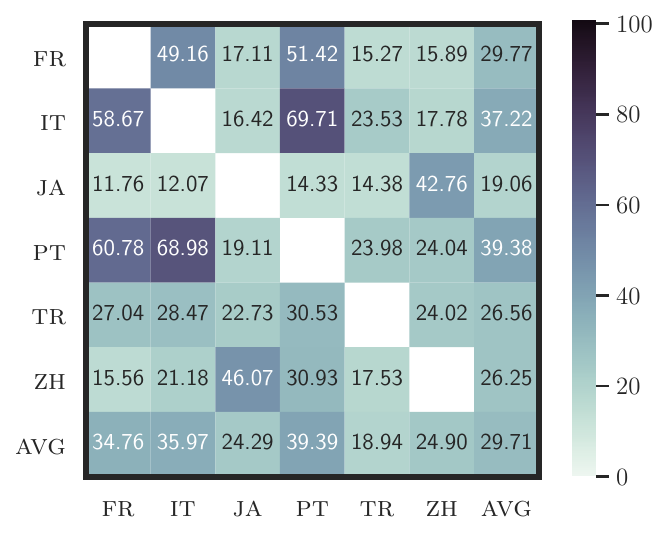}
        \caption{$\varepsilon = 8$, $l = 0$}
        \label{fig:pos_retrieval_tatoeba_8_lay0}
    \end{subfigure}
    \begin{subfigure}[b]{0.21\textwidth}
        \centering
        \includegraphics[width=\textwidth]{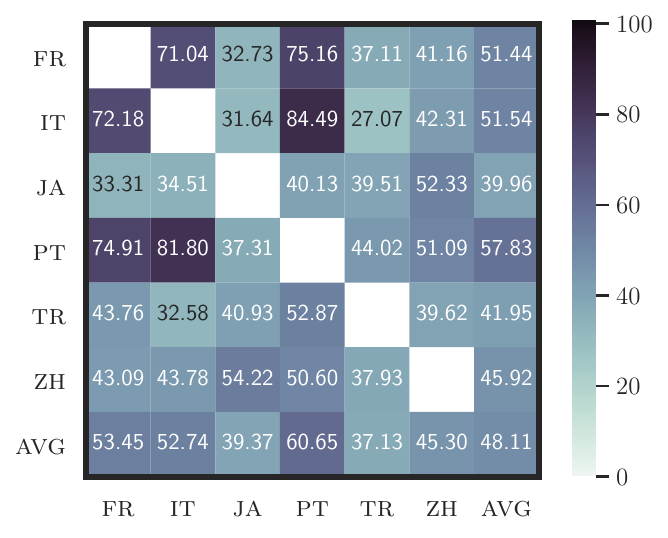}
        \caption{$\varepsilon = 8$, $l = 8$}
        \label{fig:pos_retrieval_tatoeba_8_lay8}
    \end{subfigure}
    \hspace{0.21\textwidth}
    \hspace{0.21\textwidth}
    \begin{subfigure}[b]{0.21\textwidth}
        \centering
        \includegraphics[width=\textwidth]{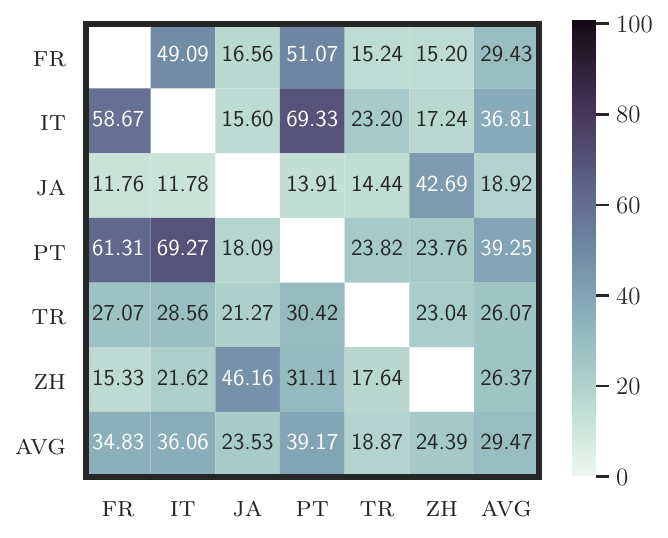}
        \caption{$\varepsilon = 15$, $l = 0$}
        \label{fig:pos_retrieval_tatoeba_15_lay0}
    \end{subfigure}
    \begin{subfigure}[b]{0.21\textwidth}
        \centering
        \includegraphics[width=\textwidth]{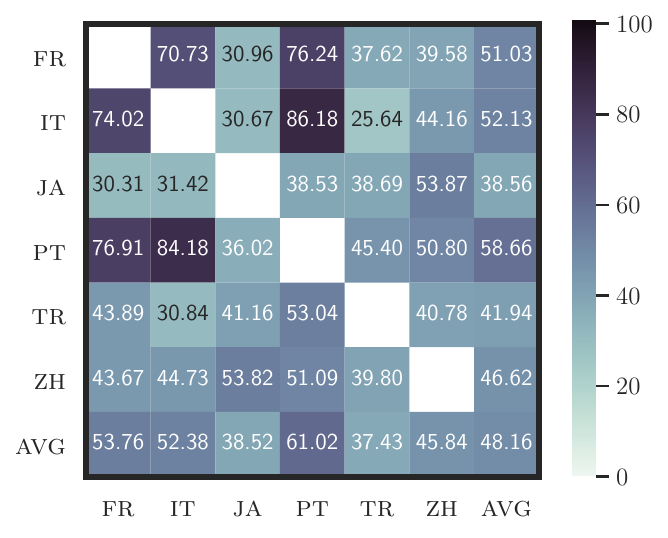}
        \caption{$\varepsilon = 15$, $l = 8$}
        \label{fig:pos_retrieval_tatoeba_15_lay8}
    \end{subfigure}
    \hspace{0.21\textwidth}
    \hspace{0.21\textwidth}
    \begin{subfigure}[b]{0.21\textwidth}
        \centering
        \includegraphics[width=\textwidth]{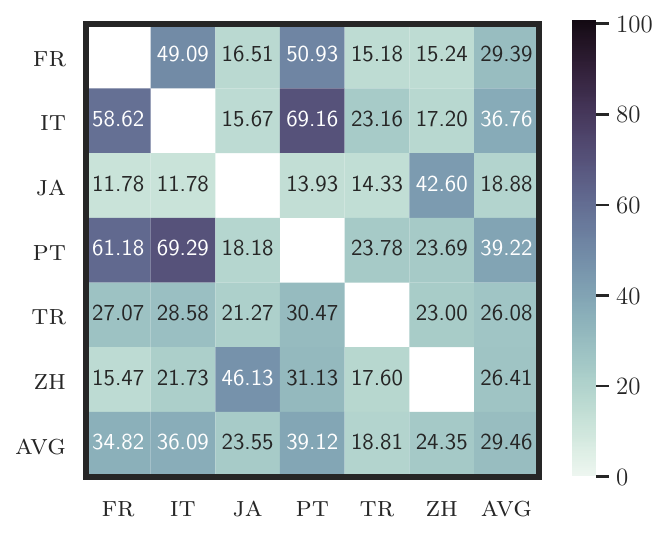}
        \caption{$\varepsilon = 30$, $l = 0$}
        \label{fig:pos_retrieval_tatoeba_30_lay0}
    \end{subfigure}
    \begin{subfigure}[b]{0.21\textwidth}
        \centering
        \includegraphics[width=\textwidth]{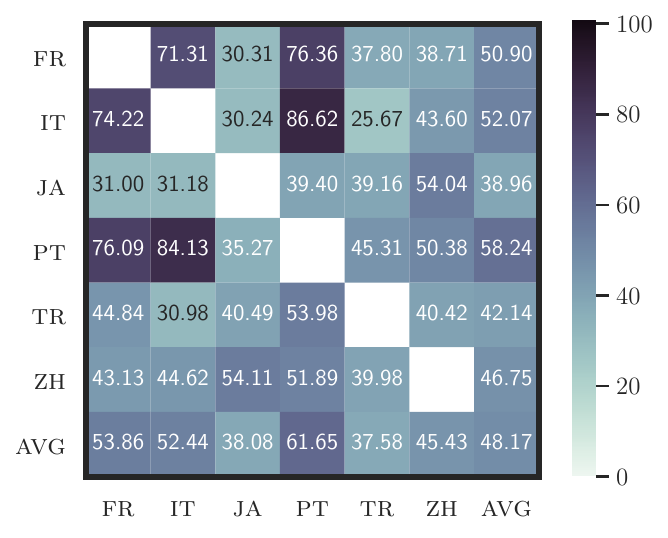}
        \caption{$\varepsilon = 30$, $l = 8$}
        \label{fig:pos_retrieval_tatoeba_30_lay8}
    \end{subfigure}
    \hspace{0.21\textwidth}
    \hspace{0.21\textwidth}
    \begin{subfigure}[b]{0.21\textwidth}
        \centering
        \includegraphics[width=\textwidth]{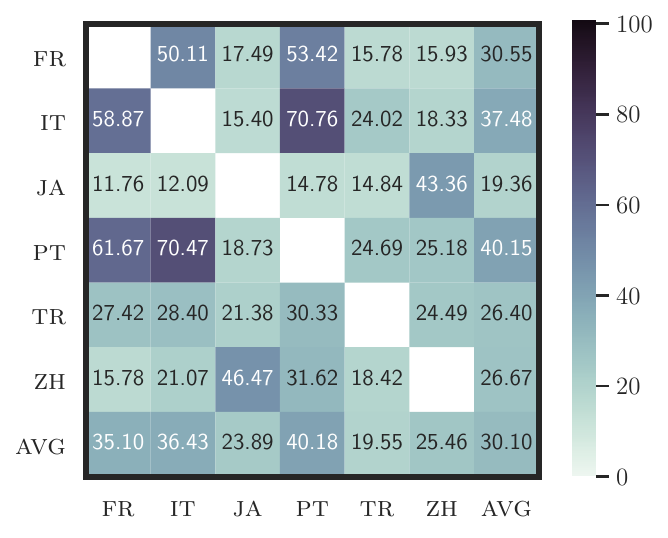}
        \caption{$\varepsilon = \infty$, $l = 0$}
        \label{fig:pos_retrieval_tatoeba_inf_lay0}
    \end{subfigure}
    \begin{subfigure}[b]{0.21\textwidth}
        \centering
        \includegraphics[width=\textwidth]{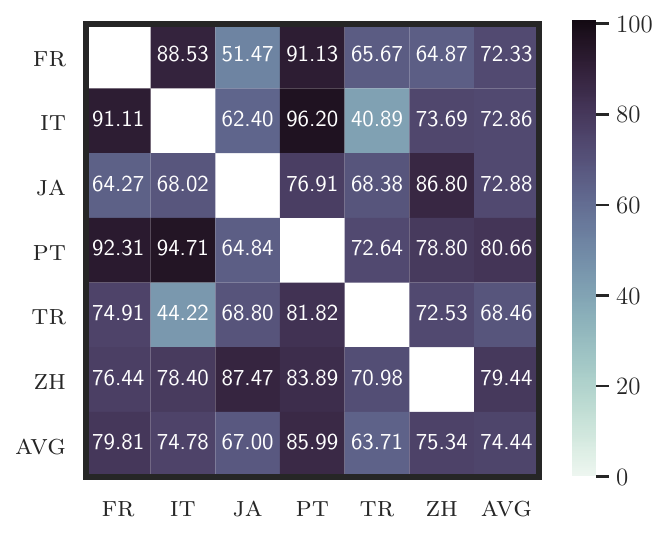}
        \caption{$\varepsilon = \infty$, $l = 8$}
        \label{fig:pos_retrieval_tatoeba_inf_lay8}
    \end{subfigure}
    \caption{\textbf{POS} sentence retrieval results for the Tatoeba dataset and different combinations of privacy budgets ($\varepsilon$) and layers ($l$). Each heatmap cell corresponds to the average over 5 random seeds. We observe that the overall patterns are highly similar across all levels of privacy, particularly at layer 0.}
    \label{fig:pos_retrieval_tatoeba}
\end{figure*}

\begin{figure*}[ht!]
    \centering
    \begin{subfigure}[b]{0.21\textwidth}
        \centering
        \includegraphics[width=\textwidth]{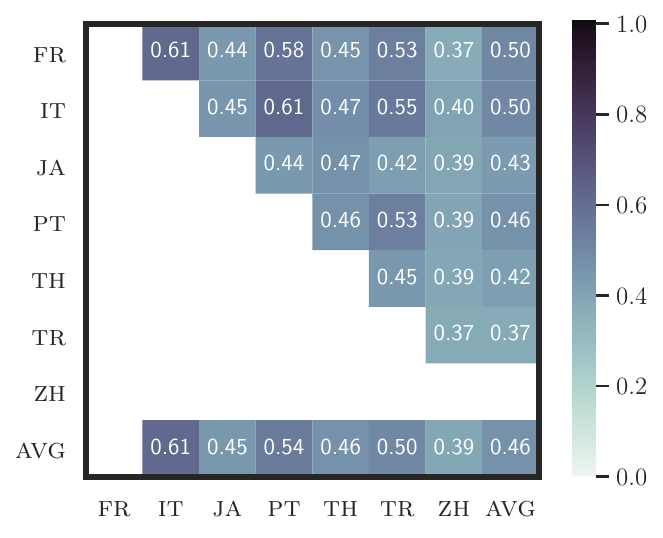}
        \caption{TED, $\varepsilon = 1$, $l = 0$}
        \label{fig:pos_cka_ted2020_1_lay0}
    \end{subfigure}
    \begin{subfigure}[b]{0.21\textwidth}
        \centering
        \includegraphics[width=\textwidth]{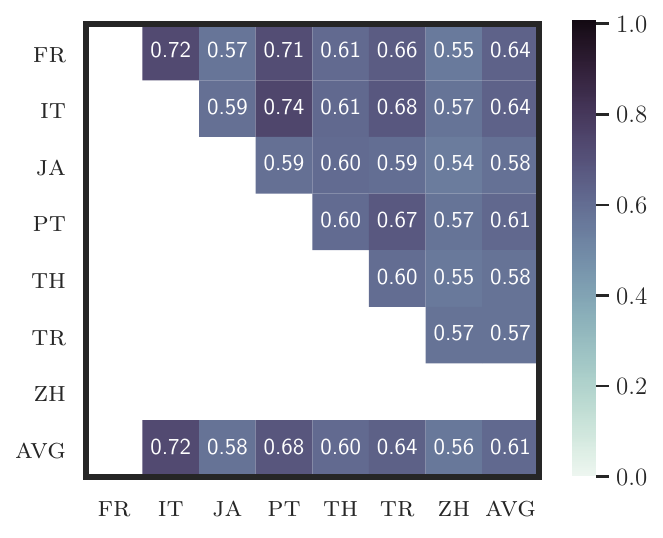}
        \caption{TED, $\varepsilon = 1$, $l = 8$}
        \label{fig:pos_cka_ted2020_1_lay8}
    \end{subfigure}
        \begin{subfigure}[b]{0.21\textwidth}
        \centering
        \includegraphics[width=\textwidth]{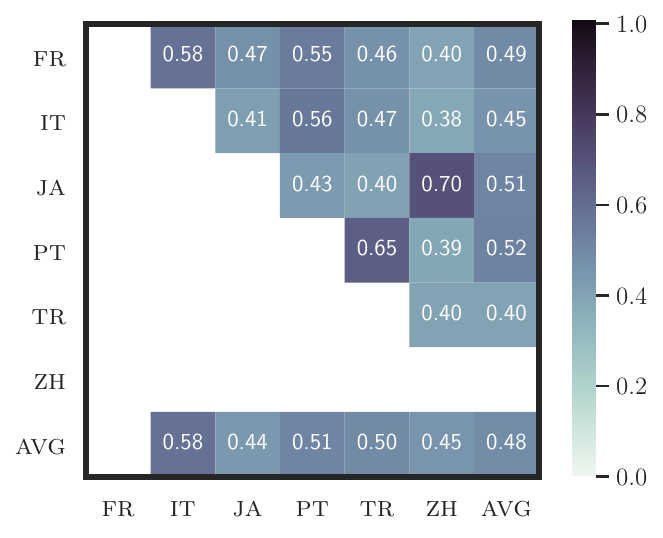}
        \caption{WM, $\varepsilon = 1$, $l = 0$}
        \label{fig:pos_cka_wikimatrix_1_lay0}
    \end{subfigure}
    \begin{subfigure}[b]{0.21\textwidth}
        \centering
        \includegraphics[width=\textwidth]{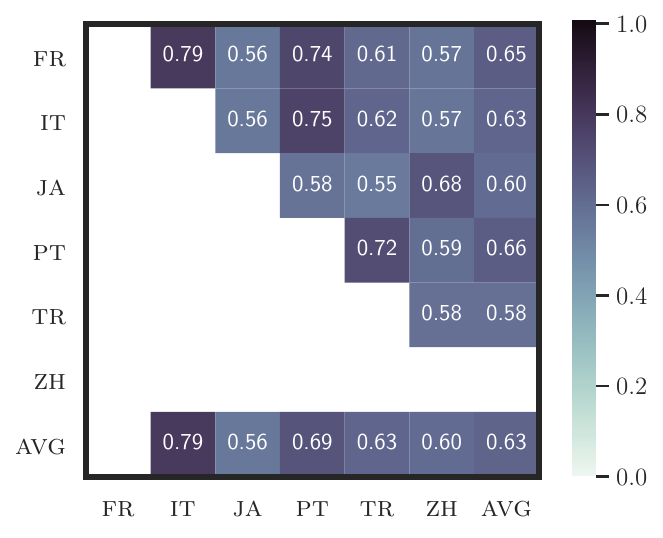}
        \caption{WM, $\varepsilon = 1$, $l = 8$}
        \label{fig:pos_cka_wikimatrix_1_lay8}
    \end{subfigure}
    \begin{subfigure}[b]{0.21\textwidth}
        \centering
        \includegraphics[width=\textwidth]{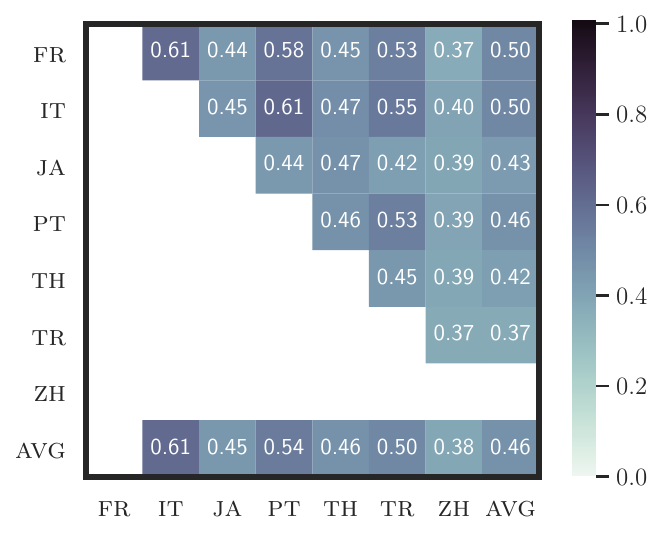}
        \caption{TED, $\varepsilon = 3$, $l = 0$}
        \label{fig:pos_cka_ted2020_3_lay0}
    \end{subfigure}
    \begin{subfigure}[b]{0.21\textwidth}
        \centering
        \includegraphics[width=\textwidth]{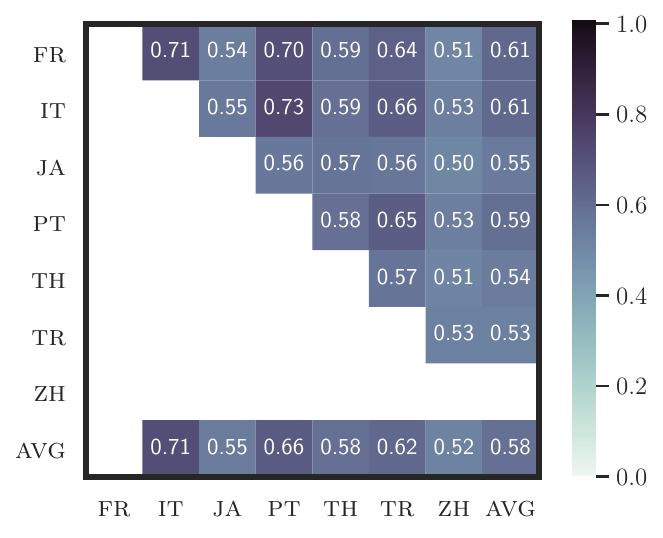}
        \caption{TED, $\varepsilon = 3$, $l = 8$}
        \label{fig:pos_cka_ted2020_3_lay8}
    \end{subfigure}
        \begin{subfigure}[b]{0.21\textwidth}
        \centering
        \includegraphics[width=\textwidth]{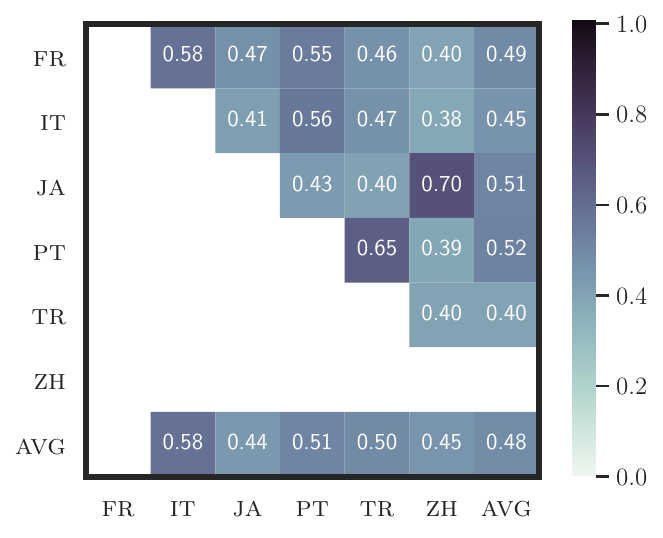}
        \caption{WM, $\varepsilon = 3$, $l = 0$}
        \label{fig:pos_cka_wikimatrix_3_lay0}
    \end{subfigure}
    \begin{subfigure}[b]{0.21\textwidth}
        \centering
        \includegraphics[width=\textwidth]{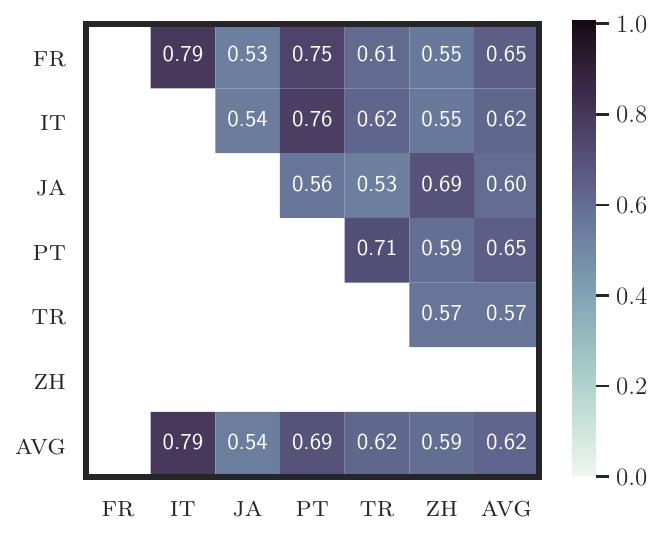}
        \caption{WM, $\varepsilon = 3$, $l = 8$}
        \label{fig:pos_cka_wikimatrix_3_lay8}
    \end{subfigure}
    \begin{subfigure}[b]{0.21\textwidth}
        \centering
        \includegraphics[width=\textwidth]{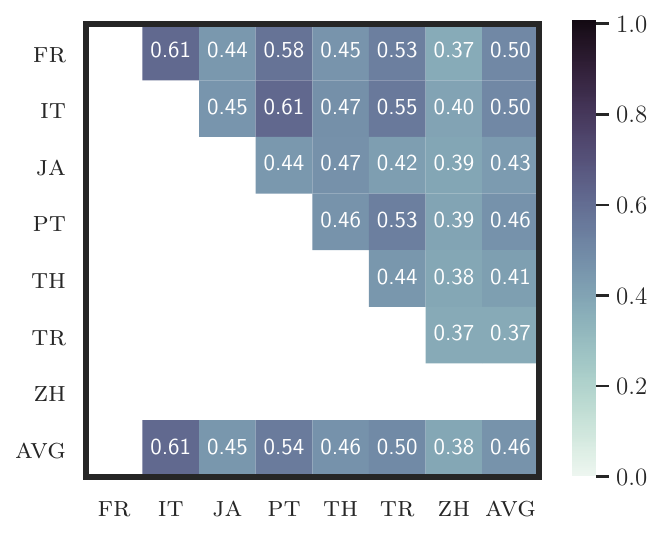}
        \caption{TED, $\varepsilon = 8$, $l = 0$}
        \label{fig:pos_cka_ted2020_8_lay0}
    \end{subfigure}
    \begin{subfigure}[b]{0.21\textwidth}
        \centering
        \includegraphics[width=\textwidth]{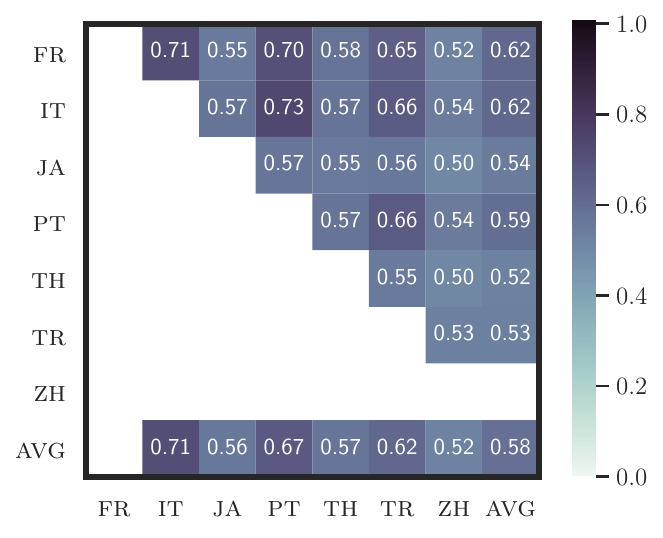}
        \caption{TED, $\varepsilon = 8$, $l = 8$}
        \label{fig:pos_cka_ted2020_8_lay8}
    \end{subfigure}
        \begin{subfigure}[b]{0.21\textwidth}
        \centering
        \includegraphics[width=\textwidth]{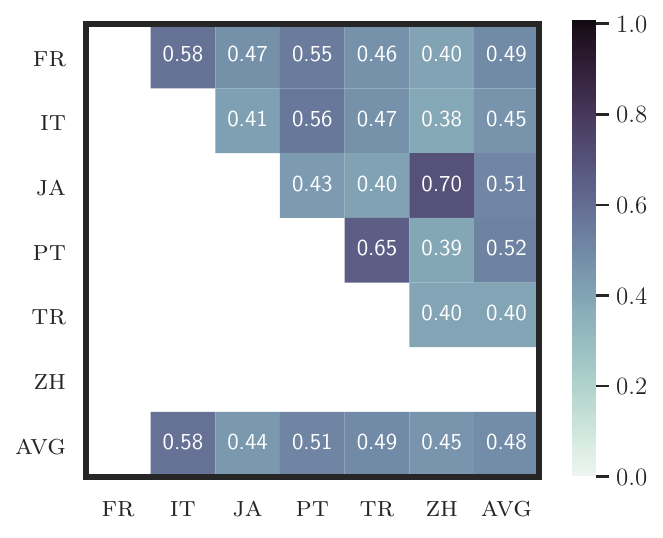}
        \caption{WM, $\varepsilon = 8$, $l = 0$}
        \label{fig:pos_cka_wikimatrix_8_lay0}
    \end{subfigure}
    \begin{subfigure}[b]{0.21\textwidth}
        \centering
        \includegraphics[width=\textwidth]{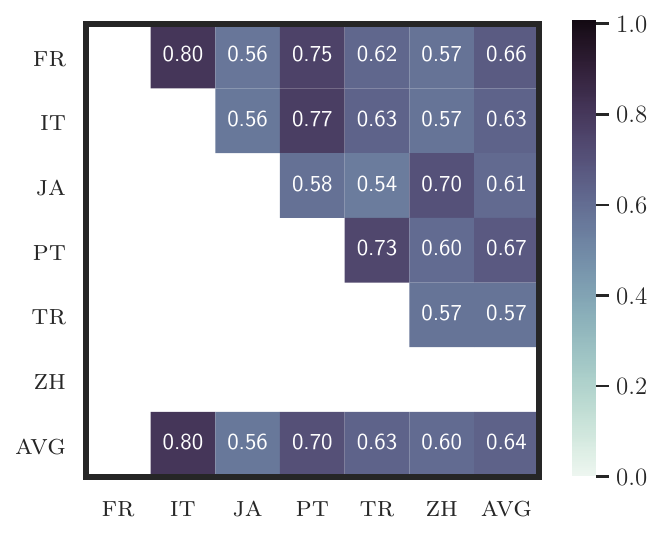}
        \caption{WM, $\varepsilon = 8$, $l = 8$}
        \label{fig:pos_cka_wikimatrix_8_lay8}
    \end{subfigure}
    \begin{subfigure}[b]{0.21\textwidth}
        \centering
        \includegraphics[width=\textwidth]{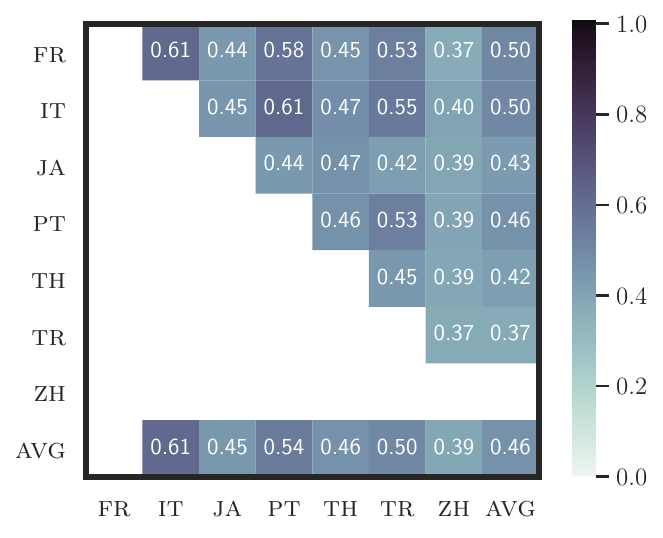}
        \caption{TED, $\varepsilon = 15$, $l = 0$}
        \label{fig:pos_cka_ted2020_15_lay0}
    \end{subfigure}
    \begin{subfigure}[b]{0.21\textwidth}
        \centering
        \includegraphics[width=\textwidth]{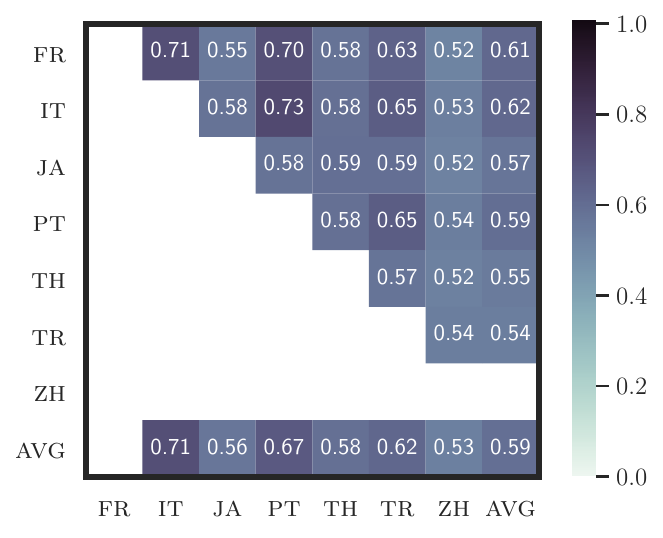}
        \caption{TED, $\varepsilon = 15$, $l = 8$}
        \label{fig:pos_cka_ted2020_15_lay8}
    \end{subfigure}
        \begin{subfigure}[b]{0.21\textwidth}
        \centering
        \includegraphics[width=\textwidth]{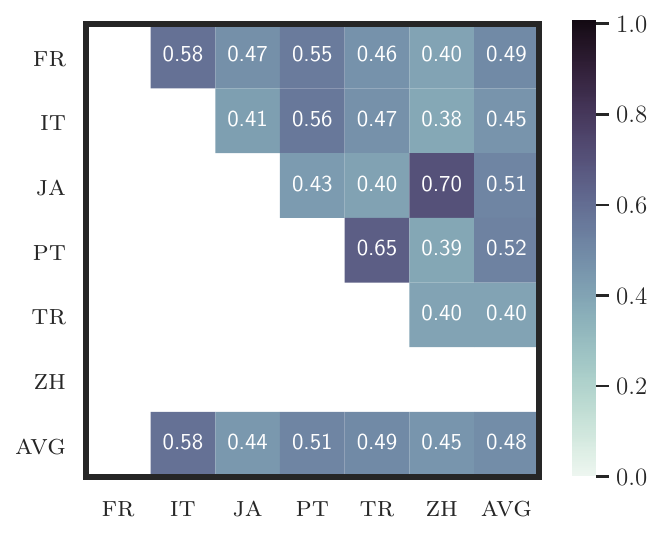}
        \caption{WM, $\varepsilon = 15$, $l = 0$}
        \label{fig:pos_cka_wikimatrix_15_lay0}
    \end{subfigure}
    \begin{subfigure}[b]{0.21\textwidth}
        \centering
        \includegraphics[width=\textwidth]{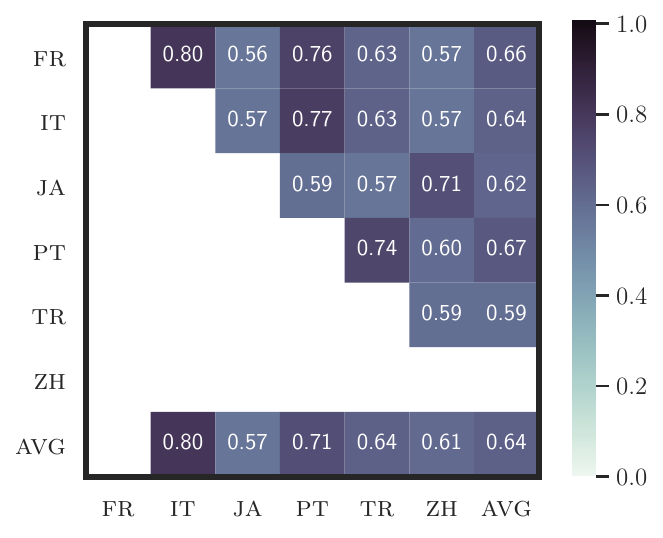}
        \caption{WM, $\varepsilon = 15$, $l = 8$}
        \label{fig:pos_cka_wikimatrix_15_lay8}
    \end{subfigure}
    \begin{subfigure}[b]{0.21\textwidth}
        \centering
        \includegraphics[width=\textwidth]{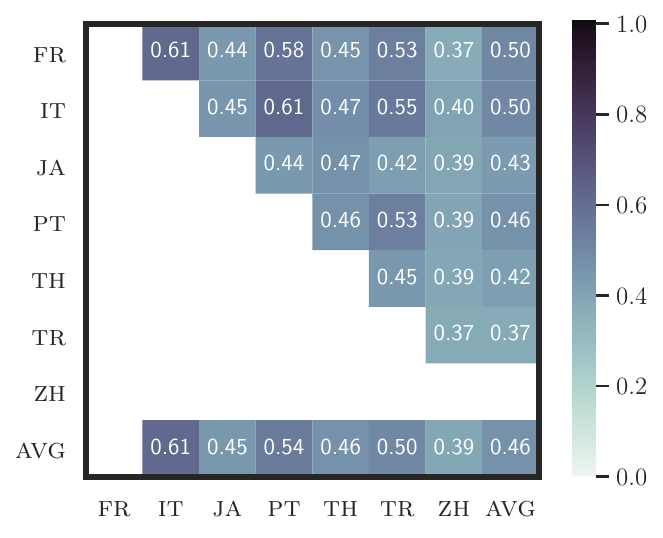}
        \caption{TED, $\varepsilon = 30$, $l = 0$}
        \label{fig:pos_cka_ted2020_30_lay0}
    \end{subfigure}
    \begin{subfigure}[b]{0.21\textwidth}
        \centering
        \includegraphics[width=\textwidth]{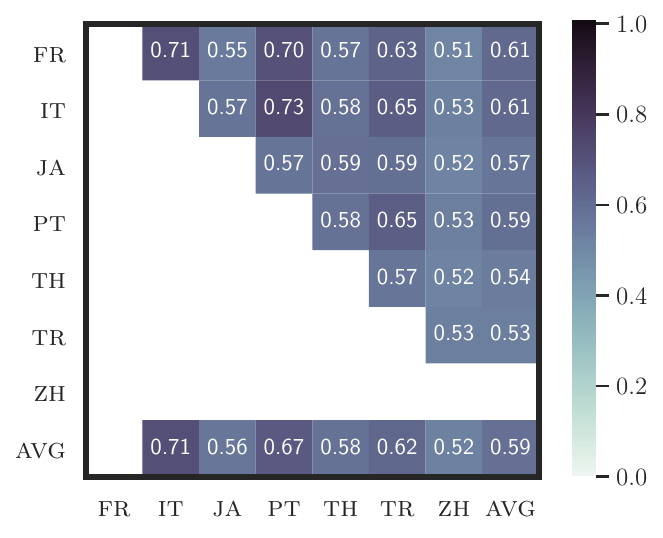}
        \caption{TED, $\varepsilon = 30$, $l = 8$}
        \label{fig:pos_cka_ted2020_30_lay8}
    \end{subfigure}
        \begin{subfigure}[b]{0.21\textwidth}
        \centering
        \includegraphics[width=\textwidth]{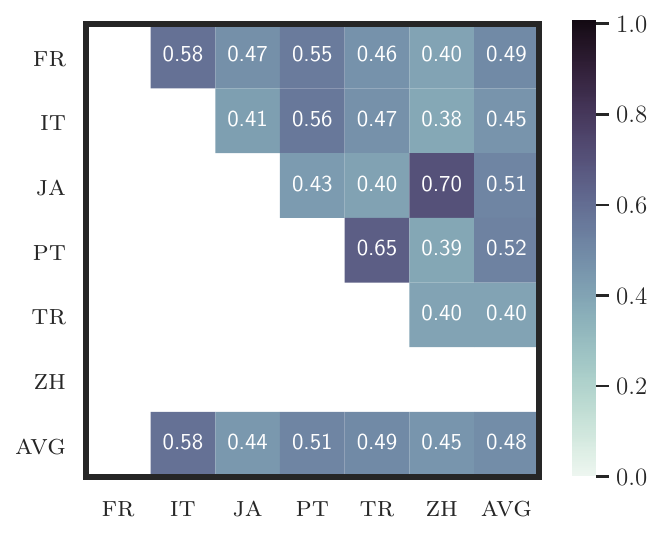}
        \caption{WM, $\varepsilon = 30$, $l = 0$}
        \label{fig:pos_cka_wikimatrix_30_lay0}
    \end{subfigure}
    \begin{subfigure}[b]{0.21\textwidth}
        \centering
        \includegraphics[width=\textwidth]{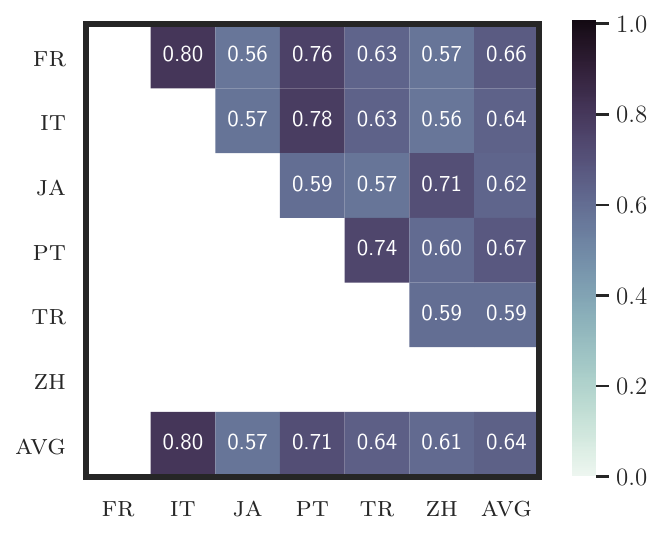}
        \caption{WM, $\varepsilon = 30$, $l = 8$}
        \label{fig:pos_cka_wikimatrix_30_lay8}
    \end{subfigure}
    \begin{subfigure}[b]{0.21\textwidth}
        \centering
        \includegraphics[width=\textwidth]{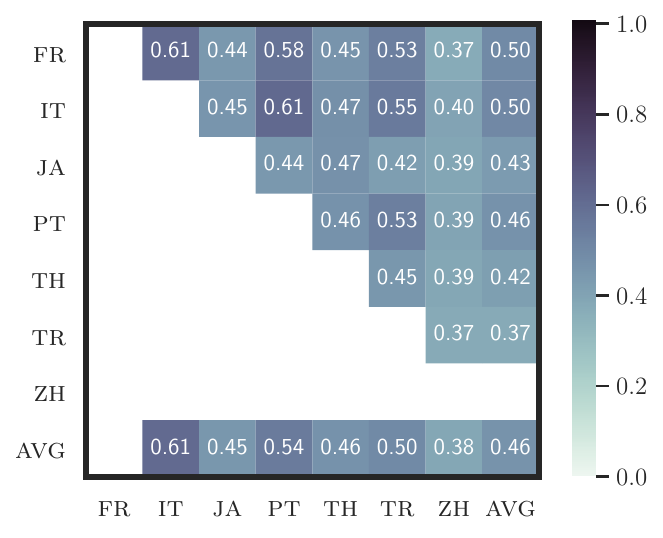}
        \caption{TED, $\varepsilon = \infty$, $l = 0$}
        \label{fig:pos_cka_ted2020_inf_lay0}
    \end{subfigure}
    \begin{subfigure}[b]{0.21\textwidth}
        \centering
        \includegraphics[width=\textwidth]{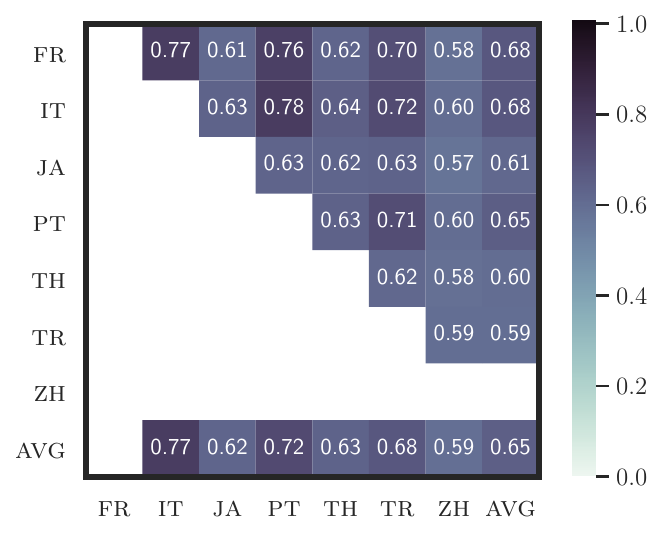}
        \caption{TED, $\varepsilon = \infty$, $l = 8$}
        \label{fig:pos_cka_ted2020_inf_lay8}
    \end{subfigure}
        \begin{subfigure}[b]{0.21\textwidth}
        \centering
        \includegraphics[width=\textwidth]{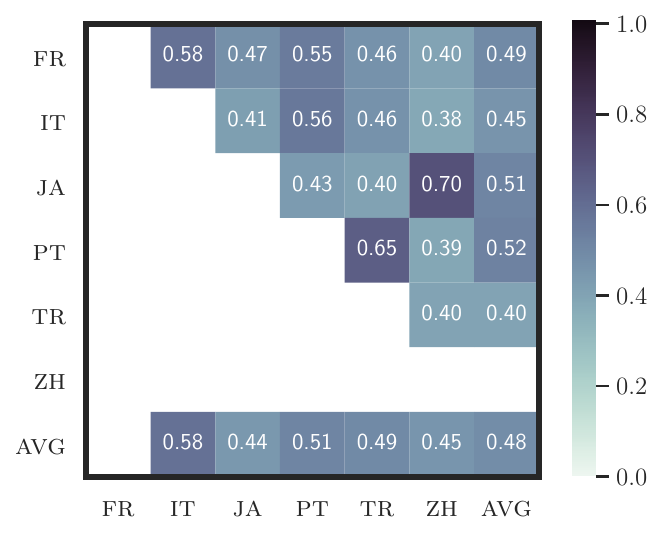}
        \caption{WM, $\varepsilon = \infty$, $l = 0$}
        \label{fig:pos_cka_wikimatrix_inf_lay0}
    \end{subfigure}
    \begin{subfigure}[b]{0.21\textwidth}
        \centering
        \includegraphics[width=\textwidth]{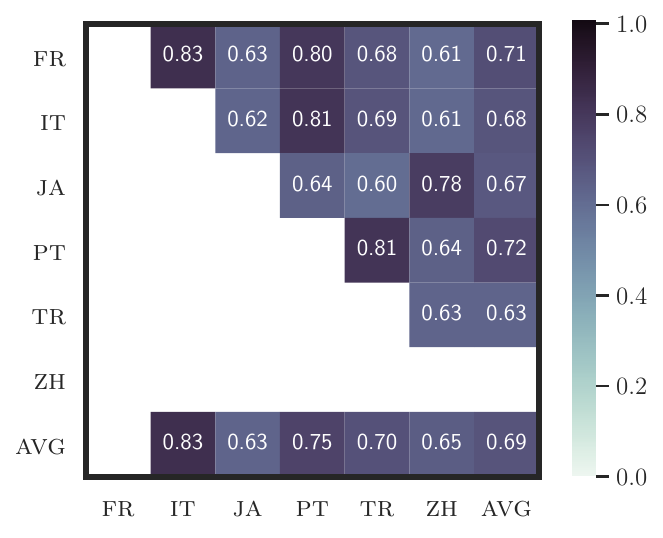}
        \caption{WM, $\varepsilon = \infty$, $l = 8$}
        \label{fig:pos_cka_wikimatrix_inf_lay8}
    \end{subfigure}
    \caption{\textbf{POS} CKA results for the TED 2020 (TED) and WikiMatrix (WM) datasets and different combinations of privacy budgets ($\varepsilon$) and layers ($l$). Each heatmap cell corresponds to the average over 5 random seeds. We observe that the overall patterns are highly similar across all levels of privacy, particularly at layer 0.}
    \label{fig:pos_cka_tedwm}
\end{figure*}

\begin{figure*}[ht!]
    \centering
    \begin{subfigure}[b]{0.21\textwidth}
        \centering
        \includegraphics[width=\textwidth]{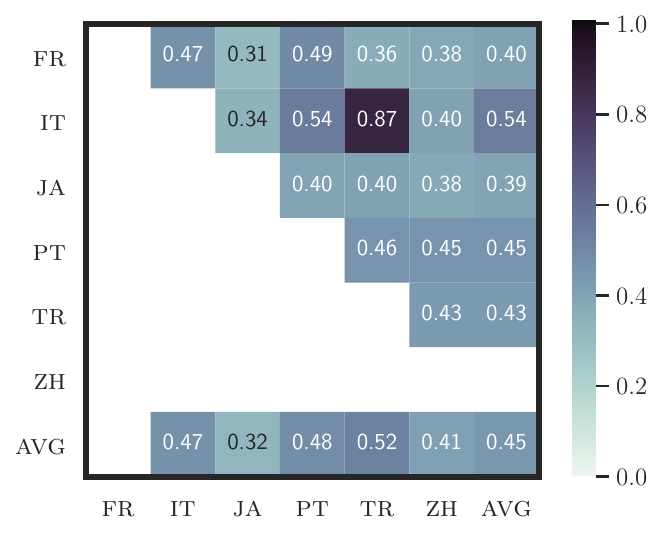}
        \caption{$\varepsilon = 1$, $l = 0$}
        \label{fig:pos_cka_tatoeba_1_lay0}
    \end{subfigure}
    \begin{subfigure}[b]{0.21\textwidth}
        \centering
        \includegraphics[width=\textwidth]{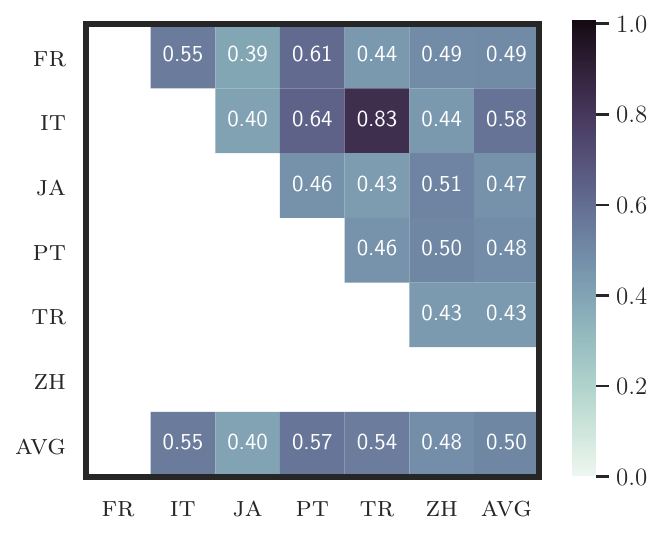}
        \caption{$\varepsilon = 1$, $l = 8$}
        \label{fig:pos_cka_tatoeba_1_lay8}
    \end{subfigure}
    \hspace{0.21\textwidth}
    \hspace{0.21\textwidth}
    \begin{subfigure}[b]{0.21\textwidth}
        \centering
        \includegraphics[width=\textwidth]{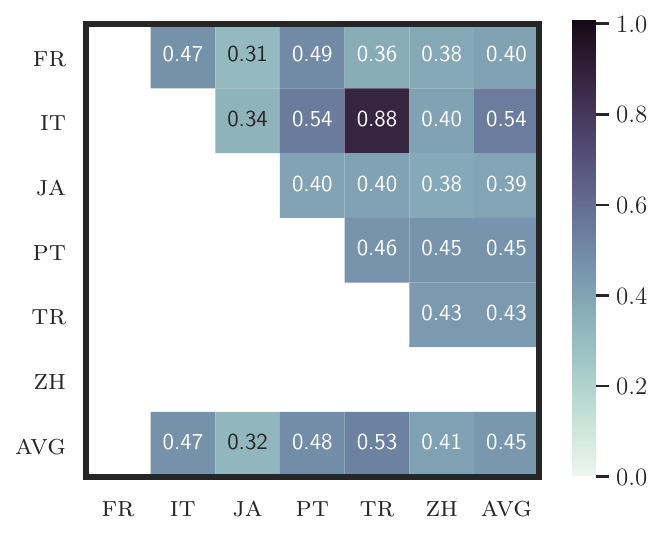}
        \caption{$\varepsilon = 3$, $l = 0$}
        \label{fig:pos_cka_tatoeba_3_lay0}
    \end{subfigure}
    \begin{subfigure}[b]{0.21\textwidth}
        \centering
        \includegraphics[width=\textwidth]{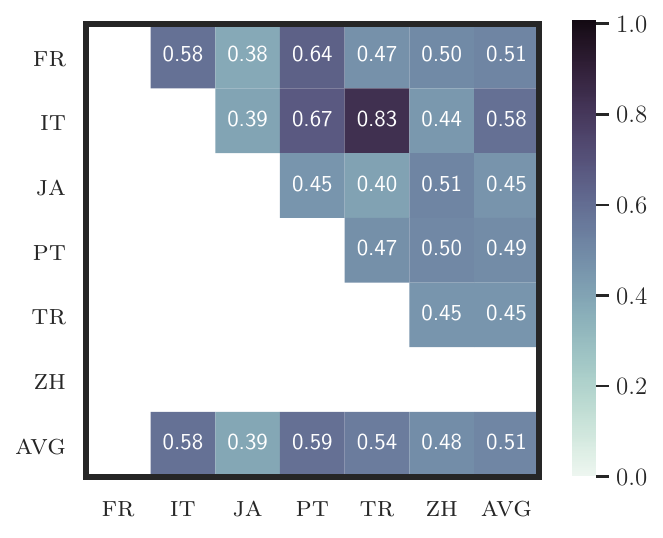}
        \caption{$\varepsilon = 3$, $l = 8$}
        \label{fig:pos_cka_tatoeba_3_lay8}
    \end{subfigure}
    \hspace{0.21\textwidth}
    \hspace{0.21\textwidth}
    \begin{subfigure}[b]{0.21\textwidth}
        \centering
        \includegraphics[width=\textwidth]{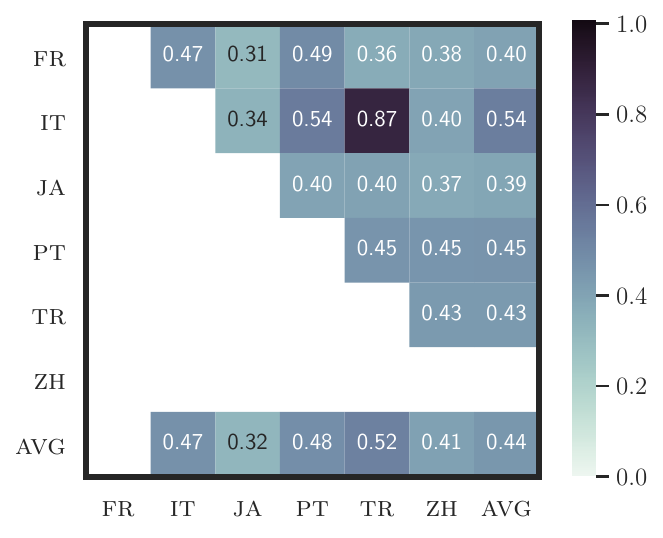}
        \caption{$\varepsilon = 8$, $l = 0$}
        \label{fig:pos_cka_tatoeba_8_lay0}
    \end{subfigure}
    \begin{subfigure}[b]{0.21\textwidth}
        \centering
        \includegraphics[width=\textwidth]{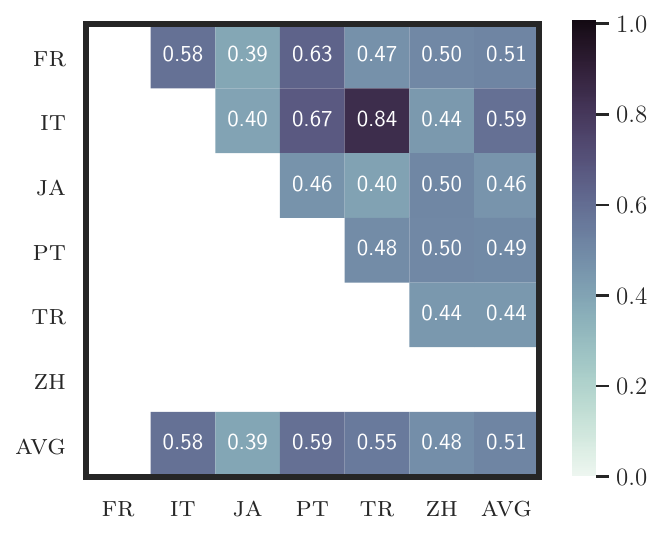}
        \caption{$\varepsilon = 8$, $l = 8$}
        \label{fig:pos_cka_tatoeba_8_lay8}
    \end{subfigure}
    \hspace{0.21\textwidth}
    \hspace{0.21\textwidth}
    \begin{subfigure}[b]{0.21\textwidth}
        \centering
        \includegraphics[width=\textwidth]{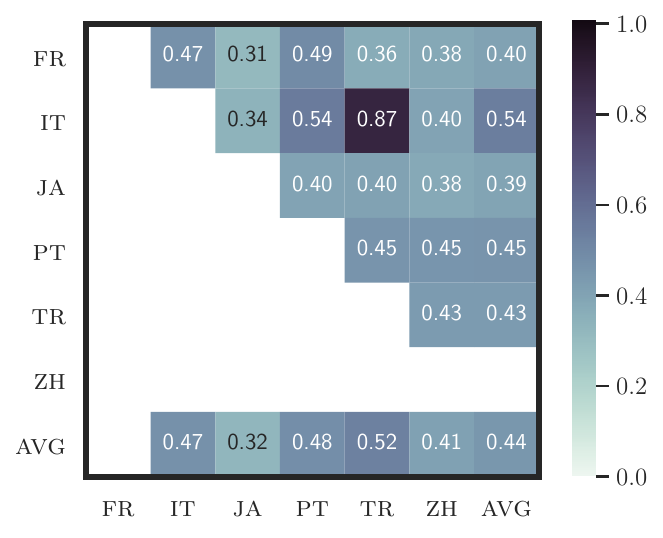}
        \caption{$\varepsilon = 15$, $l = 0$}
        \label{fig:pos_cka_tatoeba_15_lay0}
    \end{subfigure}
    \begin{subfigure}[b]{0.21\textwidth}
        \centering
        \includegraphics[width=\textwidth]{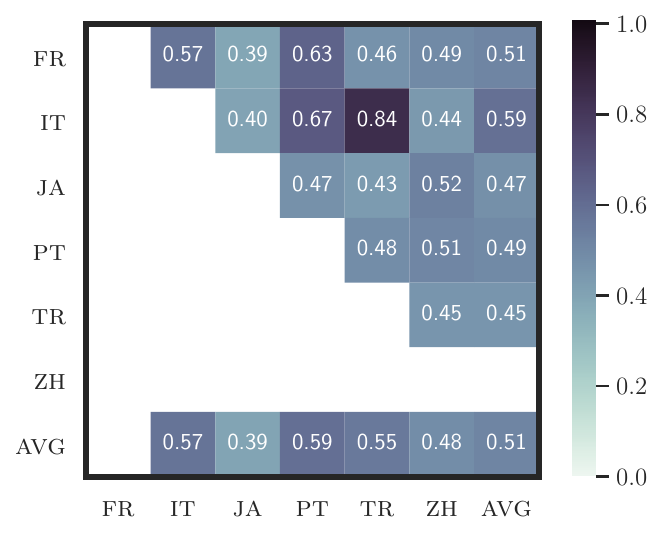}
        \caption{$\varepsilon = 15$, $l = 8$}
        \label{fig:pos_cka_tatoeba_15_lay8}
    \end{subfigure}
    \hspace{0.21\textwidth}
    \hspace{0.21\textwidth}
    \begin{subfigure}[b]{0.21\textwidth}
        \centering
        \includegraphics[width=\textwidth]{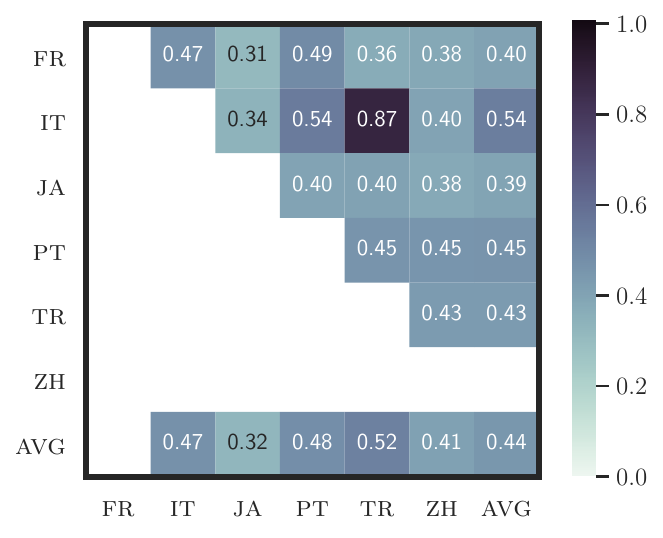}
        \caption{$\varepsilon = 30$, $l = 0$}
        \label{fig:pos_cka_tatoeba_30_lay0}
    \end{subfigure}
    \begin{subfigure}[b]{0.21\textwidth}
        \centering
        \includegraphics[width=\textwidth]{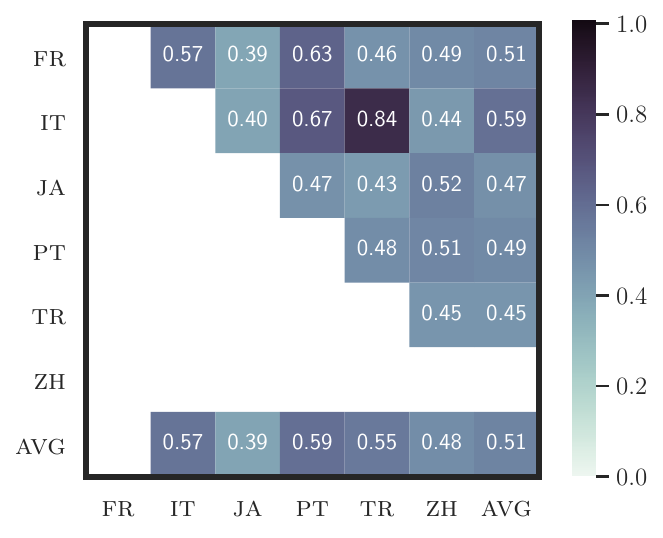}
        \caption{$\varepsilon = 30$, $l = 8$}
        \label{fig:pos_cka_tatoeba_30_lay8}
    \end{subfigure}
    \hspace{0.21\textwidth}
    \hspace{0.21\textwidth}
    \begin{subfigure}[b]{0.21\textwidth}
        \centering
        \includegraphics[width=\textwidth]{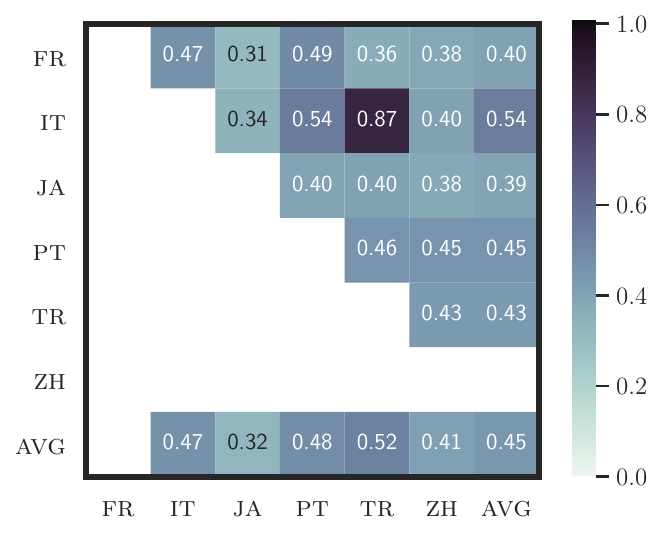}
        \caption{$\varepsilon = \infty$, $l = 0$}
        \label{fig:pos_cka_tatoeba_inf_lay0}
    \end{subfigure}
    \begin{subfigure}[b]{0.21\textwidth}
        \centering
        \includegraphics[width=\textwidth]{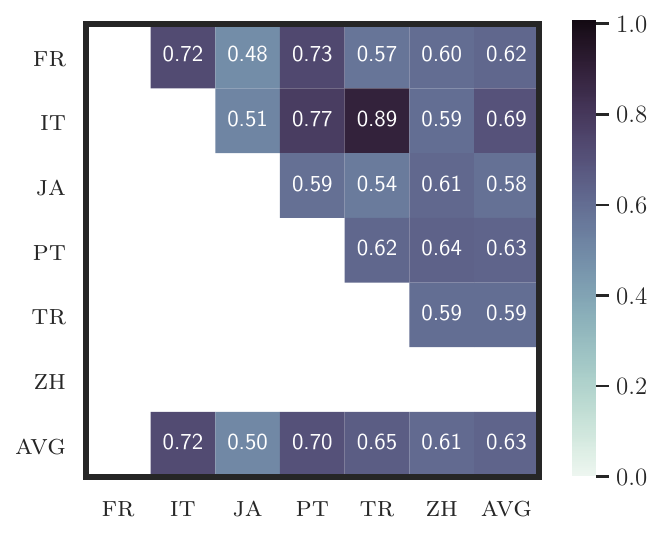}
        \caption{$\varepsilon = \infty$, $l = 8$}
        \label{fig:pos_cka_tatoeba_inf_lay8}
    \end{subfigure}
    \caption{\textbf{POS} CKA results for the Tatoeba dataset and different combinations of privacy budgets ($\varepsilon$) and layers ($l$). Each heatmap cell corresponds to the average over 5 random seeds. We observe that the overall patterns are highly similar across all levels of privacy, particularly at layer 0.}
    \label{fig:pos_cka_tatoeba}
\end{figure*}

\begin{figure*}[ht!]
    \centering
    \begin{subfigure}[b]{0.21\textwidth}
        \centering
        \includegraphics[width=\textwidth]{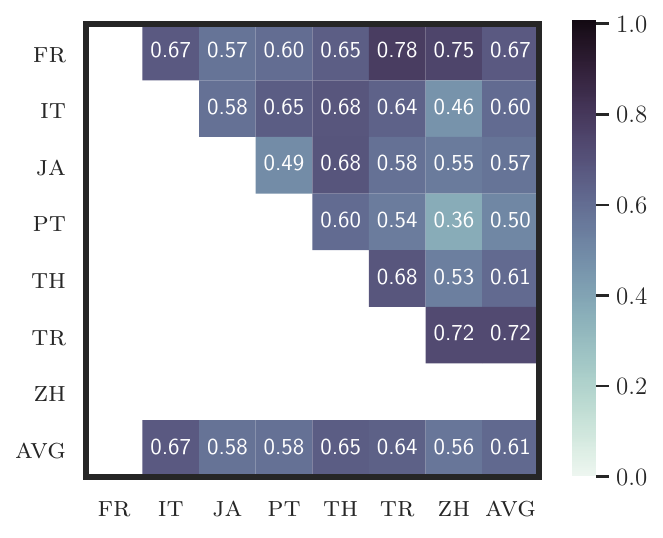}
        \caption{TED, $\varepsilon = 1$, $l = 0$}
        \label{fig:pos_rsa_ted2020_1_lay0}
    \end{subfigure}
    \begin{subfigure}[b]{0.21\textwidth}
        \centering
        \includegraphics[width=\textwidth]{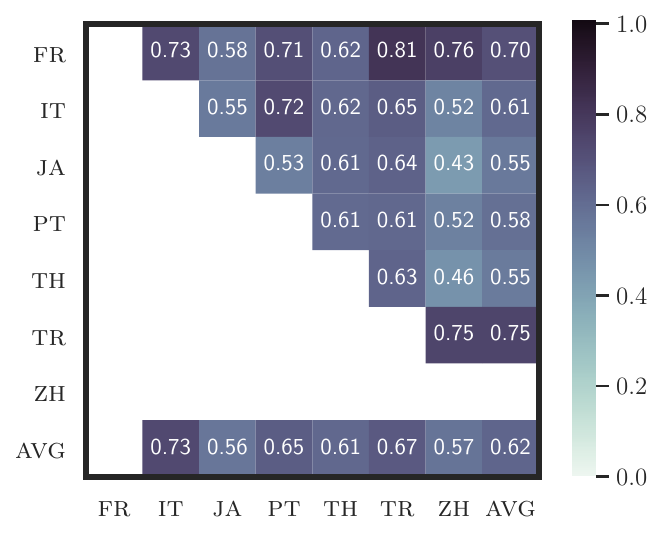}
        \caption{TED, $\varepsilon = 1$, $l = 8$}
        \label{fig:pos_rsa_ted2020_1_lay8}
    \end{subfigure}
        \begin{subfigure}[b]{0.21\textwidth}
        \centering
        \includegraphics[width=\textwidth]{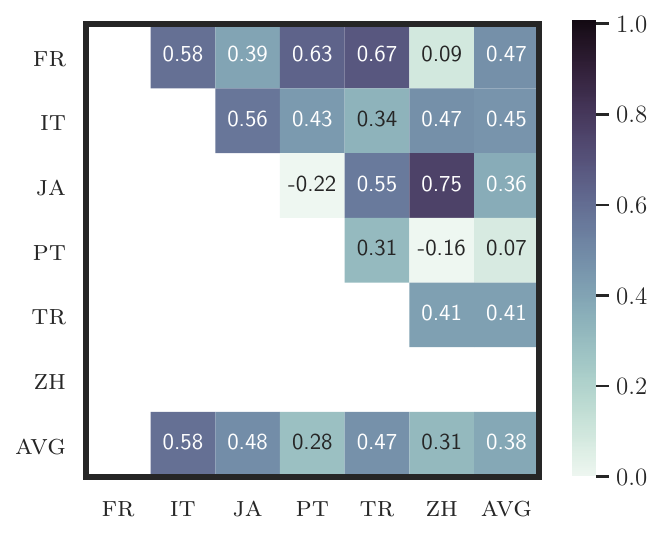}
        \caption{WM, $\varepsilon = 1$, $l = 0$}
        \label{fig:pos_rsa_wikimatrix_1_lay0}
    \end{subfigure}
    \begin{subfigure}[b]{0.21\textwidth}
        \centering
        \includegraphics[width=\textwidth]{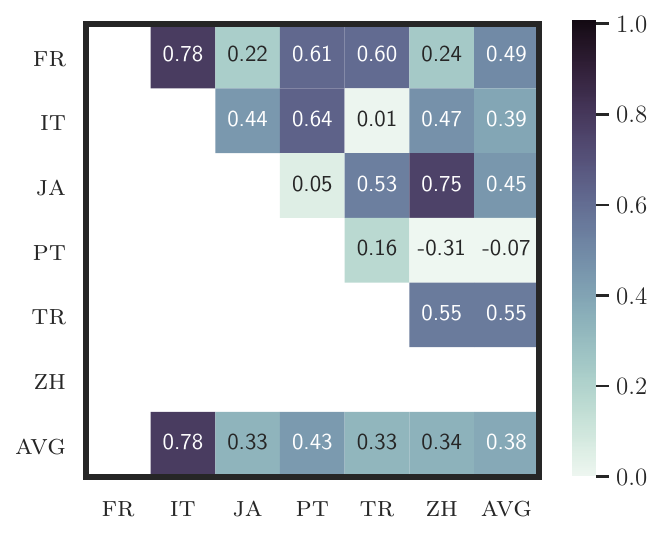}
        \caption{WM, $\varepsilon = 1$, $l = 8$}
        \label{fig:pos_rsa_wikimatrix_1_lay8}
    \end{subfigure}
    \begin{subfigure}[b]{0.21\textwidth}
        \centering
        \includegraphics[width=\textwidth]{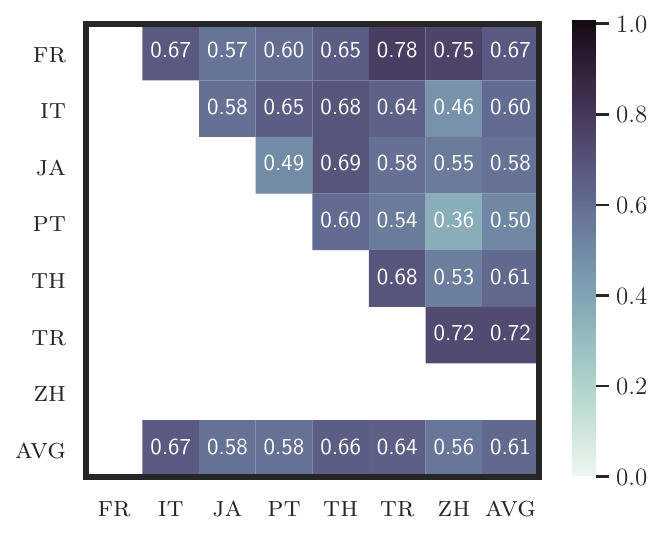}
        \caption{TED, $\varepsilon = 3$, $l = 0$}
        \label{fig:pos_rsa_ted2020_3_lay0}
    \end{subfigure}
    \begin{subfigure}[b]{0.21\textwidth}
        \centering
        \includegraphics[width=\textwidth]{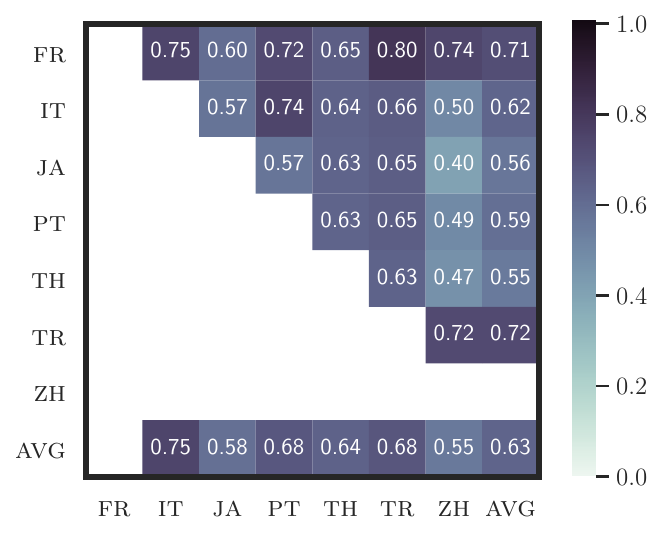}
        \caption{TED, $\varepsilon = 3$, $l = 8$}
        \label{fig:pos_rsa_ted2020_3_lay8}
    \end{subfigure}
        \begin{subfigure}[b]{0.21\textwidth}
        \centering
        \includegraphics[width=\textwidth]{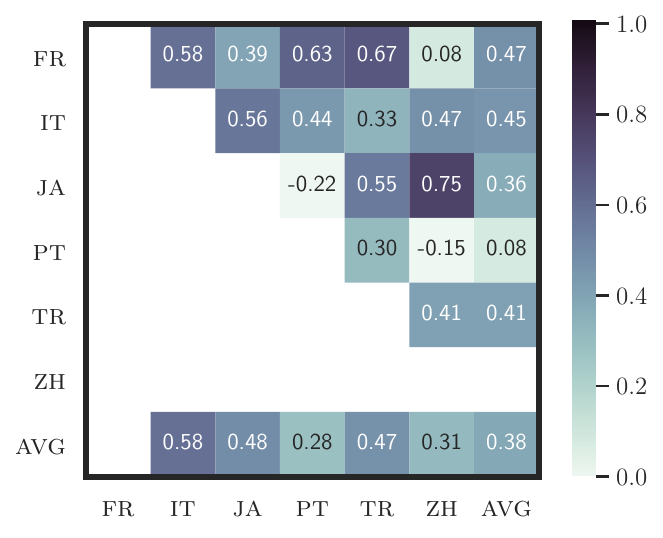}
        \caption{WM, $\varepsilon = 3$, $l = 0$}
        \label{fig:pos_rsa_wikimatrix_3_lay0}
    \end{subfigure}
    \begin{subfigure}[b]{0.21\textwidth}
        \centering
        \includegraphics[width=\textwidth]{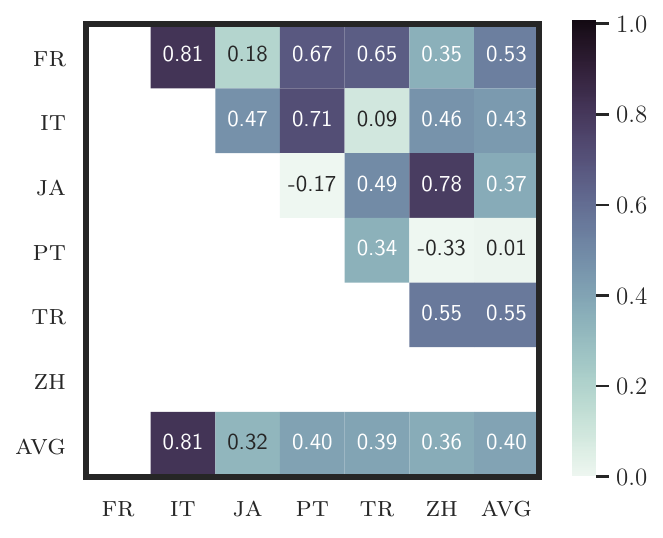}
        \caption{WM, $\varepsilon = 3$, $l = 8$}
        \label{fig:pos_rsa_wikimatrix_3_lay8}
    \end{subfigure}
    \begin{subfigure}[b]{0.21\textwidth}
        \centering
        \includegraphics[width=\textwidth]{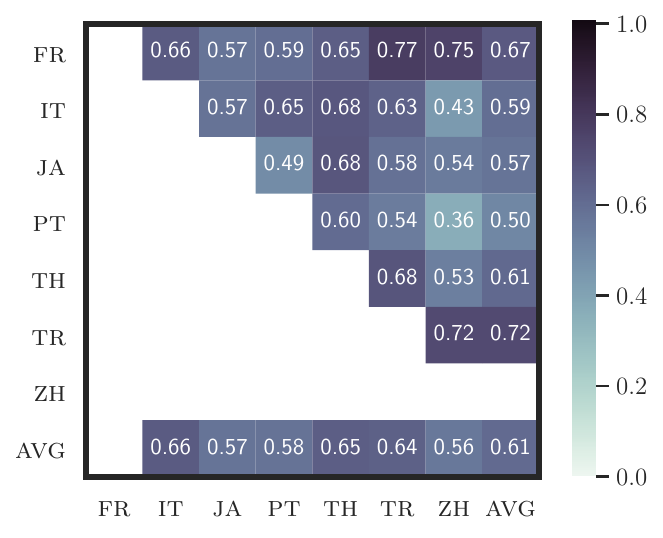}
        \caption{TED, $\varepsilon = 8$, $l = 0$}
        \label{fig:pos_rsa_ted2020_8_lay0}
    \end{subfigure}
    \begin{subfigure}[b]{0.21\textwidth}
        \centering
        \includegraphics[width=\textwidth]{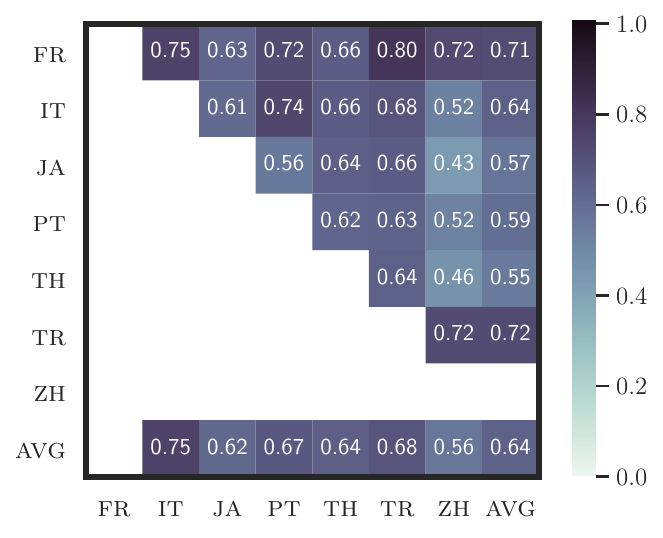}
        \caption{TED, $\varepsilon = 8$, $l = 8$}
        \label{fig:pos_rsa_ted2020_8_lay8}
    \end{subfigure}
        \begin{subfigure}[b]{0.21\textwidth}
        \centering
        \includegraphics[width=\textwidth]{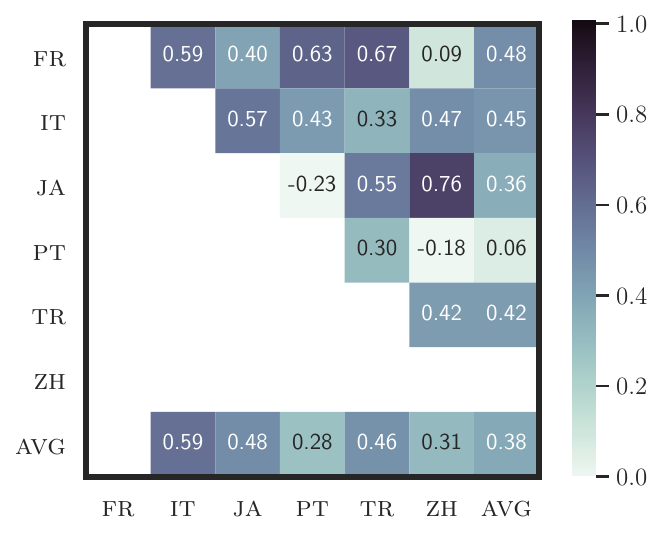}
        \caption{WM, $\varepsilon = 8$, $l = 0$}
        \label{fig:pos_rsa_wikimatrix_8_lay0}
    \end{subfigure}
    \begin{subfigure}[b]{0.21\textwidth}
        \centering
        \includegraphics[width=\textwidth]{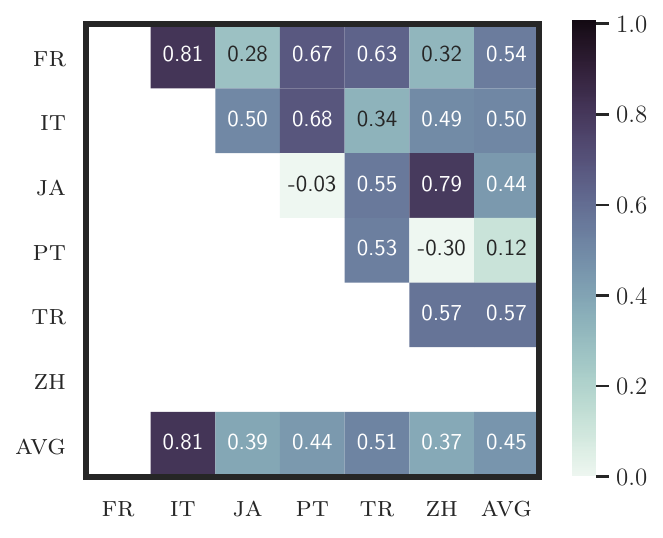}
        \caption{WM, $\varepsilon = 8$, $l = 8$}
        \label{fig:pos_rsa_wikimatrix_8_lay8}
    \end{subfigure}
    \begin{subfigure}[b]{0.21\textwidth}
        \centering
        \includegraphics[width=\textwidth]{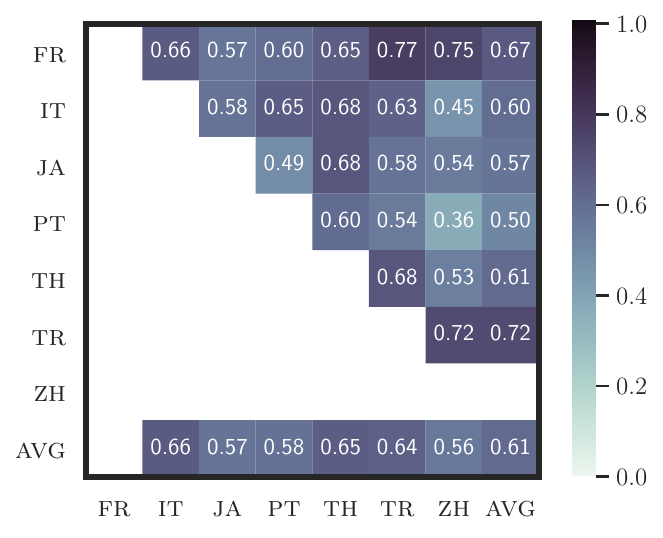}
        \caption{TED, $\varepsilon = 15$, $l = 0$}
        \label{fig:pos_rsa_ted2020_15_lay0}
    \end{subfigure}
    \begin{subfigure}[b]{0.21\textwidth}
        \centering
        \includegraphics[width=\textwidth]{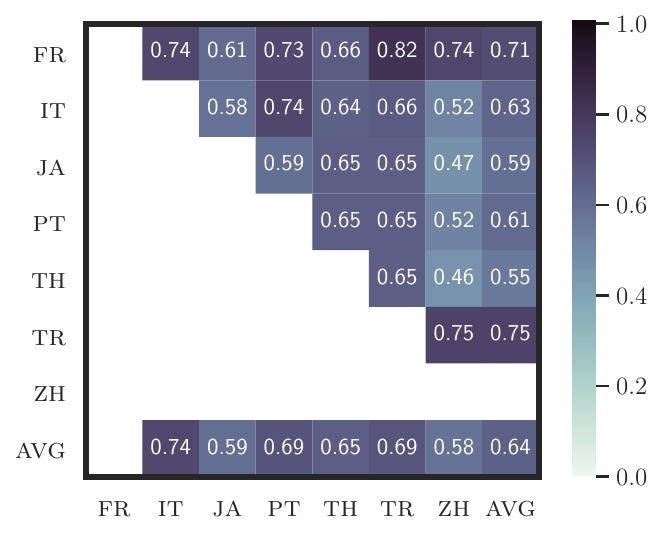}
        \caption{TED, $\varepsilon = 15$, $l = 8$}
        \label{fig:pos_rsa_ted2020_15_lay8}
    \end{subfigure}
        \begin{subfigure}[b]{0.21\textwidth}
        \centering
        \includegraphics[width=\textwidth]{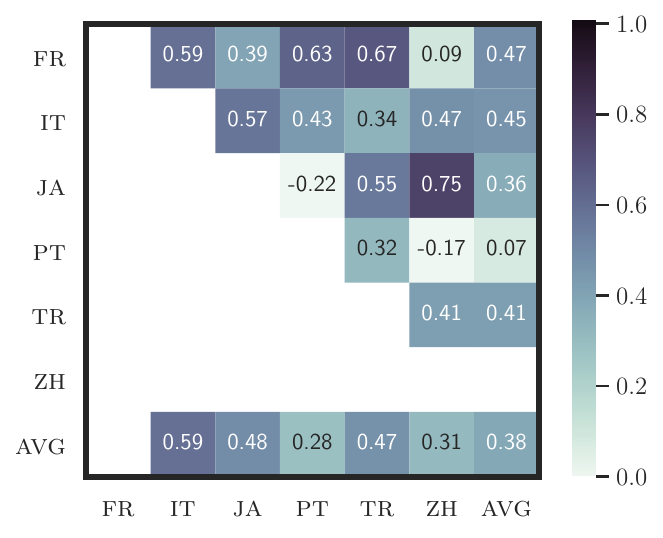}
        \caption{WM, $\varepsilon = 15$, $l = 0$}
        \label{fig:pos_rsa_wikimatrix_15_lay0}
    \end{subfigure}
    \begin{subfigure}[b]{0.21\textwidth}
        \centering
        \includegraphics[width=\textwidth]{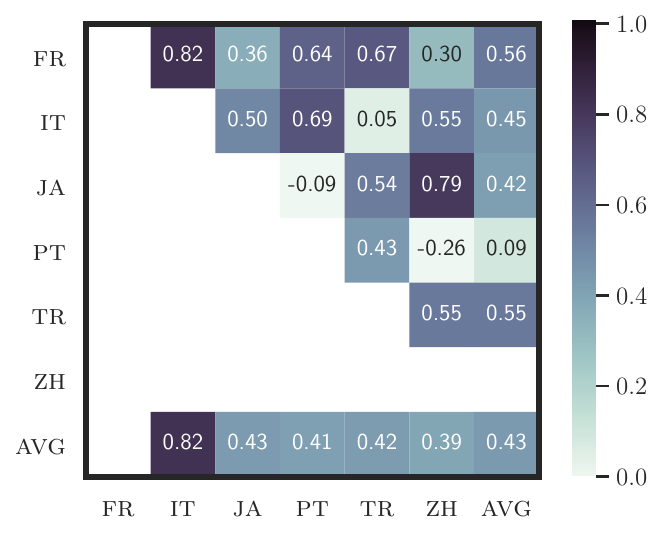}
        \caption{WM, $\varepsilon = 15$, $l = 8$}
        \label{fig:pos_rsa_wikimatrix_15_lay8}
    \end{subfigure}
    \begin{subfigure}[b]{0.21\textwidth}
        \centering
        \includegraphics[width=\textwidth]{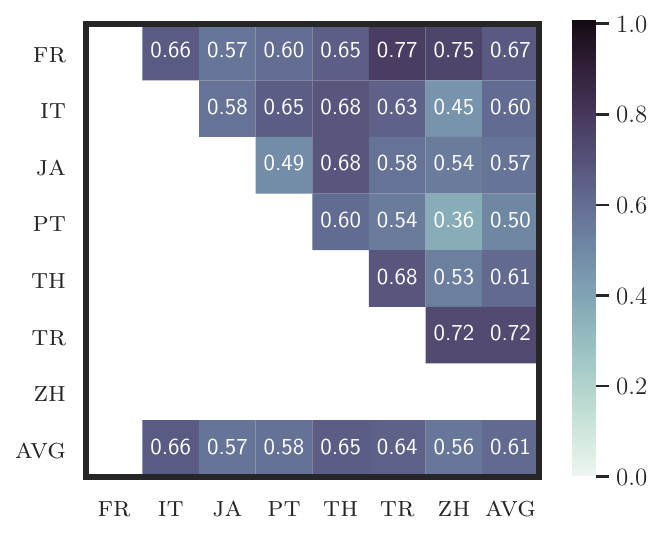}
        \caption{TED, $\varepsilon = 30$, $l = 0$}
        \label{fig:pos_rsa_ted2020_30_lay0}
    \end{subfigure}
    \begin{subfigure}[b]{0.21\textwidth}
        \centering
        \includegraphics[width=\textwidth]{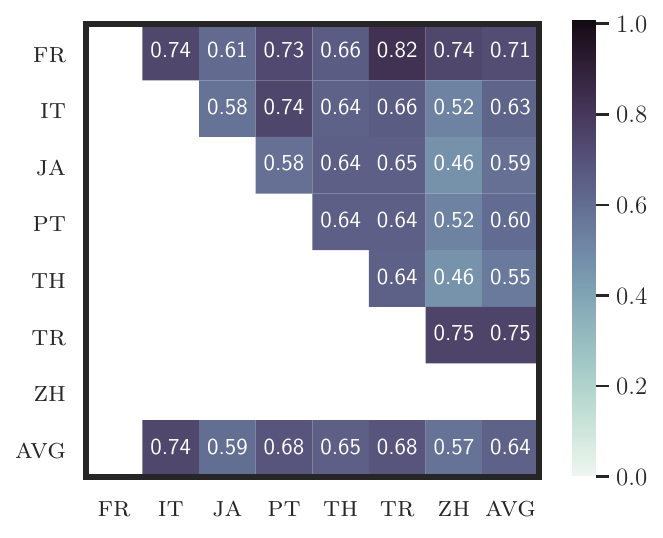}
        \caption{TED, $\varepsilon = 30$, $l = 8$}
        \label{fig:pos_rsa_ted2020_30_lay8}
    \end{subfigure}
        \begin{subfigure}[b]{0.21\textwidth}
        \centering
        \includegraphics[width=\textwidth]{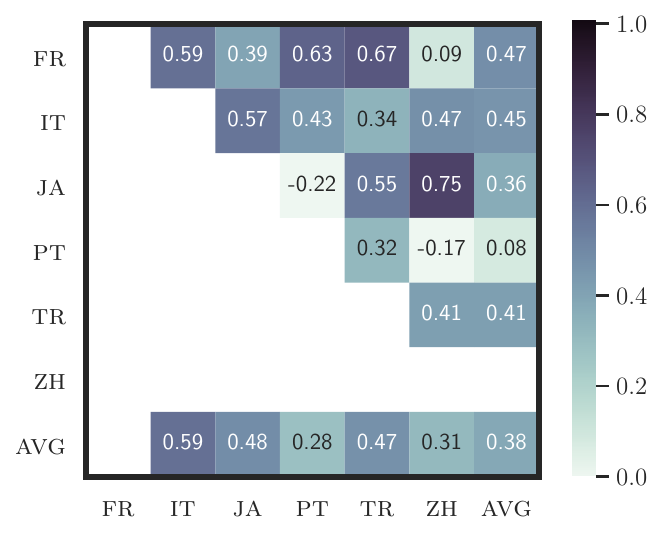}
        \caption{WM, $\varepsilon = 30$, $l = 0$}
        \label{fig:pos_rsa_wikimatrix_30_lay0}
    \end{subfigure}
    \begin{subfigure}[b]{0.21\textwidth}
        \centering
        \includegraphics[width=\textwidth]{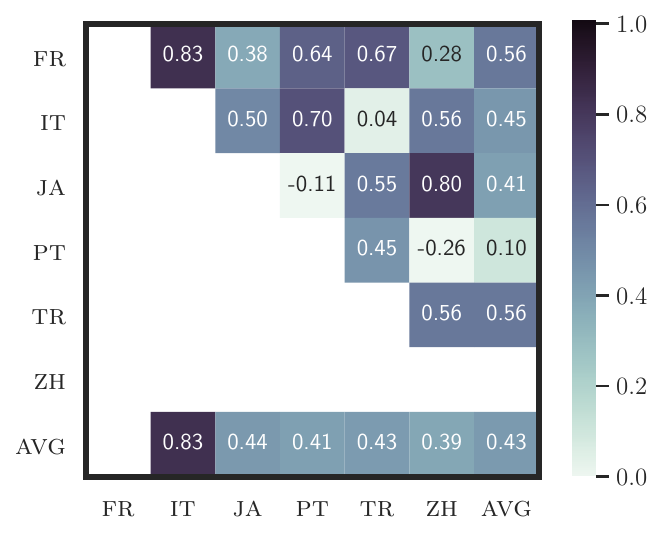}
        \caption{WM, $\varepsilon = 30$, $l = 8$}
        \label{fig:pos_rsa_wikimatrix_30_lay8}
    \end{subfigure}
    \begin{subfigure}[b]{0.21\textwidth}
        \centering
        \includegraphics[width=\textwidth]{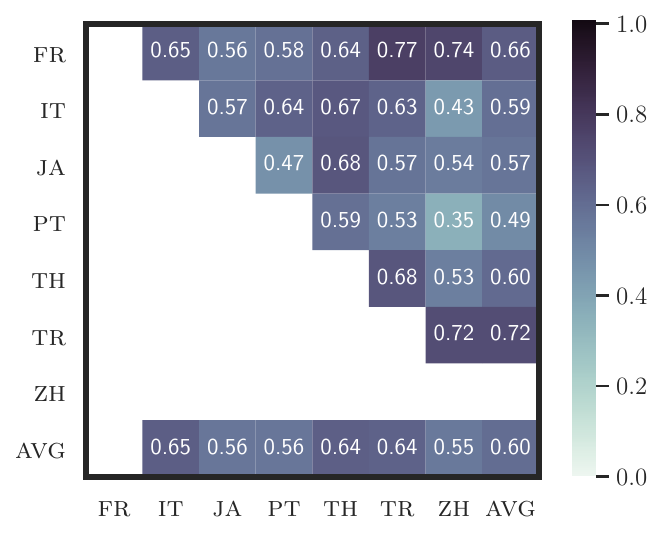}
        \caption{TED, $\varepsilon = \infty$, $l = 0$}
        \label{fig:pos_rsa_ted2020_inf_lay0}
    \end{subfigure}
    \begin{subfigure}[b]{0.21\textwidth}
        \centering
        \includegraphics[width=\textwidth]{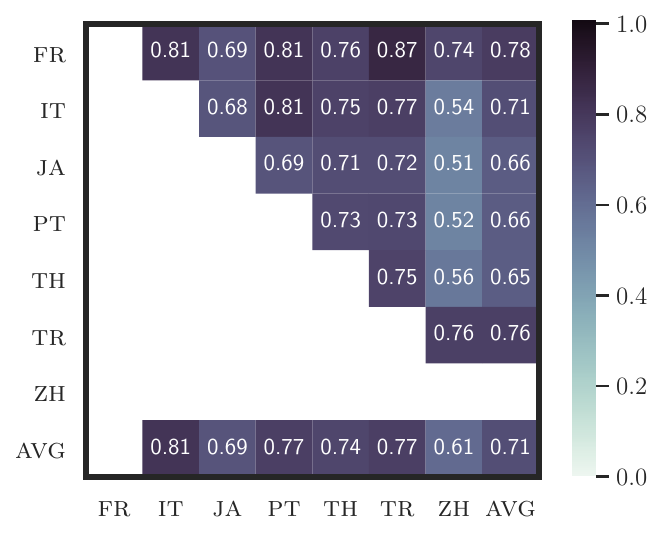}
        \caption{TED, $\varepsilon = \infty$, $l = 8$}
        \label{fig:pos_rsa_ted2020_inf_lay8}
    \end{subfigure}
        \begin{subfigure}[b]{0.21\textwidth}
        \centering
        \includegraphics[width=\textwidth]{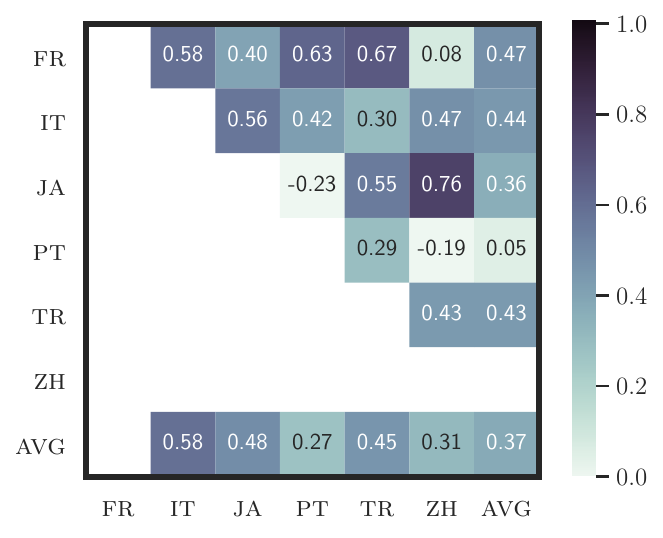}
        \caption{WM, $\varepsilon = \infty$, $l = 0$}
        \label{fig:pos_rsa_wikimatrix_inf_lay0}
    \end{subfigure}
    \begin{subfigure}[b]{0.21\textwidth}
        \centering
        \includegraphics[width=\textwidth]{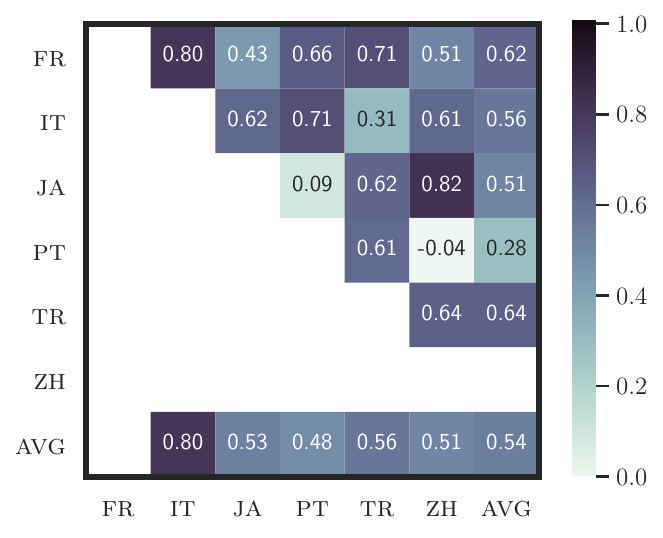}
        \caption{WM, $\varepsilon = \infty$, $l = 8$}
        \label{fig:pos_rsa_wikimatrix_inf_lay8}
    \end{subfigure}
    \caption{\textbf{POS} RSA results for the TED 2020 (TED) and WikiMatrix (WM) datasets and different combinations of privacy budgets ($\varepsilon$) and layers ($l$). Each heatmap cell corresponds to the average over 5 random seeds. We observe that the overall patterns are highly similar across all levels of privacy, particularly at layer 0.}
    \label{fig:pos_rsa_tedwm}
\end{figure*}

\begin{figure*}[ht!]
    \centering
    \begin{subfigure}[b]{0.21\textwidth}
        \centering
        \includegraphics[width=\textwidth]{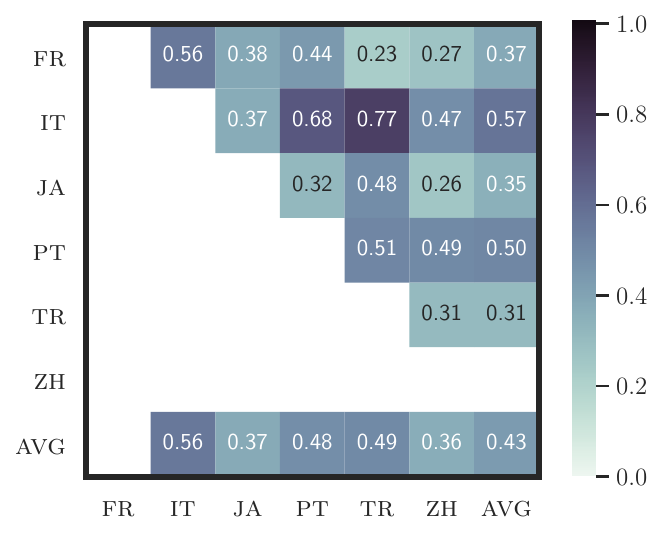}
        \caption{$\varepsilon = 1$, $l = 0$}
        \label{fig:pos_rsa_tatoeba_1_lay0}
    \end{subfigure}
    \begin{subfigure}[b]{0.21\textwidth}
        \centering
        \includegraphics[width=\textwidth]{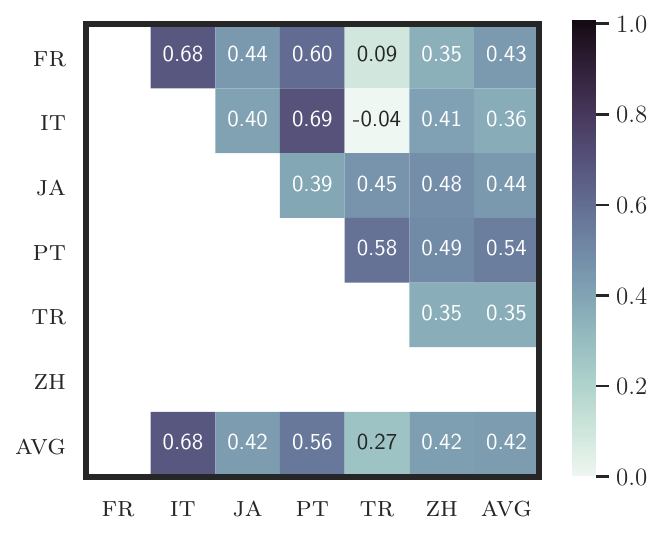}
        \caption{$\varepsilon = 1$, $l = 8$}
        \label{fig:pos_rsa_tatoeba_1_lay8}
    \end{subfigure}
    \hspace{0.21\textwidth}
    \hspace{0.21\textwidth}
    \begin{subfigure}[b]{0.21\textwidth}
        \centering
        \includegraphics[width=\textwidth]{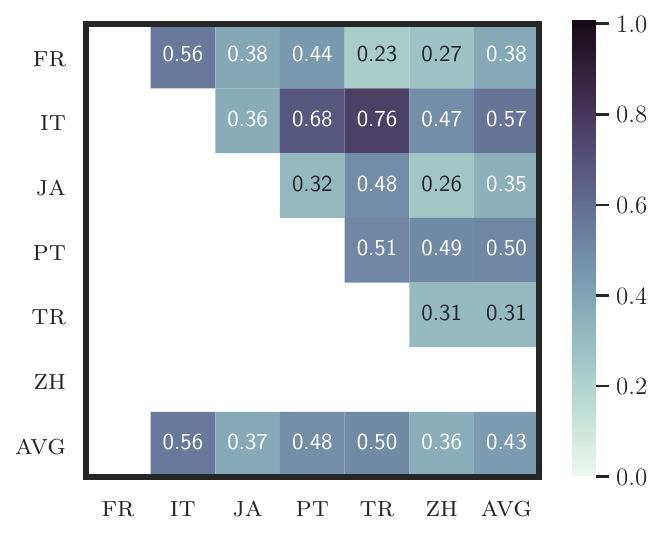}
        \caption{$\varepsilon = 3$, $l = 0$}
        \label{fig:pos_rsa_tatoeba_3_lay0}
    \end{subfigure}
    \begin{subfigure}[b]{0.21\textwidth}
        \centering
        \includegraphics[width=\textwidth]{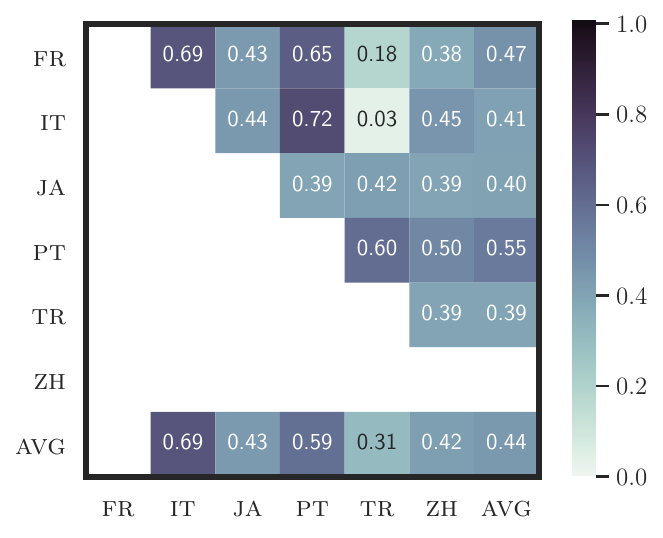}
        \caption{$\varepsilon = 3$, $l = 8$}
        \label{fig:pos_rsa_tatoeba_3_lay8}
    \end{subfigure}
    \hspace{0.21\textwidth}
    \hspace{0.21\textwidth}
    \begin{subfigure}[b]{0.21\textwidth}
        \centering
        \includegraphics[width=\textwidth]{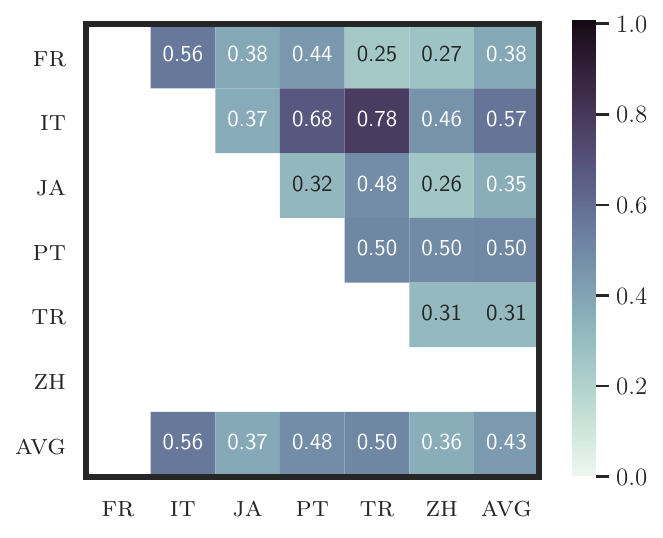}
        \caption{$\varepsilon = 8$, $l = 0$}
        \label{fig:pos_rsa_tatoeba_8_lay0}
    \end{subfigure}
    \begin{subfigure}[b]{0.21\textwidth}
        \centering
        \includegraphics[width=\textwidth]{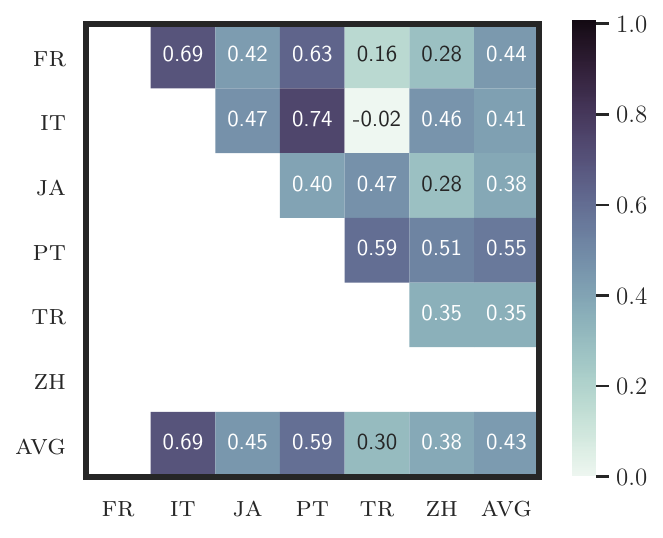}
        \caption{$\varepsilon = 8$, $l = 8$}
        \label{fig:pos_rsa_tatoeba_8_lay8}
    \end{subfigure}
    \hspace{0.21\textwidth}
    \hspace{0.21\textwidth}
    \begin{subfigure}[b]{0.21\textwidth}
        \centering
        \includegraphics[width=\textwidth]{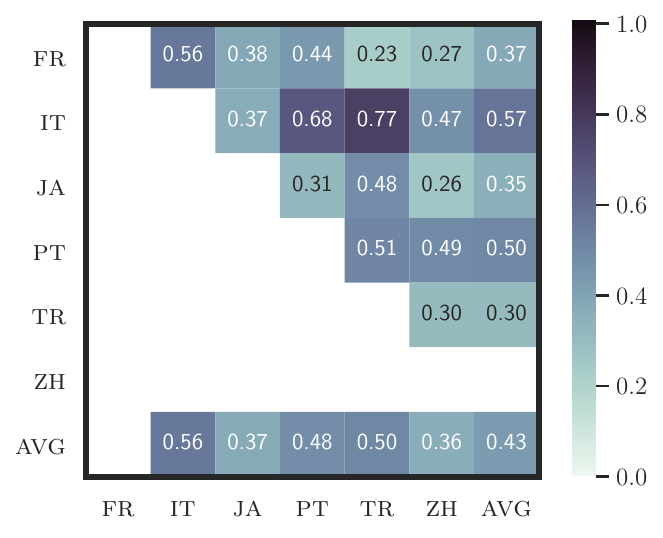}
        \caption{$\varepsilon = 15$, $l = 0$}
        \label{fig:pos_rsa_tatoeba_15_lay0}
    \end{subfigure}
    \begin{subfigure}[b]{0.21\textwidth}
        \centering
        \includegraphics[width=\textwidth]{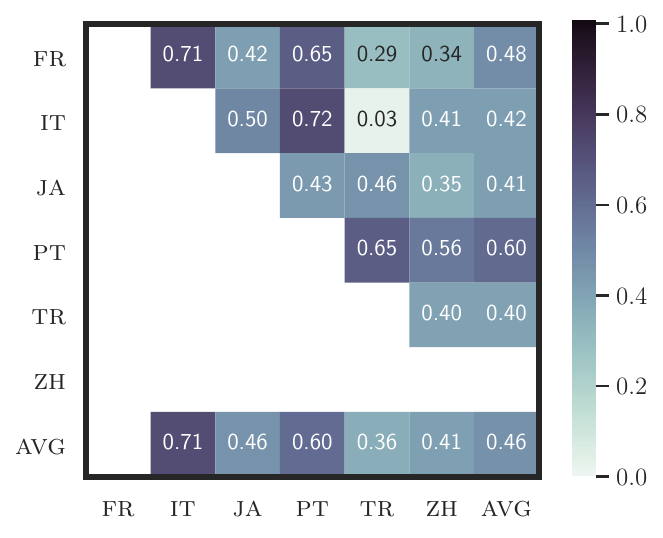}
        \caption{$\varepsilon = 15$, $l = 8$}
        \label{fig:pos_rsa_tatoeba_15_lay8}
    \end{subfigure}
    \hspace{0.21\textwidth}
    \hspace{0.21\textwidth}
    \begin{subfigure}[b]{0.21\textwidth}
        \centering
        \includegraphics[width=\textwidth]{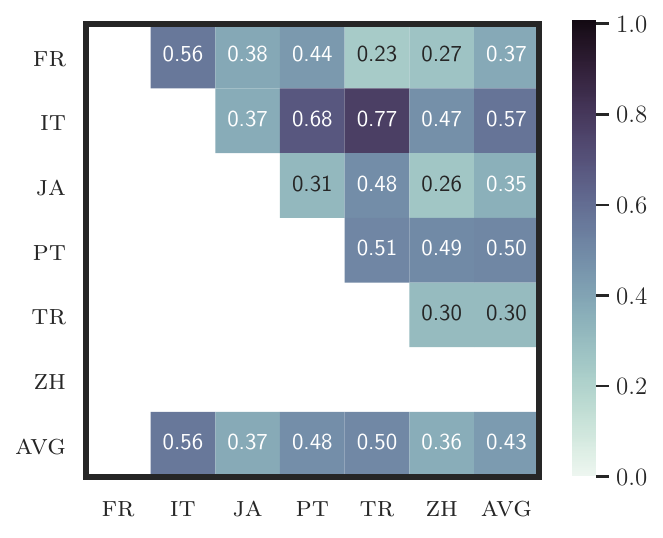}
        \caption{$\varepsilon = 30$, $l = 0$}
        \label{fig:pos_rsa_tatoeba_30_lay0}
    \end{subfigure}
    \begin{subfigure}[b]{0.21\textwidth}
        \centering
        \includegraphics[width=\textwidth]{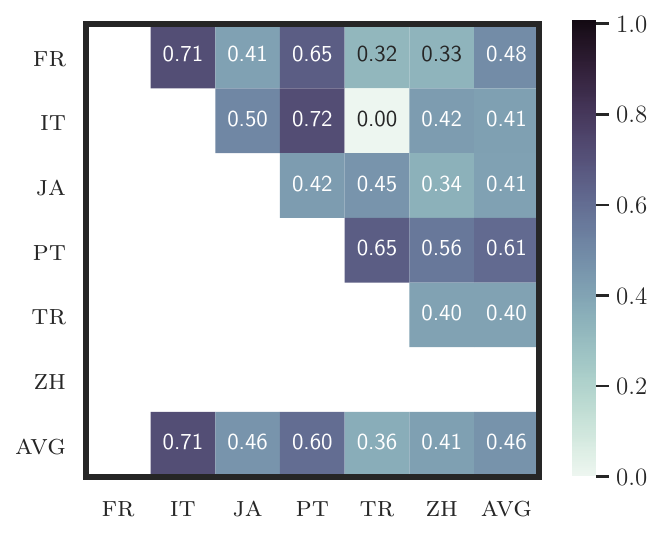}
        \caption{$\varepsilon = 30$, $l = 8$}
        \label{fig:pos_rsa_tatoeba_30_lay8}
    \end{subfigure}
    \hspace{0.21\textwidth}
    \hspace{0.21\textwidth}
    \begin{subfigure}[b]{0.21\textwidth}
        \centering
        \includegraphics[width=\textwidth]{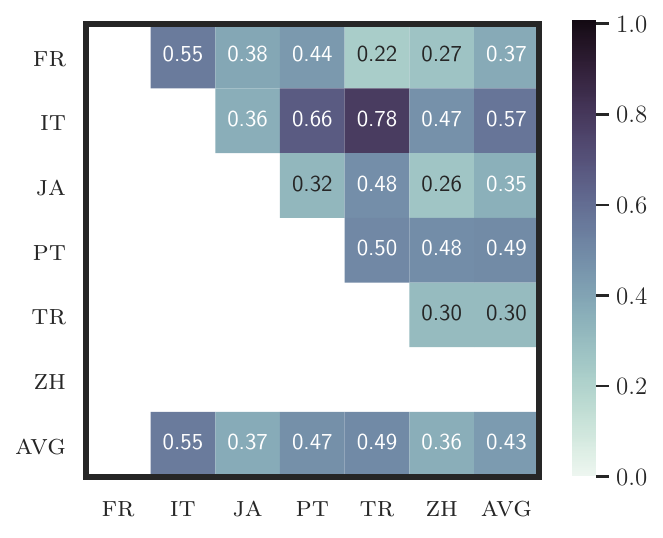}
        \caption{$\varepsilon = \infty$, $l = 0$}
        \label{fig:pos_rsa_tatoeba_inf_lay0}
    \end{subfigure}
    \begin{subfigure}[b]{0.21\textwidth}
        \centering
        \includegraphics[width=\textwidth]{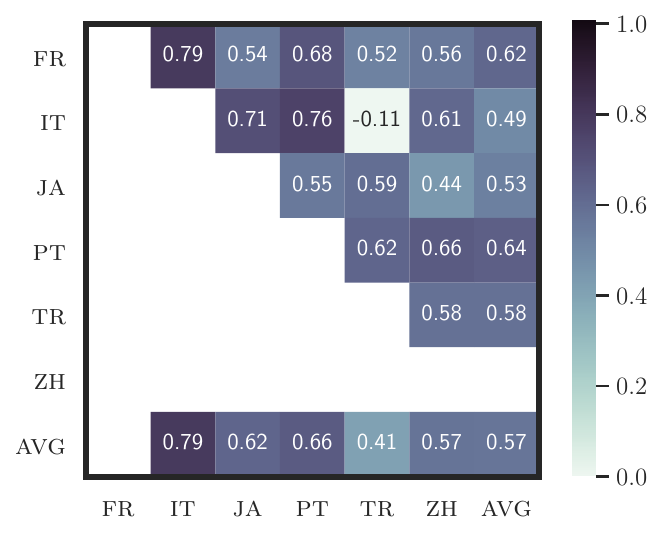}
        \caption{$\varepsilon = \infty$, $l = 8$}
        \label{fig:pos_rsa_tatoeba_inf_lay8}
    \end{subfigure}
    \caption{\textbf{POS} RSA results for the Tatoeba dataset and different combinations of privacy budgets ($\varepsilon$) and layers ($l$). Each heatmap cell corresponds to the average over 5 random seeds. We observe that the overall patterns are highly similar across all levels of privacy, particularly at layer 0. Also note that, unlike in CKA (Figure~\ref{fig:pos_cka_tatoeba}), the similarity between \textsc{it} and \textsc{tr} is high at layer 0 but low at layer 8.}
    \label{fig:pos_rsa_tatoeba}
\end{figure*}

\begin{figure*}[ht!]
    \centering
    \begin{subfigure}[b]{0.21\textwidth}
        \centering
        \includegraphics[width=\textwidth]{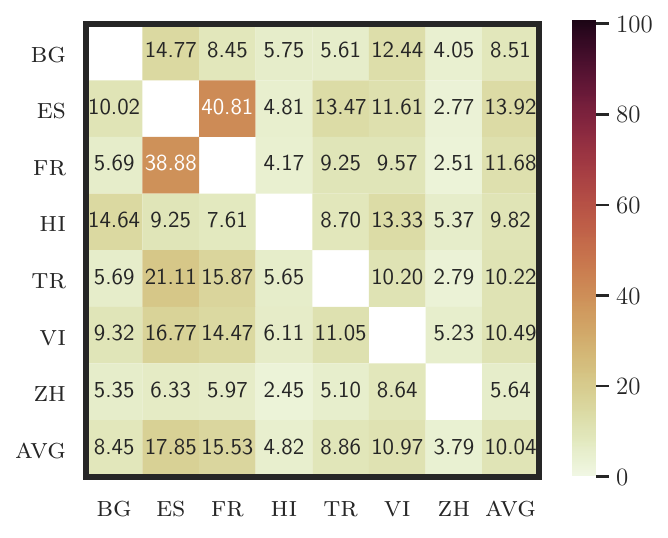}
        \caption{TED, $\varepsilon = 1$, $l = 0$}
        \label{fig:retrieval_ted2020_1_lay0}
    \end{subfigure}
    \begin{subfigure}[b]{0.21\textwidth}
        \centering
        \includegraphics[width=\textwidth]{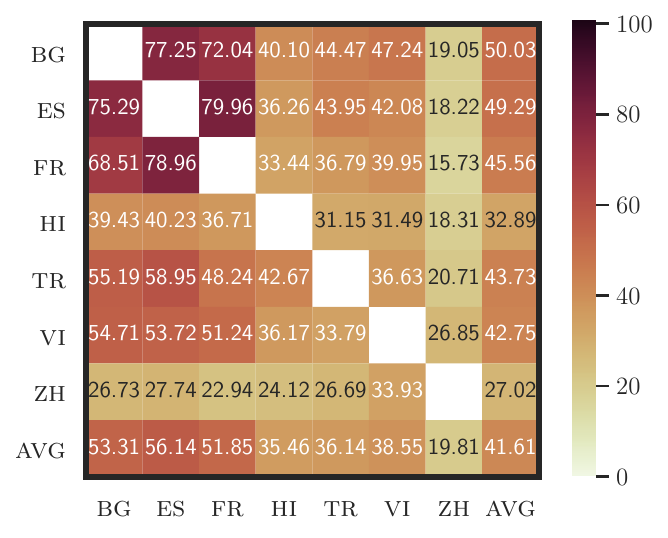}
        \caption{TED, $\varepsilon = 1$, $l = 8$}
        \label{fig:retrieval_ted2020_1_lay8}
    \end{subfigure}
        \begin{subfigure}[b]{0.21\textwidth}
        \centering
        \includegraphics[width=\textwidth]{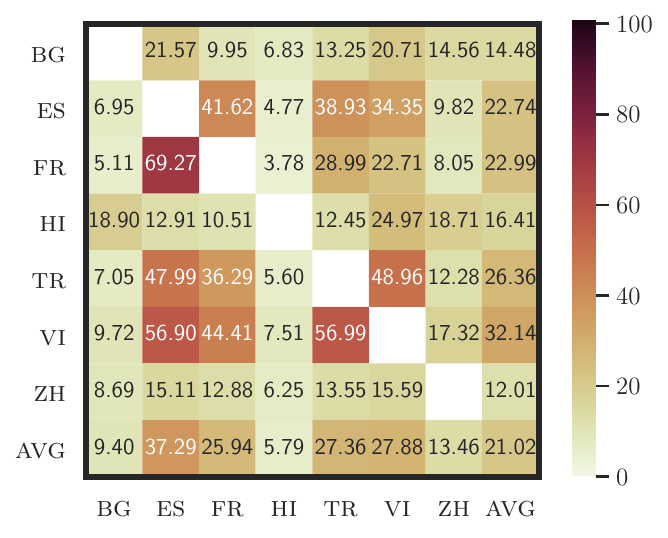}
        \caption{WM, $\varepsilon = 1$, $l = 0$}
        \label{fig:retrieval_wikimatrix_1_lay0}
    \end{subfigure}
    \begin{subfigure}[b]{0.21\textwidth}
        \centering
        \includegraphics[width=\textwidth]{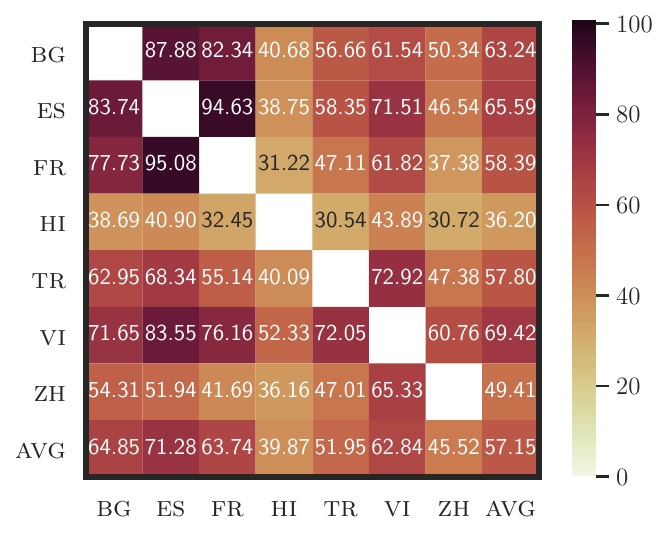}
        \caption{WM, $\varepsilon = 1$, $l = 8$}
        \label{fig:retrieval_wikimatrix_1_lay8}
    \end{subfigure}
    \begin{subfigure}[b]{0.21\textwidth}
        \centering
        \includegraphics[width=\textwidth]{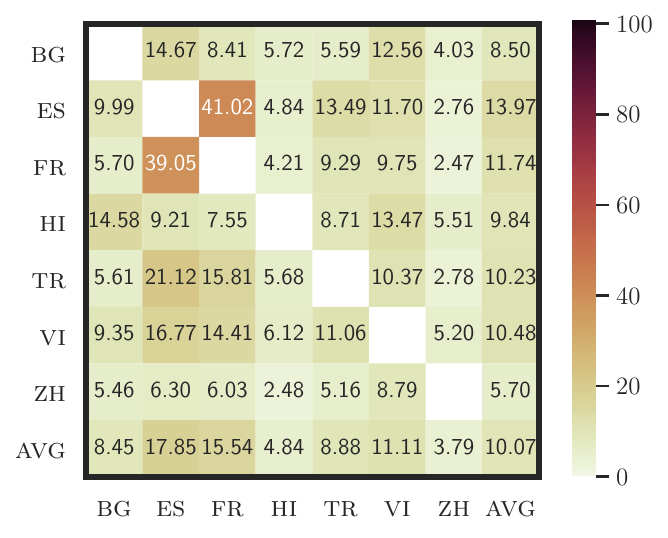}
        \caption{TED, $\varepsilon = 3$, $l = 0$}
        \label{fig:retrieval_ted2020_3_lay0}
    \end{subfigure}
    \begin{subfigure}[b]{0.21\textwidth}
        \centering
        \includegraphics[width=\textwidth]{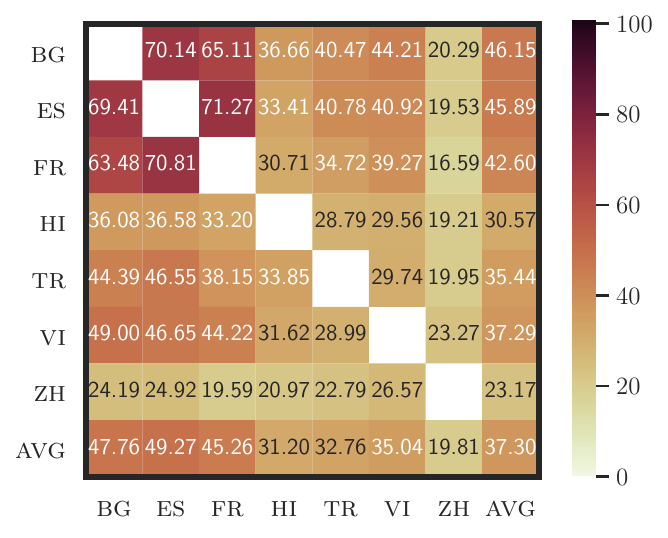}
        \caption{TED, $\varepsilon = 3$, $l = 8$}
        \label{fig:retrieval_ted2020_3_lay8}
    \end{subfigure}
        \begin{subfigure}[b]{0.21\textwidth}
        \centering
        \includegraphics[width=\textwidth]{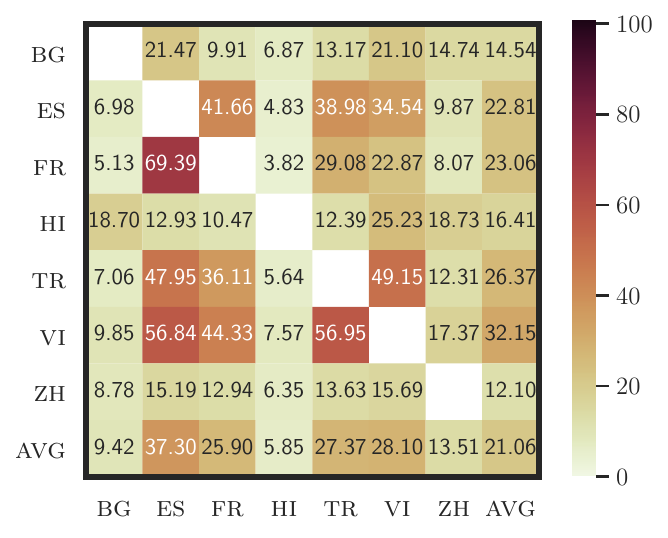}
        \caption{WM, $\varepsilon = 3$, $l = 0$}
        \label{fig:retrieval_wikimatrix_3_lay0}
    \end{subfigure}
    \begin{subfigure}[b]{0.21\textwidth}
        \centering
        \includegraphics[width=\textwidth]{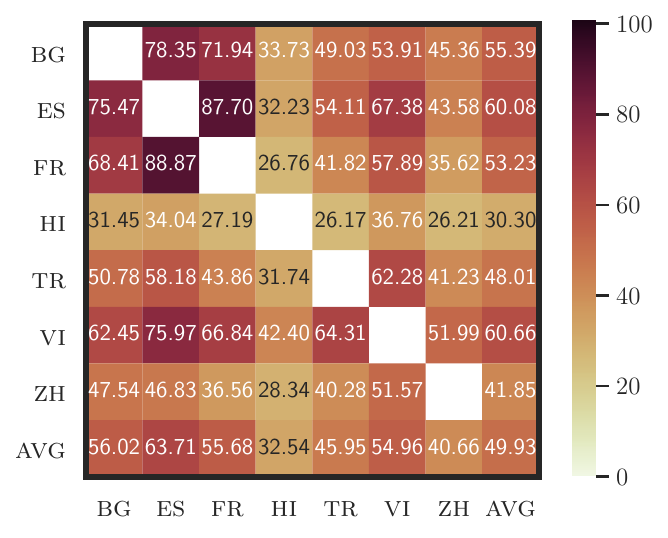}
        \caption{WM, $\varepsilon = 3$, $l = 8$}
        \label{fig:retrieval_wikimatrix_3_lay8}
    \end{subfigure}
    \begin{subfigure}[b]{0.21\textwidth}
        \centering
        \includegraphics[width=\textwidth]{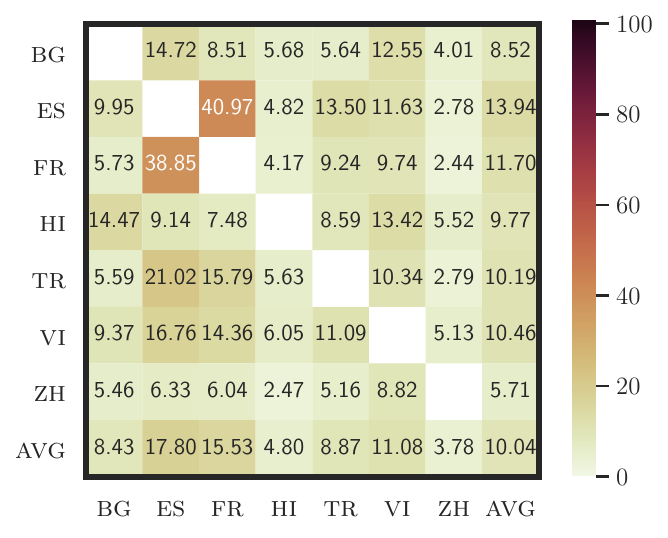}
        \caption{TED, $\varepsilon = 8$, $l = 0$}
        \label{fig:retrieval_ted2020_8_lay0}
    \end{subfigure}
    \begin{subfigure}[b]{0.21\textwidth}
        \centering
        \includegraphics[width=\textwidth]{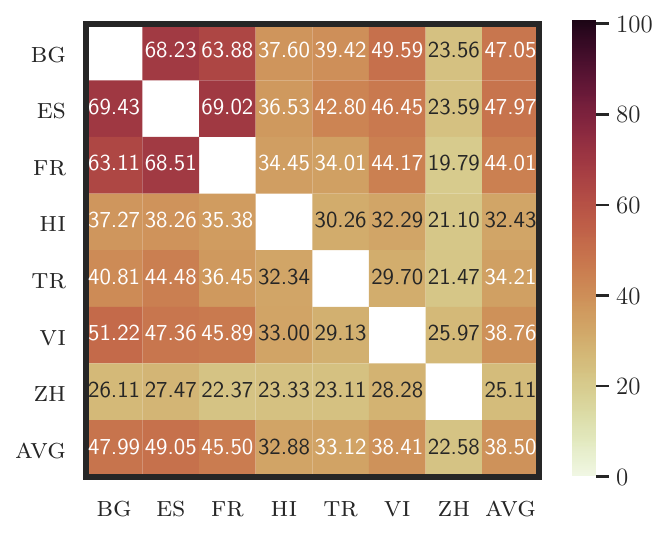}
        \caption{TED, $\varepsilon = 8$, $l = 8$}
        \label{fig:retrieval_ted2020_8_lay8}
    \end{subfigure}
        \begin{subfigure}[b]{0.21\textwidth}
        \centering
        \includegraphics[width=\textwidth]{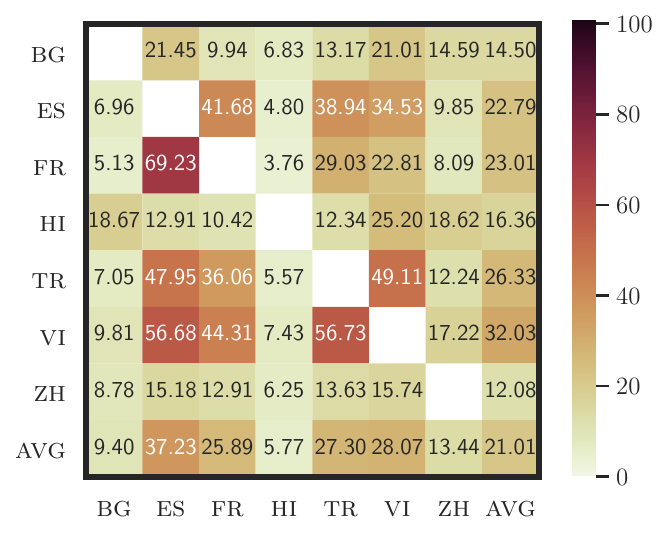}
        \caption{WM, $\varepsilon = 8$, $l = 0$}
        \label{fig:retrieval_wikimatrix_8_lay0}
    \end{subfigure}
    \begin{subfigure}[b]{0.21\textwidth}
        \centering
        \includegraphics[width=\textwidth]{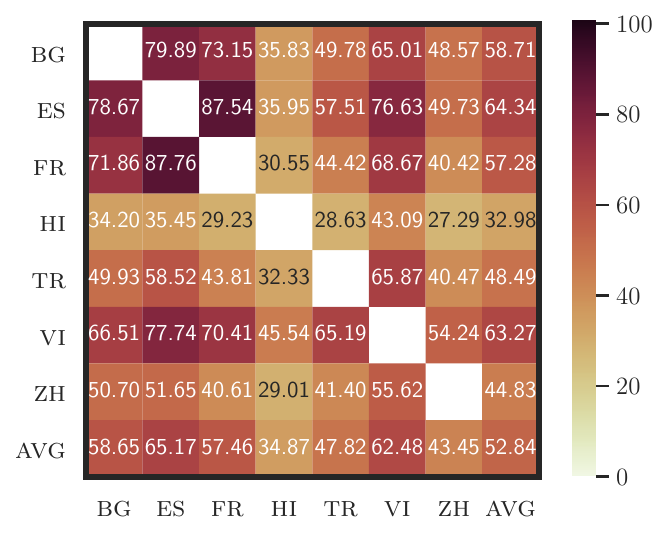}
        \caption{WM, $\varepsilon = 8$, $l = 8$}
        \label{fig:retrieval_wikimatrix_8_lay8}
    \end{subfigure}
    \begin{subfigure}[b]{0.21\textwidth}
        \centering
        \includegraphics[width=\textwidth]{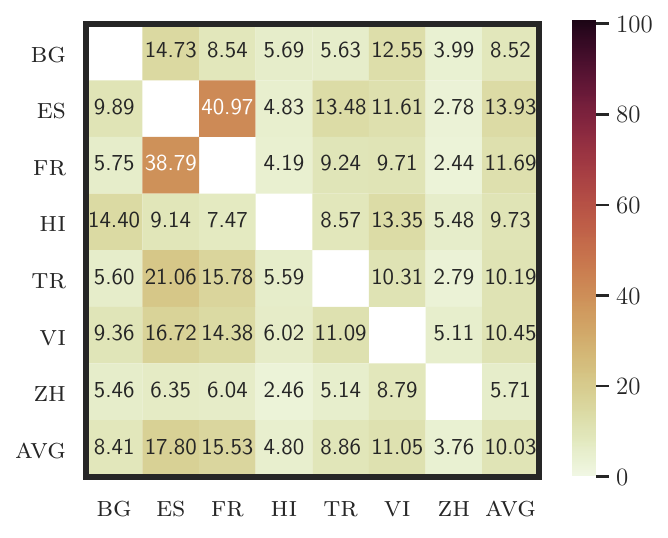}
        \caption{TED, $\varepsilon = 15$, $l = 0$}
        \label{fig:retrieval_ted2020_15_lay0}
    \end{subfigure}
    \begin{subfigure}[b]{0.21\textwidth}
        \centering
        \includegraphics[width=\textwidth]{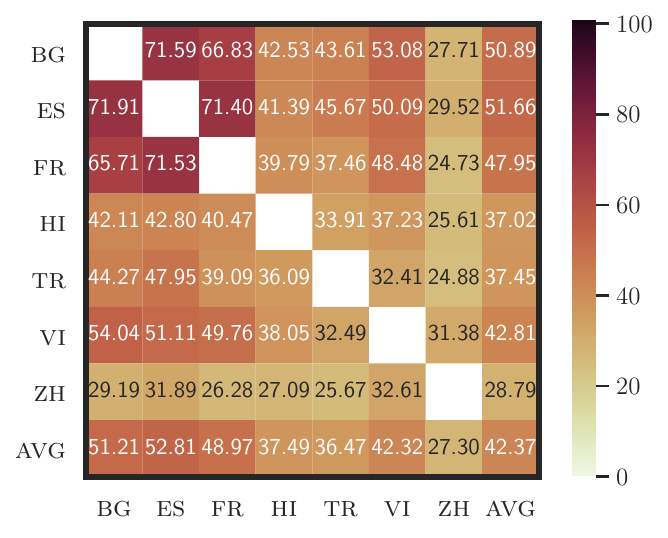}
        \caption{TED, $\varepsilon = 15$, $l = 8$}
        \label{fig:retrieval_ted2020_15_lay8}
    \end{subfigure}
        \begin{subfigure}[b]{0.21\textwidth}
        \centering
        \includegraphics[width=\textwidth]{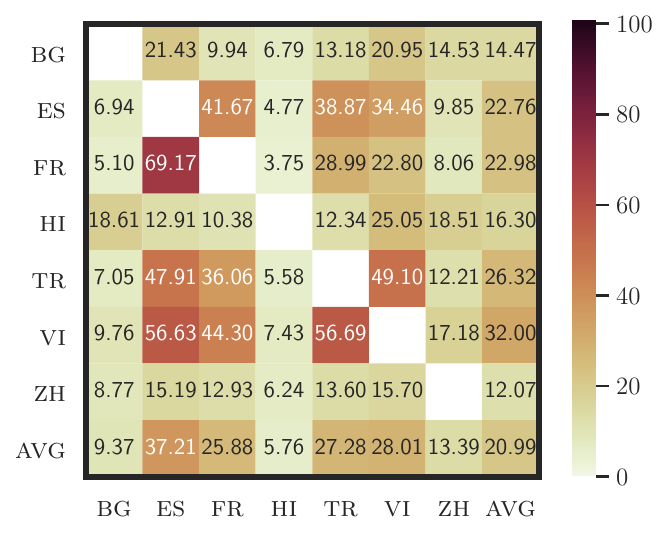}
        \caption{WM, $\varepsilon = 15$, $l = 0$}
        \label{fig:retrieval_wikimatrix_15_lay0}
    \end{subfigure}
    \begin{subfigure}[b]{0.21\textwidth}
        \centering
        \includegraphics[width=\textwidth]{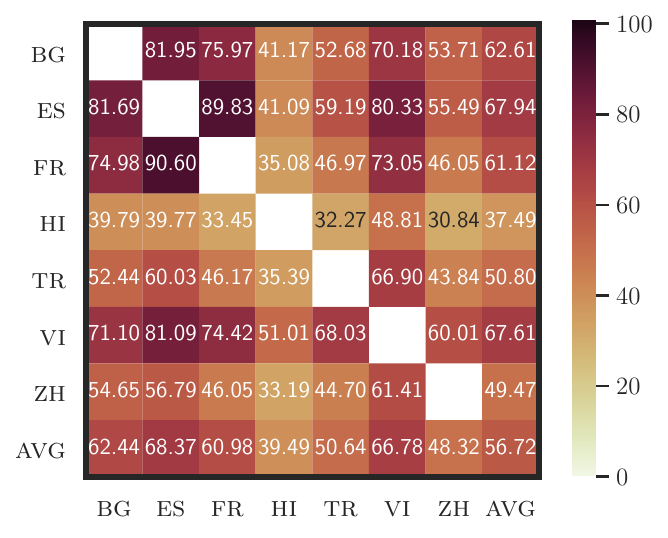}
        \caption{WM, $\varepsilon = 15$, $l = 8$}
        \label{fig:retrieval_wikimatrix_15_lay8}
    \end{subfigure}
    \begin{subfigure}[b]{0.21\textwidth}
        \centering
        \includegraphics[width=\textwidth]{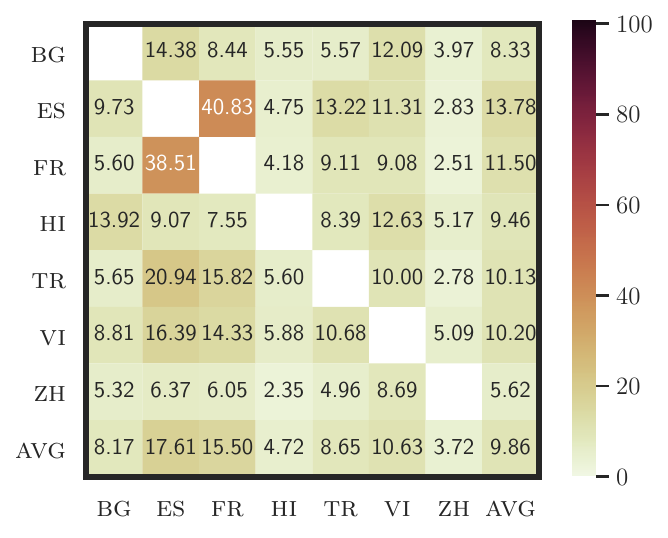}
        \caption{TED, $\varepsilon = 30$, $l = 0$}
        \label{fig:retrieval_ted2020_30_lay0}
    \end{subfigure}
    \begin{subfigure}[b]{0.21\textwidth}
        \centering
        \includegraphics[width=\textwidth]{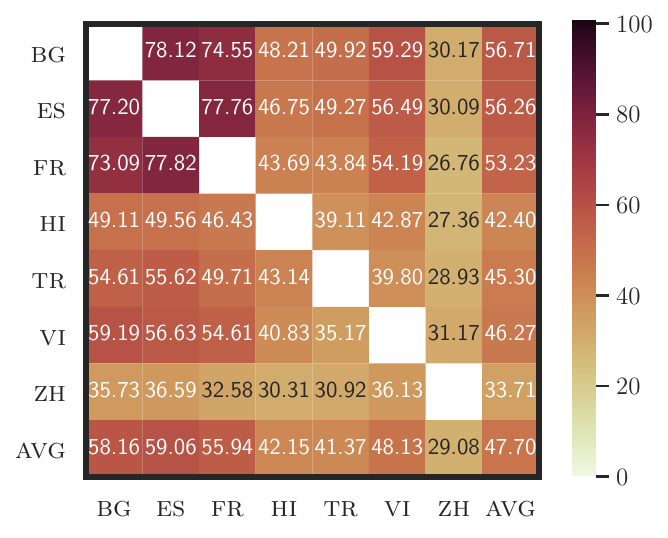}
        \caption{TED, $\varepsilon = 30$, $l = 8$}
        \label{fig:retrieval_ted2020_30_lay8}
    \end{subfigure}
        \begin{subfigure}[b]{0.21\textwidth}
        \centering
        \includegraphics[width=\textwidth]{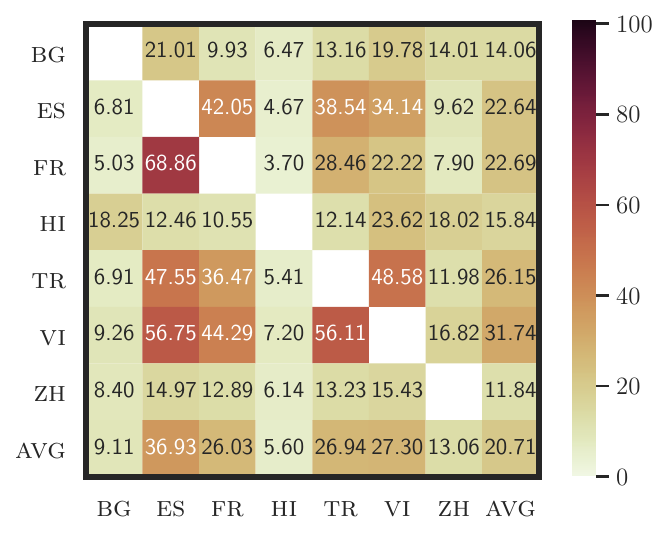}
        \caption{WM, $\varepsilon = 30$, $l = 0$}
        \label{fig:retrieval_wikimatrix_30_lay0}
    \end{subfigure}
    \begin{subfigure}[b]{0.21\textwidth}
        \centering
        \includegraphics[width=\textwidth]{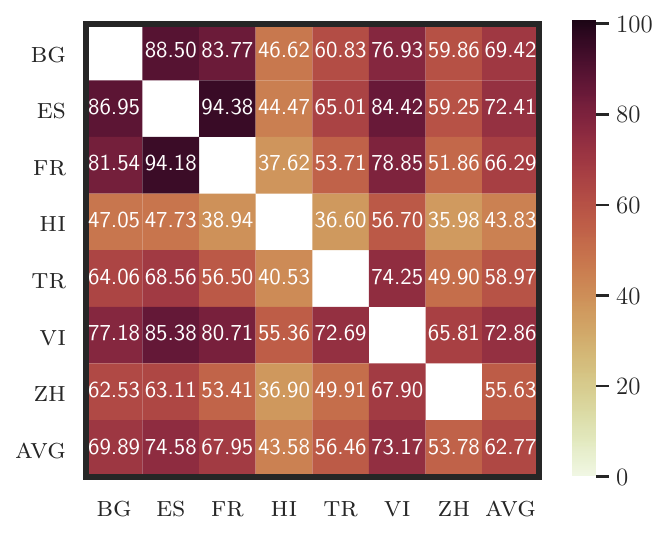}
        \caption{WM, $\varepsilon = 30$, $l = 8$}
        \label{fig:retrieval_wikimatrix_30_lay8}
    \end{subfigure}
    \begin{subfigure}[b]{0.21\textwidth}
        \centering
        \includegraphics[width=\textwidth]{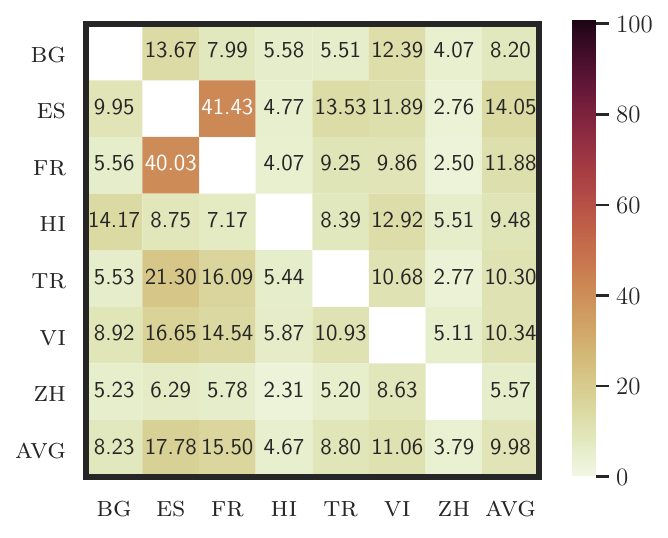}
        \caption{TED, $\varepsilon = \infty$, $l = 0$}
        \label{fig:retrieval_ted2020_inf_lay0}
    \end{subfigure}
    \begin{subfigure}[b]{0.21\textwidth}
        \centering
        \includegraphics[width=\textwidth]{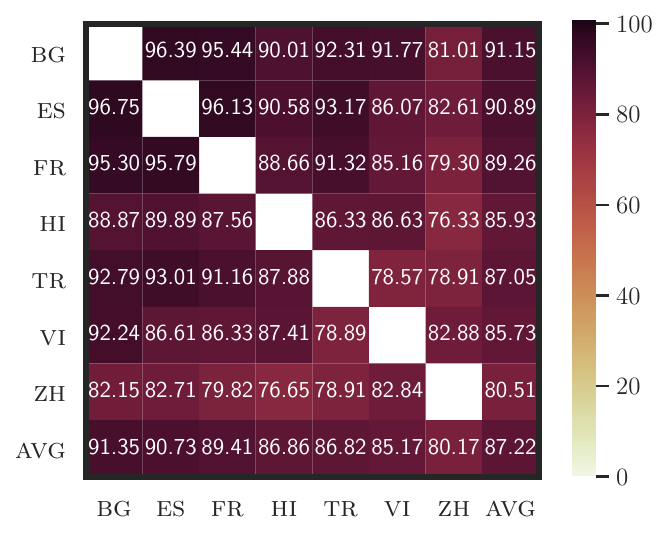}
        \caption{TED, $\varepsilon = \infty$, $l = 8$}
        \label{fig:retrieval_ted2020_inf_lay8}
    \end{subfigure}
        \begin{subfigure}[b]{0.21\textwidth}
        \centering
        \includegraphics[width=\textwidth]{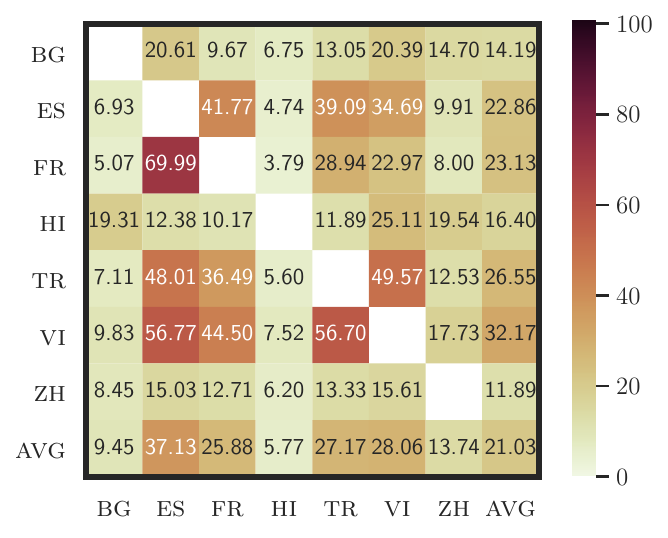}
        \caption{WM, $\varepsilon = \infty$, $l = 0$}
        \label{fig:retrieval_wikimatrix_inf_lay0}
    \end{subfigure}
    \begin{subfigure}[b]{0.21\textwidth}
        \centering
        \includegraphics[width=\textwidth]{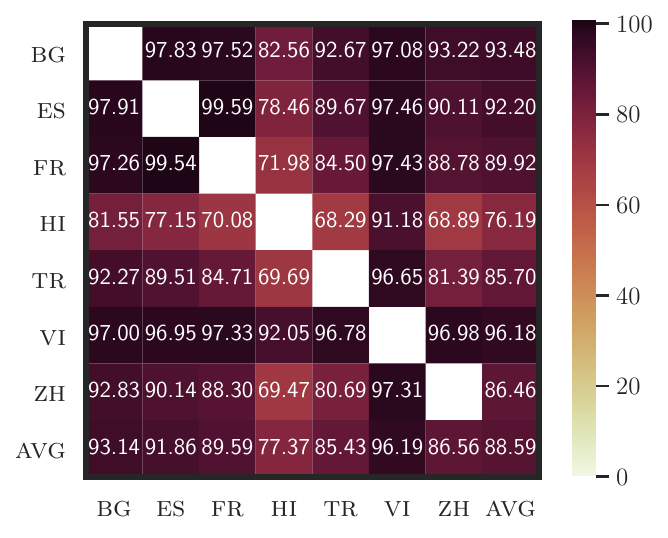}
        \caption{WM, $\varepsilon = \infty$, $l = 8$}
        \label{fig:retrieval_wikimatrix_inf_lay8}
    \end{subfigure}
    \caption{\textbf{XNLI} Sentence retrieval results for the TED 2020 (TED) and WikiMatrix (WM) datasets and different combinations of privacy budgets ($\varepsilon$) and layers ($l$). Each heatmap cell corresponds to the average over 5 random seeds. We observe that the overall patterns are highly similar across all levels of privacy, particularly at layer 0.}
    \label{fig:xnli_retrieval_full}
\end{figure*}

\begin{figure*}[ht!]
    \centering
    \begin{subfigure}[b]{0.21\textwidth}
        \centering
        \includegraphics[width=\textwidth]{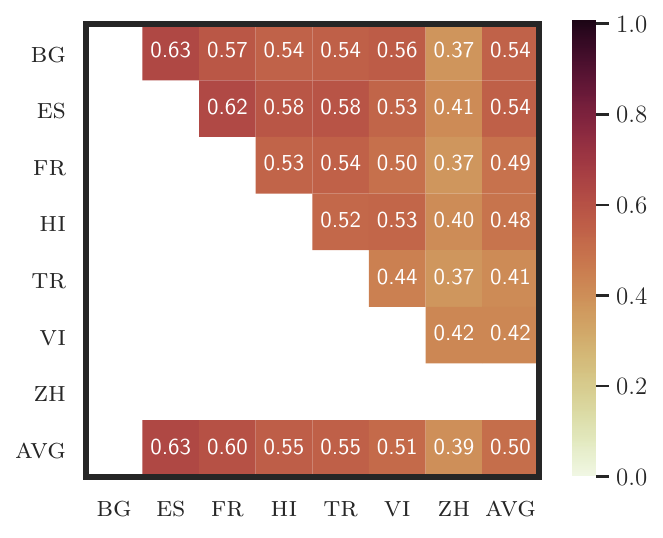}
        \caption{TED, $\varepsilon = 1$, $l = 0$}
        \label{fig:cka_ted2020_1_lay0}
    \end{subfigure}
    \begin{subfigure}[b]{0.21\textwidth}
        \centering
        \includegraphics[width=\textwidth]{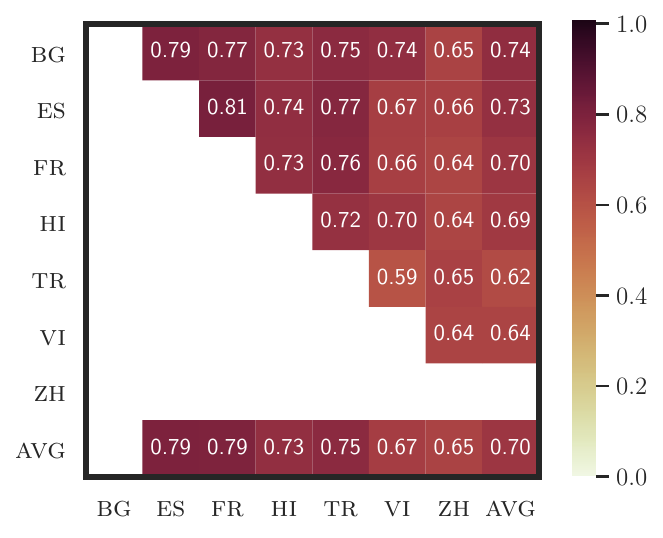}
        \caption{TED, $\varepsilon = 1$, $l = 8$}
        \label{fig:cka_ted2020_1_lay8}
    \end{subfigure}
        \begin{subfigure}[b]{0.21\textwidth}
        \centering
        \includegraphics[width=\textwidth]{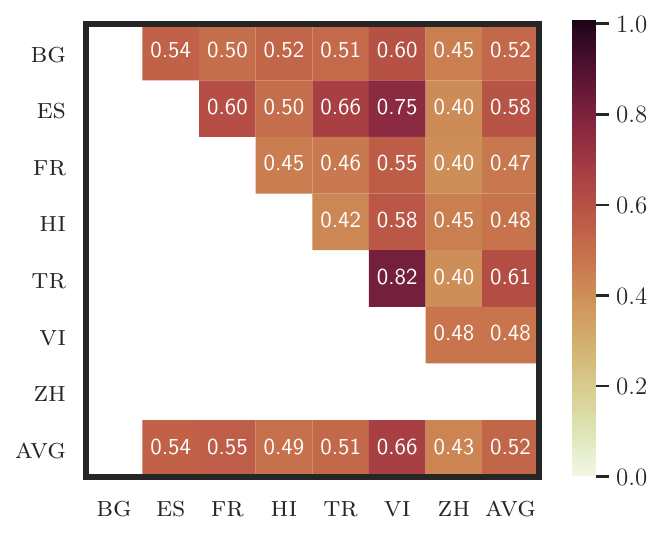}
        \caption{WM, $\varepsilon = 1$, $l = 0$}
        \label{fig:cka_wikimatrix_1_lay0}
    \end{subfigure}
    \begin{subfigure}[b]{0.21\textwidth}
        \centering
        \includegraphics[width=\textwidth]{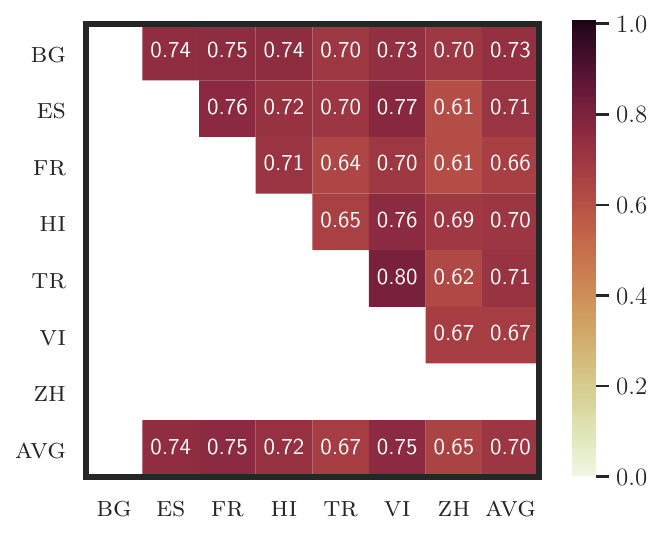}
        \caption{WM, $\varepsilon = 1$, $l = 8$}
        \label{fig:cka_wikimatrix_1_lay8}
    \end{subfigure}
    \begin{subfigure}[b]{0.21\textwidth}
        \centering
        \includegraphics[width=\textwidth]{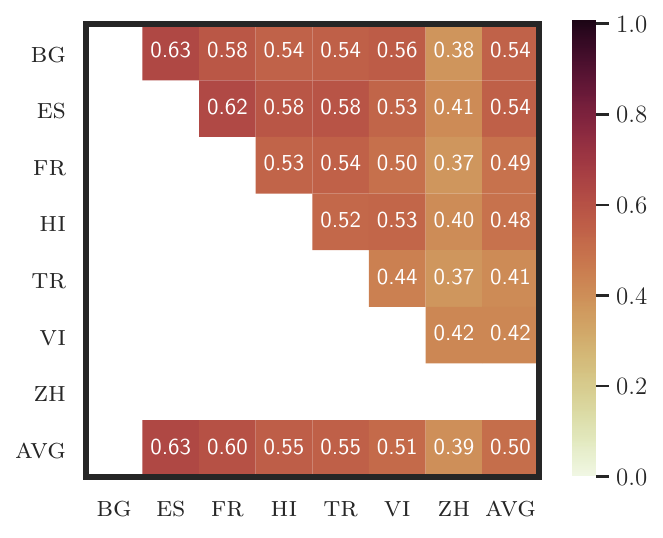}
        \caption{TED, $\varepsilon = 3$, $l = 0$}
        \label{fig:cka_ted2020_3_lay0}
    \end{subfigure}
    \begin{subfigure}[b]{0.21\textwidth}
        \centering
        \includegraphics[width=\textwidth]{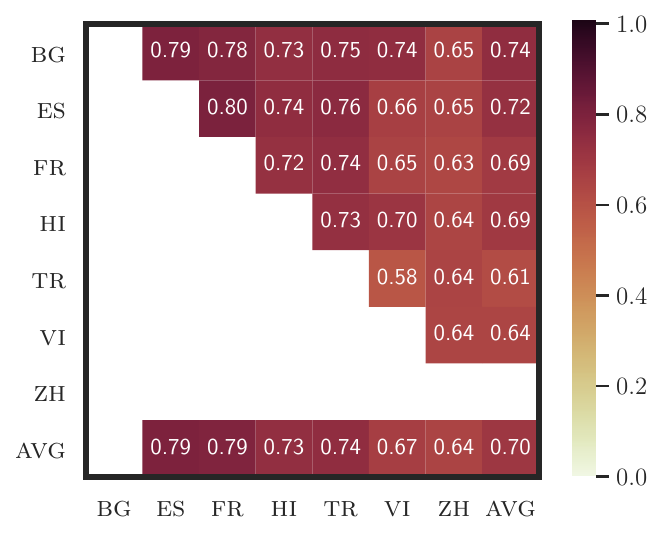}
        \caption{TED, $\varepsilon = 3$, $l = 8$}
        \label{fig:cka_ted2020_3_lay8}
    \end{subfigure}
        \begin{subfigure}[b]{0.21\textwidth}
        \centering
        \includegraphics[width=\textwidth]{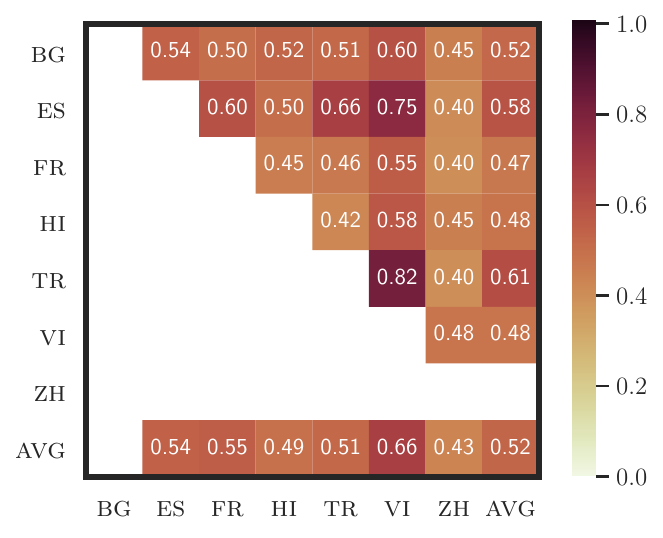}
        \caption{WM, $\varepsilon = 3$, $l = 0$}
        \label{fig:cka_wikimatrix_3_lay0}
    \end{subfigure}
    \begin{subfigure}[b]{0.21\textwidth}
        \centering
        \includegraphics[width=\textwidth]{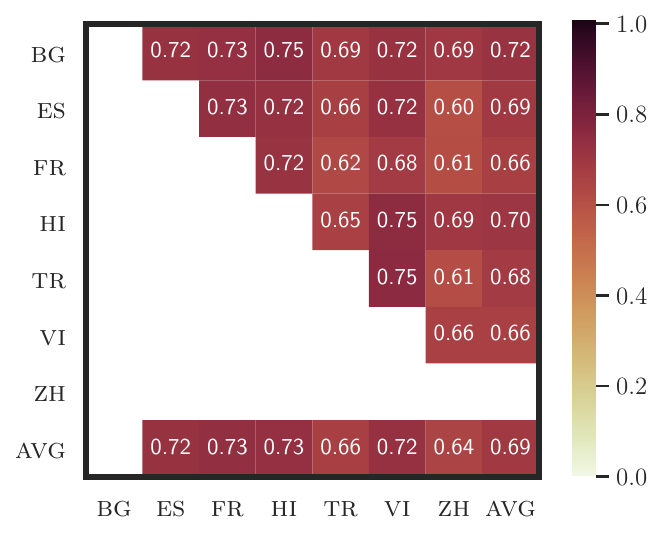}
        \caption{WM, $\varepsilon = 3$, $l = 8$}
        \label{fig:cka_wikimatrix_3_lay8}
    \end{subfigure}
    \begin{subfigure}[b]{0.21\textwidth}
        \centering
        \includegraphics[width=\textwidth]{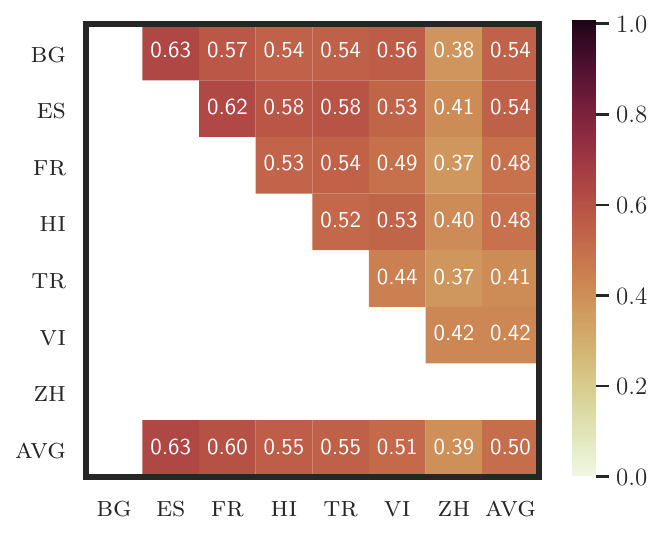}
        \caption{TED, $\varepsilon = 8$, $l = 0$}
        \label{fig:cka_ted2020_8_lay0}
    \end{subfigure}
    \begin{subfigure}[b]{0.21\textwidth}
        \centering
        \includegraphics[width=\textwidth]{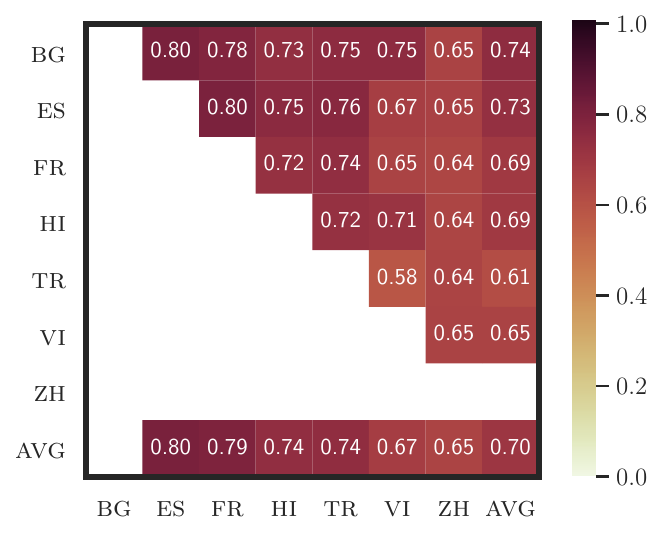}
        \caption{TED, $\varepsilon = 8$, $l = 8$}
        \label{fig:cka_ted2020_8_lay8}
    \end{subfigure}
        \begin{subfigure}[b]{0.21\textwidth}
        \centering
        \includegraphics[width=\textwidth]{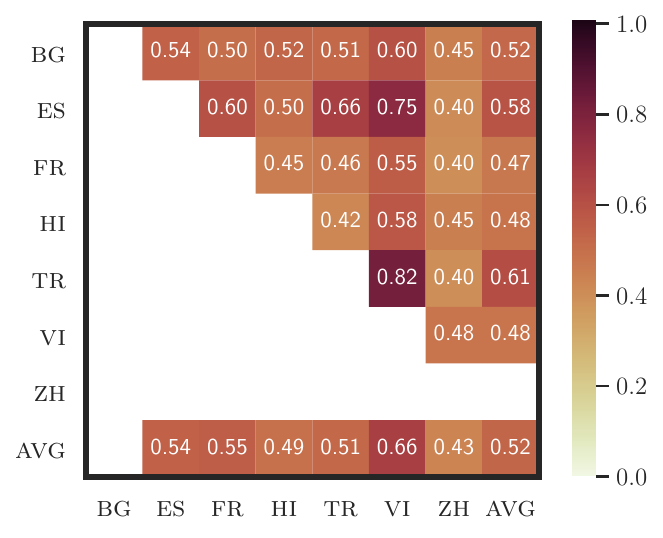}
        \caption{WM, $\varepsilon = 8$, $l = 0$}
        \label{fig:cka_wikimatrix_8_lay0}
    \end{subfigure}
    \begin{subfigure}[b]{0.21\textwidth}
        \centering
        \includegraphics[width=\textwidth]{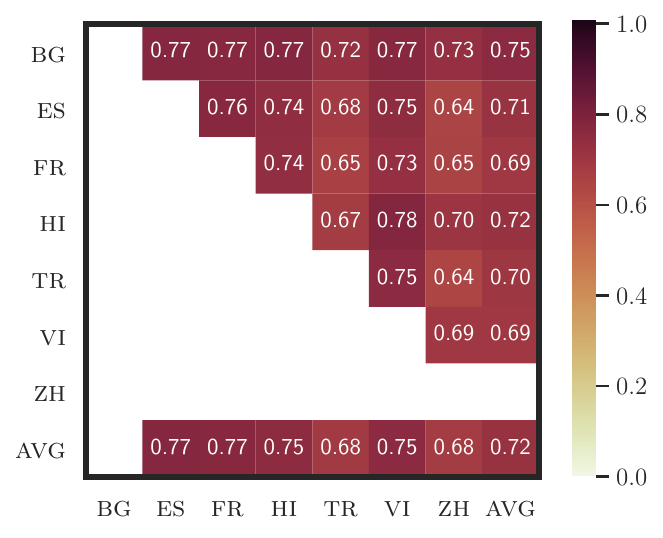}
        \caption{WM, $\varepsilon = 8$, $l = 8$}
        \label{fig:cka_wikimatrix_8_lay8}
    \end{subfigure}
    \begin{subfigure}[b]{0.21\textwidth}
        \centering
        \includegraphics[width=\textwidth]{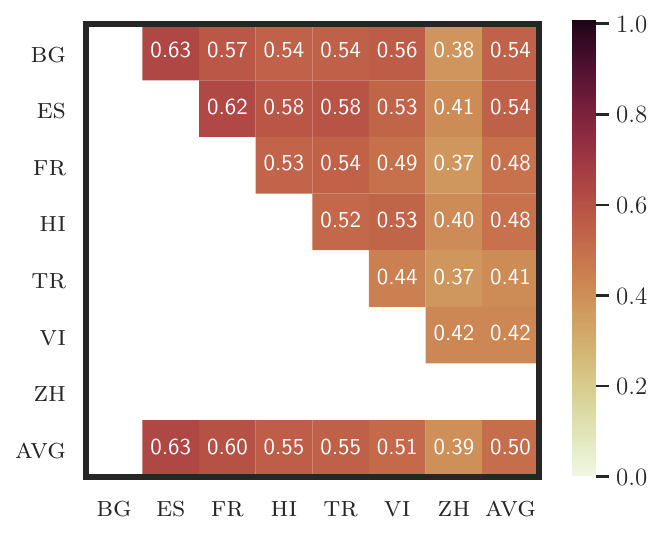}
        \caption{TED, $\varepsilon = 15$, $l = 0$}
        \label{fig:cka_ted2020_15_lay0}
    \end{subfigure}
    \begin{subfigure}[b]{0.21\textwidth}
        \centering
        \includegraphics[width=\textwidth]{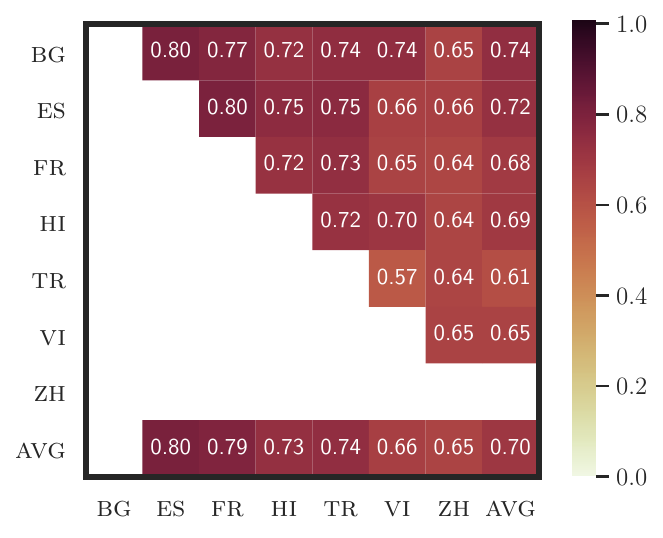}
        \caption{TED, $\varepsilon = 15$, $l = 8$}
        \label{fig:cka_ted2020_15_lay8}
    \end{subfigure}
        \begin{subfigure}[b]{0.21\textwidth}
        \centering
        \includegraphics[width=\textwidth]{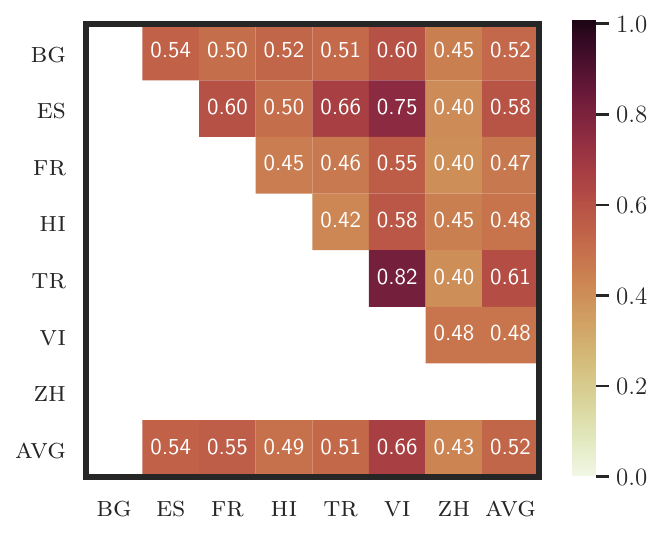}
        \caption{WM, $\varepsilon = 15$, $l = 0$}
        \label{fig:cka_wikimatrix_15_lay0}
    \end{subfigure}
    \begin{subfigure}[b]{0.21\textwidth}
        \centering
        \includegraphics[width=\textwidth]{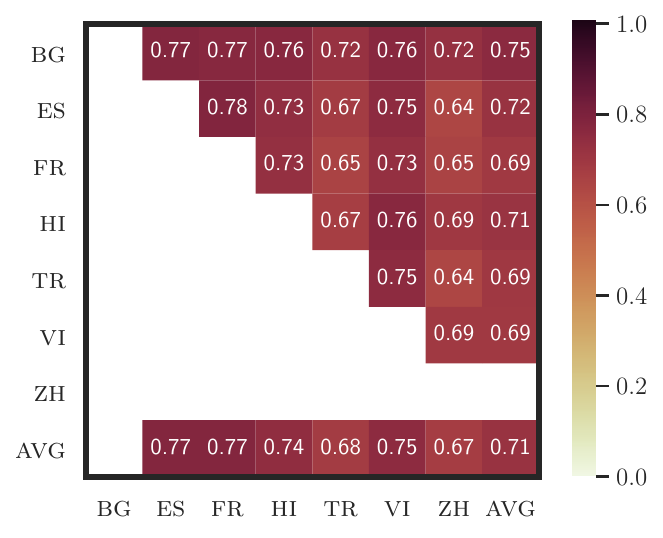}
        \caption{WM, $\varepsilon = 15$, $l = 8$}
        \label{fig:cka_wikimatrix_15_lay8}
    \end{subfigure}
    \begin{subfigure}[b]{0.21\textwidth}
        \centering
        \includegraphics[width=\textwidth]{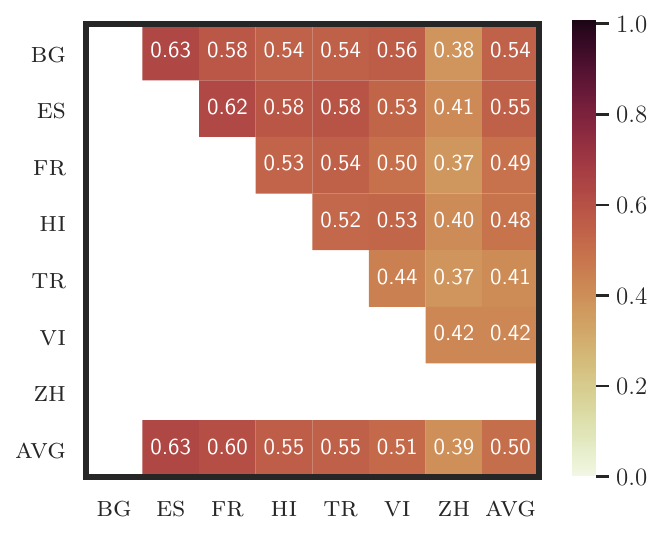}
        \caption{TED, $\varepsilon = 30$, $l = 0$}
        \label{fig:cka_ted2020_30_lay0}
    \end{subfigure}
    \begin{subfigure}[b]{0.21\textwidth}
        \centering
        \includegraphics[width=\textwidth]{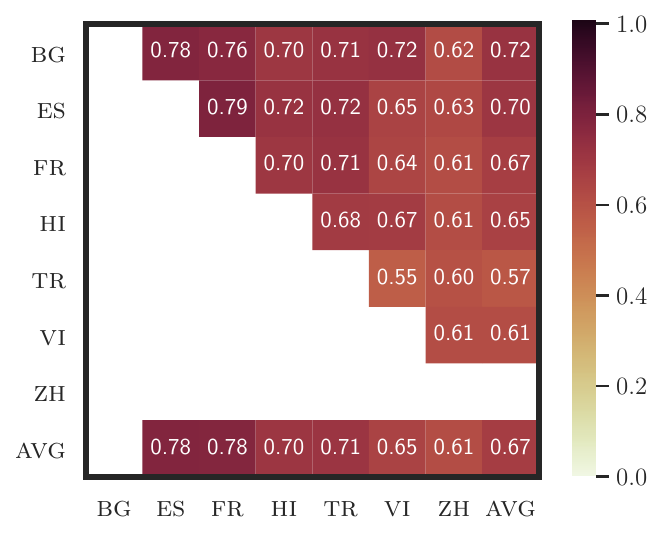}
        \caption{TED, $\varepsilon = 30$, $l = 8$}
        \label{fig:cka_ted2020_30_lay8}
    \end{subfigure}
        \begin{subfigure}[b]{0.21\textwidth}
        \centering
        \includegraphics[width=\textwidth]{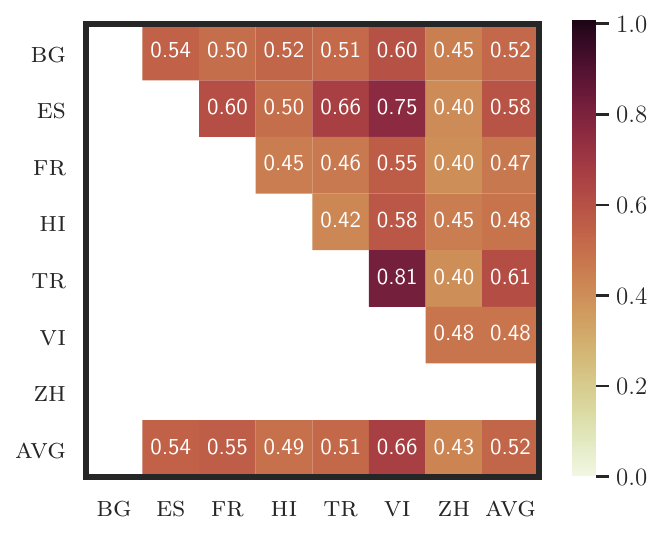}
        \caption{WM, $\varepsilon = 30$, $l = 0$}
        \label{fig:cka_wikimatrix_30_lay0}
    \end{subfigure}
    \begin{subfigure}[b]{0.21\textwidth}
        \centering
        \includegraphics[width=\textwidth]{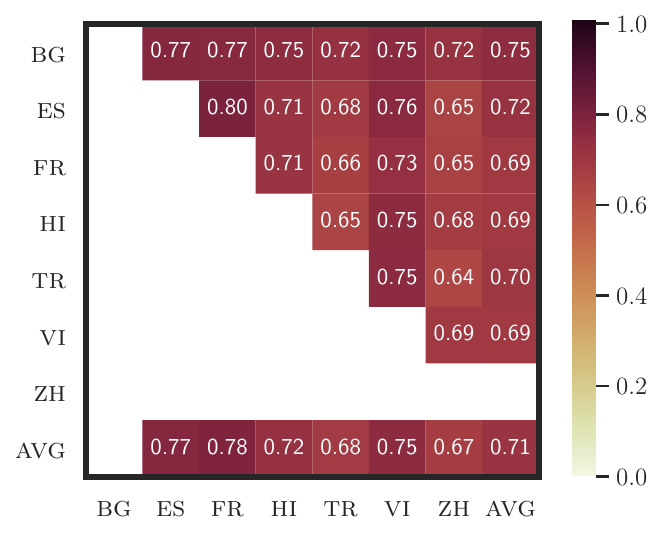}
        \caption{WM, $\varepsilon = 30$, $l = 8$}
        \label{fig:cka_wikimatrix_30_lay8}
    \end{subfigure}
    \begin{subfigure}[b]{0.21\textwidth}
        \centering
        \includegraphics[width=\textwidth]{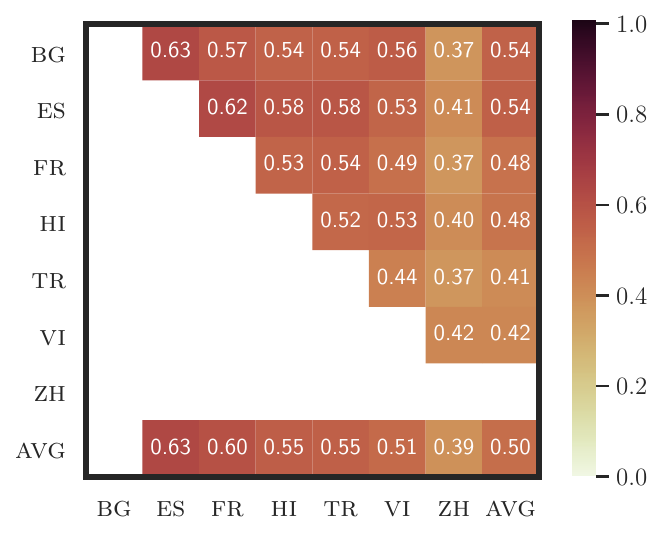}
        \caption{TED, $\varepsilon = \infty$, $l = 0$}
        \label{fig:cka_ted2020_inf_lay0}
    \end{subfigure}
    \begin{subfigure}[b]{0.21\textwidth}
        \centering
        \includegraphics[width=\textwidth]{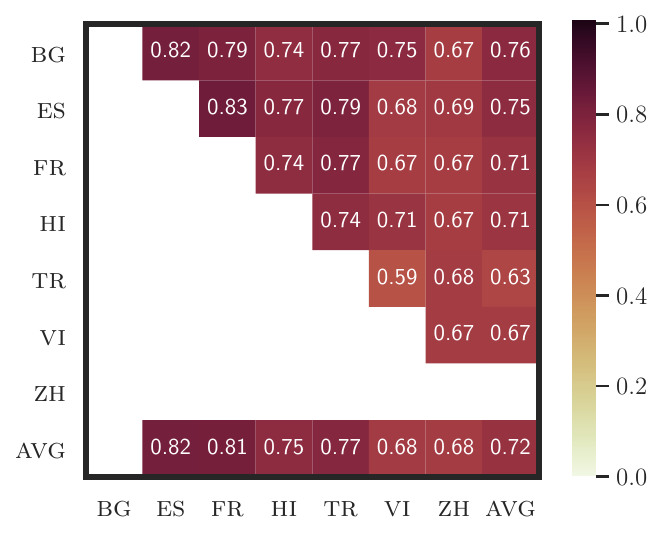}
        \caption{TED, $\varepsilon = \infty$, $l = 8$}
        \label{fig:cka_ted2020_inf_lay8}
    \end{subfigure}
        \begin{subfigure}[b]{0.21\textwidth}
        \centering
        \includegraphics[width=\textwidth]{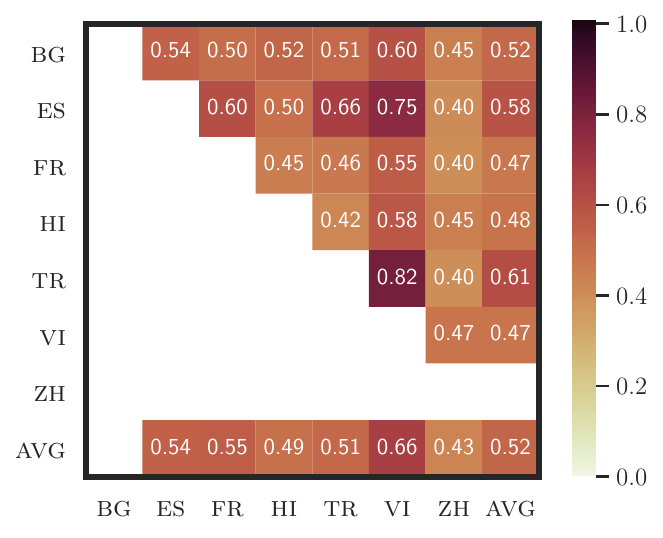}
        \caption{WM, $\varepsilon = \infty$, $l = 0$}
        \label{fig:cka_wikimatrix_inf_lay0}
    \end{subfigure}
    \begin{subfigure}[b]{0.21\textwidth}
        \centering
        \includegraphics[width=\textwidth]{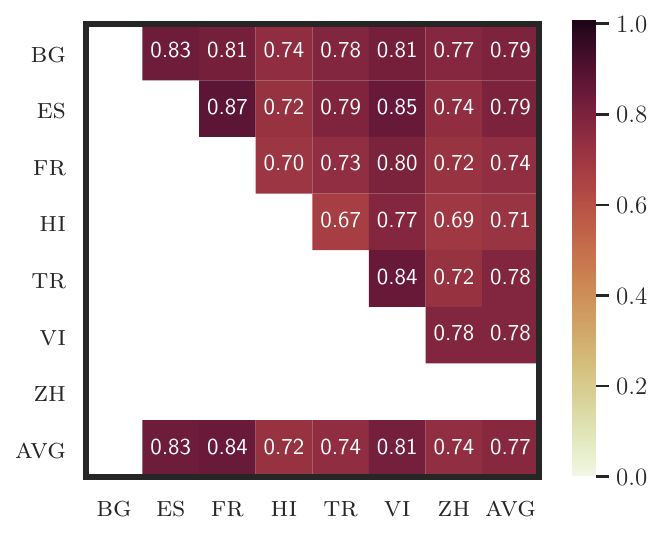}
        \caption{WM, $\varepsilon = \infty$, $l = 8$}
        \label{fig:cka_wikimatrix_inf_lay8}
    \end{subfigure}
    \caption{\textbf{XNLI} CKA results for the TED 2020 (TED) and WikiMatrix (WM) datasets and different combinations of privacy budgets ($\varepsilon$) and layers ($l$). Each heatmap cell corresponds to the average over 5 random seeds. We observe that the overall patterns are highly similar across all levels of privacy, particularly at layer 0.}
    \label{fig:xnli_cka_full}
\end{figure*}

\begin{figure*}[ht!]
    \centering
    \begin{subfigure}[b]{0.21\textwidth}
        \centering
        \includegraphics[width=\textwidth]{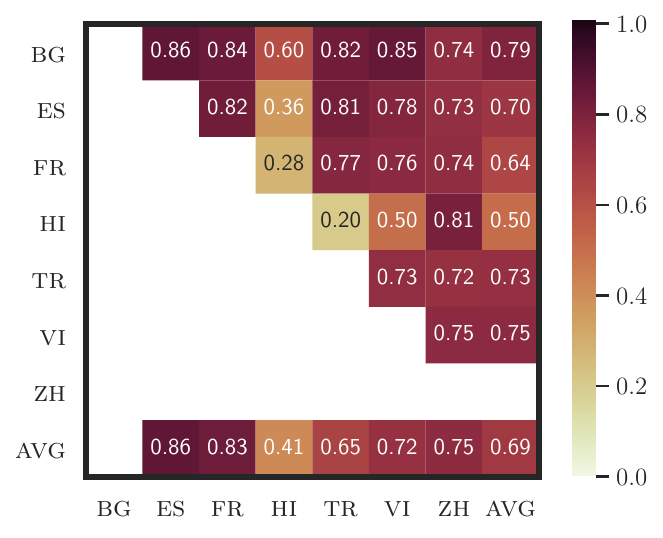}
        \caption{TED, $\varepsilon = 1$, $l = 0$}
        \label{fig:rsa_ted2020_1_lay0}
    \end{subfigure}
    \begin{subfigure}[b]{0.21\textwidth}
        \centering
        \includegraphics[width=\textwidth]{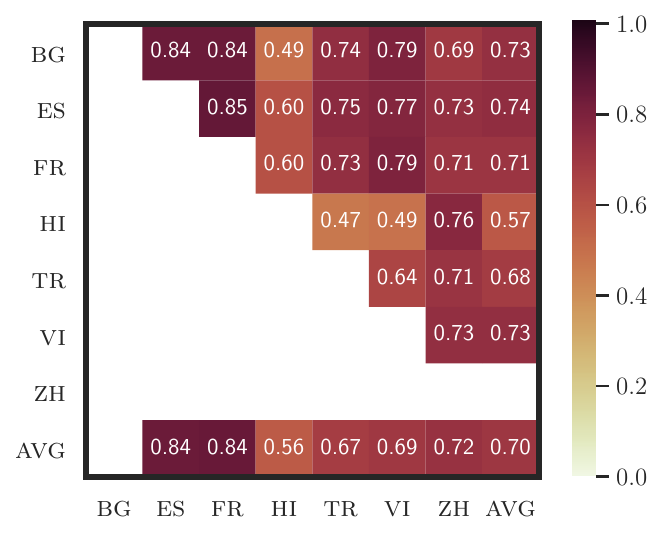}
        \caption{TED, $\varepsilon = 1$, $l = 8$}
        \label{fig:rsa_ted2020_1_lay8}
    \end{subfigure}
        \begin{subfigure}[b]{0.21\textwidth}
        \centering
        \includegraphics[width=\textwidth]{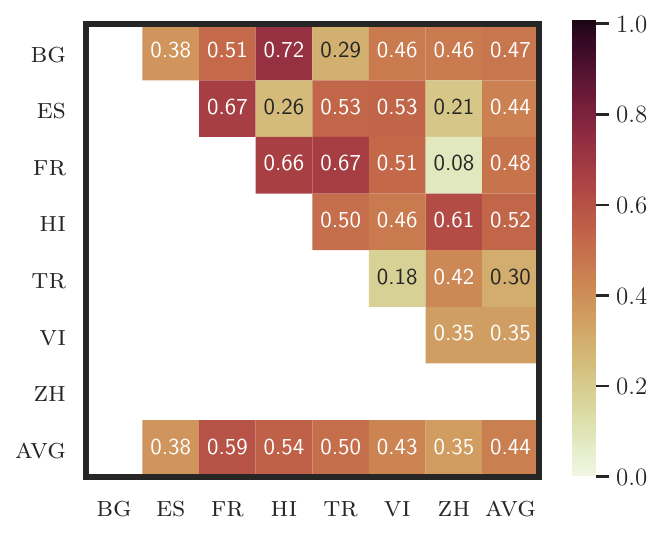}
        \caption{WM, $\varepsilon = 1$, $l = 0$}
        \label{fig:rsa_wikimatrix_1_lay0}
    \end{subfigure}
    \begin{subfigure}[b]{0.21\textwidth}
        \centering
        \includegraphics[width=\textwidth]{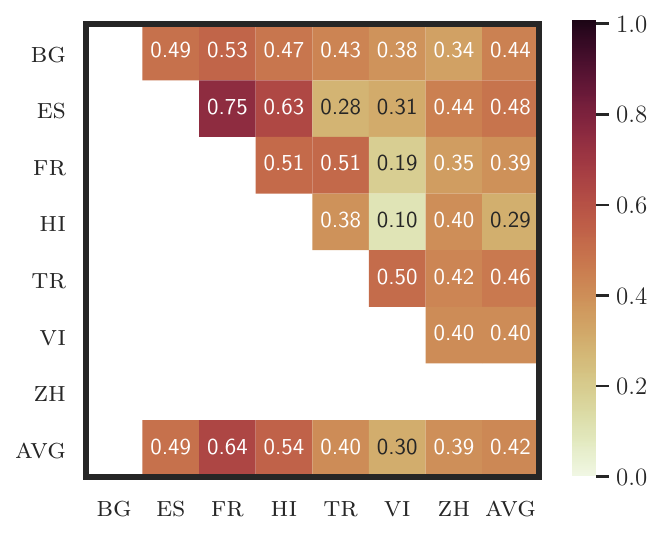}
        \caption{WM, $\varepsilon = 1$, $l = 8$}
        \label{fig:rsa_wikimatrix_1_lay8}
    \end{subfigure}
    \begin{subfigure}[b]{0.21\textwidth}
        \centering
        \includegraphics[width=\textwidth]{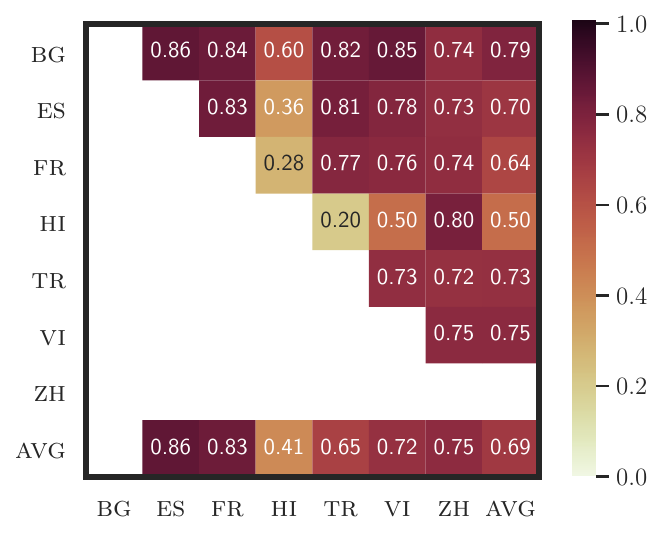}
        \caption{TED, $\varepsilon = 3$, $l = 0$}
        \label{fig:rsa_ted2020_3_lay0}
    \end{subfigure}
    \begin{subfigure}[b]{0.21\textwidth}
        \centering
        \includegraphics[width=\textwidth]{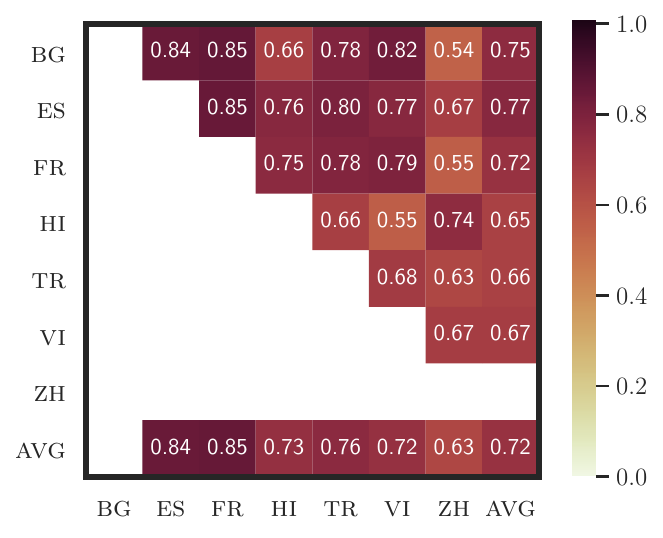}
        \caption{TED, $\varepsilon = 3$, $l = 8$}
        \label{fig:rsa_ted2020_3_lay8}
    \end{subfigure}
        \begin{subfigure}[b]{0.21\textwidth}
        \centering
        \includegraphics[width=\textwidth]{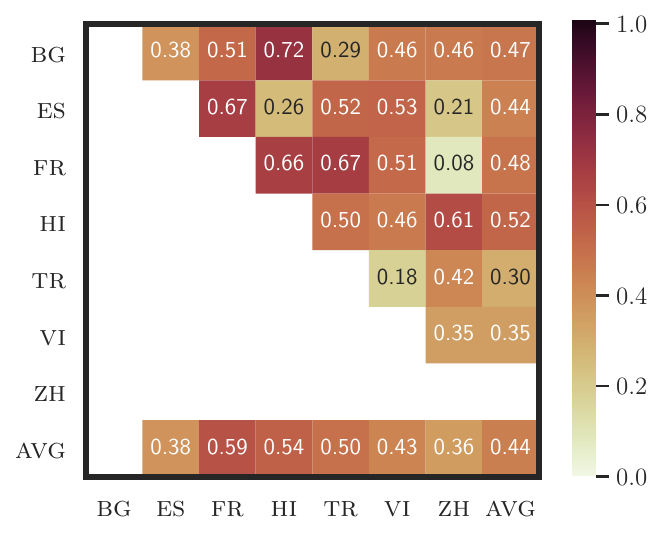}
        \caption{WM, $\varepsilon = 3$, $l = 0$}
        \label{fig:rsa_wikimatrix_3_lay0}
    \end{subfigure}
    \begin{subfigure}[b]{0.21\textwidth}
        \centering
        \includegraphics[width=\textwidth]{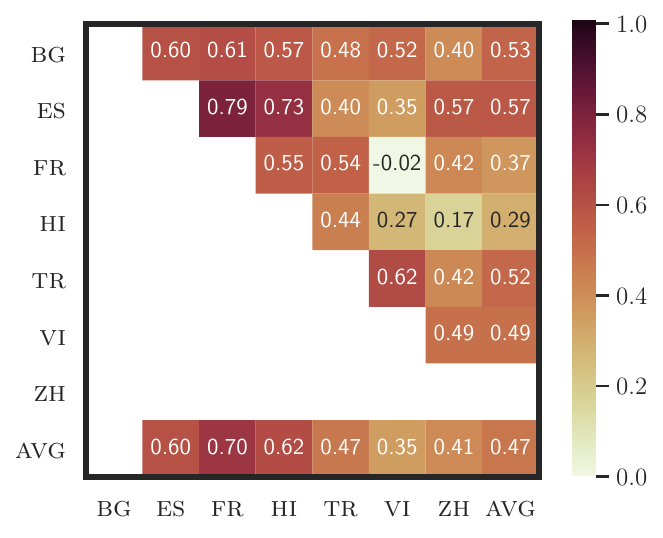}
        \caption{WM, $\varepsilon = 3$, $l = 8$}
        \label{fig:rsa_wikimatrix_3_lay8}
    \end{subfigure}
    \begin{subfigure}[b]{0.21\textwidth}
        \centering
        \includegraphics[width=\textwidth]{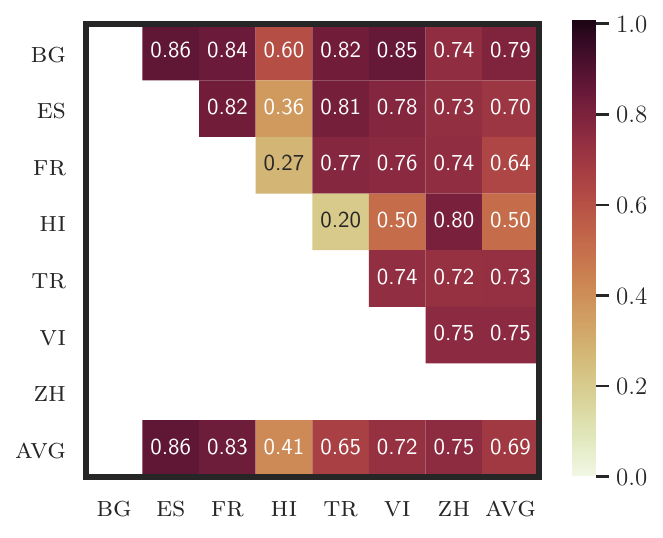}
        \caption{TED, $\varepsilon = 8$, $l = 0$}
        \label{fig:rsa_ted2020_8_lay0}
    \end{subfigure}
    \begin{subfigure}[b]{0.21\textwidth}
        \centering
        \includegraphics[width=\textwidth]{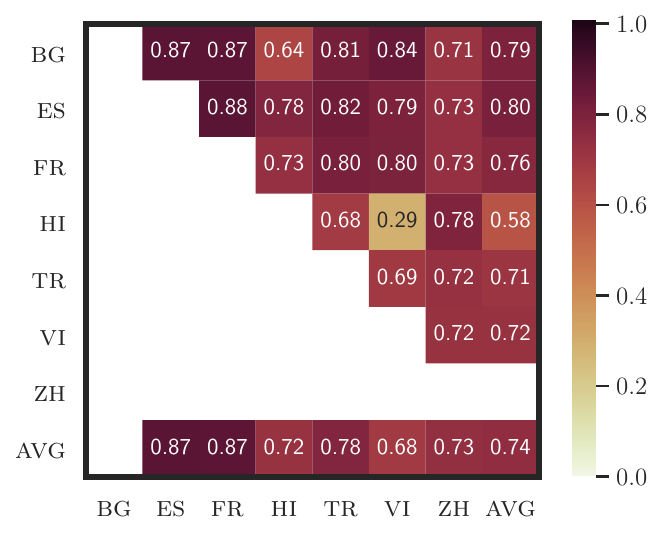}
        \caption{TED, $\varepsilon = 8$, $l = 8$}
        \label{fig:rsa_ted2020_8_lay8}
    \end{subfigure}
        \begin{subfigure}[b]{0.21\textwidth}
        \centering
        \includegraphics[width=\textwidth]{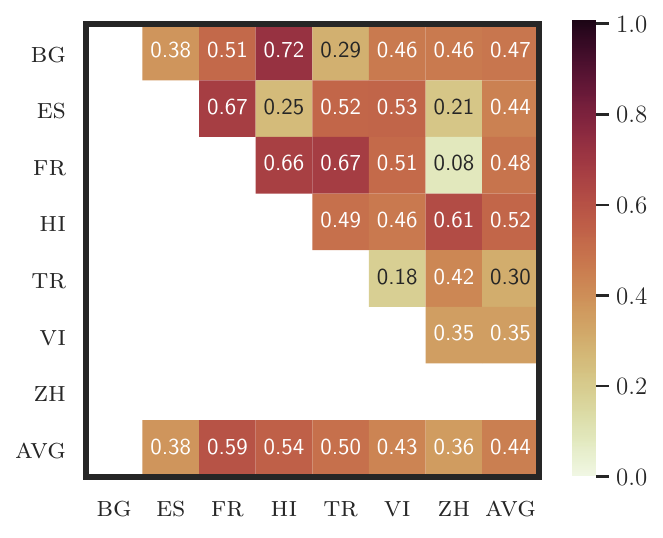}
        \caption{WM, $\varepsilon = 8$, $l = 0$}
        \label{fig:rsa_wikimatrix_8_lay0}
    \end{subfigure}
    \begin{subfigure}[b]{0.21\textwidth}
        \centering
        \includegraphics[width=\textwidth]{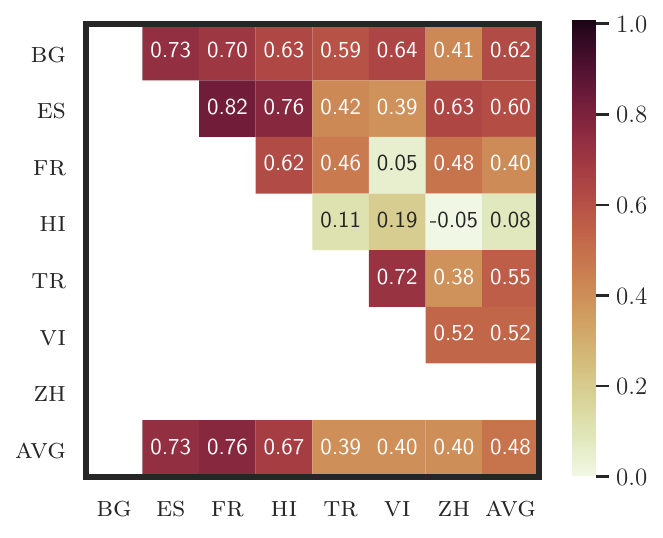}
        \caption{WM, $\varepsilon = 8$, $l = 8$}
        \label{fig:rsa_wikimatrix_8_lay8}
    \end{subfigure}
    \begin{subfigure}[b]{0.21\textwidth}
        \centering
        \includegraphics[width=\textwidth]{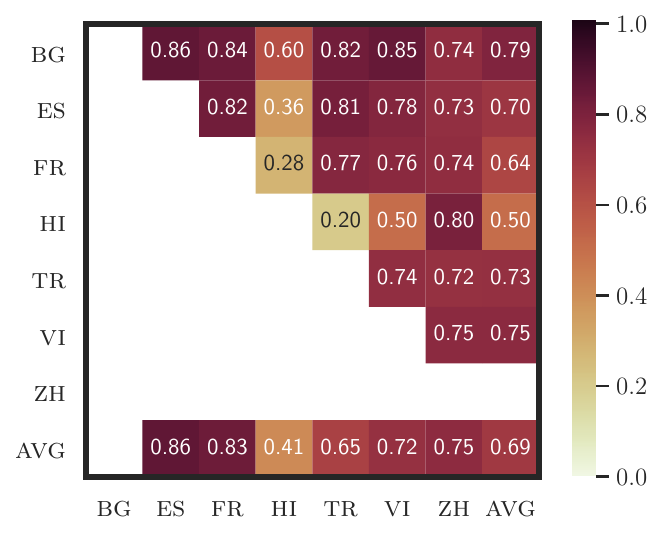}
        \caption{TED, $\varepsilon = 15$, $l = 0$}
        \label{fig:rsa_ted2020_15_lay0}
    \end{subfigure}
    \begin{subfigure}[b]{0.21\textwidth}
        \centering
        \includegraphics[width=\textwidth]{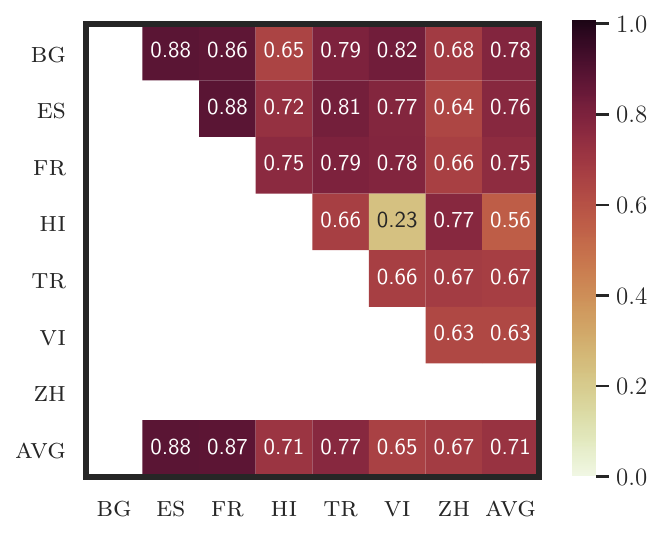}
        \caption{TED, $\varepsilon = 15$, $l = 8$}
        \label{fig:rsa_ted2020_15_lay8}
    \end{subfigure}
        \begin{subfigure}[b]{0.21\textwidth}
        \centering
        \includegraphics[width=\textwidth]{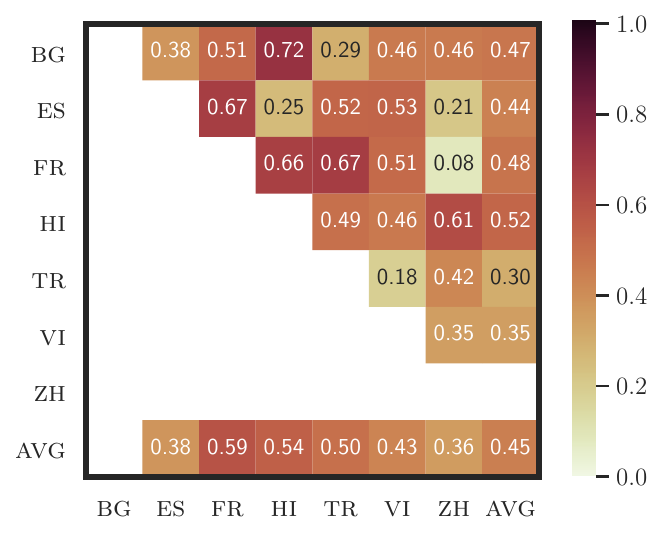}
        \caption{WM, $\varepsilon = 15$, $l = 0$}
        \label{fig:rsa_wikimatrix_15_lay0}
    \end{subfigure}
    \begin{subfigure}[b]{0.21\textwidth}
        \centering
        \includegraphics[width=\textwidth]{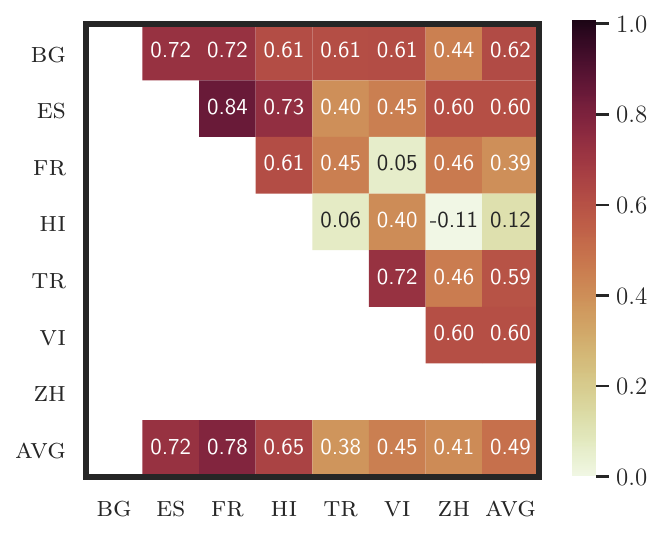}
        \caption{WM, $\varepsilon = 15$, $l = 8$}
        \label{fig:rsa_wikimatrix_15_lay8}
    \end{subfigure}
    \begin{subfigure}[b]{0.21\textwidth}
        \centering
        \includegraphics[width=\textwidth]{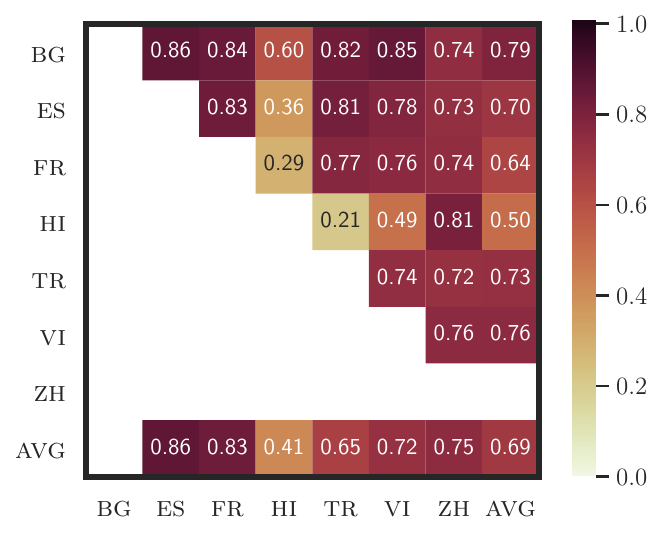}
        \caption{TED, $\varepsilon = 30$, $l = 0$}
        \label{fig:rsa_ted2020_30_lay0}
    \end{subfigure}
    \begin{subfigure}[b]{0.21\textwidth}
        \centering
        \includegraphics[width=\textwidth]{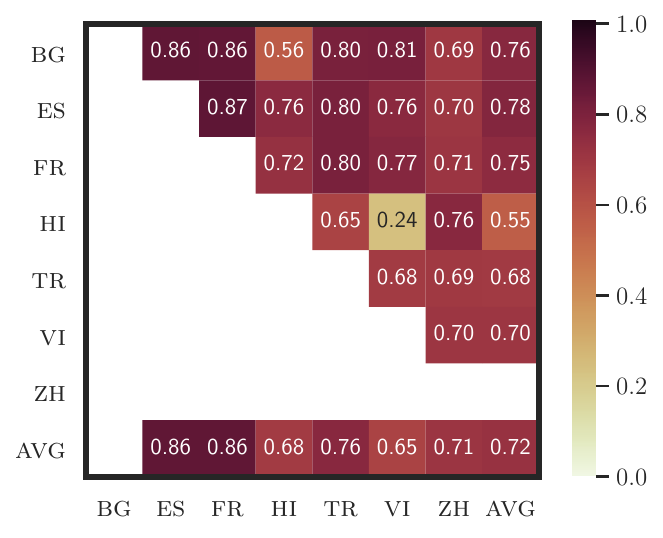}
        \caption{TED, $\varepsilon = 30$, $l = 8$}
        \label{fig:rsa_ted2020_30_lay8}
    \end{subfigure}
        \begin{subfigure}[b]{0.21\textwidth}
        \centering
        \includegraphics[width=\textwidth]{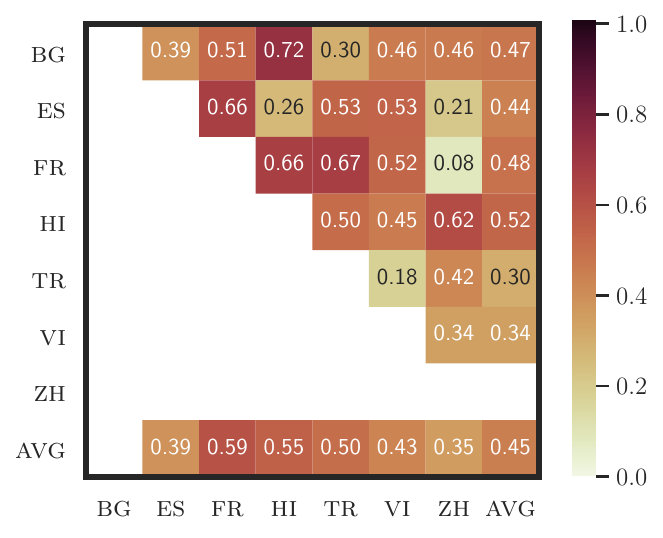}
        \caption{WM, $\varepsilon = 30$, $l = 0$}
        \label{fig:rsa_wikimatrix_30_lay0}
    \end{subfigure}
    \begin{subfigure}[b]{0.21\textwidth}
        \centering
        \includegraphics[width=\textwidth]{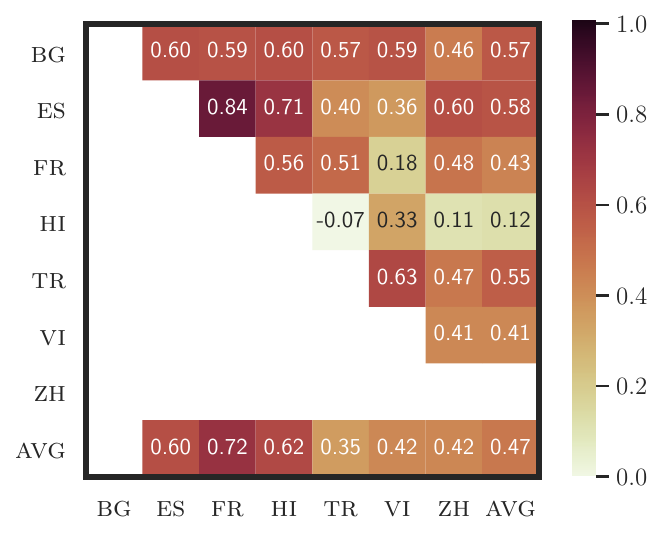}
        \caption{WM, $\varepsilon = 30$, $l = 8$}
        \label{fig:rsa_wikimatrix_30_lay8}
    \end{subfigure}
    \begin{subfigure}[b]{0.21\textwidth}
        \centering
        \includegraphics[width=\textwidth]{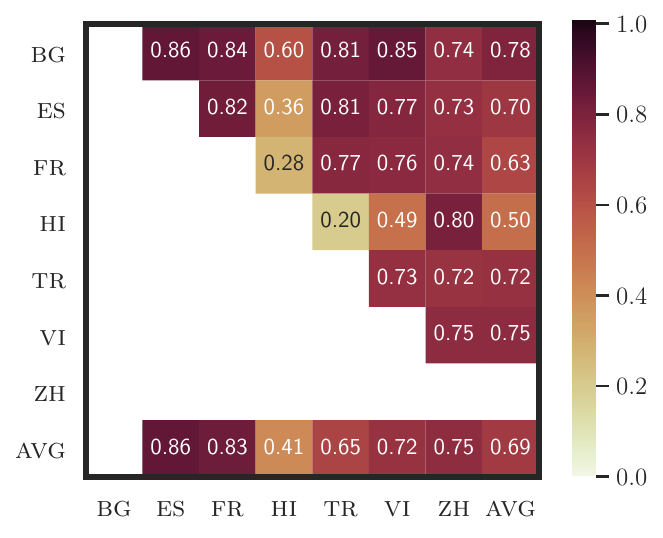}
        \caption{TED, $\varepsilon = \infty$, $l = 0$}
        \label{fig:rsa_ted2020_inf_lay0}
    \end{subfigure}
    \begin{subfigure}[b]{0.21\textwidth}
        \centering
        \includegraphics[width=\textwidth]{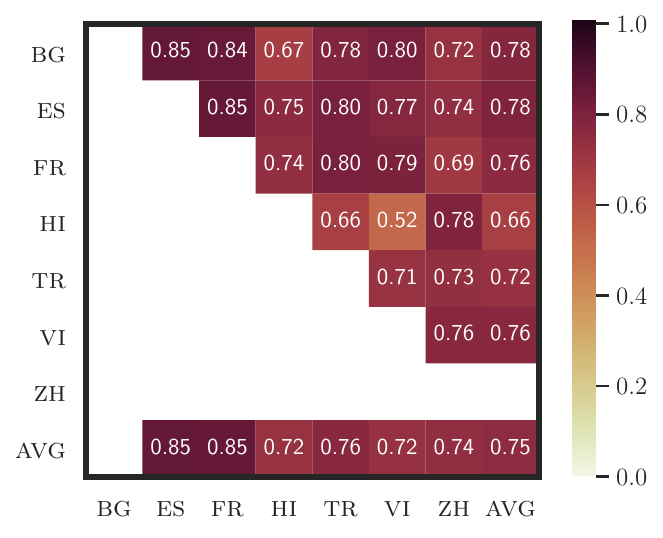}
        \caption{TED, $\varepsilon = \infty$, $l = 8$}
        \label{fig:rsa_ted2020_inf_lay8}
    \end{subfigure}
        \begin{subfigure}[b]{0.21\textwidth}
        \centering
        \includegraphics[width=\textwidth]{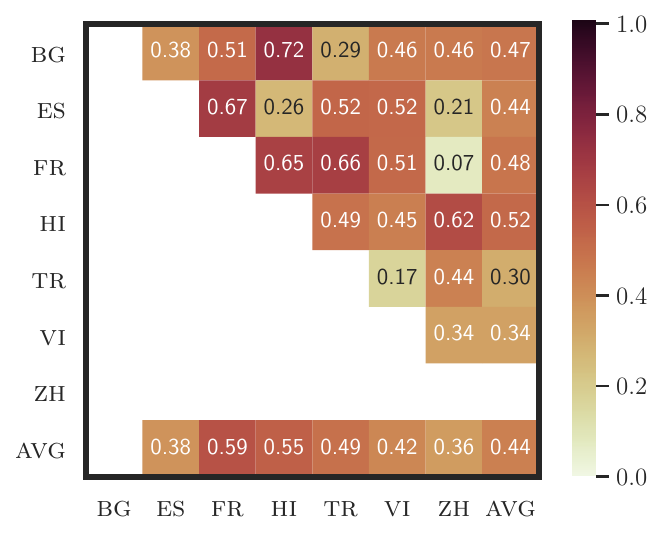}
        \caption{WM, $\varepsilon = \infty$, $l = 0$}
        \label{fig:rsa_wikimatrix_inf_lay0}
    \end{subfigure}
    \begin{subfigure}[b]{0.21\textwidth}
        \centering
        \includegraphics[width=\textwidth]{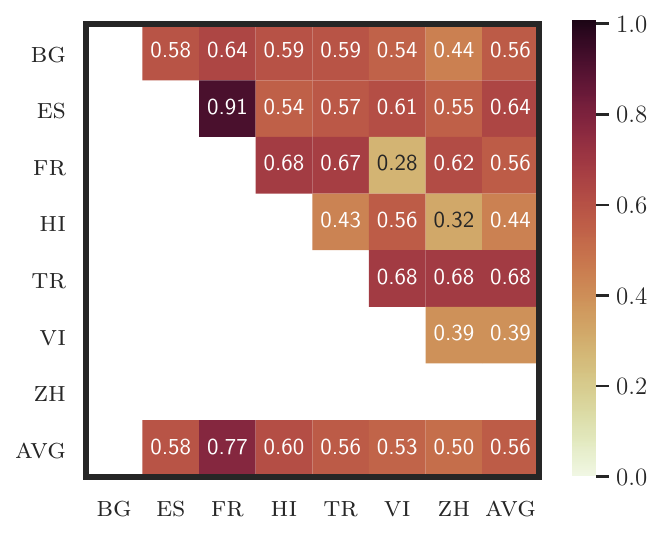}
        \caption{WM, $\varepsilon = \infty$, $l = 8$}
        \label{fig:rsa_wikimatrix_inf_lay8}
    \end{subfigure}
    \caption{\textbf{XNLI} RSA results for the TED 2020 (TED) and WikiMatrix (WM) datasets and different combinations of privacy budgets ($\varepsilon$) and layers ($l$). Each heatmap cell corresponds to the average over 5 random seeds. We observe that the overall patterns are highly similar across all levels of privacy, particularly at layer 0.}
    \label{fig:xnli_rsa_full}
\end{figure*}

\begin{table*}[htp]
\centering
\caption{\textbf{POS} IsoScores for different combinations of privacy budgets ($\varepsilon$) and layers ($l$). We show results averaged over 5 random seeds, except for \textsc{rnd} and \textsc{pre}. \textsc{rnd} and \textsc{pre} (added for comparison) denote XLM-R with randomly initialized weights and the original pretrained XLM-R, respectively. We see that the isotropy is fairly uniform across privacy budgets at layer 0 and generally higher at layer 0 than at layer 8. At layer 8, it peaks for non-private ($\varepsilon = \infty$) and our most private ($\varepsilon = 1 $) models.}
\label{tab:pos_all_isoscores}
\resizebox{0.6\textwidth}{!}{%
\begin{tabular}{@{}cccccll@{}}
\toprule
\multirow{2}{*}{\textbf{$\pmb{\varepsilon}$}} & \multicolumn{2}{c}{\textbf{TED 2020}} & \multicolumn{2}{c}{\textbf{WikiMatrix}} & \multicolumn{2}{c}{\textbf{Tatoeba}} \\
              & $l=0$ & $l=8$ & $l=0$ & $l=8$ & $l=0$ & $l=8$ \\ \midrule
\textsc{rnd} & 0.141 & 0.132 & 0.114 & 0.111 & 0.054 & 0.061 \\
\textsc{pre}  & 0.187 & 0.130 & 0.198 & 0.112 & 0.134 & 0.075 \\ \midrule
1             & 0.188 & 0.054 & 0.199 & 0.046 & 0.135 & 0.033 \\
3             & 0.188 & 0.044 & 0.199 & 0.038 & 0.135 & 0.027 \\
8             & 0.187 & 0.045 & 0.197 & 0.038 & 0.133 & 0.027 \\
15            & 0.187 & 0.047 & 0.199 & 0.040 & 0.135 & 0.028 \\
30            & 0.187 & 0.047 & 0.199 & 0.040 & 0.135 & 0.028 \\
$\infty$      & 0.188 & 0.087 & 0.199 & 0.070 & 0.135 & 0.051 \\ \bottomrule
\end{tabular}%
}

\vspace*{\floatsep}
\vspace*{\floatsep}
\vspace*{\floatsep}
\vspace*{\floatsep}

\centering
\caption{\textbf{XNLI} IsoScores for different combinations of privacy budgets ($\varepsilon$) and layers ($l$). We show results averaged over 5 random seeds, except for \textsc{rnd} and \textsc{pre}. \textsc{rnd} and \textsc{pre} (added for comparison) denote XLM-R with randomly initialized weights and the original pretrained XLM-R, respectively. We see that the isotropy is fairly uniform across privacy budgets at layer 0 and generally higher at layer 0 than at layer 8. At layer 8, it peaks for non-private ($\varepsilon = \infty$) and our most private ($\varepsilon = 1 $) models.}
\label{tab:xnli_all_isoscores}
\resizebox{0.4\textwidth}{!}{%
\begin{tabular}{@{}ccccc@{}}
\toprule
\multirow{2}{*}{\textbf{$\pmb{\varepsilon}$}} & \multicolumn{2}{c}{\textbf{TED 2020}} & \multicolumn{2}{c}{\textbf{WikiMatrix}} \\
              & $l=0$  & $l=8$  & $l=0$  & $l=8$  \\ \midrule
\textsc{rnd} & 0.144 & 0.134 & 0.130 & 0.124 \\
\textsc{pre}  & 0.195 & 0.138 & 0.210 & 0.129 \\ \midrule
1             & 0.195 & 0.121 & 0.211 & 0.120 \\
3             & 0.196 & 0.101 & 0.211 & 0.104 \\
8             & 0.196 & 0.074 & 0.212 & 0.079 \\
15            & 0.196 & 0.071 & 0.212 & 0.077 \\
30            & 0.194 & 0.087 & 0.210 & 0.089 \\
$\infty$      & 0.195 & 0.182 & 0.211 & 0.166 \\ \bottomrule
\end{tabular}%
}
\end{table*}

\end{document}